\documentclass[
    a5paper,
    footinclude=true,
    headinclude=true,
    draft=false,
    twoside=semi,
    parskip=full,  
]{scrbook}
\usepackage{etex}

\usepackage[utf8]{inputenc}
\makeatletter
\def\input@path{{utils/}}
\makeatother
\newcommand{\myTitle}{Provable Non-Convex Optimization and Algorithm Validation via Submodularity}
\newcommand{\myName}{Yatao Bian}

\newcommand{\myUni}{\protect{ETH Zürich}}

\newcommand{\myTime}{2019}

\def\utilsdir{utils}

\usepackage[american, german]{babel}
\usepackage[utf8]{inputenc}
\usepackage[T1]{fontenc}
\usepackage{lmodern}

\usepackage{tabularx}
\usepackage{xspace}

\usepackage[title, titletoc]{appendix}

\usepackage{amsthm}
\usepackage{amsmath}
\usepackage{amsfonts}
\usepackage{amssymb}
\usepackage{bm}
\usepackage{mathtools}
\usepackage{dsfont}

\usepackage[%
    dottedtoc,
    floatperchapter,
    parts
]{\utilsdir/classicthesis}

\PassOptionsToPackage{
eulerchapternumbers=false,
}{classicthesis}

\usepackage{geometry}
\usepackage{pdflscape}

\usepackage{tabularx}

\usepackage{csquotes}
\usepackage[%
    style=authoryear,%
    dashed=false,%
    doi=false,%
    isbn=false,%
    url=false,
    hyperref=auto,%
    backref,%
    backrefstyle=three,%
    maxbibnames=99,%
    maxcitenames=1,%
    backend=bibtex,
    natbib=true,
    uniquename=mininit,
    uniquelist=minyear
]{biblatex}
\bibliography{\utilsdir/thesis_biber}
\preto\fullcite{\AtNextCite{\defcounter{maxnames}{99}}}

\usepackage{mathtools}
\usepackage[vlined, ruled, linesnumbered]{algorithm2e}

\SetCommentSty{mycommfont}
\SetKwInOut{Input}{input}
\SetKwInOut{Output}{output}

\usepackage{mparhack}
\usepackage{fixltx2e}

\usepackage{makeidx}
\makeindex

\usepackage{graphicx}
\graphicspath{{\dir/figs/}}

\usepackage{caption}
\usepackage{subfig}
\usepackage{import} 

\usepackage{multirow}
\usepackage{booktabs}
\usepackage{hyperref}
\hypersetup{%
    colorlinks=true, linktocpage=true, pdfstartpage=1, pdfstartview=FitV,%
    breaklinks=true, pdfpagemode=UseNone, pageanchor=true, pdfpagemode=UseOutlines,%
    plainpages=false, bookmarksnumbered, bookmarksopen=true, bookmarksopenlevel=1,%
    hypertexnames=true, pdfhighlight=/O,
    urlcolor=webbrown,
    linkcolor=RoyalBlue,
    citecolor=webgreen,
    pdftitle={\myTitle},%
    pdfauthor={\myName, \myUni},%
    pdfsubject={},%
    pdfkeywords={},%
}

\usepackage{thmtools}
\usepackage{thm-restate}

\usepackage{cleveref} 
\crefname{section}{Section}{Sections}
\crefname{theorem}{Theorem}{Theorems}
\crefname{lemma}{Lemma}{Lemmas}
\crefname{equation}{Equation}{Equations}
\crefname{proposition}{Proposition}{Propositions}
\crefname{claim}{Claim}{Claims}
\crefname{appendix}{Appendix}{Appendices}
\crefname{algorithm}{Algorithm}{Algorithms}
\crefname{figure}{Figure}{Figures}
\crefname{table}{Table}{Tables}
\crefname{remark}{Remark}{Remarks}
\crefname{definition}{Definition}{Definitions}
\crefname{equatinon}{Equation}{Equations}
\crefname{corollary}{Corollary}{Corollaries}
\crefname{observation}{Observation}{Observations}

\theoremstyle{plain}
\newtheorem{theorem}{Theorem}[chapter]
\newtheorem{lemma}[theorem]{Lemma}
\newtheorem{remark}[theorem]{Remark}
\newtheorem{conjecture}[theorem]{Conjecture}
\newtheorem{proposition}[theorem]{Proposition}

\newtheorem{corollary}[theorem]{Corollary}
\newtheorem{claim}[theorem]{Claim}
\newtheorem{observation}[theorem]{Observation}

\theoremstyle{definition}
\newtheorem{definition}[theorem]{Definition}

\PassOptionsToPackage{printonlyused,smaller}{acronym}
\usepackage{\utilsdir/acronym} 
\renewcommand{\bflabel}[1]{{#1}\hfill} 

\newcounter{dummy}

\usepackage{pdfpages}

\PassOptionsToPackage{usenames}{xcolor}
\definecolor{cite_color}{RGB}{0, 0, 255}
\definecolor{link_color}{RGB}{153, 0,0}  
\definecolor{url_color}{RGB}{153, 102,  0}
\definecolor{emp_color}{RGB}{0,0,255}

\usepackage{mathrsfs}

\let \oldtextcircled \textcircled
\renewcommand{\textcircled}[1]{\oldtextcircled{\footnotesize #1}}

\usepackage{enumitem}
\setlist[itemize]{leftmargin=9mm}

\newcommand{\todo}[1]{\textcolor{red}{@todo:}#1}

\usepackage{wrapfig}

\newcommand{\submodulardg}{\algname{Submodular-DoubleGreedy}\xspace}
\newcommand{\drdg}{\algname{DR-DoubleGreedy}\xspace}

\newcommand{\stepsize}{step size\xspace}
\newcommand{\stepsizes}{step sizes\xspace}

\DeclarePairedDelimiter\floor{\lfloor}{\rfloor}

\newcommand{\pga}{\algname{PGA}\xspace}
\newcommand{\nonconvexfw}{\algname{Non-convex FW}\xspace}
\newcommand{\submodularfw}{\algname{Submodular FW}\xspace}
\newcommand{\shrunkenfw}{\algname{Shrunken FW}\xspace}
\newcommand{\twophasefw}{\algname{Two-Phase FW}\xspace}
\newcommand{\twophase}{\algname{Two-Phase}\xspace}
\newcommand{\mesuper}{{IR-supermodular}\xspace}
\newcommand{\me}{{IR}\xspace}
\newcommand{\irsuper}{IR-supermodular\xspace}
\newcommand{\ir}{IR\xspace}

\newcommand{\alg}{\ensuremath{\mathscr{A}}\xspace}

\def\C{{\cal C}}

\def\E{{\mathbb E}}
\def\X{{\cal X}}

\def\H{{\mathbb  H}}
\def\Z{{\mathbb Z}}

\def\cone{{\cal K}}

\def\A{{\mathscr A}}
\def\M{{\cal M}}
\def \I {{\cal{I}}}
\def\T {\mathbb{T}}
\def\opt{\ensuremath{\Omega^*}}
\def\optcont{\ensuremath{\x^*}}
\newcommand{\opti}[1]{x^*_{#1}} 

\def \c{\mathbf{c}}
\def \v{\mathbf{v}}

\def \r{\mathbf{r}}
\def \b{\mathbf{b}}
\def \d{\mathbf{d}}
\def \a{\mathbf{a}}
\def \p{\mathbf{p}}
\def \q{\mathbf{q}}
\def \x{\mathbf{x}}
\def \y{\mathbf{y}}
\def \s{\mathbf{s}}
\def \e{\mathbf{e}}
\def \u{\mathbf{u}}

\def \t{\mathbf{t}}
\def \z{\mathbf{z}}
\def \h{\mathbf{h}}
\def \bh{\mathbf{h}}

\def \BA{\mathbf{A}}
\def \BB{\mathbf{B}}
\def \BX{\mathbf{X}}

\def \BI{\mathbf{I}}
\def \BC{\mathbf{C}}
\def \BD{\mathbf{D}}

\def \BH{\mathbf{H}}
\def \BL{\mathbf{L}}

\def \BR{\mathbf{R}}

\def \BU{\mathbf{U}}

\def \BW{\mathbf{W}}

\def \Q{{\cal{Q}}}

\def \R{{\mathbb{R}}}
\def \trans{\top}
\newcommand{\tr}[1]{\text{tr}(#1)}
\newcommand{\pare}[1]{{#1}}  

\def \D{{\mathcal{D}}}

\newcommand{\graphrv}{\ensuremath{\mathbb{G}}\xspace}

\def \E{{\mathbb{E}}} 

\newcommand{\MAXCUT}{\textsc{MaxCut}\xspace}

\newcommand{\lovasz}{Lov{\'a}sz \xspace}
\newcommand{\maxcut}{\textsc{MaxCut}\xspace}

\newcommand{\argmin}{{\arg\min}}
\newcommand{\diag}{{\text{diag}}}

\newcommand{\sign}[1]{{\text{sign}(#1)}}
\newcommand{\argmax}{{\arg\max}}
\newcommand{\algname}[1]{{\texttt{#1}}}

\newcommand{\spt}[1]{{\texttt{supp}}(#1)}
\newcommand{\dtp}[2]{\langle #1, #2\rangle}
\newcommand{\de}[1]{\text{det}\left(#1\right)}
\newcommand{\fracpartial}[2]{\frac{\partial #1}{\partial #2}}
\newcommand{\fracppartial}[3]{\frac{\partial^2 #1}{\partial #2 \partial #3}}

\newcommand{\lleq}{\preceq}
\newcommand{\ggeq}{\succeq}

\newcommand{\bas}{\mathbf{e}} 

\newcommand{\groundset}{\ensuremath{\mathcal{V}}}
\newcommand{\gndset}{\ensuremath{\mathcal{V}}}
\newcommand{\dims}{n} %
\newcommand{\ltwo}[1]{\|#1\|}
\newcommand{\lone}[1]{\|#1\|_1}

\def \nmf {\algname{CoordinateAscent}\xspace}
\def \ca {\algname{CoordinateAscent}\xspace}
\def \dgmf {\algname{DG-MeanField}\xspace}
\def \palong {{Posterior Agreement}\xspace} 

\newcommand{\NEWDR}{\texttt{weak DR}}

\newcommand{\set}[1]{\{#1\}} 
\newcommand{\sete}[3]{\mathbf #1|_{#2} (#3)}

\newcommand{\parti}{\text{Z}} 
\newcommand{\zero}{\mathbf{0}} 
\newcommand{\one}{\mathbf{1}} 
\newcommand{\ele}{v} 
\newcommand{\multi}{\ensuremath{f_{\text{mt}}}} 

\newcommand{\prob}{\ensuremath{\mathbb{P}}} 

\renewcommand{\P}{\mathcal{P}}

\newcommand{\bigo}[1]{\mathcal {O}\! \left(#1\right)}

\newcommand{\epe}[2][]{\underset{#1}{\mathbb E}\left[#2\right]}  

\newcommand{\pa}{\text{PA}\xspace}
\newcommand{\intermed}{\mathbf{o}}
\def \bu{\mathbf{u}}

\def \m{\mathbf{m}}

\def \flweights {R} 

\def \data {\text{D} }

\def \nmf {\text{CoordinateAscent}\xspace}
\def \dgmf {\text{DG-MeanField}\xspace}

\renewcommand{\mid}{|}
\newcommand{\kl}[2]{\mathbb{KL}(#1\|#2)}
\newcommand{\entropy}[1]{\mathbb{H}(#1)}
\def \chara{\mathbf{e}} 

\def \bmA{\mathbf{A}}
\def \bmH{\mathbf{H}}
\def \bmL{\mathbf{L}}
\def \bmX{\mathbf{X}}
\def \bmY{\mathbf{Y}}
\def \bmI{\mathbf{I}}
\def \bmW{\mathbf{W}}
\def \bmZ{\mathbf{Z}}

\def \cg{L}
\def\A{\mathscr{A}}

\def \bmtheta{\bm{\theta}}

\def \bmalpha{\bm\alpha}

\def\C{{\cal C}}

\def\E{{\mathbb E}}
\def\X{{\cal X}}
\def\H{{\cal H}}

\newcommand{\sett}[1]{\{#1\}}

\def\M{{\mathscr M}}
\def \I {{\cal{I}}} 
\def\T {\mathbb{T}}   
\newcommand{\CB}{T}   
\def\da{{G}}
\newcommand{\datilde}{{\tilde{\da}}}
\def \t{\tau} 

\def \bmtheta{\bm{\theta}}

\def \se {\mathfrak{S}}
\def \re {\mathfrak{R}}
\def \pg {\mathfrak{PG}}
\def \bmtheta{\bm{\theta}}
\def \bmalpha{\bm{\alpha}}

\def \trans{\top}
\def \tr{\text{tr}}
\def \h{c}
\newcommand{\kernel}{k}
\newcommand{\kert}{\tilde{\kernel}}

\newcommand{\ALG}{Algorithm }
\def\E{{\mathbb E}}
\def\X{{\cal X}}
\def\H{{\cal C}}
\def\Z{{\cal G}}

\newcommand{\cardH}{{\textstyle\frac{|\H|}{1+\rho}}}

\def \bmtheta{\bm{\theta}}

\def \bmalpha{\bm{\alpha}}

\def \bmA{\mathbf{A}}

\def \bmX{\mathbf{X}}
\def \bmI{\mathbf{I}}

\def \bmL{\mathbf{L}}

\newcommand{\geqco}{\ensuremath{\gtrsim}}
\newcommand{\leqco}{\ensuremath{\lesssim}}

\makeatletter
\definecolor{safe}{HTML}{3a87ad}

\ifthenelse{\boolean{@linedheaders}}%
{
	\titleformat{\chapter}[display]%
	{\relax}
	{\raggedleft{\color{halfgray}\chapterNumber\thechapter} \\ }
	{0pt}%
	{\color{safe}\LARGE\titlerule\vspace*{.9\baselineskip}\raggedright\spacedallcaps}
	[\normalcolor\normalsize\vspace*{.8\baselineskip}\titlerule]%
}
{
	\titleformat{\chapter}[display]%
	{\relax}
	{\mbox{}\oldmarginpar{\vspace*{-3\baselineskip}\color{halfgray}\chapterNumber\thechapter}}
	{0pt}%
	{\color{safe}\LARGE\raggedright\spacedallcaps}
	[\normalcolor\normalsize\vspace*{.8\baselineskip}\titlerule]%
}
\titleformat{\section}
{\normalfont\large\bfseries}
{\thesection.}
{.3em}
{\phantomsection}

\titleformat{\subsection}
{\normalfont\large}{\thesubsection}{.6em}{\phantomsection}
\titleformat{\subsubsection}
{\normalfont\large\itshape}{\thesubsubsection}{.9em}{\phantomsection}
\makeatother

\makeatletter
\renewcommand{\@chapapp}{}
\newenvironment{chapquote}[2][2em]
{\setlength{\@tempdima}{#1}%
	\def\chapquote@author{#2}%
	\parshape 1 \@tempdima \dimexpr\textwidth-2\@tempdima\relax%
	\itshape}
{\par\normalfont\hfill--\ \chapquote@author\hspace*{\@tempdima}\par\bigskip}
\makeatother

\newcommand*{\blankpage}{%
	\par\vspace*{\fill}%
	{\centering\emph{This page was intentionally left blank.}\par}
	\vspace{\fill}%
}
\makeatletter
\renewcommand*{\cleardoubleoddstandardpage}{%
	\clearpage
	\if@twoside
	\ifodd\c@page
	\else
	\blankpage
	\thispagestyle{empty}%
	\newpage
	\if@twocolumn\hbox{}\newpage\fi
	\fi
	\fi
}
\makeatother

\begin{document}

\selectlanguage{american}

\frontmatter

\begin{titlepage}
    \begin{center}
        \large
        \begingroup
            \spacedlowsmallcaps{Diss. ETH No. 26496}
        \endgroup

        \hfill

        \vfill

        \begingroup
        {\LARGE \bfseries  
         Provable Non-Convex Optimization and Algorithm Validation via Submodularity}
        \endgroup

        \vfill

        \begingroup
            A thesis submitted to attain the degree of\\[1em]
            \spacedallcaps{Doctor of Sciences }of\spacedallcaps{ ETH Zurich}\\
            (Dr. sc. ETH Zurich)
        \endgroup

        \vfill

        \begingroup
            presented by\\[1em]
            \normalsize{\spacedallcaps{Yatao (An) Bian}}\\[1em]
            Master of Science in Engineering\\
            Shanghai Jiao Tong University\\
            born on 17.01.1990\\
            citizen of China\\ 
        \endgroup

        \vfill

        \begingroup
            accepted on the recommendation of\\[1em]
            Prof.\ Dr.\ Joachim M. Buhmann, examiner\\
            Prof.\ Dr.\ Andreas Krause, co-examiner\\
            Prof.\ Dr.\ Yisong Yue, co-examiner
        \endgroup

        \vfill

        \myTime

        \vfill

    \end{center}
\end{titlepage}

\thispagestyle{empty}

\hfill

\vfill

\noindent\myName: \textit{\myTitle,} 
\textcopyright\ \myTime

%
%

%
%
%
%
%

\cleardoublepage

\cleardoublepage
\def\dir{frontbackmatter}
\pdfbookmark[1]{Abstract}{Abstract}
\begingroup
\let\clearpage\relax
\let\cleardoublepage\relax
\let\cleardoublepage\relax

\chapter*{Abstract}

Submodularity is one of the most well-studied properties of problem
classes in combinatorial optimization and many applications of machine
learning and data mining, with strong implications for guaranteed
optimization.  In this thesis, we investigate the role of
submodularity in provable non-convex optimization and validation of
algorithms.

A profound understanding which classes of functions can be tractably
optimized remains a central challenge for non-convex optimization. By
advancing the notion of submodularity to continuous domains (termed
``continuous submodularity''), we characterize a class of generally
non-convex and non-concave functions -- \emph{continuous submodular
  functions}, and derive algorithms for approximately maximizing them
with strong approximation guarantees.  Meanwhile, continuous
submodularity captures a wide spectrum of applications, ranging from
revenue maximization with general marketing strategies, MAP inference
for DPPs to mean field inference for probabilistic log-submodular
models, which renders it as a valuable domain knowledge in optimizing
this class of objectives.

Validation of algorithms is an information-theoretic framework to
investigate the robustness of algorithms to fluctuations in the input
/ observations and their generalization ability. We
investigate various algorithms for one of the paradigmatic
unconstrained submodular maximization problem: \MAXCUT.  Due to
submodularity of the \MAXCUT objective, we are able to present
efficient approaches to calculate the algorithmic information content
of \MAXCUT algorithms.  The results provide insights into the
robustness of different algorithmic techniques for \MAXCUT.

\vfill

\pagebreak

\selectlanguage{german}

\pdfbookmark[1]{Zusammenfassung}{Zusammenfassung}

\chapter*{Zusammenfassung}
Submodularit{\" a}t ist eine der am besten erforschten Eigenschaften von
Problemklassen in der
kombinatorischen Optimierung. 
Sie findet Anwendung in Bereichen des maschinellen Lernens und des
Data-Minings. Submodularit{\" a}t liefert ausserdem wesentliche Grundlagen für algorithmische Garantien in der Optimierung. In dieser
Arbeit untersuchen wir die Rolle von Submodularit{\" a}t in
nicht-konvexer Optimierung sowie in der Validierung von Algorithmen.

Eine zentrale Herausforderung im Bereich der nicht-konvexen
Optimierung liegt darin, das Verst{\" a}ndnis {\"u}ber
Funktionsklassen, welche nachweislich optimiert werden k{\"o}nnen, zu
erweitern. Indem wir den Begriff von Submodularit{\" a}t auf den
kontinuierlichen Bereich {\"u}bertragen (bezeichnet als
„kontinuierliche Submodularit{\" a}t”), k{\"o}nnen wir eine allgemeine
Klasse von nicht-konvexen und nicht-konkaven Funktionen
beschreiben. Wir entwickeln Algorithmen, die diese kontinuierlichen
submodularen Funktionen mit beweisbaren Garantien approximativ
optimieren k{\"o}nnen. Die kontinuierliche Submodularität eröffnet ein
breites Anwendungsspektrum, das von Umsatzmaximierung mit allgemeinen
Vermarktungsstrategien, MAP-Inferenz f{\"u}r DPPs bis hin zur
approximativen Inferenz mittels der „Mean-field'' Näherung f{\"u}r
probabilistische log-submodulare Modelle reicht.

Die Validierung von Algorithmen ist ein informationstheoretisches
Konzept, das die Robustheit gegen{\"u}ber Fluktuationen in den
Eingabe-Daten bzw. Beobachtungen {\"u}berpr{\"u}ft. Das Konzept
untersucht damit die Generalisierungsfähigkeit eines Algorithmus. Wir untersuchen verschiedene Algorithmen  f{\"u}r eines
der paradigmatischen submodularen
Maximierungsprobleme: \maxcut. Aufgrund der Submodularität der
\maxcut\ Kostenfunktion k{\"o}nnen wir effiziente Ansätze zur Berechnung des
algorithmischen Informationsgehaltes von \maxcut-Algorithmen
herleiten. Die Resultate liefern Einblicke in die Robustheit der
verschiedenen algorithmischen Verfahren f{\"u}r \maxcut.

\selectlanguage{american}

\endgroup

\vfill

\cleardoublepage
\def\dir{frontbackmatter}
\pdfbookmark[1]{Publications}{Publications}
\chapter*{Publications}

The following publications\footnote{My name was also written as (Andrew) An Bian due to a name change. My ORCID iD is \href{https://orcid.org/0000-0002-2368-4084}{orcid.org/0000-0002-2368-4084}.} are included in this thesis:
\begin{itemize}
	\item \fullcite{bian2019optimalmeanfield}

    \item \fullcite{bian2017guaranteed}
    
    \item \fullcite{biannips2017nonmonotone}
       
     \item \fullcite{bian2016information}
     
     \item \fullcite{ITW15_BianGB}
\end{itemize}


The following publications were part of my PhD research, are however not covered in this thesis. The topics of these publications are
outside of the scope of the material covered here:
\begin{itemize}
    \item \fullcite{bian2013parallel}
    
    \item \fullcite{he2018cola}\\
    {\small\textit{$^*$ Authors contributed equally.}}
    
    \item \fullcite{Celestine2018trust}
    
    \item \fullcite{bianicml2017guarantees}
    
    \item \fullcite{gorbach2017model}.\\
    {\small\textit{$^*$ Authors contributed equally.}}
\end{itemize}

\cleardoublepage
\def\dir{frontbackmatter}
\pdfbookmark[1]{Acknowledgments}{acknowledgments}

%
%

\chapter*{Acknowledgments}

I am deeply indebted to my supervisor, Prof. Joachim M. Buhmann, for his boundless generosity of encouragement, patience, advice and enthusiasm. I would like to thank him for providing the opportunity to work in his group and, allowing much freedom in exploring various topics. He always provides me support and guidance in both research and life, whenever I came to the door of his office. I am deeply grateful to Prof. Andreas Krause, for his generosity of time, insight, and friendship, who provides much more than a co-examiner and a collaborator could; To Prof. Yisong Yue, for taking time to read through the draft of my thesis, giving valuable comments and examine me.  To Prof. Martin Jaggi, for his always patience and kindness when interacting with me; To Rita Klute, who cares for us like her own children; To Rebekka Burkholz, for the warmth, encouragement and optimism she brings to us; To Yuxin Chen, who treats me like a brother, for his always patience and constant support whenever I had a difficulty; To Kaixiang Zhang, for being my best friend and brother; to Shuangying Jiang, for the encouragement and deep communications we had ever since the high school, for being a friend like my sister; To Alex Gronskiy, for giving me advice on my first research program; To Sebastian Tschiatschek, for sharing with me the joy of his son; To Hadi Daneshmand, for the warm chats with him and support from him; To Luis Haug, for the happy chats while we were drinking together; To Jie Song, for his generous help ever since I started my PhD program and being one of my best friends; To Kfir Levy, for letting me know the pure joy of doing research; To David Balduzzi, for his generous suggestions and recommendations; To Lie He, for his smart questions which drive me to think deeper; To Gabriel Krummenacher, for teaching me how to be a TA; To Gideon Dresdner, for his humor that blends American and Chinese cultures; To Max Paulus,  for always ``pushing’’ me to join the rowing team; To Hamed Hassani, for his advice when I encountered a difficult rebuttal; To Baharan Mirzasoleiman, for her patient discussions when I started to work on submodularity; To Dima Laptev, for training me to be a qualified IT coordinator; To Yannic Kilcher, for letting me know the charm of a ``super condi’’; To Mohammad Reza Karimi, for his positive attitude towards life and everyone around; To Alina Dubatovka, for her sense of responsibility and frankness when interacting with us; To Nico Gorbach, for letting me know how to live a balanced life; To Djordje Miladinovic, for introducing me cool bars and ``interesting’’ places; To Stefan Bauer, for sharing with me the pain and joy of a doctoral program during lunches and dinners; To Viktor Wegmayr,  for his encouraging words and optimism he inspires; To Aytunc Sahin, for his humor and support; To Zeke Wang, for the joint dinners, travels and sports; To Yuheng Zhang, the only philosopher I know, for leading me to think beyond techniques; To Jianrong Wen, for organizing various sport events in Zurich; To Han Wu, the best mathematician I know, for his generous help in solving a difficult geometric problem; To Philippe Wenk, for sharing with me encouraging stories when I had a bad mood; To Stefan Stark, for sharing with me the story of being a Stark (of GOT).

Many thanks to my other colleagues in the Institute for Machine Learning, who taught me a lot during the numerous occasions, Peter Schüffler, Judith Zimmermann, Josip Djolonga, Paolo Penna, Luca Corinzia, Fabian Laumer, Ivan Ovinnikov, Adish Singla,  Xinrui Lyu, Felix Berkenkamp, Zalán Borsos, Charlotte Bunne, Sebastian Curi, Johannes Kirschner, Anastasia Makarova, Mojmír Mutný, Matteo Turchetta, Aurelien Lucchi,  Celestine Dünner, Carsten Eickhoff, Octavian Ganea, Paulina Grnarova, Florian Schmidt,  Jonas Kohler, Stephanie Hyland, Matthias Hüser, Harun Mustafa, Vincent Fortuin, Natalia Marciniak, Mikhail Karasikov, for the great time we spent together.

Lots of thanks also to countless other friends (there are too many to list, so I will sample some randomly): Yanan Sui, Liwei Wang, Wen Li, Johann Gangji,  Xu Chen, Jinlong Tu, Mengmeng Deng, Ning Yang, Xiangyang Liu, Benjamin Fischer, Bin Huang, Xuanlong Guo, Xinlei Qiu, Bernd Deffner, Meng Li, Jing Yang, Guang Lu, Meijun Liu, Meng Liu, Lysie Champion, Yuhua Chen,  Wuyan Wang, Cen Nan, Jiajia Liu, Stanley Chan, Chen Chen, Feng Lue, Zhonghai Wang,  Peidong Liu, for their support and for the wonderful time we spent together and still spend together.

I also would like to thank Prof. Yuncai Liu, who guided me to the realm of research during my master program; To Jian Song, one of the best programmers I know, who led me into the area of parallel computing; To Xiong Li, for the early guidance of doing scientific research; To Junchi Yan, who gave me countless suggestions; To Prof. Ming-Hsuan Yang, for the instructions of writing a scientific paper.

I owe a lot to my family, for their unconditional support and love, without which nothing would be possible. I am grateful to my father, who provided me love, tolerance and guidance when I was young; To my sister for her caring, for always listening to my complaints and joys; Especially to my mother for her incalculable effort in taking care of the family by herself, for her faith in me and her dedication to my success – It is to her I dedicate this dissertation. Lastly, my utmost appreciation goes to my beloved girlfriend, for her caring, love and understanding during my good and bad times. Holding a PhD herself, she understands me more than anyone else could; She always cheers me up when I have a hard time; Without her nothing would be worthwhile.

\cleardoublepage
\def\dir{frontbackmatter}
\refstepcounter{dummy}
\pdfbookmark[1]{\contentsname}{tableofcontents}
\setcounter{tocdepth}{2} 
\setcounter{secnumdepth}{3} 
\manualmark
\markboth{\spacedlowsmallcaps{\contentsname}}{\spacedlowsmallcaps{\contentsname}}
\tableofcontents
\automark[section]{chapter}
\renewcommand{\chaptermark}[1]{\markboth{\spacedlowsmallcaps{#1}}{\spacedlowsmallcaps{#1}}}
\renewcommand{\sectionmark}[1]{\markright{\thesection\enspace\spacedlowsmallcaps{#1}}}
\clearpage

\begingroup
    \let\clearpage\relax
    \let\cleardoublepage\relax
    \let\cleardoublepage\relax
    \refstepcounter{dummy}
    \pdfbookmark[1]{\listfigurename}{lof}
    \listoffigures

	\vfill 
    \pagebreak 

    \refstepcounter{dummy}
    \pdfbookmark[1]{\listtablename}{lot}
    \listoftables
\endgroup

\cleardoublepage

\mainmatter

\def\dir{chapters/1-introduction}
\chapter{Introduction}
\label{chapter_intro}

\begin{chapquote}{Confucius}
I hear and I forget. I see and I remember. I do and I understand.
\end{chapquote}

\section{What is Submodularity over Binary Domains?}

Submodularity  is  a  structural   property  usually  associated  with
\emph{set  functions}, with  important  implications for  optimization
\citep{nemhauser1978analysis}.  The general  setup
requires a groundset $\groundset$ containing $n$ items, which could
be, for instance, all the features in supervised learning problems, or
all sensor locations in sensor placement. Usually we have an objective
function which maps a subset of $\groundset$ to a real value:
$F(X): 2^\groundset \rightarrow \R_+$, which often measures utility,
coverage, relevance etc.

Equivalently, one can express any subset $X$ as a binary vector
$\x\in \{0, 1\}^n$: component $i$ of $\x,\;x_i=1$ means that item $i$
is inside $X$, otherwise item $i$ is outside of $X$. This binary
representation associates the powerset of $\groundset$
with all vertices of an $n$-dimensional hypercube.  Because of this, we
also call submodularity of set functions ``submodularity over binary
domains''.

Over binary domains, there are two famous definitions of
submodularity: the submodularity definition and the diminishing
returns (DR) definition.

\begin{definition}[Submodularity definition]
  A set function $F(X): 2^\groundset \mapsto \R$ is submodular iff
  $\forall X, Y \subseteq \groundset$, it holds:
\begin{align}
	F(X) + F(Y) \geq F(X\cup Y) + F(X \cap Y).
\end{align}
\end{definition}

One can easily show that it is equivalent to the following DR definition:

\begin{definition}[DR definition]
  A set function $F(X): 2^\groundset \mapsto \R$ is submodular iff
  $\forall A \subseteq B \subseteq \groundset$ and
  $\forall v \in \groundset \setminus B$, it holds:
	\begin{align}
	F(A \cup \set{v}) - F(A) \geq F(B\cup \set{v}) - F(B).
	\end{align}
\end{definition}

Optimizing submodular set functions has found numerous applications in
machine learning, including variable selection
\citep{DBLP:conf/uai/KrauseG05}, dictionary learning
\citep{krause2010submodular,das2011submodular}, sparsity inducing
regularizers \citep{bach2010structured}, summarization
\citep{lin2011class,mirzasoleiman2013distributed} and variational
inference \citep{djolonga2014map}. Submodular set functions can be
efficiently minimized \citep{iwata2001combinatorial}, and there are
strong guarantees for approximate maximization
\citep{nemhauser1978analysis,krause2012submodular}.

\section{Why Do We Need  Continuous Submodularity?}

Continuous submodularity essentially captures the weak diminishing returns
phenomenon over continuous domains.  In summary, there are two
motivations for studying continuous submodularity: i) It is an
important modeling ingredient for many real-world applications; ii) It
captures a subclass of well-behaved non-convex optimization problems,
which admits guaranteed  approximate optimization with algorithms running in
polynomial time.

\subsection{Natural Prior Knowledge for Modeling}

In order to illustrate the first motivation, let us consider a virtual
scenario here. Suppose you got stuck in the desert one day, and became
extremely thirsty. After two days of exploration you found a bottle of
water, what is even better is that you also found a bottle of coke.

At this very moment, let us use a two-dimensional function
$f([x_1; x_2])$ to quantize the ``happiness'' gained by having $x_1$
quantity of water and $x_2$ quantity of coke. Let
$\delta = [50 \text{ml water} ; 50 \text{ml coke}]$. Now it is natural
to see that the following inequality shall hold:
$f([1ml; 1ml] + \delta) - f([1ml; 1ml]) \geq f([100ml;
100ml] + \delta) - f([100ml; 100ml])$.
Due to the diminishing returns property, the LHS of the inequality measures the marginal
gain of happiness by having $\delta$ more [water, coke] based on a
\emph{small} context ([1ml; 1ml]), while the RHS means the marginal
gain based on a \emph{large} context ([100ml; 100ml]).
The diminishing returns (DR) property models the context sensitive
expectation that adding one more unit of resource
contributes more in the small context than in a large context.

Now it is straightforward to see that DR is a natural component in
many real-world models. For example, user preference in recommender
systems, customer satisfaction, influence in social advertisements
etc.

\subsection{A Provable Non-Convex Structure}

Non-convex optimization delineates the new frontier in machine
learning, since it arises in numerous learning tasks from training
deep neural networks to latent variable models
\citep{anandkumar2014tensor}.  A fundamental problem in non-convex
optimization is to reach a stationary point assuming smoothness of the
objective for unconstrained optimization
\citep{sra2012scalable,li2015accelerated,reddi2016fast,allen2016variance}
or constrained optimization problems
\citep{ghadimi2016mini,lacoste2016convergence}.
However, without proper assumptions, a stationary point may not lead
to any global approximation guarantee.  It remains a challenging
problem to understand which classes of non-convex objectives can be
tractably optimized.

In pursuit of solving this challenging problem, we show that
continuous submodularity provides a natural structure for
provable non-convex optimization problems.
It shows up in various important non-convex objectives. Let us look at
a simple example by considering a classical quadratic program (QP):
$f(\x) = \frac{1}{2}\x^\trans \BH \x + \bh^\trans \x + c$. When $\BH$
is symmetric, we know that the Hessian matrix is $\nabla^2 f =
\BH$.
Let us consider a specific two dimensional example, where
$\BH = [-1, -2; -2, -1]$, one can verify that its eigenvalues are
$[1; -3]$. So it is an indefinite quadratic program, which is
neither convex, nor concave. However, it will soon be clear that it is a
DR-submodular function after you have read the definitions in
\cref{chapter_characterizations}, and we have proposed polynomial-time
solvers to optimize it with strong approximation guarantees.

This structure has been used in various non-convex objectives, which
might been known for decades. People may have developed different
algorithms to solve them. However, previously researchers did not
realize that they share this common structure. Examples include but
are not limited to the QPs  studied in \citet{kim2003exact}, the  \lovasz \citep{lovasz1983submodular} and
multilinear extensions \citep{calinescu2007maximizing} of submodular
set functions, or to the softmax extensions \citep{gillenwater2012near} for
DPP (determinantal point process) MAP inference.

\section[Algorithmic Information Content]{Analysis of \MAXCUT\
  Algorithms via Algorithmic Information Content}
\label{sec_intro_alg_information}

Algorithmic information content is originally motivated by the
approximation set coding (ASC) framework \citep{Buhmann10isit,
  JB:mcpr:2011,buhmann2013simbad}, and it measures the amount of
information that an algorithm can extract from noisy observations of data
instances. So it is a natural criterion for studying the robustness of
algorithms.

For algorithmic analysis in the general setting, we investigate the
generalization ability of an algorithm $\A$ under the
\emph{two-instance scenario}, which assumes a
generative process of data instances: i) Generate a
``master instance'' $G$, e.g., a complete graph with Gaussian distributed edge
weights; ii) Generate two data instances
$G^\prime,\;G^{\prime\prime}$ by independently applying a noise
process to the master instance $G$.
With an abuse of notation, we use $\graphrv, \graphrv'$ and $\graphrv''$ to denote the corresponding  random variables in this generative process, and use $G, G', G''$ to represent the realizations. The dependence relationship of these random variables can be described by the graphical model in \cref{fig_two_instance}.

\begin{figure}[htbp]
	\centering
	\includegraphics[width=0.5\textwidth]{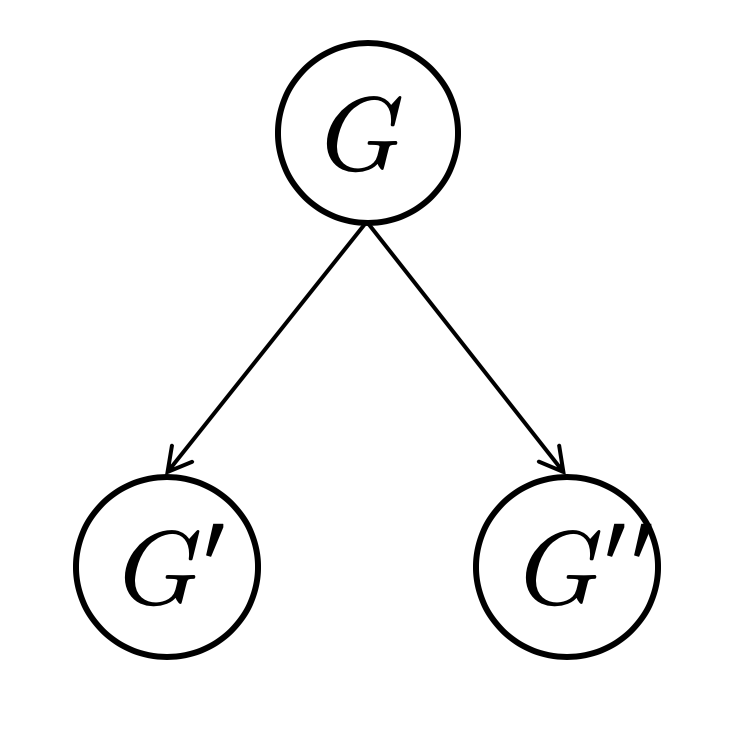}
	\caption{Graphical model induced by the two-instance scenario.}
	\label{fig_two_instance}
\end{figure}

The algorithm \alg  then calculates a sequence of
posteriors $\{\prob_t^{\A}(c|G')\}$, $\{\prob_t^{\A}(c|G'')\}$ as a
function of time $t$.  The variable $c$ denotes a solution in the hypothesis/solution
space $\H$.  The \emph{posterior agreement} (PA) criterion is
defined to measure the overlap between the two posteriors at time $t$,
\begin{equation}\label{eq:pa}
k_t^{\A}(G', G''):= \sum\nolimits_{c\in \H}\prob_t^{\A}(c|G')\prob_t^{\A}(c|G''). \quad  \texttt{(PA)}
\end{equation}

We define the \textit{information content} of an algorithm $\A$ as the
maximal \textit{temporal} information content $I_t^{\alg}(\graphrv'; \graphrv'')$
at time $t$:
\begin{align}\label{eq_ic}
I^{\alg}(\graphrv'; \graphrv'') &:=
\max_t  I_t^{\mathscr{A}}(\graphrv'; \graphrv'')  \\\notag
& = \max_t  \mathbb{E}_{G', G''}
\bigl[
\log \bigl( |\H| k_t^{\A}(G', G'') \bigr)
\bigr].
\end{align}

It generalizes the algorithmic information content of
\citet{informativemst}.
$ I^{\mathscr{A}}_t (\graphrv'; \graphrv'')$ measures how much
information is extracted by ${\mathscr{A}}$ at time $t$ from the
input data that is relevant to the output data, thus
reflecting the generalization ability.
Note that the definition can be easily generalized for continuous
algorithms by interpreting $t$ as the running time.

The algorithmic information content  naturally suggests the following algorithm regularization and
validation strategy:
\begin{itemize}
\item \textit{Regularize} an algorithm $\A$ by stopping it at the
  optimal time, which is defined as
  $t^* = \arg\max_t \mathbb{E}_{G', G''} \bigl[ \log \bigl( |\H|
  k_t^{\A}(G', G'') \bigr) \bigr]$.
  It corresponds to the well-known early-stopping strategy
  \citep{giles2001overfitting};
\item \textit{Validation}: Use $I^\A$ to measure the generalization ability of an algorithm \alg. According to this measure, we can, for example, search for generalizable algorithms
  under a specific
  data generation process.
\end{itemize}

\MAXCUT is one typical instance  of the unconstrained submodular maximization (USM) problem.
It is used in various scenarios, such as
semi-supervised learning (\cite{Wang:2013:SLU:2502581.2502605}),   opinion mining in
social networks  (\cite{agrawal2003mining}),  statistical physics and
 circuit layout design (\cite{barahona1988application}).
Beside
\MAXCUT, USM captures many practical problems such as
\textsc{MaxDiCut} \citep{halperin2001combinatorial}, variants of
\textsc{MaxSat} and the maximum facility location problem
\citep{cornuejols1977uncapacitated,ageev19990}.

Submodularity plays an important role in information-content based
analysis for \MAXCUT algorithms.  Due to the submodular nature of the
\MAXCUT objective, we can design efficient methods to calculate the
algorithmic information content of several \MAXCUT algorithms, so as
to conduct efficient analysis of these algorithms.

\section{Contributions and Thesis Structure}
\label{sec_contributions_structure}

\subsection{Contributions}

In this work we investigate the role of submodularity in guaranteed
non-convex optimization and algorithm validation, which results in the
following contributions:

\textbf{For non-convex optimization: }

\begin{enumerate}

\item By lifting the notion of submodularity to continuous domains, we
  identify a subclass of tractable non-convex optimization problems:
  continuous submodular optimization. We provide a thorough
  characterization of continuous submodularity, which results in
  $0^{\text{th}}$ order, $1^{\text{st}}$ order and $2^{\text{nd}}$
  order definitions.

\item We propose hardness results and provable algorithms for
  constrained submodular maximization in three settings: i) Maximizing
  monotone functions with down-closed convex constraints; ii)
  Maximizing non-monotone functions with box constraints; iii)
  Maximizing non-monotone functions with down-closed convex
  constraints.

\item We present representative applications with the studied
  continuous submodular objectives, and extensively evaluate the
  proposed algorithms on these applications.

\end{enumerate}

\textbf{For algorithm validation:}

\begin{enumerate}
\item Motivated by the ``coding by posterior'' framework, we formulate
  the posterior agreement (PA) objective as a criterion for algorithm
  validation.

\item We present efficient approaches to evaluate the PA objective for
  various algorithms of the \MAXCUT problem, which is one classical instance of the unconstrained submodular maximization
  problem. The studied \MAXCUT algorithms involve different
  algorithmic techniques, such as greedy heuristics and semidefinite
  programming relaxation.

\item We validate the \MAXCUT algorithms with extensive experiments on
  different synthetic graph instances.
\end{enumerate}

\subsection{Thesis Structure}

In \cref{sec-background} we present notations, background and related
work.  In \cref{chapter_characterizations} we firstly give a thorough
characterization of the class of continuous submodular and
DR-submodular\footnote{A DR-submodular function is a submodular
  function with the additional diminishing returns (DR) property, which
  will be formally defined in \cref{sec_defs_continuous_submodurity}.}
functions, then present some intriguing properties for the problem of
constrained DR-submodular maximization, such as the local-global
relation.  In \cref{chapter_applications} we illustrate representative
applications of continuous submodular optimization.

In the next three chapters we discuss hardness results and algorithmic
techniques for constrained DR-submodular maximization in different
settings: \cref{chapter_max_monotone} illustrates how to maximize
monotone continuous DR-submodular functions,
\cref{chapter_doublegreedy} studies box-constrained non-monotone  continuous
submodular maximization and \cref{chapter_constrained_nonmonotone}
provides techniques on maximizing non-monotone DR-submodular functions
with a down-closed convex constraint.

Chapters \labelcref{chapter_greedy_maxcut,chapter_Goemans_Williamson,chapter_mean_field}
contain details on algorithm and model validation with submodular
objectives: \cref{chapter_greedy_maxcut} shows efficient methods for
calculating the posterior agreement  of greedy \MAXCUT algorithms,
\cref{chapter_Goemans_Williamson} presents approximating techniques
for evaluating the posterior agreement for the classical
Geomans-Williamson's \MAXCUT algorithm, \cref{chapter_mean_field}
illustrates provable continuous submodular maximization algorithms to
approximately maximize the mean field lower bound of posterior
agreement.

Lastly, \cref{chapter_disc_future_work} discusses potential future
directions and concludes the thesis.

\def\dir{chapters/2-background}
\chapter{Background} \label{sec-background}

\begin{chapquote}{Lao Tzu}
	A journey of a thousand miles begins with a single step.
\end{chapquote}

We will introduce important notations, background and
related work in this chapter.

\section{Notation}

Throughout this work we assume
$ \groundset =\{\ele_1, \ele_2,..., \ele_n\}$ being the ground set of
$n$ elements, and $\chara_i\in\R^n$ is the characteristic vector for
element $\ele_i$ (also the standard $i^\text{th}$ basis vector).
We use boldface letters $\x\in \R^\groundset$ and $\x\in \R^n$
interchangebly to indicate an $n$-dimensional vector, where $x_i$ is
the $i^\text{th}$ entry of $\x$. We use a boldface captial letter
$\BA\in\R^{m\times n}$ to denote an $m$ by $n$ matrix and use $A_{ij}$
to denote its ${ij}^\text{th}$ entry.
By default, $f(\cdot)$ is used to denote a continuous function, and
$F(\cdot)$ to represent a set function.  
For a differentiable function $f(\cdot)$, $\nabla f(\cdot)$ denotes its gradient, and for a twice differentiable function $f(\cdot)$, $\nabla^2 f(\cdot)$ denotes its Hessian. 
$[n]:= \{1,...,n\}$ for an
integer $n \geq 1$.  $\|\cdot\|$ means the Euclidean norm by default.
Given two vectors $\x,\y$, $\x\leqco \y$ means
$x_i\leq y_i, \forall i$.  $\x\vee \y$ and $\x \wedge \y $ denote
coordinate-wise maximum and coordinate-wise minimum, respectively.
$\sete{x}{i}{k}$ is the operation of setting the
$i^\text{th}$ element of $\x$ to $k$, while keeping all other elements
unchanged, i.e., $\sete{x}{i}{k}=\x-x_i \bas_i + k\bas_i$.

For the two-instance scenario in algorithm validation, we use \alg to denote an algorithm. With an abuse of notation, we use \graphrv to denote the random variable of a graph, and use $G$ as its realization. $I^{\alg}$ represents the algorithmic information content of \alg, and $\I$ denotes the classical mutual information.

\section{Related Work on Validation of Models and Algorithms}
\label{sec_validation_related_work}

Both model and algorithm validations are based on the posterior
agreement objective. It is motivated by the ``coding by posterior''
framework, which will be formally verified in \cref{sec_coding}.
On a high level, it is motivated by an analogue to the noisy communication
channel in Shannon's information theory \citep{cover2012elements}.

\citet{Buhmann10isit,JB:mcpr:2011} propose the approximation set
coding (ASC) framework to conduct model selection for K-means
clustering. Then it is used as a criterion to determine the rank for a
truncated singular value decomposition \citep{frank2011selecting} and
do model selection for spectral clustering
\citep{chehreghani2012information}.
It is further developed as a principled way to evaluate generalization
of algorithms for sorting algorithms \citep{busse2012information},
minimum spanning tree algorithms
\citep{informativemst,gronskiy2018statistical} and greedy \MAXCUT
algorithms \citep{ITW15_BianGB}.

Posterior agreement  (PA) is a generalization of the ASC framework.  For model
validation, it determines an optimal trade-off between the
expressiveness of a model and robustness by measuring the overlap
between posteriors of the model parameter conditioned on the two data
instances.  It has been employed to conduct model selection for
Gaussian processes regression \citep{gorbach2017model} and algorithm
validation \citep{bian2016information}.
Recently, \citet{buhmannaposterior} prove rigorous asymptotics of PA
on two combinatorial problems: Sparse minimum bisection and
Lawler's quadratic assignment problem.

\section{Related Work on Submodular Optimization}

\subsection{Submodularity over Discrete Domains}

Submodularity is often viewed as a discrete analogue of convexity, and
provides computationally effective structure so that many discrete
problems with this property are efficiently solvable or approximable.
Of particular interest is a $(1-1/e)$-approximation for maximizing a
monotone submodular set function subject to a cardinality, a matroid,
or a knapsack constraint \citep{nemhauser1978analysis,
  DBLP:conf/stoc/Vondrak08, sviridenko2004note}.  For non-monotone
submodular functions, a 0.325-approximation under cardinality and
matroid constraints \citep{gharan2011submodular}, and a
0.2-approximation under knapsack constraint has been shown
\citep{lee2009non}.  Another result is unconstrained maximization of
non-monotone submodular set functions, for which
\citet{buchbinder2012tight} propose the deterministic double greedy
algorithm with a 1/3 approximation guarantee, and the randomized double
greedy algorithm which achieves the tight 1/2 approximation guarantee.

Although most commonly associated with set functions, in many
practical scenarios, it is natural to consider generalizations of
submodular set functions, including \textit{bisubmodular} functions,
\textit{$k$-submodular} functions, \textit{tree-submodular} functions,
\textit{adaptive submodular} functions, as well as submodular
functions defined over integer lattices.

\citet{golovin2011adaptive} introduce the notion of adaptive submodularity to
generalize submodular set functions to adaptive policies. 
\citet{kolmogorov2011submodularity} studies tree-submodular functions
and presents a polynomial-time algorithm for minimizing them. 
For distributive lattices, it is well-known that the combinatorial
polynomial-time algorithms for minimizing a submodular set function
can be adopted to minimize a submodular function over a bounded
integer lattice \citep{fujishige2005submodular}. 

Recently, maximizing a submodular function over integer lattices has
attracted considerable attention. In particular,
\citet{soma2014optimal} develop a $(1-1/e)$-approximation algorithm
for maximizing a monotone DR-submodular integer function under a
knapsack constraint. For non-monotone submodular functions over the
bounded integer lattice, \citet{gottschalk2015submodular} provide a
1/3-approximation algorithm. Approximation algorithms for maximizing
bisubmodular functions and $k$-submodular functions have also been
proposed by \citet{singh2012bisubmodular,ward2014maximizing}.
Recently, \citet{soma2018maximizing} present a continuous extension
for maximizing monotone integer submodular functions, which is
non-smooth.

\subsection{Submodularity over Continuous Domains}

Even though submodularity is most widely considered in the discrete
realm, the notion can be generalized to arbitrary lattices
\citep{fujishige2005submodular}.  
\citet{DBLP:journals/mor/Wolsey82} considers maximizing a special
class of continuous submodular functions subject to one knapsack
constraint, in the context of solving location problems. That class of
functions are additionally required to be monotone, piecewise linear
and \textit{concave}.
\citet{calinescu2007maximizing} and \citet{DBLP:conf/stoc/Vondrak08}
discuss a subclass of continuous submodular functions, which is termed
smooth submodular functions\footnote{A function
  $f: [0,1]^n \rightarrow \R$ is smooth submodular if it has second
  partial derivatives everywhere and all entries of its Hessian matrix
  are non-positive.}, to describe the multilinear extension of a
submodular set function.  They propose the continuous greedy
algorithm, which has a $(1-1/e)$ approximation guarantee on maximizing
a smooth submodular functions under a down-monotone polytope
constraint.  Recently, \citet{bach2015submodular} considers the
minimization of a continuous submodular function, and proves that
efficient techniques from convex optimization may be used for
minimization.

Recently, \citet{ene2016reduction} provide a reduction from an
integer DR-submodular function maximization problem to a submodular
set function maximization problem, which suggests a way to optimize
continuous submodular functions over \textit{simple} continuous
constriants: Discretize the continuous function and constraint to be
an integer instance, and then optimize it using the
reduction. However, for monotone DR-submodular functions maximization,
this method can not handle the general continuous constraints
discussed in this work, i.e., arbitrary down-closed convex sets. And
for general submodular function maximization, this method cannot be
applied, since the reduction needs the additional diminishing returns
property.  Therefore we focus on continuous methods in this work.

Very recently, \citet{niazadeh2018optimal} present optimal algorithms for
non-monotone submodular maximization with a box constraint.
Continuous submodular maximization is also well studied in the
stochastic setting \citep{hassani2017gradient,mokhtari2018stochastic},
online setting \citep{chen2018online}, bandit setting
\citep{durr2019non} and decentralized setting
\citep{mokhtari2018decentralized}.

\section{Classical Frank-Wolfe Style Algorithms} \label{sec_frank_wolfe}

Since the workhorse algorithms for continuous  DR-submodular maximization are
Frank-Wolfe style algorithms, we give a brief introduction of
classical Frank-Wolfe algorithms in this section.

The Frank-Wolfe algorithm \citep{frank1956algorithm} (also known as
Conditional Gradient algorithm or the Projection-Free algorithm) is
one of the 
classical algorithms for constrained convex
optimization. It has seen a revival in recent years due to its
projection free feature and its ability to exploit structured
constraints \citep{jaggi13}.

The Frank-Wolfe algorithm solves the following constrained
optimization problem:
\begin{align}
	\min_{\x\in \R^n, \; \x\in \D} f(\x),
\end{align}
where $f$ is differentiable with $L$-Lipschitz gradients and the
constraint $\D$ is convex and compact.

A sketch of the Frank-Wolfe algorithm is presented in
\cref{alg_classical_fw}.  It needs an initializer $\x^\pare{0}\in \D$.
Then it runs for $T$ iterations. In each iteration: in
Step \labelcref{alg_fw_lmo} it solves a linear minimization problem
whose objective is defined by the current gradient $\nabla f(\x^{t})$,
this step is often called the linear minimization/maximization oracle
(LMO); In Step
\labelcref{alg_fw_stepsize} a \stepsize $\gamma$ is chosen; Then it
updates the solution $\x$ to be a convex combination of the current
solution and the LMO output $\s$.

\begin{algorithm}[ t]
  \caption{Classical Frank-Wolfe algorithm for constrained convex
    optimization \citep{frank1956algorithm}}\label{alg_classical_fw}
	
	\KwIn{$\min_{\x\in \R^n, \x\in \D} f(\x)$; $\x^\pare{0}\in \D$ }
\For{$t=0\dots T$}{
  {Compute
    $\s^\pare{t} := \argmin_{\s\in \D} \left\langle \s, \nabla
      f(\x^{t}) \right\rangle$ \label{alg_fw_lmo} \tcp*{LMO}}

	{Choose \stepsize $\gamma \in (0, 1]$\;  \label{alg_fw_stepsize}} 
	
{	Update $\x^\pare{t+1}:= (1-\gamma)\x^{t}+\gamma\s^\pare{t}$\;}
}
	\KwOut{$\x^\pare{T}$\;}
\end{algorithm}

There are several popular rules to choose the \stepsize in Step 
\labelcref{alg_fw_stepsize}. For a short summary: i)
$\gamma_t := \frac2{t+2}$, which is often called the ``oblivious''
rule since it does not depend on any information of the optimization
problem; ii)
$\gamma_t = \min\{1, \frac{g_\pare{t}}{L \|\s^\pare{t} -
  \x^\pare{t}\|} \}$,
where
$g_\pare{t} := -\dtp{\nabla f(\x^{t})}{\s^\pare{t} - \x^\pare{t}} $ is
the so-called Frank-Wolfe gap, which is an upper bound of the
suboptimality if $f$ is convex; iii) Line search rule:
$\gamma_t := \argmin_{\gamma\in [0, 1]} f(\x^\pare{t} + \gamma
(\s^\pare{t} - \x^\pare{t}) )$.

\subsection{Frank-Wolfe Algorithm for Non-Convex Optimization}

Recently, Frank-Wolfe algorithms have been extended for smooth
non-convex optimization problems with
constraints. \citet{lacoste2016convergence} analyzed the Frank-Wolfe
method for general constrained non-convex optimization problems, where
he used the Frank-Wolfe gap as the non-stationarity measure.
\citet{reddi2016stochastic} studied Frank-Wolfe methods for non-convex
stochastic and finite-sum optimization problems. They also used the
Frank-Wolfe gap as the non-stationarity measure.

\section{Existing Structures for  Non-Convex Optimization}

\subsection{Quasi-Convexity}

A function $f: \D \mapsto \R$ defined on a convex subset $\D$ of a
real vector space is quasi-convex if for all $\x, \y\in \D$ and
$\lambda\in [0, 1]$ it holds,
\begin{align}
	f(\lambda \x + (1- \lambda)\y) \leq \max\{ f(\x), f(\y) \}.
\end{align}
Quasi-convex optimization problems appear in different areas, such as
industrial organization \citep{wolfstetter1999topics} and computer
vision \citep{ke2007quasiconvex}.
Quasi-convex optimization problems can be solved by a series of convex
feasibility problems \citep{boyd2004convex}.
\citet{hazan2015beyond} studied stochastic quasi-convex optimization,
where they proved that a stochastic version of the normalized gradient
descent can converge to a global minimium for quasi-convex functions
that are locally Lipschitz.

\subsection{Geodesic Convexity}

Geodesic convex functions are a class of generally
non-convex functions in Euclidean space. However, they
still enjoy the nice property that local optimum implies global optimum. 
\citet{sra2016geometric} provided a brief introduction to 
geodesic convex  optimization with machine learning applications.  
Recently, \citet{vishnoi2018geodesic} 
collected details on various aspects of
geodesic convex optimization.

\begin{definition}[Geodesically convex functions]
Let $(\M, g)$ be a Riemannian manifold and $K\subseteq \M$ be a totally convex set with respect to $g$. A function $f: K \rightarrow \R$ is a geodesically convex function with respect to $g$ if $\forall \p, \q \in K$, and for all geodesic $\gamma_{\p \q}: [0, 1]\rightarrow K$ that joins $\p$ to $\q$, it holds,
\begin{align}
\forall t\in [0, 1], f(\gamma_{\p \q}(t)) \leq (1- t) f(\p) + t f(\q).
\end{align} 
\end{definition}

Various applications with non-convex objectives in Euclidean space can be resolved with geodesic convex optimization methods, such as Gaussian mixture models \citep{hosseini2015matrix}, 
metric learning \citep{zadeh2016geometric} and matrix square root \citep{sra2015matrix}. 
By deriving explicit expressions for the smooth manifold structure, such as  inner products, gradients, vector transport and Hessian, various optimization methods have been developed. \citet{jeuris2012survey} presented conjugate gradient, BFGS and trust-region methods. \citet{qi2010riemannian} proposed the Riemannian BFGS (RBFGS) algorithm for general retraction and vector transport. \citet{ring2012optimization} proved its local superlinear rate of convergence. \citet{sra2015conic} presented a limited memory version of RBFGS.

\def\dir{chapters/3-characterizations}
\chapter{Characterizations and Properties of Continuous Submodular Functions}
\chaptermark{Characterizations \& Properties of Continuous Submodularity}
\label{chapter_characterizations}

\begin{chapquote}{Confucius}
	By three methods we may learn wisdom: First, by reflection, which is noblest; Second, by imitation, which is easiest; and third by experience, which is the bitterest.
\end{chapquote}

In order to systematically study continuous submodular optimization,
the first thing would be to investigate the characterizations of it. 
Similar as the definitions of convexity, continuous
submodularity can be described using $0^{\text{th}}$ order,
$1^{\text{st}}$ order and $2^{\text{nd}}$ order conditions, which will
be elaborated in \cref{sec_defs_continuous_submodurity}.
\cref{statement_continuous_submodular_max} states the problem of
constrained submodular maximization in continuous domains and
summarizes necessary assumptions of the analysis. In
\cref{sec_underlying_properties} we present several intriguing
properties of constrained DR-submodular maximization problems,
including concavity along non-negative/non-positive directions and the
local-global relation.  Finally, we investigate a generalized class of
submodular functions on ``conic'' lattices in \cref{sec_lattice}.
This focus allows us to model a larger class of non-trivial
applications that include logistic regression with a non-convex
separable regularizer, non-negative PCA, etc (for details see
\cref{subsec_moti_lattice}).
To optimize them, we provide a reduction that enables to invoke
algorithms for continuous submodular optimization problems.

\section{Characterizations of  Continuous Submodular  Functions}
\label{sec_defs_continuous_submodurity}

Continuous submodular functions are defined on subsets of $\R^n$:
$\X = \prod_{i=1}^n \X_i$, where each $\X_i$ is a compact subset of
$\mathbb{R}$ \citep{topkis1978minimizing, bach2015submodular}.  A
function $f: \X \rightarrow \R$ is submodular \textit{iff} for all
$(\x, \y)\in \X \times \X$,

\begin{equation}\label{eq1}
   f(\x) + f(\y) \geq f(\x \vee \y) + f(\x \wedge \y),  \quad
   (\emph{submodularity}) 
\end{equation}
 
where $\wedge$ and $\vee$ are the coordinate-wise minimum and maximum
operations, respectively.  Specifically, $\X_i$ could be a finite set,
such as $\{0, 1 \}$ (in which case $f(\cdot)$ is called a \textit{set}
function), or $\{0, ..., k_i-1 \}$ (called \textit{integer} function),
where the notion of continuity is vacuous; $\X_i$ can also be an
interval, which is referred to as a continuous domain. In this
section, we consider the interval by default, but it is worth noting
that the properties introduced in this section can be applied to
$\X_i$ being a general compact subset of $\R$.
 
When twice-differentiable, $f(\cdot)$ is submodular iff all
off-diagonal entries of its Hessian are non-positive\footnote{Notice
  that an equilavent definition of (\ref{eq1}) is that
  $\forall \x\in \X$, $\forall i \neq j$ and $a_i, a_j\geq 0$
  s.t. $x_i +a_i\in \X_i, x_j+a_j\in \X_j$, it holds
  $f(\x+a_i\bas_i) + f(\x+a_j\bas_j) \geq f(\x) + f(\x+a_i\bas_i +
  a_j\bas_j)$.
  With $a_i$ and $a_j$ approaching zero, one get (\ref{eq2}).}
\citep{bach2015submodular},
\begin{equation}\label{eq2}
  \forall \x\in \X, \;\; \frac{\partial^2 f(\x)}{\partial x_i \partial x_j}
  \leq 0, \;\; \forall i \neq j. 
\end{equation}

The class of continuous submodular functions contains a subset of both
convex and concave functions, and shares some useful properties with
them (illustrated in \cref{fig_sub}).  Examples include submodular and
convex functions of the form $\phi_{ij}(x_i - x_j)$ for $\phi_{ij}$
convex; %
submodular and concave functions of the form $\x
\mapsto g(\sum_{i=1}^{n} \lambda_i x_i)$ for
$g$
concave and $\lambda_i$
non-negative.  Lastly, indefinite quadratic functions of the form
$f(\x)
= \frac{1}{2} \x^\trans \BH \x + \bh^\trans \x +
c$ with all off-diagonal entries of
$\bmH$
non-positive are examples of submodular but non-convex/non-concave
functions.
Interestingly, characterizations of continuous submodular functions
are in correspondence to those of convex functions, which are
summarized in \cref{tab_comparison}.

 \begin{table}[t]
 	\begin{center}
\caption{Comparison of definitions of submodular and convex functions \citep{bian2017guaranteed}}
\label{tab_comparison} 		

\begin{tabularx}{\textwidth}{|l|X|X|}
  \hline
  Definitions &   Continuous submodular  function $f(\cdot)$ & Convex function $g(\cdot)$, $\forall \lambda\in [0,1]$ \\
  \hline \hline
  $0^{\text{th}}$ order  &  $f(\x) + f(\y) \geq f( \x \vee \y) + f(\x \wedge \y)$ &  $\lambda g(\x) + (1-\lambda)g(\y) \geq g(\lambda \x + (1-\lambda) \y)$   \\
  \hline
  $1^{\text{st}}$ order & {\texttt{weak DR} property (\cref{def_weakdr}), or $\nabla f(\cdot)$ {is a \textcolor{link_color}{weak antitone} mapping} (\cref{lemma_weak_antitone})} & $g(\y) \geq g(\x) +   \dtp{\nabla g(\x)}{\y-\x}$   \\
  \hline
  $2^{\text{nd}}$ order & $\frac{\partial^2 f(\x)}{\partial x_i \partial x_j} \leq 0$, \textcolor{link_color}{$\forall i \neq j$}   & $\nabla^2 g(\x)  \succeq 0$ (symmetric positive semidefinite)   \\
  \hline
\end{tabularx}
\end{center}
\end{table}

\subsection{The DR Property and DR-Submodular Functions}
 
The Diminishing Returns (DR) property  was introduced
when studying set and integer functions.
 We generalize the {DR} property to general functions defined
 over $\X$. It will soon be clear that the {DR} property
 defines a subclass of submodular functions.  All of the proofs can be
 found in \cref{app_proof}.

\begin{figure}
	\centering
	\includegraphics[width=0.6\textwidth]{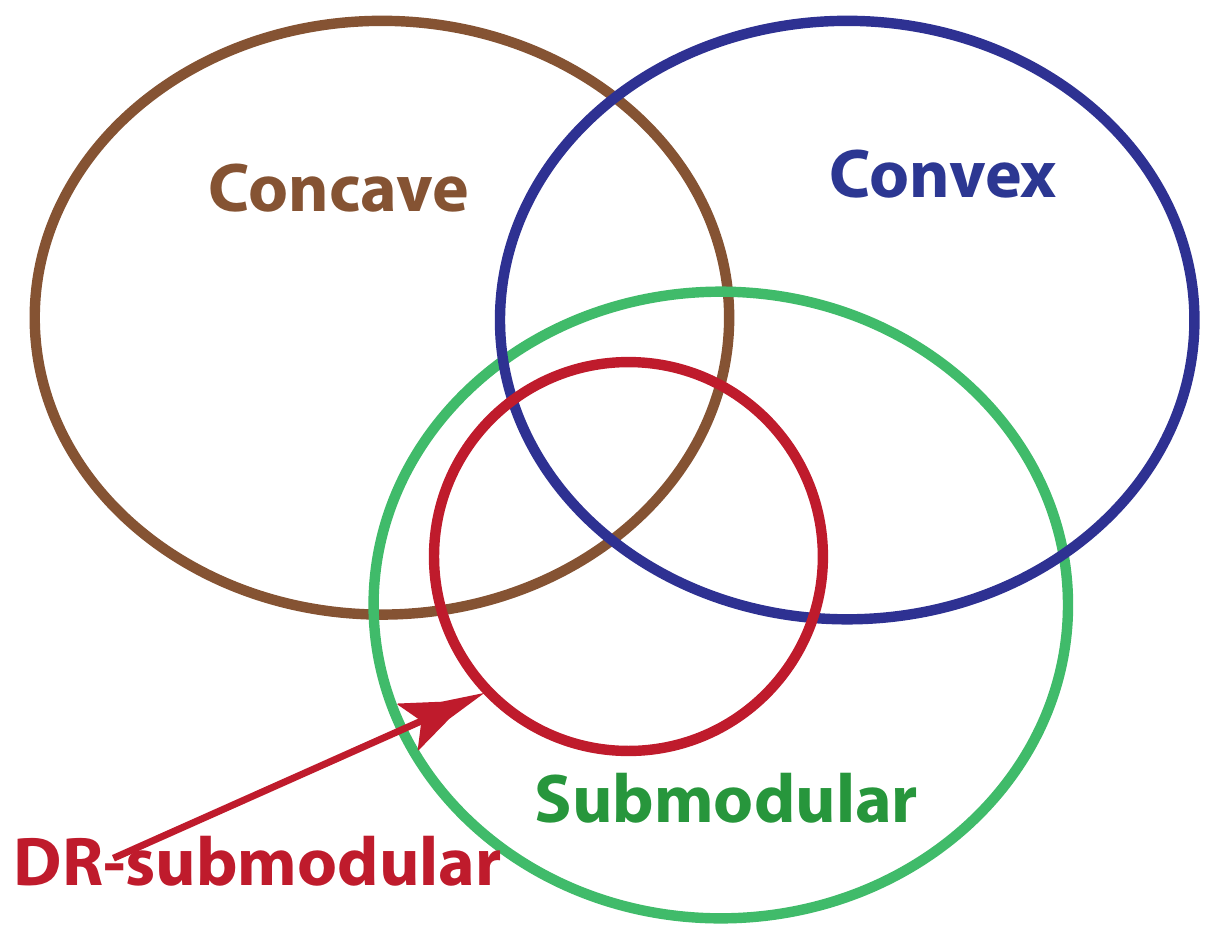}
	\caption{Venn diagram for concavity, convexity, submodularity and DR-submodularity.}
	\label{fig_sub}
\end{figure}

\begin{definition}[{DR} property and DR-submodular
  functions]\label{def_dr}
  A function $f(\cdot)$ defined over $\X$ satisfies the
  \emph{diminishing returns (DR)} property if
  $\forall \a \leqco {\b} \in \X$, $\forall i\in [n]$,
  $\forall k\in \R_+$ such that $(k\chara_i+ \a)$ and $(k\bas_i + \b)$
  are still in $\X$, it holds,
\begin{equation} 
f(k\chara_i+ \a) - f(\a) \geq f(k\chara_i + \b) - f(\b).   
\end{equation}	

This function $f(\cdot)$ is called a DR-submodular\footnote{Note that
  DR property implies submodularity and thus the name
  ``DR-submodular'' contains redundant information about submodularity
  of a function, but we keep this terminology to be consistent with
  previous literature on integer submodular functions.} function.
If $- f(\cdot)$ is DR-submodular, we call $f(\cdot)$ an
\textcolor{link_color}{\irsuper} function, where \ir stands for ``Increasing Returns''.
\end{definition}

One immediate observation is that for a differentiable DR-submodular
function $f(\cdot)$, we have that $\forall \a\leqco \b\in \X$,
$\nabla f(\a)\geqco \nabla f(\b)$, i.e., the gradient
$\nabla f(\cdot)$ is an \emph{antitone} mapping from $\R^n$ to
$\R^n$. This observation can be formalized below:

\begin{lemma}[Antitone mapping]\label{lemma_dr_antitone}
  If $f(\cdot)$ is continuously differentiable, then $f(\cdot)$ is
  DR-submodular iff $\nabla f(\cdot)$ is an \emph{antitone mapping} 
  from $\R^n$ to $\R^n$, i.e., $\forall \a\leqco \b\in \X$,
  $\nabla f(\a)\geqco \nabla f(\b)$.
\end{lemma}

Recently, the {DR} property is explored by \citet{NIPS2016_6073} to
achieve the worst-case competitive ratio for an online concave
maximization problem.  The {DR} property is also closely related to a
sufficient condition on a concave function $g(\cdot)$ \citep[Section
5.2]{bilmes2017deep}, to ensure submodularity of the corresponding set
function generated by giving $g(\cdot)$ boolean input vectors.

\subsection{The Weak DR Property and Its Equivalence to Submodularity}

It is well known that for set functions, the {DR} property is
equivalent to submodularity, while for integer functions,
submodularity does not in general imply the {DR} property
\citep{soma2014optimal,DBLP:conf/nips/SomaY15,soma2015maximizing}. However,
it was unclear whether there exists a diminishing-return-style
characterization that is equivalent to submodularity of
integer functions. In this work we give a positive answer to
this open problem by proposing the \textit{weak diminishing returns}
(\texttt{weak DR}) property for general functions defined over $\X$,
and prove that \texttt{weak DR} gives a sufficient and necessary
condition for a general function to be submodular.

\begin{definition}[{Weak DR} property]\label{def_weakdr}
  A function $f(\cdot)$ defined over $\X$ has the \textit{weak
    diminishing returns} property \emph{(weak DR)} if
  $\forall \a\leqco \b\in \X$,
  $\color{blue}\forall i\in \groundset \text{ such that } a_{i} =
  b_{i}$,
  $\forall k\in \R_+$ such that $(k\chara_i+\a)$ and $(k\chara_i+\b)$
  are still in $\X$, it holds,
  \begin{equation} \label{def_supp_dr2} 
    f(k\chara_i+\a) - f(\a) \geq
    f(k\chara_i + \b) - f(\b).
  \end{equation}
\end{definition}

The following proposition shows that for all set functions, as well as
integer and continuous functions, submodularity is equivalent
to the \NEWDR\ property.
\begin{proposition}[\texttt{submodularity}) $\Leftrightarrow$
  (\NEWDR]\label{lemma_support_dr} A function $f(\cdot)$ defined over
  $\X$ is submodular \emph{iff} it satisfies the \textit{weak DR}
  property.
 \end{proposition}
Given \cref{lemma_support_dr}, 
one can treat \NEWDR\ as the first order definition of submodularity: 
Notice that for a continuously differentiable  function $f(\cdot)$ with the 
\texttt{weak DR} property,   we have that $\forall  \a\leqco \b\in \X$, $\forall i\in \groundset \text{ s.t. } a_{i} = b_{i}$, it holds   $\nabla_i f(\a)\geq \nabla_i f(\b)$,  i.e., $\nabla f(\cdot)$ is a \textit{weak} antitone mapping. 
Formally,

\begin{lemma}[Weak antitone mapping]\label{lemma_weak_antitone}
  If $f(\cdot)$ is continuously differentiable, then $f(\cdot)$ is
  submodular iff $\nabla f(\cdot)$ is a weak \emph{antitone mapping}
   from $\R^n$ to $\R^n$, i.e., $\forall \a\leqco \b\in \X$,
  $\forall i\in \groundset \text{ s.t. } a_{i} = b_{i}$,
  $\nabla_i f(\a)\geq \nabla_i f(\b)$.
\end{lemma}

Now we show that the \texttt{DR} property is stronger than the
\texttt{weak DR} property, and the class of DR-submodular functions is
a proper subset of that of submodular functions, as indicated by
\cref{fig_sub}.

\begin{proposition}[\texttt{submodular/weak DR}) +
  (\texttt{coordinate-wise concave}) $\Leftrightarrow$
  (\texttt{DR}]\label{lemma_dr} A function $f(\cdot)$ defined over
  $\X$ satisfies the \texttt{DR} property iff $f(\cdot)$ is submodular
  and coordinate-wise concave, where the \texttt{coordinate-wise
    concave} property is defined as: $\forall \x\in \X$,
  $\forall i\in \groundset$, $\forall k, l\in \R_+$ s.t.
  $(k\chara_i+\x), (l\chara_i + \x), ((k+l)\chara_i + \x)$ are still
  in $\X$, it holds,
  \begin{align}\label{def_coordinatewise_concave}
f(k\chara_i+\x) - f(\x) \geq f((k+l)\chara_i + \x) - f(l\chara_i + \x),
  \end{align}
	or equivalently (if twice differentiable)
        $\frac{\partial^2 f(\x)}{\partial x_i^2} \leq 0, \forall i\in
        \groundset$.
\end{proposition}
Proposition \ref{lemma_dr} shows that a twice differentiable function $f(\cdot)$
is DR-submodular iff $\forall \x\in \X,  \frac{\partial^2 f(\x)}{\partial x_i \partial x_j}
\leq 0, \forall i, j\in \groundset$, which does not necessarily imply the concavity of $f(\cdot)$. 
Given Proposition \ref{lemma_dr}, we also have the characterizations
of continuous DR-submodular functions, which are summarized in
\cref{tab_dr}.
\begin{table}[t]
  \begin{center}
\caption{Summarization of definitions of continuous DR-submodular
	functions \citep{bian2017guaranteed}}
\label{tab_dr}  	
    \begin{tabularx}{\textwidth}{|l|X|X|}
      \hline Definitions & Continuous DR-submodular function
      $f(\cdot)$, $\forall \x, \y\in \X$\\
      \hline \hline $0^{\text{th}}$ order &
      $f(\x) + f(\y) \geq f( \x \vee \y) +
      f(\x \wedge \y)$, and $f(\cdot)$ is coordinate-wise concave (see  \labelcref{def_coordinatewise_concave}) \\
      \hline $1^{\text{st}}$ order & {\texttt{DR} property
        (\cref{def_dr}), or $\nabla f(\cdot)$ {is an \textcolor{link_color}{antitone} mapping} (\cref{lemma_dr_antitone})}  \\
      \hline $2^{\text{nd}}$ order &
      $\frac{\partial^2 f(\x)}{\partial x_i \partial x_j} \leq 0$,
      \textcolor{link_color}{$\forall i , j$} (all entries of the
      Hessian matrix  being non-positive)  \\
      \hline
\end{tabularx}
\end{center}
\end{table}

\subsection{A Simple Visualization}

\cref{fig_2d_softmax} shows the contour of a 2-D continuous submodular
function
$[x_1; x_2]\mapsto 0.7 (x_1 - x_2)^2 + e^{-4(2x_1 - \frac{5}{3})^2} +
0.6e^{-4(2x_1 - \frac{1}{3})^2}+ e^{-4(2x_2 - \frac{5}{3})^2}+
e^{-4(2x_2 - \frac{1}{3})^2}$ and a 2-D DR-submodular function
\begin{align}
\x \mapsto \log\de{\diag(\x)(\bmL-\bmI) +\bmI }, \x\in [0,1]^2,
\end{align}
where $\bmL = [2.25, 3; 3, 4.25]$.
We can see that both of them are neither convex, nor concave. Notice
that along each of the coordinate, the continuous submodular function
may behave pretty arbitrarily. While for the DR-submdular function, it
is always concave along any single coordinate.

\begin{figure}[htbp]
	\centering
	\includegraphics[width=0.49\textwidth]{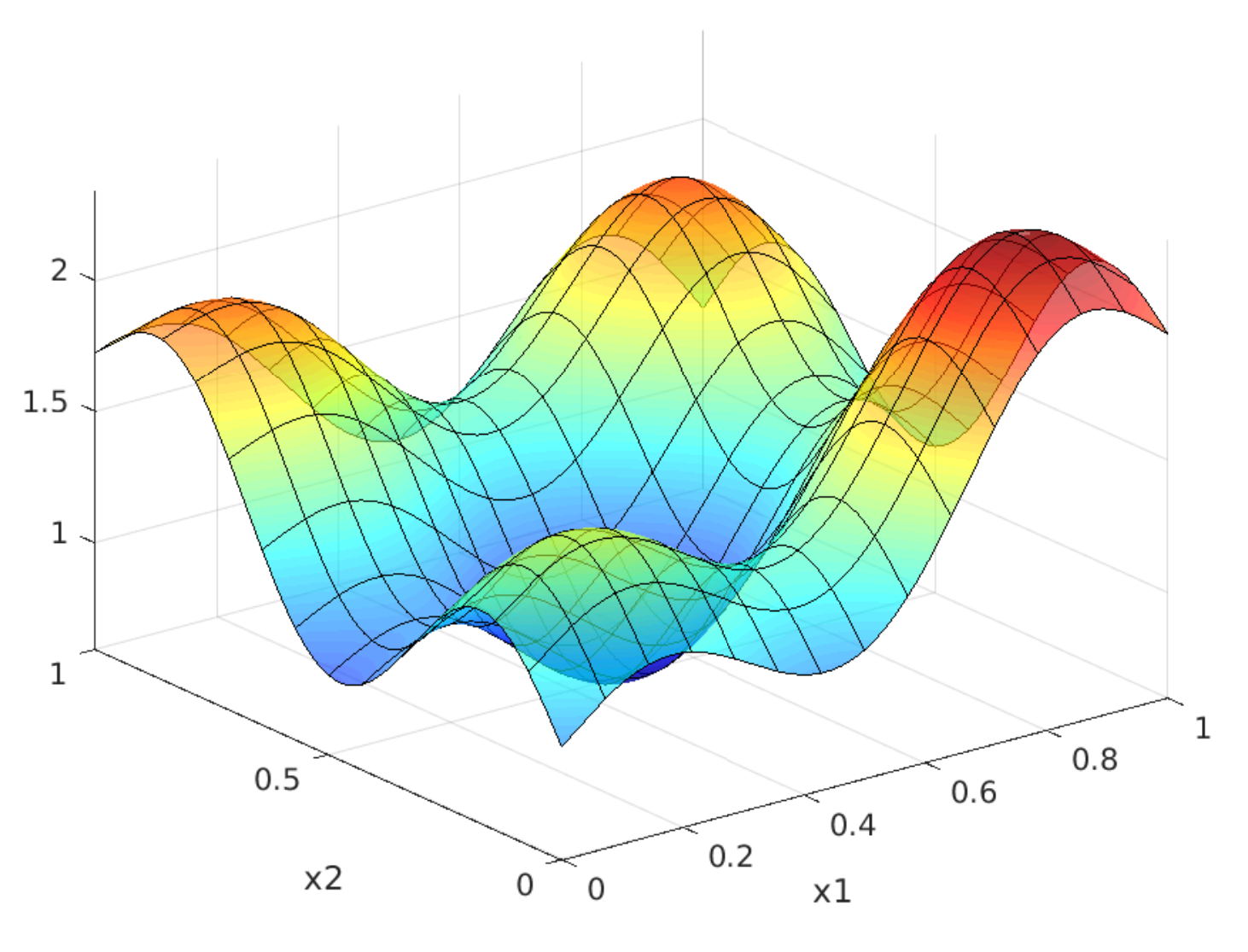}
	\includegraphics[width=0.49\textwidth]{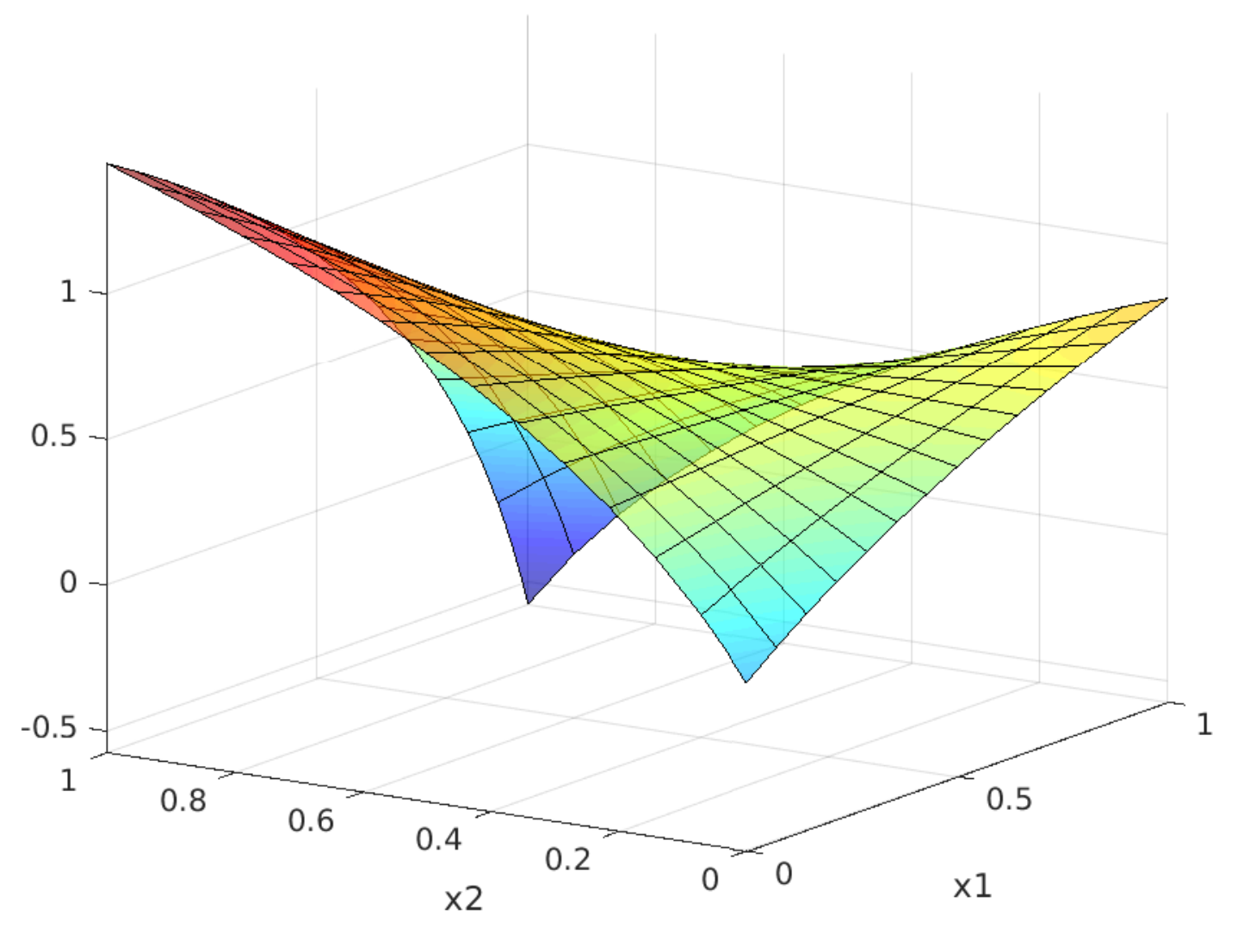}
	\caption{Left: A 2-D continuous submodular function:
          $[x_1; x_2]\mapsto 0.7 (x_1 - x_2)^2 + e^{-4(2x_1 -
            \frac{5}{3})^2} + 0.6e^{-4(2x_1 - \frac{1}{3})^2}+
          e^{-4(2x_2 - \frac{5}{3})^2}+ e^{-4(2x_2 - \frac{1}{3})^2}$.
          Right: A 2-D softmax extension, which is continuous
          DR-submodular.
          $\x \mapsto \log\de{\diag(\x)(\bmL-\bmI) +\bmI }, \x\in
          [0,1]^2$, where $\bmL = [2.25, 3; 3, 4.25]$.}
	\label{fig_2d_softmax}
\end{figure}

\section[Problem Statement of Continuous Submodular
Maximization]{Problem Statement of Continuous Submodular Function
  Maximization}
\label{statement_continuous_submodular_max}

The general setup of constrained 
continuous
submodular function  maximization is,
\begin{align}\label{setup}
\max_{\x \in \P\subseteq \X} f(\x),     \tag{P}
\end{align}
where $f: \X \rightarrow \R$ is continuous submodular or
DR-submodular, $\X = [\underline{\u}, \bar \u]$
\citep{bian2017guaranteed}.  One can assume $f$ is non-negative over
$\X$, since otherwise one just needs to find a lower bound for the
minimum function value of $f$ over $\X$ (because box-constrained
submodular minimization can be solved to arbitrary precision in
polynomial time \citep{bach2015submodular}). Let the lower bound be
$f_{\text{min}}$, then working on a new function
$f'(\x):= f(\x) - f_{\text{min}}$ will not change the solution
structure of the original problem \labelcref{setup}.

The constraint set $\P\subseteq \X$ is assumed to be a
\emph{down-closed} {convex} set, since without this property one
cannot reach any constant factor approximation guarantee of the
problem \labelcref{setup} \citep{vondrak2013symmetry}. Formally,
down-closedness of a convex set is defined bellow:
\begin{definition}[Down-closedness]
A down-closed convex set is a convex set $\P$ associated with a
lower bound $\underline{\bu}\in \P$, such that: 
\begin{enumerate}
\item  $\forall \y\in \P$, $\underline{\bu} \leqco \y$; 

\item $\forall \y\in\P$, $\x\in \R^n$,
  $\underline{\bu} \leqco \x\leqco \y$ implies that $\x\in \P$.
\end{enumerate}
\end{definition}

Without loss of generality, we assume $\P$ lies in the postitive
orthant and has the lower bound $\zero$, since otherwise we can always
define a new set
$\P' = \{\x \;|\; \x = \y - \underline{\bu}, \y\in \P \}$ in the
positive orthant, and a corresponding continuous submdular function
$f'(\x) := f(\x + \underline{\bu})$, and all properties of the
function are still preserved.

The diameter of $\P$ is $D:= \max_{\x,\y\in\P}\|\x-\y\|$, and it holds
that $D \leq \|\bar \u \|$.  We use $\x^*$ to denote the
global maximum of \labelcref{setup}.
In some applications we know that $f$ satisfies the monotonicity property:

\begin{definition}[Monotonicity]
A function $f(\cdot)$  is monotone nondecreasing if,
\begin{align}
\forall \a \leqco \b, f(\a) \leq f(\b).
\end{align}
In the sequel, by ``monotonicity'', we mean monotone nondecreasing by default.
\end{definition}

We also assume that $f$ has
Lipschitz gradients,

\begin{definition}[Lipschitz gradients]
  A differentiable function $f(\cdot)$ has $L$-Lipschitz gradients if
  for all $\x,\y \in \X$ it holds that,
  \begin{align}\label{eq_smooth}
    \| \nabla f(\x)- \nabla f(\y) \| \leq L \|\x - \y\|.
  \end{align}
  According to \citet[Lemma 1.2.3]{nesterov2013introductory}, if
  $f(\cdot)$ has $L$-Lipschitz gradients, then
  \begin{align}\label{eq_quad_lower_bound}
    |f(\x + \v) - f(\x)  - \dtp{\nabla f(\x)}{\v}| \leq \frac{L}{2} \|\v\|^2.
  \end{align}	
\end{definition}

For Frank-Wolfe style algorithms, the notion of curvature usually
gives a tighter bound than just using the Lipschitz gradients.
\begin{definition}[Curvature]
  The curvature of a differentiable function $f(\cdot)$ w.r.t. a
  constraint set $\P$ is,
  \begin{align}
    C_f(\P) : = \sup_{\x, \v\in \P, \gamma\in (0,  1],  \y = \x +
    \gamma (\v - \x )}\frac{2}{\gamma^2}\left[f(\y) - f(\x) - {(\y -
    \x)^\trans}{\nabla f(\x)}\right]. 
  \end{align}
\end{definition}

If a differentiable function $f(\cdot)$ has $L$-Lipschitz gradients,
one can easily show that $C_f(\P) \leq LD^2$, given \citet[Lemma
1.2.3]{nesterov2013introductory}.

\section[Properties of Constrained DR-Submodular
Maximization]{Underlying Properties of Constrained DR-Submodular
  Maximization }
\label{sec_underlying_properties}

\vspace{-0.1cm} In this section we present several properties arising
in DR-submodular function maximization. First we show properties
related to concavity of the objective along certain directions, then
we establish the relation between locally stationary points and the
global optimum (thus called ``local-global relation'').  These
properties will be used to derive guarantees for the algorithms in the
following chapters. All omitted proofs are in \cref{app_proof}.

\subsection{Properties Along Non-Negative/Non-Positive Directions}

Though in general a DR-submodular function $f$ is neither convex, nor
concave, it is \emph{concave} along some directions:

\begin{proposition}[\cite{bian2017guaranteed}]\label{prop_concave}
  A continuous {DR}-submodular function $f(\cdot)$ is concave along
  any non-negative direction $\v \geqco \zero$, and any non-positive
  direction $\v \leqco \zero$.
\end{proposition}

Notice that DR-submodularity is a stronger condition than concavity
along directions $\v \in \pm \R_+^n$: for instance, a concave function
is concave along any direction, but it may not be a DR-submodular
function.

\paragraph{Strong DR-submodularity.}
DR-submodular objectives may be strongly concave along directions
$\v \in \pm \R_+^n$, e.g., for DR-submodular quadratic functions.  We
will show that such additional structure may be exploited to obtain
stronger guarantees for the local-global relation.

\begin{definition}[Strong DR-submodularity] \label{eq_strongly_dr} A
  function $f$ is $\mu$-strongly DR-submodular ($\mu\geq 0$) if for
  all $\x\in \X$ and $\v \in \pm \R_+^n$, it holds that,
  \begin{align}\label{eq_strong_dr} f(\x+\v) \leq f(\x) +
    \dtp{\nabla f(\x)}{\v} - \frac{\mu }{2}\|\v\|^2.
  \end{align}
\end{definition}

\subsection{Relation Between Approximately Stationary Points and
  Global Optimum: Local-Global Relation}
\label{subsec_local_global}

First of all, we present the following \namecref{lemma_3_1}, which
will motivate us to consider a  non-stationarity measure
for general constrained optimization problems. 

\begin{proposition}\label{lemma_3_1}
  If $f$ is $\mu$-strongly DR-submodular, then for any two points
  $\x$, $\y$ in $\X$, it holds:
\begin{align}\label{non_stationarity}
  (\y-\x)^{\trans}\nabla f(\x) \geq f(\x\vee\y) + f(\x\wedge \y) -
  2f(\x) + \frac{\mu}{2}\|\x -\y\|^2. 
\end{align}
\end{proposition}  
\cref{lemma_3_1} implies that if $\x$ is stationary (i.e.,
$\nabla f(\x)= \zero$), then
$2f(\x) \geq f(\x\vee\y) + f(\x\wedge \y) + \frac{\mu}{2}\|\x
-\y\|^2$,
which gives an implicit relation between $\x$ and $\y$. While in
practice finding an exact stationary point is not easy, usually
non-convex solvers will arrive at an approximately stationary point,
thus requiring a proper measure of non-stationarity for the
constrained optimization problem.

\paragraph{Non-stationarity measure.}
Looking at the LHS of \labelcref{non_stationarity}, it naturally
suggests to use $\max_{\y\in\P}(\y-\x)^{\trans}\nabla f(\x)$ as the
non-stationarity measure, which happens to coincide with the measure
used by \citet{lacoste2016convergence,reddi2016stochastic}, and it can
be calculated for free for {Frank-Wolfe}-style algorithms (e.g.,
\cref{alg_classical_fw}).

In order to adapt it to the local-global relation, we give a slightly
more general definition here: For any constraint set $\Q\subseteq \X$,
the non-stationarity of a point $\x\in \Q$ is,
\begin{align}\label{non_stationary}
  g_{\Q}(\x) := \max_{\v\in\Q}\dtp{\v - \x}{\nabla f(\x)}. \qquad
  \text{(\emph{non-stationarity})}
\end{align}
It always holds that $g_{\Q}(\x)\geq 0$. If $g_{\Q}(\x)=0$, we call $\x$  a ``stationary''
point in $\Q$. \labelcref{non_stationary} is a
natural generalization of the non-stationarity measure
$\|\nabla f(\x)\|$ for unconstrained optimization problems.

As the following statements show, $g_{\Q}(\x)$ plays an important
role in characterizing the local-global relation.

\subsubsection{Local-Global Relation in Monotone Setting}

\begin{corollary}[Local-Global Relation: \emph{Monotone Setting}]\label{coro_1half}
  Let $\x$ be a point in $\P$ with non-stationarity $g_{\P}(\x)$.  If
  $f$ is monotone nondecreasing and $\mu$-strongly DR-submodular, then
  it holds that,
  \begin{align}
    f(\x) \geq \frac{1}{2}\left[f(\x^*) -g_{\P}(\x) \right ]  +   \frac{\mu}{4}\|\x -\optcont\|^2.
  \end{align}
\end{corollary}
\cref{coro_1half} indicates that any stationary point is a 1/2
approximation, which also shows up in \citet{hassani2017gradient} with
$\mu = 0$. Furthermore, if $f$ is $\mu$-strongly DR-submodular, the
quality of $\x$ will be boosted a lot: if $\x$ is close to $\optcont$,
it should be close to being optimal since $f$ is smooth; if $\x$ is
far away from $\optcont$, the term $\frac{\mu}{4}\|\x -\optcont\|^2$
will boost the bound significantly.
We provide here a very succinct proof based on \cref{lemma_3_1}.

\begin{proof}[Proof of \cref{coro_1half}]
  Let $\y = \optcont$ in \cref{lemma_3_1}, one can easily reach
  \begin{align}
    f(\x) \geq \frac{1}{2}\left[f(\x^* \vee \x) + f(\optcont \wedge
    \x) -g_{\P}(\x) \right ]  +   \frac{\mu}{4}\|\x -\x^*\|^2.
  \end{align}
	
  Because of monotonicity and $\x^* \vee \x \geqco \x^*$, we know that
  $f(\x^* \vee \x)\geq f(\x^*)$. From non-negativity,
  $f(\optcont \wedge \x) \geq 0$.  Then we reach the conclusion.
\end{proof}

\subsubsection{Local-Global Relation in Non-Monotone Setting}

\begin{proposition}[Local-Global Relation:
  \emph{Non-Monotone Setting}]\label{local_global}
  Let $\x$ be a point in $\P$ with non-stationarity $g_{\P}(\x)$, and
  ${\Q} :=\P \cap  \{\y | \y\leqco \bar \u - \x\}$.
  Let $\z$ be a point in $\Q$ with non-stationarity $g_{\Q}(\z)$.  It
  holds that,
  \begin{flalign}
    &\max\{f(\x), f(\z) \} \geq \\\notag 
    &\frac{1}{4}\left[f(\x^*)
      -g_{\P}(\x) -g_{\Q}(\z)\right ] + \frac{\mu}{8}\left(\|\x
      -\x^*\|^2 + \|\z -\z^*\|^2\right ),
  \end{flalign}
  where $\z^*:= \x\vee \x^* -\x$.
\end{proposition}
\cref{fig_local_global_nonmotone} provides a two dimensional visualization of \cref{local_global}.
Notice that the smaller constraint $\Q$ is generated after the first stationary point $\x$ is calculated. 
\begin{figure}[htbp]
	\centering
	\includegraphics[width=0.6\textwidth]{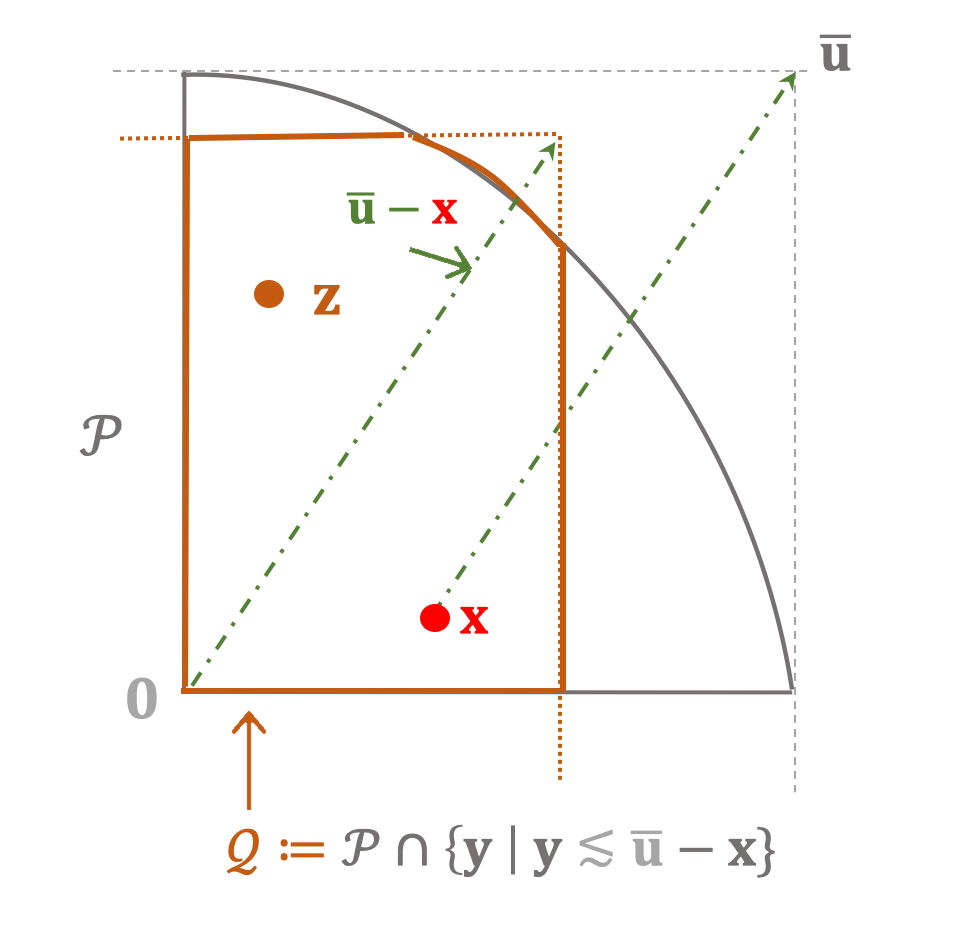}
	\caption{Visualization of the local-global relation in non-monotone setting.}
	\label{fig_local_global_nonmotone}
\end{figure}

\paragraph{Proof sketch of \cref{local_global}:}
The proof uses \cref{lemma_3_1}, the non-stationarity
in \labelcref{non_stationary} and a key observation in the following
\namecref{claim_key}.  The detailed proof is deferred to 
\cref{app_claim_proof}.
\begin{restatable}[]{claim}{keyclaim}
	\label{claim_key}
	Under the setting of \cref{local_global}, it holds that, 
	\begin{align}
		f(\x\vee \x^*) + f(\x \wedge \x^*) +  f(\z\vee \z^*) + f(\z \wedge \z^*) \geq f(\x^*).
	\end{align}
\end{restatable}
Note that \citet{chekuri2014submodular,gillenwater2012near} propose a
similar relation for the special cases of multilinear/softmax
extensions by mainly proving the same conclusion as in
\cref{claim_key}. Their relation does not incorporate the properties
of non-stationarity or strong DR-submodularity.  They both use the
proof idea of constructing a complicated auxiliary set function
tailored to specific DR-submodular functions.  We present a different
proof method by directly utilizing the DR property on carefully
constructed auxiliary points (e.g., $(\x+\z)\vee \x^*$ in the proof of
\cref{claim_key}), this is arguably more succint and straightforward
than that of \citet{chekuri2014submodular,gillenwater2012near}.

\section[Generalized Submodularity and The Reduction]{Generalized
  Submodularity on Conic Lattices and the Reduction to Continuous
  Submodularity}
\label{sec_lattice}

Continuous submodular functions can already model many scenarios. Yet,
there are several interesting cases which are in general not
(DR-)submodular, but can still be captured by a generalized
notion. 
This generalized notion of submodularity is defined over lattices induced by conic inequalities. 
It  enables us to develop polynomial-time 
algorithms with guarantees by using ideas from continuous submodular
optimization. We present representative applications in
\cref{subsec_moti_lattice}.

In the rest of this section, we firstly define the class of general
continuous submodular functions over lattices induced by conic
inequalities. Furthermore we provide a reduction to the original
(DR-)submodular optimization problem.

\subsection{Poset and Conic Lattice}

\paragraph{Proper Cone and Conic Inequality.}

Let us consider at the proper cone that will be used to define a conic
inequality.
A cone $\cone\subseteq \R^n$ is a \emph{proper cone} if it is convex,
closed, solid (having nonempty interior) and pointed (contains no
line, i.e., $\x\in \cone, -\x\in \cone$ implies $\x=\zero$).
A proper cone $\cone$ can be used to define a conic inequality (a.k.a.
generalized inequality \citep[Chapter 2.4]{boyd2004convex}):
$\a\preceq_{\cone} \b$ iff $\b-\a \in \cone$, which also defines a
partial ordering since the binary relation $\lleq_{\cone}$ is
reflexive, antisymmetric and transitive.  Then it is easy to see that
$(\X, \lleq_\cone)$ is a partially ordered set (poset).

\paragraph{Lattice and Lattice Cone.}
If two elements $\a, \b \in \X$ have a least upper bound (greatest
lower bound), it is denoted as the ``join'': $\a\vee \b$ (the
``meet'': $\a\wedge \b$).  A lattice is a poset that contains the join
and meet of each pair of its elements \citep{garg2015introduction}.
A ``lattice cone'' \citep{fuchssteiner2011convex} is the proper cone
that can be used to define a lattice.
Note that not all conic inequalities can be used to define a
lattice. {For example, the positive semidefine cone
  $\cone_{\text{PSD}} = \{\bmA\in \R^{n\times n} | \bmA \text{ is
    symmetric, } \bmA \succeq 0\}$
  is a proper cone, but its induced ordering
  can not be used to define a lattice.  We provide a simple counter
  example to verify this argument in \cref{app_sec_counter_psd}.}

Specifically, we name the lattice that can be defined through 
a conic inequality as ``conic lattice'', 
since it is of particular interest for modeling  the real-world 
applications in this thesis.

\begin{definition}[Conic Lattice \citep{biannips2017nonmonotone}] \label{def_conic_lattice}
	Given a poset $(\X, \lleq_{\cone})$ induced by 
	the conic inequality $\lleq_{\cone}$, if there exist  join and meet operations  
	for every pair of elements $(\a,\b)$ in $\X\times \X$, 
	s.t. $\a \vee \b$ and $\a \wedge \b$ are still in $\X$, then 
	 $(\X, \lleq_{\cone})$ is a conic lattice. 
\end{definition}
In one word, a conic lattice $(\X, \lleq_{\cone})$ is a lattice induced  by a conic inequality $\preceq_{\cone}$.
\subsection{A Specific Conic Lattice and Submodularity on It}

In the following we introduce a class of conic lattices to  model 
the applications in this work. We further provide  a general 
characterization about submodularity on this conic lattice. 

\paragraph{Orthant conic lattice.} 
Given a sign vector $\bmalpha\in \{\pm1\}^n$, the orthant cone is
defined as
$\cone_{\bmalpha}:= \{\x\in \R^n \;|\; x_i\alpha_i \geq 0, \forall
i\in [n]\}$. One can verify that $\cone_{\bmalpha}$ is a proper cone.
For any two points $\a, \b\in \X $, one can further define the join
and meet operations:
$(\a\vee \b)_i := \alpha_i \max\{\alpha_ia_i, \alpha_ib_i \}$,
$(\a\wedge \b)_i :=\alpha_i \min\{\alpha_ia_i, \alpha_ib_i \}$,
$\forall i\in [n]$. Then it is easy to show that the poset
$(\X, \lleq_{\cone_{\bmalpha}})$ is a valid conic lattice.

A function $f:\X\mapsto \R$ is submodular on a lattice
\citep{topkis1978minimizing,fujishige2005submodular} if for all
$(\x, \y)\in \X \times \X$, it holds that,
\begin{align}\label{eq_sub_lattice}
f(\x) + f(\y) \geq f(\x \vee \y)  + f(\x \wedge \y).
\end{align}

One can establish the characterizations of submodularity on the
orthant conic lattice $(\X, \preceq_{\cone_{\bmalpha}})$ similarly as
that in \citet{bian2017guaranteed}:
\begin{proposition}[Characterizations of Submodularity on Orthant
  Conic Lattice
  $(\X, \preceq_{\cone_{\bmalpha}})$]\label{prop_submod_orthant}
  If a function $f$ is submodular on the lattice  $(\X,
  \preceq_{\cone_{\bmalpha}})$   (called
  \emph{$\cone_{\bmalpha}$-submodular}), then we have  the following
  two equivalent characterizations:\\ 
  a) $\forall \a,\b \in \X$ s.t. $\a \preceq_{\cone_{\bmalpha}} \b$,
  $\forall i$ s.t. $a_i =b_i$, $\forall k\in \R_+$ s.t. $(k\e_i+\a)$
  and $(k\e_i+\b)$ are still in $\X$, it holds that,
  \begin{align}\label{eq_general_weakdr}
    \alpha_i [	f(k\e_i+\a) - f(\a)] \geq \alpha_i [f(k\e_i+\b) - f(\b)].  \quad \texttt{ \emph{(weak DR)} }
  \end{align}

  b) If $f$ is twice differentiable, then $\forall \x\in \X$ it holds,
  \begin{flalign}\label{eq_general_sub}
    \alpha_i\alpha_j\nabla_{ij}^2 f(\x) \leq 0,\; \forall i,j \in [n],
    \textcolor{blue}{i\neq j}.
  \end{flalign}
\end{proposition}
\cref{prop_submod_orthant} can be proved by directly generalizing the
proof of \cref{lemma_support_dr}, so the detailed proof is omitted
here due to the high similarity.

Next, we generalize the definition of DR-submodularity to the conic
lattice $(\X, \preceq_{\cone_{\bmalpha}})$:
\begin{definition}[$\cone_{\bmalpha}$-DR-submodular]
  \label{def_general_dr}
  A function $f:\X\mapsto \R$ is $\cone_{\bmalpha}$-DR-submodular if
  $\forall \a, \b\in \X$ s.t. $\a \preceq_{\cone_{\bmalpha}}\b$,
  $\forall i \in [n], \forall k\in \R_+$ s.t. $(k\e_i+\a)$ and
  $(k\e_i+\b)$ are still in $\X$, it holds that,
  \begin{align}\label{eq_general_dr}
    \alpha_i [	f(k\e_i+\a) - f(\a)] \geq \alpha_i [f(k\e_i+\b) - f(\b)].
  \end{align}
\end{definition}

In correspondence to the relation between DR-submodularity and
submodularity over continuous domains (\cref{lemma_dr}), one can
easily get the similar relation (with highly similar proof) bellow:
\begin{proposition}[$\cone_{\bmalpha}$-submodular + coordinate-wise
  concave $\Leftrightarrow$\\
  $\cone_{\bmalpha}$-DR-submodular]\label{prop_relation}
  A function $f$ is $\cone_{\bmalpha}$-DR-submodular iff it is
  $\cone_{\bmalpha}$-submodular and coordinate-wise concave.
\end{proposition}
Combining \labelcref{eq_general_sub} and \cref{prop_relation}, one can
show that if $f$ is twice differentiable and
$\cone_{\bmalpha}$-DR-submodular, then $\forall \x\in \X$ it holds
that,
\begin{flalign}\label{eq_general_conedr}
  \alpha_i\alpha_j\nabla_{ij}^2 f(\x) \leq 0,\; \forall i, j\in [n].
\end{flalign}
Similarly, a function $f$ is $\cone_{\bmalpha}$-\irsuper iff
$-f$ is $\cone_{\bmalpha}$-DR-submodular.

\begin{remark}
We only consider the orthant conic lattice $(\X, \preceq_{\cone_{\bmalpha}})$ here, since it can
already model the applications in this work. However, it
is noteworthy that the framework can be generalized 
to arbitrary conic lattices, which may be of interest 
to model more complex applications. 
\end{remark}

\subsection{A Reduction to Optimizing Submodular Functions over
  Continuous Domains}

To be succint, in this section we only discuss the reduction for the
$\cone_{\bmalpha}$-DR-submodular maximization problems. However, it is
easy to see that the reduction works for all kinds of
$\cone_{\bmalpha}$-submodular optimization problems, e.g.,
$\cone_{\bmalpha}$-submodular minimization problem.

Suppose $g$ is a $\cone_{\bmalpha}$-DR-submodular function, and the
$\cone_{\bmalpha}$-DR-submodular maximization problem is
$\max_{\y\in\P'} g(\y)$, where
$\P' = \{ \y\in \R^n|h_i(\y)\leq b_i, \forall i\in [m], \y
\ggeq_{\cone_{\bmalpha}} \mathbf{0} \}$
is down-closed w.r.t. the conic inequality $\lleq_{\cone_{\bmalpha}}$.
The down-closedness here means if $\a\in \P'$ and
$\mathbf{0}\lleq_{\cone_{\bmalpha}} \b \lleq_{\cone_{\bmalpha}} \a$,
then $\b\in \P'$ as well.

Let $\bmA:=\diag(\bmalpha)$, and a function $f(\x):= g(\bmA \x)$. One
can see that if $g$ is $\cone_{\bmalpha}$-DR-submodular, then $f$ is
DR-submodular: assume wlog.\footnote{If twice differentiability is not satisfied, one can still
  use other equivalent characterizations, for instance, the
  characterization in \labelcref{eq_sub_lattice} or
  in \labelcref{eq_general_weakdr} to formulate this.} that $g$ is
twice differentiable, then
$\nabla^2f(\x) = \bmA^\trans \nabla^2 g \bmA$, and
$\nabla^2_{ij}f(\x) = \alpha_i\alpha_j \nabla^2_{ij} g \leq 0$, so $f$
is DR-submodular.

By the affine transformation $\y:=\bmA\x$, one can transform the
$\cone_{\bmalpha}$-DR-submodular maximization problem to be a
DR-submodular maximization problem $\max_{\x\in\P} g(\bmA\x)$, where 
$\P = \{ \x\in \R^n|h_i(\bmA\x)\leq b_i, \forall i\in [m], \bmA\x
\ggeq_{\cone_{\bmalpha}} \mathbf{0} \}$
is down-closed w.r.t. the ordinary component-wise inequality $\leqco$.
To verify the down-closedness of $\P$ w.r.t. to the ordinary
inequality $\leqco$ here, let $\y_1 = \bmA \x_1 \in \P'$ (so
$\x_1\in \P$). Suppose there is a point $\y_2 = \bmA \x_2$ s.t.
$\mathbf{0}\lleq_{\cone_{\bmalpha}} \y_2 \lleq_{\cone_{\bmalpha}}
\y_1$.
From the down-closedness of $\P'$, we know that $\y_2\in \P'$, thus
$\x_2\in \P$.  Looking at
$\mathbf{0}\lleq_{\cone_{\bmalpha}} \y_2 \lleq_{\cone_{\bmalpha}}
\y_1$,
it is equivalent to $0\leqco \x_2\leqco \x_1$. Thus we establish the
down-closedness of $\P$.

Given the reduction, we can reuse the algorithms for the original
DR-submodular maximization problem \labelcref{setup}.

\section{Conclusions}

In this chapter we presented detailed characterizations of continuous
submodular functions. By introducing the weak DR property, we make it
possible to describe submodularity for general functions (set, integer
and continuous functions) using a DR-style characterization. 
After a formal statement of the class of
continuous submodular maximization problems, we illustrated intriguing
properties of this class of problems. It includes concavity along
certain directions and the local-global relation.  These
characterizations and properties will be heavily used in proofs of the
subsequent chapters.

\section{Additional Proofs}
\label{app_proof}

Since $\X_i$ is a compact subset of $\R$, we denote its lower bound
and upper bound to be $\underline{u}_i$ and $\bar u_i$, respectively.

\subsection{Proofs of \cref{lemma_dr_antitone} and
  \cref{lemma_weak_antitone}}

\begin{proof}[Proof of \cref{lemma_dr_antitone}]
  \emph{Sufficiency}: For any dimension $i$,
  \begin{align}
    \nabla_i f(\a) = \lim_{k\rightarrow 0} \frac{f(k\bas_i + \a)
    - f(\a)}{k}  \geq \lim_{k\rightarrow 0} \frac{f(k\bas_i +
    \b) - f(\b)}{k} = \nabla_i f(\a).
  \end{align}

  \emph{Necessity}:
	
  Firstly, we show that for any $\c \geqco \zero$, the function
  $g(\x) := f(\c + \x) - f(\x)$ is monotonically non-increasing.
  \begin{align}
    \nabla g(\x) =  \nabla f(\c + \x) - \nabla f(\x) \leqco \zero.
  \end{align}
	
  Taking $\c = k\bas_i$, since $g(\a) \leq g(\b)$, we reach the
  DR-submodularity definition.
\end{proof}

\begin{proof}[Proof of \cref{lemma_weak_antitone}]
  Similar as the proof of \cref{lemma_dr_antitone}, we have the
  following:
	
  \emph{Sufficiency}: For any dimension $i$
  $\text{ s.t. } a_{i} = b_{i}$,
  \begin{align}
    \nabla_i f(\a) = \lim_{k\rightarrow 0} \frac{f(k\bas_i + \a) -
    f(\a)}{k}  \geq \lim_{k\rightarrow 0} \frac{f(k\bas_i + \b) -
    f(\b)}{k} = \nabla_i f(\a).
  \end{align}

  \emph{Necessity}:
	
  We show that for any $k \geq 0$, the function
  $g(\x) := f(k\bas_i + \x) - f(\x)$ is monotonically non-increasing.
  \begin{align}
    \nabla g(\x) =  \nabla f(k\bas_i + \x) - \nabla f(\x) \leqco \zero.
  \end{align}
	
  Since $g(\a) \leq g(\b)$, we reach the weak DR definition.
\end{proof}

\subsection{Alternative Formulation of the \NEWDR\ Property}

First of all, we will prove that \NEWDR\ has the following alternative
formulation, which will be used to prove Proposition \ref{lemma_support_dr}.
\begin{lemma}[Alternative formulation of \NEWDR]
  The \NEWDR\ property (\cref{def_supp_dr2}, denoted as
  \emph{Formulation I}) has the following equilvalent formulation
  (\cref{def_supp_dr}, denoted as \emph{Formulation II}):
  $\forall \a\leqco \b\in \X$,
  $\forall i\in \{i'|a_{i'} = b_{i'}=\underline{u}_{i'} \}, \forall
  k'\geq l'\geq 0$
  s.t. $(k'\chara_i + \a)$, $(l'\chara_i + \a)$, $(k'\chara_i + \b)$
  and $(l'\chara_i + \b)$ are still in $\X$, the following inequality
  is satisfied,
  \begin{equation} \label{def_supp_dr} f(k'\chara_i + \a) -
    f(l'\chara_i + \a) \geq f(k'\chara_i+ \b) - f(l'\chara_i+ \b).
    \quad \emph{(Formulation II)}
  \end{equation}
\end{lemma}
\begin{proof}

  Let $D_1 = \{i| a_i = b_i = \underline{u}_i \}$,
  $D_2 = \{i|\underline{u}_i < a_i = b_i < \bar u_i \}$, and
  $D_3 = \{i| a_i = b_i = \bar u_i \}$.

  1) \texttt{Formulation II} $\Rightarrow$ \texttt{Formulation I}

  When $i\in D_1$, set $l' = 0$ in \texttt{Formulation II} one can get
  $f(k'\chara_i+ \a) - f(\a) \geq f(k'\chara_i+ \b) - f(\b)$.

  When $i\in D_2$, $\forall k\geq 0$, let
  $l' = a_i - \underline{u}_i = b_i- \underline{u}_i >0,$
  $k' = k + l' = k +(a_i - \underline{u}_i)$, and let
  $\bar \a = (\sete{a}{i}{\underline{u}_i}), \bar \b =
  (\sete{b}{i}{\underline{u}_i})$.
  It is easy to see that $\bar \a \leqco \bar \b$, and
  $\bar a_i = \bar{b}_i = \underline{u}_i$. Then from
  \texttt{Formulation II},
  \begin{flalign}
    & f(k'\chara_i + \bar \a) - f(l'\chara_i + \bar \a) = f(k\chara_i
    + \a) - f(\a) \\\notag 
    & \geq f(k'\chara_i + \bar \b) -
    f(l'\chara_i + \bar \b) = f(k\chara_i + \b) - f(\b).
  \end{flalign}
  When $i\in D_3$, \cref{def_supp_dr2} holds trivially.

  The above three situations proves the \texttt{Formulation I}.

  2) \texttt{Formulation II} $\Leftarrow$ \texttt{Formulation I}

  $\forall \a\leqco \b$, $\forall i\in D_1$, one has
  $a_i = b_i = \underline{u}_i$.
  $\forall k'\geq l' \geq 0$, let
  $\hat \a = l'\chara_i + \a, \hat \b = l'\chara_i + \b$, let
  $k = k'-l' \geq 0$, it can be verified that $\hat \a\leqco \hat \b$
  and $\hat a_i = \hat b_i$, from \texttt{Formulation I},
  \begin{flalign}
    &f(k\chara_i + \hat \a) - f(\hat \a) = f(k'\chara_i + \a) -
    f(l'\chara_i + \a)\\\notag \geq 
    & f(k\chara_i + \hat \b) - f(\hat
    \b) = f(k'\chara_i + \b) - f(l'\chara_i + \b).
  \end{flalign}
  which proves \texttt{Formulation II}.
\end{proof}

\subsection{Proof of Proposition \ref{lemma_support_dr}}

\begin{proof}

  1) \texttt{submodularity} $\Rightarrow$ \texttt{weak DR}:

  Let us prove the \texttt{Formulation II} (\cref{def_supp_dr}) of
  \texttt{weak DR}, which is,

  $\forall \a\leqco \b\in \X$,
  $\forall i\in \{i'|a_{i'} = b_{i'}=\underline{u}_{i'} \}, \forall
  k'\geq l'\geq 0$, the following inequality holds,
  \begin{equation}
    f(k'\chara_i+ \a) - f(l'\chara_i + \a) \geq f(k'\chara_i+ \b) - f(l' \chara_i + \b).
  \end{equation}

  And $f$ is a submodular function iff $\forall \x, \y\in \X$,
  $f(\x)+f(\y) \geq f( \x \vee \y) + f(\x \wedge \y)$, so
  $f(\y) - f(\x\wedge \y) \geq f(\x\vee \y) - f(\x)$.
	
  Now $\forall \a \leqco \b \in\X$, one can set $\x = l'\chara_i + \b$
  and $\y = k'\chara_i + \a$. It can be easily verified that
  $\x\wedge \y =l'\chara_i + \a$ and $\x\vee \y = k'\chara_i + \b$.
	Substituting all the above equalities into
	 $f(\y) - f(\x\wedge \y) \geq f(\x\vee \y) - f(\x)$ one can get 	
	 $f(k'\chara_i+ \a) - f(l'\chara_i + \a) \geq f(k'\chara_i+ \b) - f(l' \chara_i + \b)$.

     2) \texttt{submodularity}
     $\Leftarrow$  \texttt{weak DR}:

     Let us use \texttt{Formulation I} (\cref{def_supp_dr2}) of
     \texttt{weak DR} to prove the \texttt{submodularity} property.

     $\forall \x, \y\in \X$, let $D := \{e_1, \cdots, e_d\}$  be the
     set of elements for which $y_e > x_e$, let
     $k_{e_i}: = y_{e_i} - x_{e_i}$.  Now set
     $\a^0 := \x\wedge \y, \b^0 := \x$ and
     $\a^i = (\sete{a^{i-1}}{e_i}{y_{e_i}}) = k_{e_i}\chara_i +
     \a^{i-1}, \b^i = (\sete{b^{i-1}}{e_i}{y_{e_i}}) = k_{e_i}\chara_i
     + \b^{i-1}$,
     for $i = 1, \cdots, d$.  
     
     One can verify that
     $\a^i\leqco \b^i, a^i_{e_{i'}} = b^i_{e_{i'}}$ for all
     $i'\in D, i=0, \cdots, d$, and that
     $\a^d = \y, \b^d = \x\vee \y$.

     Applying \cref{def_supp_dr2} of the \texttt{weak DR} property for
     $i = 1,\cdots, d$ one can get
     \begin{flalign}
       &f(k_{e_1}\chara_{e_1} + \a^0) - f(\a^0) \geq
       f(k_{e_1}\chara_{e_1} + \b^0) - f(\b^0) \\
       &f(k_{e_2}\chara_{e_2} + \a^1) - f(\a^1) \geq
       f(k_{e_2}\chara_{e_2} + \b^1) - f(\b^1) \\\notag
       &\cdots\\ 
       &f(k_{e_d}\chara_{e_d} + \a^{d-1}) -
       f(\a^{d-1}) \geq f(k_{e_d}\chara_{e_d} + \b^{d-1}) -
       f(\b^{d-1}).
     \end{flalign}
     Taking a sum over all the above $d$ inequalities, one can get
     \begin{flalign}
       & f(k_{e_d}\chara_{e_d} + \a^{d-1}) - f(\a^{0}) \geq
       f(k_{e_d}\chara_{e_d} + \b^{d-1}) - f(\b^{0})\\\notag 
      &  \Leftrightarrow\\ 
       & f(\y) - f(\x\wedge \y) \geq f(\x\vee
       \y) - f(\x)\\\notag
       &  \Leftrightarrow\\ 
       & f(\x) + f(\y) \geq
       f(\x\vee \y) + f(\x\wedge \y),
     \end{flalign}
     which proves the submodularity property.
\end{proof}

\subsection{Proof of Proposition \ref{lemma_dr}}

\begin{proof}
	
	1) \texttt{submodular} + \texttt{coordinate-wise concave}
		$\Rightarrow$ \texttt{DR}:
	
	From coordinate-wise concavity we have $f(\a + k\chara_i) - f(\a) \geq f(\a+(b_i - a_i + k)\chara_i) - f(\a+(b_i - a_i)\chara_i)$. Therefore, to prove \text{DR} it suffices to show that 
	\begin{flalign}\label{eq_12}
		f(\a+(b_i - a_i + k)\chara_i) - f(\a+(b_i - a_i)\chara_i) \geq f(\b + k\chara_i) - f(\b).
	\end{flalign}
	Let $\x:=\b, \y:=(\a+(b_i - a_i + k)\chara_i)$, so $\x\wedge\y  = (\a+(b_i - a_i)\chara_i), \x\vee \y = (\b + k\chara_i)$.
	From submodularity, one can see that inequality \labelcref{eq_12} holds.
	
	2) \texttt{DR} $\Rightarrow$ \texttt{submodular} + \texttt{coordinate-wise concave}:

	  From \texttt{DR} property, the  \texttt{weak DR} (\cref{def_supp_dr2}) property is implied, which 
	equivalently proves the \textit{submodularity} property.
	
	To prove \textit{coordinate-wise concavity}, one just need to set $\b:=\a+l\chara_i$, then we have  $f(\a + k\chara_i) - f(\a) \geq f(\a + (k+l)\chara_i) - f(\a + l\chara_i)$. 
\end{proof}

\subsection{Proof of  \cref{prop_concave}}\label{supp_prop_concave}

\begin{proof}[Proof of  \cref{prop_concave}]
	Consider a univariate  function
	\begin{align}
	 g(\xi):= f(\x+\xi \v^*), \xi\geq 0, \v^* \geqco \zero. 
	\end{align}
	
	We know that 
	\begin{align}
			\frac{d g(\xi)}{d \xi} = \dtp{\v^*}{\nabla f(\x+\xi \v^*)}.
	\end{align}

	It can be verified that: 
	
	$g(\xi)$ is concave $\Leftrightarrow$ 
	\begin{flalign} 
	\frac{d^2 g(\xi)}{d \xi^2} = (\v^*)^\trans \nabla^2 f(\x+\xi \v^*) \v^* = \sum_{i\neq j} v^*_i v^*_j \nabla^2_{ij} f + \sum_i (v_i^*)^2\nabla_{ii}^2f \leq 0.
	\end{flalign}
	The non-positiveness of $\nabla^2_{ij} f $ is ensured by submodularity of $f(\cdot)$, and the non-positiveness of $\nabla^2_{ii} f $ results from the coordinate-wise concavity of $f(\cdot)$.
	
	The proof of concavity along any non-positive direction is similar, which is omitted here. 
\end{proof}

\subsection{Proof of \cref{lemma_3_1}}

\begin{proof}[Proof of \cref{lemma_3_1}]
   Since $f$ is DR-submodular, so it is concave along any direction $\v\in \pm \R^n_+$. We know that $\x\vee \y - \x \geqco \zero$
    and $\x\wedge \y - \x\leqco \zero$, so from the strong DR-submodularity in \labelcref{eq_strong_dr}, 
    \begin{align} 
   & f(\x\vee\y)  - f(\x) \leq \dtp{\nabla f(\x)}{\x\vee \y - \x}  -\frac{\mu}{2}  \|\x\vee \y - \x\|^2,\\ 
    &  f(\x\wedge \y) - f(\x) \leq \dtp{\nabla f(\x)}{\x\wedge \y - \x}  -\frac{\mu}{2}  \|\x\wedge  \y - \x\|^2.
    \end{align}
    Summing the above two inequalities and notice that $\x\vee\y + \x\wedge \y = \x+\y$, we arrive,
    \begin{flalign} 
     & (\y-\x)^{\trans}\nabla f(\x)\\\notag
      & \geq f(\x\vee\y) + f(\x\wedge \y) - 2f(\x) + \frac{\mu}{2}  (\|\x\vee \y - \x\|^2 + \|\x\wedge  \y - \x\|^2)\\ 
     & = f(\x\vee\y) + f(\x\wedge \y) - 2f(\x) + \frac{\mu}{2}  \| \y - \x\|^2,
    \end{flalign}
   the last equality holds since $\|\x\vee \y - \x\|^2 + \|\x\wedge  \y - \x\|^2 =  \| \y - \x\|^2$.
\end{proof}

\subsection{Proof of  \cref{local_global}}\label{app_claim_proof}

\begin{proof}[Proof of \cref{local_global}]
Consider the point $\z^*:= \x\vee \x^* -\x = (\x^* - \x)\vee \zero$. One can see that: 1) $\zero\leqco \z^* \leqco \x^*$; 2) $\z^* \in \P$ (down-closedness); 3) $\z^*\in \Q$ (because of   $\z^*\leqco \bar \u - \x$). 
From \cref{lemma_3_1},
\begin{align}\label{eq_1718}
& \dtp{\x^*-\x}{\nabla f(\x)} +  2f(\x) \geq f(\x\vee \x^*) + f(\x \wedge \x^*) +  \frac{\mu}{2}\|\x -\x^*\|^2, \\\label{eq12}
& \dtp{\z^*-\z}{\nabla f(\z)} +  2f(\z) \geq f(\z\vee \z^*) + f(\z \wedge \z^*) +  \frac{\mu}{2}\|\z -\z^*\|^2.
\end{align}
Let us first of all prove the following  key \namecref{claim_key}.

\keyclaim* 

\begin{proof}[Proof of \cref{claim_key}]
Firstly, we are going to prove that 
\begin{align}\label{proof_part1}
f(\x \vee \x^*) + f(\z\vee \z^*) \geq f(\z^*) + f((\x+\z)\vee \x^*), 
\end{align}
which is equivalent to
$f(\x \vee \x^*) - f(\z^*) \geq f((\x+\z)\vee \x^*) - f(\z\vee \z^*)$.
It can be shown that  $\x \vee \x^*  - \z^* = (\x+\z)\vee \x^* - \z\vee \z^* $. Combining this with 
the fact that $\z^* \leqco \z\vee \z^*$, and using the DR property (see \cref{def_dr}) implies 
\labelcref{proof_part1}.
Then we establish,
\begin{align}\label{eq_EqaulityPoints}
 \x \vee \x^*  - \z^* = (\x+\z)\vee \x^* - \z\vee \z^* ~.
\end{align}
We will show that both the RHS and LHS of the above equation are equal to $\x$:  for the LHS of \labelcref{eq_EqaulityPoints} we can write 
 $\x \vee \x^*  - \z^* =  \x \vee \x^*  - \left(  \x \vee \x^* - \x\right) = \x$.
For the RHS of \labelcref{eq_EqaulityPoints} let us consider any coordinate $i\in [n]$,
\begin{align}\notag 
&(x_i+z_i)\vee x_i^* - z_i\vee z_i^* =\\
 & (x_i+z_i)\vee x_i^* - \left((x_i+z_i)-x_i\right)\vee  \left((x_i \vee x_i^*) - x_i\right) =x_i,
\end{align}
where the last equality holds easily for the two situations: $(x_i+z_i) \geq  x_i^*$ and $(x_i+z_i) < x_i^*$.

Next, we are going to prove that,
\begin{align}\label{proof_part2}
 f(\z^*) + f(\x\wedge \x^*)\geq f(\x^*) + f(\zero).
\end{align}
It is equivalent to 
$f(\z^*)   - f(\zero) \geq  f(\x^*) - f(\x\wedge \x^*)$,
which can be done similarly by the DR property: Notice that
\begin{align} 
\x^* - \x\wedge \x^* = \x\vee \x^* - \x = \z^* - \zero \text{ and } 
 \zero \leqco  \x\wedge \x^*.
\end{align}
Thus \labelcref{proof_part2} holds from the DR property. 
Combining \labelcref{proof_part1,proof_part2} one can get,
\begin{align}\notag 
 & f(\x \vee \x^*) + f(\z\vee \z^*) + f(\x\wedge \x^*) + f(\z\wedge \z^*) \\
  & \geq  f(\x^*) + f(\zero) +  f((\x+\z)\vee \x^*)+ f(\z\wedge \z^*)\\\notag 
& \geq f(\x^*).    \quad \text{(non-negativity of $f$) }
\end{align}
\end{proof} 
 
Combining \labelcref{eq_1718,eq12} and \cref{claim_key} it reads,
\begin{align}\label{eq16}
 & \dtp{\x^* -\x}{\nabla f(\x)} +  \dtp{\z^*-\z}{\nabla f(\z)}  +   2(f(\x) + f(\z) ) \\ & \geq f(\x^*) + 
 \frac{\mu}{2}(\|\x -\x^*\|^2 + \|\z -\z^*\|^2).
\end{align}

From the definition of non-stationarity in \labelcref{non_stationary} one can get, 
\begin{align}\label{eq17}
&  g_{\P}(\x) := \max_{\v\in\P}\dtp{\v - \x}{\nabla f(\x)} \overset{\x^*\in \P}{\geq}  \dtp{\x^*-\x}{\nabla f(\x)},\\\label{eq18}
& g_{\Q}(\z) := \max_{\v\in\Q}\dtp{\v - \z}{\nabla f(\z)}  \overset{\z^*\in \Q}{\geq} \dtp{\z^*-\z}{\nabla f(\z)}.
\end{align}
Putting together \cref{eq16,eq17,eq18} we can get, 
\begin{align} 
2(f(\x) + f(\z) ) \geq f(\x^*) -g_{\P}(\x) -g_{\Q}(\z) +   \frac{\mu}{2}(\|\x -\x^*\|^2 + \|\z -\z^*\|^2).
\end{align}
So it arrives 
\begin{flalign} 
& \max\{f(\x), f(\z) \} \geq \\
&\frac{1}{4}[f(\x^*) -g_{\P}(\x) -g_{\Q}(\z)]  +   \frac{\mu}{8}(\|\x -\x^*\|^2 + \|\z -\z^*\|^2).
\end{flalign}
\end{proof}

\subsection{A Counter Example to Show That PSD Cone is not a Lattice}
\label{app_sec_counter_psd}

The positive semidefine 
cone $\cone_{\text{PSD}} = \{\bmA\in \R^{n\times n} | \bmA \text{ is symmetric, } \bmA  \succeq  0\}$ is a proper cone, but not 
a lattice cone. That is, it can not be used to 
define a lattice over the space of symmetric matrices. 

Let us consider the two dimensional symmetric  matrix space $S^2$.
Specifically, the following 
two symmetric matrices,
\begin{align}\notag 
\bmX = \begin{bmatrix}
1 &0 \\
0 & 0
\end{bmatrix}, 
\bmY = \begin{bmatrix}
0 &0 \\
0 & 1
\end{bmatrix}.
\end{align}
For the conic inequality $\lleq_{\cone_{\text{PSD}}}$, 
assume that there exists a least upper bound, i.e.,
the join of $\bmX, \bmY$: $\bmZ: = \bmX \vee \bmY$. From
the definition of least upper bound, $\forall\; \bmW \in S^2$ it should hold that,
\begin{align}\label{eq33}
\bmW   \ggeq_{\cone_{\text{PSD}}} \bmX \text{ and } \bmW   \ggeq_{\cone_{\text{PSD}}} \bmY  \text{ iff }  \bmW   \ggeq_{\cone_{\text{PSD}}} \bmZ.
\end{align}
Suppose $\bmZ = \begin{bmatrix}
b &a \\
a & c
\end{bmatrix}$. Firstly, consider $\bmW$ to be diagonal 
matrices, one can verify that $\bmZ$ must be in the form of 
$ \begin{bmatrix}
1 &a \\
a & 1
\end{bmatrix}$, then considering $\bmW =\bmI$ forcing $\bmZ$
to be $\bmI$. 

Now let  $\bmW =\frac{2}{3} \begin{bmatrix}
2 &1 \\
1 & 2
\end{bmatrix}$, which is $\ggeq_{\cone_{\text{PSD}}} \bmX$
and $\ggeq_{\cone_{\text{PSD}}} \bmY$. However,  $\bmW - \bmI =\frac{1}{3} \begin{bmatrix}
1 &2 \\
2 & 1
\end{bmatrix} \notin \cone_{\text{PSD}}$, thus
contradicting \cref{eq33}.

\def\dir{chapters/applications}
\chapter{Applications of Continuous Submodular Optimization}
\label{chapter_applications}

\begin{chapquote}{Winston Churchill}
	You will never get to the end of the journey if you stop to shy a stone at every dog that barks.
\end{chapquote}

Continuous submodularity naturally finds applications in various
scenarios, ranging from
influence and revenue maximization, to DPP MAP inference and mean
field inference of probabilistic graphical models.  In this part, we
will discuss several concrete problem instances.

\section{Submodular Quadratic Programming (SQP)} 

Non-convex/non-concave QP problem of the form
$f(\x) = \frac{1}{2} \x^\trans \bmH \x + \bh^\trans \x + c$ under
convex constraints naturally arises in many applications, including
scheduling \citep{DBLP:journals/jacm/Skutella01}, inventory theory,
and free boundary problems.  A special class of QP is the submodular
QP (the minimization of which was studied in \citet{kim2003exact}), in
which all off-diagonal entries of $\bmH$ are required to be
non-positive.  {Price optimization} with continuous prices is a
DR-submodular quadratic program \citep{ito2016large}.

Another representative  class of  DR-submodular quadratic objectives
arise when computing the {stability number}  $s(G)$ of a graph $G= (V, E)$, 
${s(G)}^{-1} = \min_{\x\in \Delta}\x^\trans (\bmA + \bmI)\x$,
where $\bmA$ is the adjacency matrix of the graph $G$, $\Delta$ is the
standard simplex \citep{motzkin1965maxima}.  This instance is a
convex-constrained monotone DR-submodular maximization problem.

\section{Continuous Extensions of Submodular Set Functions}
\label{sec_app_multilinear}
  
The Lov{\'a}sz extension \citep{lovasz1983submodular} used for
submodular set function minimization is both submodular and convex
(see Appendix A of \citet{bach2015submodular}).

The multilinear extension \citep{calinescu2007maximizing} is
extensively used for submodular set function maximization.  It is the
expected value of $F(S)$ under the surrogate distribution
$q(S|{\x}):= \prod_{i\in S}x_i \prod_{j\notin S}(1-x_j),
\x\in[0,1]^n$:
\begin{align}\label{eq_multilinear_ext}
\multi(\x) := \E_{q(S\mid \x)} [F(S)] =\sum_{S\subseteq
	\groundset}F(S)\prod_{i\in S}x_i \prod_{j\notin S}(1-x_j).
\end{align}
$\multi(\x)$ is DR-submodular and coordinate-wise linear
\citep{bach2015submodular}.
The partial derivative of $\multi(\x)$ can be expressed as,
\begin{align} 
\nabla_i \multi(\x)  & 
= \E_{q(S\mid \x, x_i = 1)} [F(S)]  - \E_{q(S\mid \x, x_i=0)} [F(S)]\\\notag 
& =\multi(\sete{x}{i}{1}) - \multi(\sete{x}{i}{0})
\\\notag  
&  =  \sum_{S\subseteq \groundset, S\ni i } F(S)
\prod_{j \in S\backslash\{i\}}x_j
\prod_{j'\notin S}(1-x_{j'})\\\notag 
& \quad  - \sum_{S\subseteq
	\groundset\backslash \{i\} }\ F(S) \prod_{j\in
	S} x_j \prod_{j'\notin S, j'\neq i}
(1-x_{j'}). 
\end{align}

At the first glance, evaluating the multilinear extension
in \cref{eq_multilinear_ext} costs
an exponential number of
operations. However, when used in practice, one can often use sampling
techniques to estimate its value and gradient. Furthermore, it is
worth noting that for several classes of practical submodular set
functions, their multilinear extensions $\multi()$ admit closed form
expressions.  We present details in the following.

\subsection{Gibbs Random Fields}\label{gibbs_multilinear}

Let us use $\v \in \{0,1\}^\groundset  $ to equivalently 
denote the $n$ binary random variables in Gibbs random fields. 
$F(\v)$ corresponds to the negative energy function 
in Gibbs random fields. If the energy function is
parameterized with a finite order of interactions, i.e., 
$F(\v) = \sum_{s\in \groundset} \theta_s v_s + \sum_{(s,t)\in
	\groundset \times \groundset} \theta_{s, t}v_s v_t + ... +
\sum_{(s_1, s_2, ..., s_d)} \theta_{s_1, s_2, ..., s_d}v_{s_1} \cdots
v_{s_d},  \; d< \infty$, then one can verify that its
multilinear extension has the following closed form,
\begin{align}
\multi(\x)
= \sum_{s\in \groundset} \theta_s x_s + \sum_{(s,t)\in \groundset
	\times \groundset} \theta_{s, t}x_s x_t + ...  \\\notag 
+  \sum_{(s_1, s_2,
	..., s_d)} \theta_{s_1, s_2, ..., s_d}x_{s_1} \cdots  x_{s_d}\,.
\end{align}

The gradient of this expression can also be easily derived.  Given
this observation, one can quickly  derive the multilinear extensions
of a large category of energy functions of Gibbs random fields, e.g.,
graph cut, hypergraph cut, Ising models, etc.  
Specifically,

\paragraph{Undirected  \maxcut.}
For undirected \maxcut, its objective is
$F(\v) = \frac{1}{2}\sum_{(i,j)\in E} w_{ij} (v_i + v_j -2v_i v_j), \v
\in \{0,1\}^\groundset $.
One can verify that its multilinear extension is
$\multi(\x) = \frac{1}{2}\sum_{(i,j)\in E} w_{ij} (x_i + x_j -2x_i
x_j), \x \in [0,1]^\groundset $.

\paragraph{Directed \maxcut.} For directed \maxcut, its objective is
$F(\v) = \sum_{(i,j)\in E} w_{ij} v_i (1- v_j), \v \in
\{0,1\}^\groundset $.
Its multilinear extension is
$\multi(\x) = \sum_{(i,j)\in E} w_{ij} x_i (1- x_j), \x \in
[0,1]^\groundset $.

\paragraph{Ising models.}
For Ising models \citep{ising1925contribution} with non-positive pairwise interactions (antiferromagnetic interactions),
$F(\v) = \sum_{s\in \groundset} \theta_s v_s + \sum_{(s,t)\in E}
\theta_{st}v_s v_t$,
$\v\in \{0, 1 \}^\groundset$, this objective can be easily verified to
be submodular.
Its multilinear extension is:
\begin{align}
  \multi(\x)= \sum_{s\in \groundset} \theta_s x_s + \sum_{(s,t)\in E}
  \theta_{st}x_s x_t, \x \in [0,1]^\groundset. 
\end{align}

\subsection{Facility Location and FLID (Facility Location Diversity)}

FLID is a diversity model \citep{Tschiatschek16diversity} that has
been designed as a computationally efficient alternative to DPPs
\citep{kulesza2012determinantal}.  It is in a more general form than
the facility location objective.  Let
$\BW\in \R_+^{|\groundset|\times D }$ be the weights, each row
correponds to the latent representation of an item, with $D$ as the
dimensionality. Then
\begin{align}\notag
  F(S) := & \sum\nolimits_ {i\in S} u_i + \sum\nolimits_{d=1}^{D} (
            \max_{i\in S} W_{i,d} - \sum\nolimits_{i\in S} W_{i,d} ) \\ 
  =& \sum\nolimits_ {i\in S} u'_i + \sum\nolimits_{d=1}^{D}
     \max_{i\in S}W_{i,d}, 
\end{align}
which models both coverage and diversity, and
$u'_i = u_i - \sum_{d=1}^{D} W_{i,d}$. If $u'_i=0$, one recovers the
facility location objective.
The computational complexity of evaluating its partition function is
$\bigo{|\groundset|^{D+1}}$ \citep{Tschiatschek16diversity}, which is
exponential in terms of $D$.

We now show the technique such that $\multi(\x)$ and
$\nabla_i \multi(\x) $ can be evaluated in $\bigo{Dn^2}$ time.
Firstly, for one $d\in [D]$, let us sort $W_{i,d}$ such that
$W_{i_d(1), d} \leq W_{i_d(2), d} \leq \cdots \leq W_{i_d(n), d} $.
After this sorting, there are $D$ permutations to record:
$i_d(l), l=1,...,n, \forall d\in [D]$.  Now, one can verify that,
\begin{align} 
& \multi(\x) \\\notag 
& =  \sum_ {i \in [n]} u'_i x_i +  \sum_d
\sum_{S\subseteq \groundset }  \max_{i \in S} W_{i, d}
\prod_{m\in S}x_m \prod_{m'\notin S}(1-x_{m'}) \\\notag 
& = \sum_ {i\in [n]} u'_i x_i + \sum_{d} \sum_{l=1}^n
W_{i_d(l), d} x_{i_d(l)} \prod_{m=l+1}^n [1-
x_{i_d(m)}].  
\end{align}
Sorting costs $\bigo{Dn\log n}$, and from the above expression, one
can see that the cost of evaluating $\multi(\x)$ is $\bigo{Dn^2}$. By the
relation that
$\nabla_i \multi(\x) = \multi(\sete{x}{i}{1}) -
\multi(\sete{x}{i}{0})$,
the cost is also $\bigo{Dn^2}$.

\subsection{Set Cover Functions}
\label{supp_setcover}
Suppose there are $|C| = \{c_1, ...,c_{|C|}\}$ concepts, and $n$ items
in $\groundset$. Give a set $S\subseteq \groundset$, $\Gamma (S)$
denotes the set of concepts covered by $S$. Given a modular function
$\m: 2^C \mapsto \R_+ $, the set cover function is defined as
$F(S) = \m (\Gamma(S))$.
This function models coverage in maximization,
and also the notion of complexity in minimization problems \citep{lin2011optimal}. 
Let us define an inverse map $\Gamma^{-1}$, such that for 
each concept $c$, $\Gamma^{-1}(c)$ denotes the set 
of items $v$ such that $\Gamma^{-1}(c) \ni v$. So the 
multilinear extension is,
\begin{align}\notag 
\multi(\x)  & =  \sum\nolimits_{i \in \groundset}  \m (\Gamma(S))  \prod\nolimits_{m\in S}x_m \prod\nolimits_{m'\notin S }(1-x_{m'}) \\
& =  \sum\nolimits_ {c\in C}  m_c \left[  1- \prod\nolimits_{ i\in \Gamma^{-1}(c) }  (1- x_i) \right].
\end{align}
The last equality is achieved by considering the situations
where a concept $c$ is covered.
One can observe that both $\multi(\x)$ and $\nabla_i \multi(\x) $ can
be evaluated in $\bigo{n|C|}$ time.

\subsection{General Case: Approximation by Sampling}

In the most general case, one may only have access to the function values of $F(S)$. 
In this scenario, one can use a polynomial number of sample steps to estimate
$\multi(\x)$ and its gradients. 

Specifically: 1) Sample $k$ times
$S \sim q(S|\x)$ and evaluate function values for them, resulting in
$F(S_1), ...,F(S_k)$.  2) Return the average
$\frac{1}{k}\sum_{i=1}^{k} F(S_i)$. According to the Hoeffding bound
\citep{hoeffding1963probability}, one can easily derive that
$\frac{1}{k}\sum_{i=1}^{k} F(S_i)$ is arbitrarily close to
$\multi(\x)$ with increasingly more samples: With probability at least
$1- \exp(-k\epsilon^2/2)$, it holds that
$|\frac{1}{k}\sum_{i=1}^{k} F(S_i) - \multi(\x)| \leq \epsilon \max_S
|F(S)| $, for all $\epsilon > 0$.

\section{Influence Maximization with Marketing Strategies}
\label{app_influence_max_marketing_strategies}

\citet{kempe2003maximizing} proposed the general marketing strategy
for influence maximization, which is a very realistic setting. They
assume that there exists a number $m$ of different marketing actions
$M_i$, each of which may affect some subset of nodes by increasing
their probabilities of being activated.  A natural property should be
that the more we spend on any one action, the stronger should be its
effect.

Formally, one chooses $x_i$ investments to marketing action $M_i$, so
one marketing strategy is an $m$-dimensional vector $\x\in \R^m$.  Then
the probability that node $i$ will become activated is described by
the activation function: $a_i(\x): \R^m \rightarrow [0, 1]$. This
function should satisfy the DR property by assuming that any
marketing strategy is more effective when the targeted individual is
less ``marketing-saturated'' at that point.

Now we search for the expected size of the final active set, which
is the expected influence. We know that given a marketing strategy
$\x$, a node $i$ becomes active with probability $a_i(\x)$, so the
expected influence is:
\begin{align}\label{influence_general_marketing}
f(\x) = \sum_{S\subseteq V} F(S) \prod_{i\in S} a_i(\x)
\prod_{j\notin S} (1 - a_j(\x)). 
\end{align}
$F(S)$ is the influence with the seeding set as $S$. It is submodular
for many influence models, such the Linear Threshold model and
Independent Cascade model of \citet{kempe2003maximizing}.  One can
easily see that \cref{influence_general_marketing} is DR-submodular by
viewing it as a composition of the multilinear extension of $F(S)$ and
the activation function $a(\x)$.

\subsection{Realizations of the Activation Function}
\label{app_activations_influence_max}

For the activation function $a_i(\x)$, we consider two realizations:

\begin{enumerate}
	\item Independent marketing actions.
	
	Here we provide one action for each user, and different
	actions are independent. So we have $m = |V|$ actions, and for user
	$i$, there exists an activation function $a_i(x_i)$, which is a one
	dimensional nondecreasing DR-submodular function. A specific
	instance is that $a_i(x_i) = 1 - (1 - p_i)^{x_i}$, $p_i \in [0, 1]$
	is the probability of user $i$ become activated with one unit of
	investment.
	
	\item Bipartite marketing actions.
	
	Suppose there are $m$ marketing actions and $|V|$ users.  The
	influence relationship among actions and users are modeled as a
	bipartite graph $(M, V; W)$, where $M$ and $V$ are collections of
	marketing actions and users, respectively, and $W$ is the collection
	of weights.  The edge weight, $p_{st}\in W$, represents the
	influence probability of action $s$ to users $t$ by providing one
	unit of investment to action $s$.  So with a marketing strategy as
	$\x$, the probability of a user $t$ being activated is
	$a_t(\x) = 1- \prod_{(s, t)\in W} \left(1-p_{st} \right)^{x_s}$.
	This is a nondecreasing DR-submodular function.
	
\end{enumerate}

One may notice that the independent marketing actions is a special
case of bipartite marketing actions.

\section{Optimal Budget Allocation with Continuous Assignments}

Optimal budget allocation is a special case of the influence
maximization problem.  It can be modeled as a bipartite graph
$(S,T; W)$, where $S$ and $T$ are collections of advertising channels
and customers, respectively.
The edge weight, $p_{st}\in W$, represents the influence probability
of channel $s$ to customer $t$.
The goal is to distribute the budget (e.g., time for a TV
advertisement, or space of an inline ad) among the source nodes, and
to maximize the expected influence on the potential customers
\citep{soma2014optimal,DBLP:conf/aaai/HatanoFMK15}.  

The total influence of customer $t$ from all channels can be modeled
by a proper monotone DR-submodular function $I_t(\x)$, e.g.,
$I_t(\x) = 1- \prod_{(s, t)\in W} \left(1-p_{st} \right)^{x_s}$ where
$\x\in \R^S_+$ is the budget assignment among the advertising
channels.  For a set of $k$ advertisers, let $\x^i\in \R^S_+$ be
the budget assignment for advertiser $i$, and
$\x:= [\x^1,\cdots, \x^k]$ denote the assignments for all the
advertisers.  The overall objective is,
\begin{flalign}
  & g(\x)= \sum\nolimits_{i=1}^k \alpha_i f(\x^i) ~\text{ with }~\\
  & f(\x^i) :=\sum\nolimits_{t\in T} I_t(\x^i), \; \zero \leqco
  \x^i\leqco \bar \bu^i , \forall i = 1,..., k,
\end{flalign}
which is monotone DR-submodular.

A concrete application is defined by advertiser bidding for search
marketing, i.e., where vendors bid for the right to appear alongside
the results of different search keywords.
Here, $x^i_s$ is the volume of advertisement space allocated to the
advertiser $i$ to show his ad alongside query keyword $s$.  The search
engine company needs to distribute the budget (advertising space) to
all vendors to maximize their influence on the
customers, %
while respecting various constraints. For example, each vendor has a
specified budget limit for advertising, and the ad space associated
with each search keyword can not be too large.
All such constraints can be formulated as a down-closed polytope $\P$,
hence the \submodularfw algorithm (\cref{alg_sfmax_GradientAscend} in \cref{chapter_max_monotone})  can be used to find an
approximate solution for the problem $\max_{\x\in \P} g(\x)$.

Note that one can flexibly add regularizers in designing
$I_t(\x^i)$ as long as it remains monotone DR-submodular. 
For example, adding separable regularizers of the form
$\sum_s \phi(x^i_s)$ does not change off-diagonal entries of the
Hessian, and hence maintains submodularity. Alternatively, bounding
the second-order derivative of $\phi(x^i_s)$ ensures DR-submodularity.

\section{Softmax Extension for DPPs}  

Determinantal point processes (DPPs) are probabilistic models of
repulsion, that have been used to model diversity in machine learning
\citep{kulesza2012determinantal}. The constrained MAP (maximum a
posteriori) inference problem of a DPP is an NP-hard combinatorial
problem in general. Currently, the methods with the best approximation
guarantees are based on either maximizing the multilinear extension
\citep{calinescu2007maximizing} or the softmax extension
\citep{gillenwater2012near}, both of which are continuous DR-submodular
functions.

The multilinear extension is given as an expectation over the original
set function values, thus evaluating the objective of this extension
requires expensive sampling in general.  In contrast, the softmax extension has a
closed form expression, which is much more appealing from a
computational perspective.
Let $\bmL$ be the positive semidefinite kernel matrix of a DPP, its
softmax extension is:
\begin{flalign}\label{eq_softmax}
f(\x) = \log\de{\diag(\x)(\bmL-\bmI) +\bmI }, \x\in [0,1]^n,
\end{flalign}
where $\bmI$ is the identity matrix, $\diag(\x)$ is the diagonal
matrix with diagonal elements set as $\x$. 
Its DR-submodularity can be established by directly applying 
Lemma 3 in \citet{gillenwater2012near},  which immediately implies 
that all  entries of  $\nabla^2  f$ are non-positive, so $f(\x)$
is continuous DR-submodular.

The problem of MAP
inference in DPPs corresponds to the problem $\max_{\x\in \P} f(\x)$,
where $\P$ is a down-closed convex constraint, e.g., a matroid
polytope or a matching polytope.

\section{Mean Field Inference for Probabilistic Log-Submodular Models}

Probabilistic log-submodular models \citep{djolonga14variational} are a class of
probabilistic models over subsets of a ground set $\groundset = [n]$,
where the log-densities are submodular set functions $F(S)$:
$p(S) = \frac{1}{\parti}\exp(F(S))$. The partition function
$ \parti = \sum_{S\subseteq \groundset}\exp(F(S))$ is typically hard to
evaluate.  One can use mean field inference to approximate $p(S)$ by
some factorized distribution
$q(S|\x):= \prod_{i\in S}x_i \prod_{j\notin S}(1-x_j), \x\in
[0,1]^n$,
by minimizing the distance measured w.r.t. the Kullback-Leibler
divergence between $q$ and $p$, i.e.,
$ \sum_{S\subseteq \groundset} q(S|\x)
\log\frac{q(S|\x)}{p(S)}$. It is,

\begin{align} 
\text{KL}(\x) & = 
-\sum_{S\subseteq \groundset}F(S)  \prod_{i\in S}x_i \prod_{j\notin S}(1-x_j) +\\\notag
&  \sum\nolimits_{i=1}^{n} [x_i\log x_i + (1-x_i)\log(1-x_i)] + \log \parti.
\end{align}

$ \text{KL}(\x)$ is \mesuper w.r.t. $\x$. To see this: The first term
is the negative of a multilinear extension, so it is \mesuper. The
second term is separable, and coordinate-wise convex, so it will not
affect the off-diagonal entries of $\nabla^2 \text{KL}(\x)$, it will
only contribute to the diagonal entries.  Now, one can see that all
entries of $\nabla^2 \text{KL}(\x)$
are non-negative, so $\text{KL}(\x)$ is \mesuper w.r.t. $\x$.
Minimizing the Kullback-Leibler divergence $\text{KL}(\x)$ amounts to
maximizing a DR-submodular function.

\section{Revenue Maximization with Continuous Assignments}
\label{sec_revenue_max}

The viral marketing suggests to choose a small subset of buyers to
give them some product for free, to trigger a cascade of further
adoptions through ``word-of-mouth'' effects, in order to maximize the
total revenue \citep{hartline2008optimal}.  For some products (e.g.,
software), the seller usually gives away the product in the form of a
trial, to be used for free for a limited time period.  In this task,
except for deciding whether to choose a user or not, the sellers also
need to decide how much the free assignment should be, in which the
assignments should be modeled as continuous variables.
We call this problem \emph{revenue maximization with continuous
  assignments}.

We use a directed graph $G = (\groundset, E; \BW)$ to represent the
social connection graph. $\groundset$ contains all the $n$ users, $E$
is the edge set, and $\BW$ is the adjacency matrix. We treat the
undirected social connection graph as a special case of the directed
graph, by taking one undirected edge as two directed edge with the
same weight.

\subsection{A Variant of the Influence-and-Exploit (IE) Strategy} 

This model has been used in \citet{soma2017non} and \citet{durr2019non}.  It can be treated as a
simplified variant of the Influence-and-Exploit (IE) strategy of
\citet{hartline2008optimal}.

Specifically:

\begin{itemize}
\item \emph{Influence} stage: For each of the user $i$, we give him
  $x_i$ units of products for free, the user becomes an advocate of
  the product with probability $1 - q^{x_i}$ (independently from other
  users), where $q\in (0, 1)$ is a parameter. This is consistent with
  the intuition that with more free assignment, the user is more
  likely to advocate the product.

\item \emph{Exploit} stage: suppose that a set $S$ of users advocate
  the product while the complement set $\groundset \setminus S$ of
  users do not. Now the revenue comes from the users in
  $\groundset \setminus S$, since they will be influenced by the
  advocates with probability proportional to the edge weights. We use
  a simplified concave graph model \citep{hartline2008optimal} for the
  value function, i.e.,
  $v_j(S) = \sum_{i\in S} W_{ij}, j\in \groundset \setminus S$.
  Assume for simplicity that the users of $\groundset \setminus S$ are
  visited independently with each other.  Then the revenue is:
  \begin{align}
    R(S) =\sum_{j\in\groundset \setminus S} v_j(S) = \sum_{j\in
    \groundset\setminus S} \sum_{i\in S} W_{ij}.  
  \end{align}
  Notice that $S$ is a random set drawn according to the distribution
  specified by the continuous assignment $\x$.
\end{itemize}

With this Influence-and-Exploit (IE) strategy, the expected revenue is
a function $f: \R_+^\groundset \rightarrow \R_+$, as shown below:
\begin{align}\notag 
  f(\x) & = \epe[S]{R(S)}\\ 
	&= \epe[S]{\sum_{i\in S} \sum_{j\in \groundset\setminus
          S}W_{ij} } \\
        & = \sum_{i\in \groundset} \sum_{j\in \groundset\setminus
          \{i\}} W_{ij} (1- q^{x_i})q^{x_j}.
\end{align}

\subsection{An Alternative Model}
\label{app_revenue_max_alternative}

In addition to the Influence-and-Exploit (IE) model, we have also studied  an alternative model.  Assume there are $q$ products
and $n$ buyers/users, let $\x^i \in \R_+^n$  be the assignments of
product $i$ to the $n$ users, let $\x:= [\x^1,\cdots, \x^q]$ denote
the assignments for the $q$ products.  The revenue can be modeled as
$g(\x) = \sum_{i=1}^q f(\x^i)$ with
\begin{flalign}\label{eq_re}
  f(\x^i) := \alpha_i \sum\nolimits_{s: x^i_s =0} R_s(\x^i) + \beta_i
  \sum\nolimits_{t: x^i_t \neq 0} \phi (x^i_t) +\gamma_i
  \sum\nolimits_{t: x^i_t \neq 0} \bar R_t(\x^i),\\\notag \zero \leqco
  \x^i \leqco \bar \bu^i,
\end{flalign}
where $x^i_t$ is the assignment of product $i$ to user $t$ for free,
e.g., the amount of free trial time or the amount of the product
itself.  $R_s(\x^i)$ models revenue gain from user $s$ who did not
receive the free assignment. It can be some non-negative,
non-decreasing submodular function. $\phi (x^i_t)$ models revenue gain
from user $t$ who received the free assignment, since the more one
user tries the product, the more likely he/she will buy it after the
trial period.  $\bar R_t(\x^i)$ models the revenue loss from user $t$
(in the free trial time period the seller cannot get profits), which
can be some non-positive, non-increasing submodular function.
For products with continuous assignments, usually the cost of the
product does not increase with its amount, e.g., the product as a
software, so we only have the box constraint on each assignment. The
objective in \cref{eq_re} is generally
\emph{non-concave/non-convex}, and non-monotone submodular (see
\cref{supp_revenue} for more details).
\begin{lemma}\label{revenue}
  If $R_s(\x^i)$ is non-decreasing submodular and $\bar R_t(\x^i)$ is
  non-increasing submodular, then $f(\x^i)$ in \cref{eq_re} is
  submodular.
\end{lemma}

\section{Applications Generalized from the Discrete Setting}

Many discrete submodular problems can be naturally generalized to the
continuous setting with continuous submodular objectives.  The maximum
coverage problem and the problem of text summarization with submodular
objectives are among the examples \citep{lin2010multi}. We put details
in the sequel.

\subsection{Text Summarization}  

Submodularity-based objective functions for text summarization perform
well in practice \citep{lin2010multi}.  Let $C$ be the set of all
concepts, and $\groundset$ be the set of all sentences.  As a
typical example, the concept-based summarization aims to find a subset
$S$ of the sentences to maximize the total credit of concepts covered
by $S$. \citet{soma2014optimal} considered extending the submodular
text summarization model to the one that incorporates ``confidence''
of a sentence, which has discrete value, and modeling the objective to
be an integer submodular function.  It is also natural to model the
confidence level of sentence $i$ to be a continuous value
$x_i\in [0, 1]$. Let us use $p_i(x_i)$ to denote the set of covered
concepts when selecting sentence $i$ with confidence level $x_i$, it
can be a monotone covering function
$p_i: \R_+ \rightarrow 2^C, \forall i\in \groundset$.  Then the
objective function of the extended model is
$f(\x) = \sum_{j\in \cup_i p_i(x_i) } c_j$, where $c_j\in \R_+$ is the
credit of concept $j$. It can be verified that this objective is a
monotone continuous submodular function.

\subsection{Sensor Energy Management}
For cost-sensitive outbreak detection in sensor networks
\citep{leskovec2007cost}, one needs to place sensors in a subset of
locations selected from all the possible locations $\groundset$, to
quickly detect a set of contamination events $E$, while respecting the
cost constraints of the sensors.  For each location $v\in \groundset$
and each event $e\in E$, a value $t (v, e)$ is provided as the time it
takes for the placed sensor in $v$ to detect event
$e$. \citet{DBLP:conf/nips/SomaY15} considered the sensors with
discrete energy levels. It is natural to model the energy levels of
sensors to be a \emph{continuous} variable $\x\in \R_+^\groundset$.
For a sensor with energy level $x_v$, the success probability it
detects the event is $1-(1-p)^{x_v}$, which models that by spending
one unit of energy one has an extra chance of detecting the event with
probability $p$.  In this model, beyond deciding whether to place a
sensor or not, one also needs to decide the optimal energy levels. Let
$t_{\infty} = \max_{e\in E, v\in \groundset}t(v,e)$, let $v_e$ be the
first sensor that detects event $e$ ($v_e$ is a random variable).  One
can define the objective as the expected detection time that could be 
\textit{saved},
\begin{flalign}
  f(\x) := \mathbb{E}_{e\in E} \mathbb{E}_{v_e} [t_{\infty} - t(v_e,
  e)],
\end{flalign}
which is a monotone DR-submodular function.  Maximizing $f(\x)$ w.r.t.
the cost constraints pursues the goal of finding the optimal energy
levels of the sensors, to maximize the expected detection time that
could be saved.

\subsection{Multi-Resolution Summarization}
Suppose we have a collection of items, e.g., images
$\groundset = \{\ele_1, ..., \ele_n\}$. 
We follow the strategy to extract a representative summary, where
representativeness is defined w.r.t.~a submodular set function
$F:2^\groundset\to \mathbb{R}$. However, instead of returning a single
set, our goal is to obtain summaries at multiple levels of detail or
resolution. One way to achieve this goal is to assign each item
$\ele_i$ a nonnegative score $x_i$. Given a user-tunable threshold
$\tau$, the resulting summary $S_\tau=\{\ele_i | x_i \geq \tau\}$ is
the set of items with scores exceeding $\tau$. Thus, instead of
solving the discrete problem of selecting a fixed set $S$, we pursue
the goal to optimize over the scores, e.g., to use the following
continuous submodular function,
\begin{flalign} 
f(\x) =  \sum\nolimits_{i \in \groundset} \sum\nolimits_{j\in \groundset}  \phi(x_j) s_{i,j}
- \sum\nolimits_{i \in \groundset} \sum\nolimits_{j \in \groundset} x_i x_j s_{i,j},
\end{flalign} 
where $s_{i,j}\geq 0$ is the similarity between items $i,j$, and
$\phi(\cdot)$ is a non-decreasing concave function.

\subsection{Facility Location with Scales}

The classical discrete facility location problem can be 
generalized to the continuous case where the scale of a facility is
determined by a continuous value in interval $[\zero, \bar \bu]$.  For a
set of facilities $\groundset$, let $\x\in \R_+^\groundset$ be the
scale of all facilities.  The goal is to decide how large each
facility should be in order to optimally serve a set $T$ of
customers. For a facility $s$ of scale $x_s$, let $p_{st}(x_s)$ be the
value of service it can provide to customer $t\in T$, where
$p_{st}(x_s)$ is a normalized monotone function ($p_{st}(0) =
0$). %
Assuming each customer chooses the facility with highest value, the
total service provided to all customers is
$f(\x) = \sum_{t\in T} \max_{s\in \groundset} p_{st}(x_s)$.  It can be
shown that $f$ is monotone submodular.

\section[Exemplar Applications of Generalized Submodularity]{Exemplar
  Applications Captured by Generalized Submodularity on Conic
  Lattices}
\label{subsec_moti_lattice}

In \cref{sec_lattice} we show the technical details on a class of
continuous submodular functions over conic lattices. Here we list two
prototypical 
applications that are not continuous submodular, but continuous
submodular over a conic lattice.

\subsection{Logistic Regression with a Separable Regularizer}

Consider the logistic regression model with a \emph{non-convex}
separable regularizer.  This flexibility may result in better
statistical performance (e.g., in recovering discontinuities
\citep{antoniadis2011penalized}) compared to classical models with
convex regularizers.
Let $\z^1,..., \z^m$ in $\R^n$ be $m$ training samples with
corresponding binary labels $\y\in \{\pm 1\}^m$. Assume that the
following mild assumption is satisfied:
For any fixed dimension $i$, all the data points have the same sign,
i.e., $\sign{z^j_i}$ is the same for all $j \in [m]$ (which can be
achieved by easily scaling if not).

The task is to solve the following non-convex optimization problem,
\begin{flalign}\label{lr}
  \min_{\x\in \R^n} f(\x) := \frac{1}{m}\sum\nolimits_{j=1}^{m}f_j(\x)
  +\lambda r(\x),
\end{flalign}
where $f_j(\x) = \log(1 +\exp(-y_j \x^\trans \z^j))$ is the logistic
loss; $\lambda>0$ is the regularization parameter, and $r(\x) $ is
some non-convex separable regularizer.  Such separable regularizers
are popular in statistics, and two notable choices are
$r(\x)= \sum_{i= 1}^n \frac{\gamma x_i^2}{1+\gamma x_i^2}$, and
$r(\x) = \sum_{i= 1}^n \min \{\gamma x_i^2, 1 \}$ (see
\citet{antoniadis2011penalized} for more choices).
Let us define a vector $\bmalpha\in \{\pm 1 \}^n$ as
$\alpha_i = \text{sign}(z_i^j), i\in [n]$ and
$l(\x) :=\frac{1}{m}\sum\nolimits_{j=1}^{m}f_j(\x)$.

One can show that $l(\x)$ is not DR-submodular or \mesuper.  Yet, we
can show that $l(\x)$ is $\cone_{\bmalpha}$-\mesuper, where the latter
generalizes \me-supermodularity.

\begin{lemma}\label{claim_logistic}
Consider the logistic loss: 
\begin{flalign} \label{app_lr}
  l(\x) =\frac{1}{m}\sum\nolimits_{j=1}^{m}f_j(\x) =
  \frac{1}{m}\sum\nolimits_{j=1}^{m} \log(1 +\exp(-y_j \x^\trans
  \z^j)).
\end{flalign}	
$l(\x)$ above is $\cone_{\bmalpha}$-\mesuper.
\end{lemma}

Usually, one can assume the optimal solution $\x^*$ lies in some box
$[\underline \u, \bar \u]$. Then the problem is an instance of constrained non-monotone $\cone_{\bmalpha}$-DR-submodular maximization
problem.

\subsection{Non-Negative PCA (NN-PCA)}
\label{sec_nnpca}

NN-PCA \citep{zass2007nonnegative,sigg2008expectation,montanari2016non} is widely used as
alternative models of PCA for dimension reduction, since its
projection involves only non-negative weights -- a required property
in fields like economics, bioinformatics and computer vision.

For a given set of $m$ data points $\z^j\in \R^n, j\in [m]$, NN-PCA
aims to solve the following non-convex optimization problem:
\begin{flalign}\label{nn_pca_append}
  \min_{\|\x\|_2 \leq 1, \x\geqco \zero} f(\x) : = -\frac{1}{2}
  \x^\trans \left(\sum\nolimits_{j=1}^{m} \z^j {\z^j}^\trans\right )
  \x.
\end{flalign}
Let $\bmA= \sum\nolimits_{j=1}^{m} \z^j {\z^j}^\trans$,  
one can see that,
\begin{flalign} 
  & A_{pp} = \sum\nolimits_{j=1}^{m}(z_p^j)^2 \geq 0, \; A_{pq} =
  \sum\nolimits_{j=1}^{m} z^j_p z^j_q = A_{qp}.
\end{flalign}

Let us make the following weak assumption: 
For one dimension/feature $i$, all the data points have the same sign,
i.e., $\sign{z^j_i}$ is the same for all $j \in [m]$ (which can be
achieved by easily scaling if not).
Now, by choosing the sign vector $\bmalpha\in \{\pm 1\}^n$ to be
$\alpha_p = \sign{z^j_p}, \forall p\in [n]$, one can easily verify
that $A_{pq}\alpha_p\alpha_q \geq 0, \forall p,q\in [n]$. Notice that
$\nabla^2 f$ in \cref{nn_pca_append} is $-\bmA$, so it holds that
$\alpha_p\alpha_q\nabla^2_{pq} f \leq 0, \forall p,q\in [n]$, thus
$f(\x)$ is $\cone_{\bmalpha}$-DR-submodular according
to \labelcref{eq_general_conedr}.  Thus we can
treat \labelcref{nn_pca_append} as an instance of the constrained
$\cone_{\bmalpha}$-DR-submodular minimization problem.

\section{Conclusions}

In this chapter we have discussed various classes of applications
whose objectives fall into the class of continuous submodular
functions. They motivate us to study polynomial-time algorithms with
strong approximation guarantees, which will be presented in the subsequent chapters.

\section{Additional Details}

\subsection{Details of Revenue Maximization with Continuous
  Assignments}
\label{supp_revenue}

\subsubsection{More Details About the Model}
As discussed in the main text, $R_s(\x^i)$ should be some
non-negative, non-decreasing, submodular function; therefore, we set 
$R_s(\x^i) := \allowbreak \sqrt{\sum_{t: x^i_t \neq 0}x^i_t w_{st}}$,
where $w_{st}$ is the weight of edge connecting users $s$ and $t$.
The first part in R.H.S. of \cref{eq_re} models the revenue from users
who have not received free assignments, while the second and third
parts model the revenue from users who have gotten the free
assignments.
We use $w_{tt}$ to denote the ``self-activation rate" of user $t$:
Given certain amount of free trail to user $t$, how probable is it
that he/she will buy after the trial.  The intuition of modeling the
second part in R.H.S. of \cref{eq_re} is: Given the users more free
assignments, they are more likely to buy the product after using it.
Therefore, we model the expected revenue in this part by
$\phi(x^i_t) = w_{tt}x^i_t$; The intuition of modeling the third part
in R.H.S. of \cref{eq_re} is: Giving the users more free assignments,
the revenue could decrease, since the users use the product for free
for a longer period.  As a simple example, the decrease in the revenue
can be modeled as $\gamma \sum_{t:x^i_t\neq 0} -x^i_t$.

\subsubsection{Proof of Lemma \ref{revenue}}

\begin{proof}
	
	First of all, we prove that $g(\x) : = \sum_{s: x_s =0} R_s(\x)$
	is a non-negative submodular function.
	
	It is easy to see that $g(\x)$ is non-negative. 
	To prove that $g(\x)$ is submodular, one just need,
	\begin{flalign}\label{eq_f}
	g(\a) + g(\b) \geq g(\a\vee \b) + g(\a\wedge \b), \quad  \forall \a, \b \in [\zero, \bar \bu].
	\end{flalign}
	Let $A:= \spt{\a}, B := \spt{\b}$, where $\spt{\x}:=\{i|x_i\neq 0 \}$ is  the  support of the vector $\x$.
	First of all, because $R_s(\x)$ is non-decreasing,  and $\b\geqco \a\wedge \b$, $\a\geqco \a\wedge \b$, 
	
	\begin{flalign}\label{eq_1}
	\sum_{s\in A\backslash B} R_s(\b) + \sum_{s\in B\backslash A} R_s(\a) \geq \sum_{s\in A\backslash B} R_s(\a\wedge \b)  + \sum_{s\in B\backslash A} R_s(\a\wedge \b). 
	\end{flalign}
	By submodularity of $R_s(\x)$, and  summing over $s\in \groundset \backslash(A\cup B)$,
	\begin{flalign}\label{eq_2}
	\sum_{s\in \groundset \backslash(A\cup B)}R_s(\a) + \sum_{s\in \groundset \backslash(A\cup B)}R_s(\b) \geq \sum_{s\in \groundset \backslash(A\cup B)}R_s(\a\vee \b) + \sum_{s\in \groundset \backslash(A\cup B)}R_s(\a\wedge \b).
	\end{flalign}
	Summing Equations  \ref{eq_1} and \ref{eq_2} one can get
	\begin{flalign}\notag 
	\sum_{s\in \groundset \backslash A}R_s(\a) + \sum_{s\in \groundset \backslash B}R_s(\b) \geq \sum_{s\in \groundset \backslash(A\cup B)}R_s(\a\vee \b) + \sum_{s\in \groundset \backslash(A\cap  B)}R_s(\a\wedge \b)
	\end{flalign}
	which is equivalent to  \cref{eq_f}.
	
	Then we prove that $h(\x):=\sum_{t: x_t \neq 0} \bar R_t(\x)$ is submodular. 
	Because $\bar R_t(\x)$ is non-increasing, and $\a\leqco \a\vee \b$, 
	$\b \leqco \a\vee \b$, 
	\begin{flalign}\label{37}
	\sum_{t\in A\backslash B} \bar R_t(\a) + \sum_{t\in B\backslash A} \bar R_t(\b) \geq \sum_{t\in A\backslash B} \bar R_t(\a\vee \b) + \sum_{t\in B\backslash A} \bar R_t(\a\vee \b).
	\end{flalign}
	By submodularity of $\bar R_t(\x)$, and summing over $t\in A\cap  B$,
	\begin{flalign}\label{38}
	\sum_{t\in A\cap  B} \bar R_t(\a) + \sum_{t\in A\cap  B} \bar R_t(\b) \geq \sum_{t\in A\cap  B} \bar R_t(\a\vee \b) + \sum_{t\in A\cap  B} \bar R_t(\a\wedge \b).
	\end{flalign}
	Summing Equations \ref{37}, \ref{38} we get,
	\begin{flalign}
	\sum_{t\in A} \bar R_t(\a) + \sum_{t\in  B} \bar R_t(\b) \geq \sum_{t\in A\cup  B} \bar R_t(\a\vee \b) + \sum_{t\in A\cap  B} \bar R_t(\a\wedge \b)
	\end{flalign}
	which is equivalent to $h(\a)+h(\b)\geq h(\a\vee \b)+h(\a\wedge \b)$, $\forall \a, \b \in [\zero, \bar \bu]$, thus proving the submodularity of 
	$h(\x)$.
	
	Finally, because $f(\x)$ is the sum of two submodular functions and one 
	modular function, so it is submodular.
\end{proof}

\subsection{Proof for the  Logistic Loss in \cref{subsec_moti_lattice}}
\label{appe_dr_lr}

\begin{proof}[Proof of \cref{claim_logistic}]
	To show that $l(\x)$ is  $\cone_{\bmalpha}$-\mesuper, we can
	check the second-order condition in \labelcref{eq_general_conedr}, that is, whether it holds that 
	$\alpha_p\alpha_q\nabla_{pq}^2 l(\x) \geq  0,\; \forall p, q \in [n]$. One can verify  that, 
	\begin{align} 
	& \fracpartial{l(\x)}{x_p} = \frac{1}{m}\sum\nolimits_{j=1}^{m} \frac{-y_j z_{p}^j}{\exp{(y_j \x^\trans \z^j)}+1},\\ 
	& \fracppartial{l(\x)}{x_p}{x_q} =  \frac{1}{m}\sum\nolimits_{j=1}^{m} \frac{\exp{(y_j \x^\trans \z^j)}}{[\exp{(y_j \x^\trans \z^j)}+1]^2}z^j_p z^j_q .
	\end{align}
	Since $\alpha_p = \text{sign}(z^j_p)$, so $\alpha_p\alpha_q\nabla_{pq}^2 l(\x) \geq  0,\; \forall p, q \in [n]$. Thus 
	$l(\x)$ in \cref{app_lr} is $\cone_{\bmalpha}$-\mesuper
	according to \labelcref{eq_general_conedr}.
\end{proof}

\def\dir{chapters/mono-sub-max}
\chapter{Maximizing Monotone Continuous DR-Submodular  Functions}
\label{chapter_max_monotone}

\begin{chapquote}{Master Oogway}
	Your mind is like this water, my friend. When it is agitated, it becomes difficult to see. But if you allow it to settle, the answer becomes clear.
\end{chapquote}

In this chapter, we study the problem of maximizing a monotone
continuous DR-submodular function subject to a down-closed convex
constraint, i.e.,
\begin{align}\label{problem_monotone_max}
\max_{\x \in \P\subseteq \X} f(\x),    
\end{align}
where $f: \X \rightarrow \R$ is DR-submodular and monotone
nondecreasing.  A function $f$ is monotone nondecreasing if
$\forall \x \leqco \y\in \X$, it holds $f(\x) \leq f(\y)$.

Maximizing a monotone DR-submodular function over a down-closed convex
constraint has many real-world applications, e.g., influence
maximization with general marketing strategies and sensor energy
management. One can refer to \cref{chapter_applications} for details
on applications.

\section{Hardness and Inapproximability Results}

Though with the monotonicity assumption, solving
problem \labelcref{problem_monotone_max} is still a very challenging
task.  Actually, we prove the following hardness result:

\begin{proposition}[Hardness and Inapproximability]\label{prop_np}
  The problem of maximizing a monotone nondecreasing continuous
  DR-submodular function subject to a general down-closed
  \emph{polytope} constraint is NP-hard.  For any $\epsilon >0$, it
  cannot be approximated in polynomial time within a ratio of
  $(1-1/e+\epsilon)$ (up to low-order terms), unless RP = NP.
\end{proposition}

\begin{remark}
  Due to the NP-hardness of converging to the global optimum for
  problem \labelcref{problem_monotone_max}, in the following by
  ``convergence'' we mean converging to a solution point which has a
  constant factor approximation guarantee with respect to the global
  optimum.
\end{remark}

\section[Algorithms Based on the Local-Global Relation]{Algorithms
  Based on the Local-Global Relation: Non-Convex FW and PGA}

The first class of algorithms directly utilize the local-global
relation in \cref{coro_1half}: we know that any stationary point is a
1/2 approximate solution, thus plugging in a solver that can reach a
stationary point would result in an algorithm with a 1/2 approximation
guarantee.

\citet{hassani2017gradient} showed that the {Projected Gradient
  Ascent} algorithm (\pga) with constant \stepsize ($1/L$) can
converge to a stationary point, so it has a 1/2 approximation
guarantee.
We can also show that the \nonconvexfw of
\citet{lacoste2016convergence} has a 1/2 approximation guarantee
according to the local-global relation:

\begin{corollary}\label{coro_nonconvex_fw}
  The {non-convex Frank-Wolfe} algorithm (\nonconvexfw) of
  \citet{lacoste2016convergence} has a 1/2 approximation guarantee,
  and $1/\sqrt{k}$ rate of convergence for solving Problem \labelcref{problem_monotone_max}.
\end{corollary}

\subsection{The \nonconvexfw Algorithm}
\label{sec_nonconvex_fw}

To write the thesis in a self-contained style, we summarized the \nonconvexfw
algorithm in \cref{nonconvex_fw}.

\begin{algorithm}[htbp]
  \caption{\nonconvexfw
    $(f, \P, K, \epsilon, \x^\pare{0})$\citep{lacoste2016convergence}
    for maximizing a smooth objective}\label{nonconvex_fw}
  \KwIn{$\max_{\x \in \P} f(\x)$, $f$: a smooth function, $\P$:
    {convex} set, $K$: number of iterations, $\epsilon$: stopping
    tolerance}
  \For{$k = 0, ... , K$}{ {find
      $\v^\pare{k} \text{ s.t. } \dtp{\v^\pare{k}}{\nabla
        f(\x^\pare{k})} \geq \max_{\v\in\P} \dtp{\v}{ \nabla
        f(\x^\pare{k})}$\tcp*{\emph{LMO}}}
    {$\d^\pare{k} \leftarrow \v^\pare{k} - \x^\pare{k}$,
      $g_k := \dtp{\d_k}{\nabla f(\x^\pare{k})}$ \tcp*{$g_k$:
        non-stationarity measure}} {\bfseries{if $g_k \leq \epsilon$
        then return $\x^\pare{k}$}\;} {Option I:
      $\gamma_k \in \argmin_{\gamma\in [0, 1]}f(\x^\pare{k} + \gamma
      \d^\pare{k})$,
			
      Option II: $\gamma_k \leftarrow \min \{\frac{g_k}{C}, 1 \}$ for
      $C\geq C_f(\P)$ \;}
    {$\x^\pare{k+1}\leftarrow \x^\pare{k} + \gamma_k \d^\pare{k}$ \;}
  } \KwOut{$\x^\pare{k'}$ and $g_{k'} = \min_{0\leq k\leq K} g_k$
    \tcp*{modified output solution compared to that of
      \citet{lacoste2016convergence}}}
\end{algorithm}

\cref{nonconvex_fw} is modified from \citet{lacoste2016convergence},
the only difference lies in the output: we output the solution
$\x^\pare{k'}$ with the minimum non-stationarity, which is needed to
apply the local-global relation. While \citet{lacoste2016convergence}
outputs the solution in the last iteration.

Since $C_f(\P)$ is generally hard to evaluate, we tested with the
classical oblivious \stepsize rule ($\frac{2}{k+2}$) and the Lipschitz
\stepsize rule
($\gamma_k = \min\{1, \frac{g_\pare{k}}{L \|\d^\pare{k}\|} \}$, where
$g_\pare{k}$ is the so-called Frank-Wolfe gap) in the experiments.

\subsection{The \pga Algorithm}

\begin{algorithm}[htbp]
  \caption{\pga for maximizing a monotone DR-submodular objective
    \citep{hassani2017gradient}}\label{alg_pga}
	
  \KwIn{$\max_{\x \in \P} f(\x)$, $f$: a smooth DR-Submodular
    function, $\P$: {convex} set, $K$: number of iterations,
    $\x^\pare{0}\in \P$}
  \For{$k = 0, ... , K-1$}{ {Set \stepsize $\gamma_k$ \tcp*{i):
        ``Lipschitz'' rule $\frac{1}{L}$; ii): adaptive rule:
        $C/\sqrt{k}$}}
    {$\y^\pare{k+1} \leftarrow \x^\pare{k} + \gamma_k \nabla
      f(\x^\pare{k})$\;}
    {$\x^\pare{k+1}\leftarrow \argmin_{\x\in \P} \|\x -
      \y^\pare{k+1}\| $
      \tcp*{Projection}} } \KwOut{$\x^\pare{k'}$ with
    ${k'} = \argmax_{0\leq k\leq K} f(\x^\pare{k})$ \tcp*{modified
      output compared to that of \citet{hassani2017gradient}}}
\end{algorithm}

\cref{alg_pga} is taken from \citet{hassani2017gradient}.
It accepts a smooth DR-submodular function $f$, and a convex
constraint $\P$. Then it runs for $K$ iterations. In each iteration,
we firstly choose a \stepsize $\gamma_k$, then we update the current
solution using the current gradient to get a point $\y^\pare{k+1}$.
Lastly, it projects $\y^\pare{k+1}$ onto the convex set $\P$, which
amounts to solving a constrained quadratic program. After $K$
iterations, we output the solution with the maximal function value,
which is slightly different from that of \citet{hassani2017gradient}.

It has a 1/2 approximation guarantee and sublinear rate of
convergence:

\begin{theorem}[\cite{hassani2017gradient}]
  For \cref{alg_pga}, if one choose $\gamma_k = 1/L$, then after $K$
  iterations,
  \begin{align}
    f(\x^\pare{K}) \geq \frac{f(\optcont)}{2} - \frac{D^2 L}{2K}.
  \end{align}
\end{theorem}
It is worth noting that, in general the smoothness parameter $L$ is
difficult to estimate, so the ``Lipschitz'' \stepsize rule
$\gamma_k = 1/L$ poses a challenge for implementation. In experiments,
\citet{hassani2017gradient} also suggests the adaptive \stepsize rule
$\gamma_k = C/\sqrt{k}$, where $C$ is a constant.

\section{\texttt{Submodular FW}\xspace: Follow Concave Directions}

For DR-submodular maximization, one key intuition is that
{DR}-submodular functions are non-convex/non-concave in general,
however, it is concave along any non-negative directions, as shown by
\cref{prop_concave}.  Thus, if we design the algorithm such that it
follows a non-negative direction in each update step, we ensure that
the algorithm achieves progress in a concave direction, so the
function value is guaranteed to grow by a certain increment.  Based on
this intuition, we present the \submodularfw algorithm, which is a
generalization of the continuous greedy algorithm of
\citet{DBLP:conf/stoc/Vondrak08}, and the classical Frank-Wolfe
algorithm \citep{frank1956algorithm,DBLP:conf/icml/Jaggi13}.

\begin{algorithm}[ t]
  \caption{\submodularfw for monotone {DR}-submodular
    maximization \citep{bian2017guaranteed}}\label{alg_sfmax_GradientAscend}
	
  \KwIn{$\max_{\x\in \P} f(\x)$, $\P$ is a down-closed {convex} set in
    the positive orthant with lower bound $\zero$; prespecified
    step size $\gamma \in (0, 1]$; Error tolerances $\alpha$ and
    $\delta$. \# of iterations $K$.}
	
  {$\x^0 \leftarrow \zero$, $t\leftarrow 0$, $k\leftarrow 0$\tcp*{$k:$
      iteration index, $t$: cumulative \stepsize}} \While{$t < 1$}{
		
    {find \stepsize $\gamma_k\in (0, 1]$, e.g.,
      $\gamma_k \leftarrow \gamma $;
      set $\gamma_k \leftarrow \min\{\gamma_k, 1-t\}$\;}
		
    {find
      $\v^k \text{ s.t. }  \dtp{\v^k}{\nabla f(\x^k)} \geq
      \alpha\max_{\v\in\P} \dtp{\v}{ \nabla f(\x^k)} -
      \frac{1}{2}\delta \gamma_k L D^2$
      \tcp*{$\alpha\in(0, 1]$ is the mulplicative error level,
        $\delta\in [0, \bar \delta]$ is the additive error
        level}\label{fw_sub}}
		
    {$\x^{k+1}\leftarrow \x^k + \gamma_k \v^k$,
      $t \leftarrow t + \gamma_k$,
      $k\leftarrow k+1$\;\label{step_update}} }
	
  \KwOut{$\x^K$\;
  }
\end{algorithm}

\cref{alg_sfmax_GradientAscend} summarizes the details. Since it is a
variant of the convex Frank-Wolfe algorithm for DR-submodular
maximization, we call it \submodularfw.
In iteration $k$, the algorithm uses the linearization of $f(\cdot)$
as a surrogate, and moves in the direction of the maximizer of this
surrogate function, i.e.,
$\v^k=\arg\max_{\v \in \P} \dtp{\v}{ \nabla f(\x^k)}$.
Intuitively, it searches for the direction in which one can maximize the
improvement in the function value and still remain feasible.  Finding
such a direction requires maximizing a linear objective at each
iteration.
Meanwhile, it eliminates the need for projecting back to the feasible
set in each iteration, which is an essential step for methods such as
projected gradient ascent (\pga).
The \submodularfw algorithm updates the solution in each iteration by using \stepsize
$\gamma_k$, which can simply be set to a prespecified constant
$\gamma$.

Note that \submodularfw can tolerate both multiplicative error
$\alpha$ and additive error $\delta$ when solving the LMO subproblem
(Step \ref{fw_sub} of \cref{alg_sfmax_GradientAscend}). Setting
$\alpha = 1$ and $\delta = 0$ would recover the error-free case.

\begin{remark}
  The main difference of  \submodularfw in
  \cref{alg_sfmax_GradientAscend} and  the classical 
  Frank-Wolfe algorithm in \cref{alg_classical_fw} lies in the  update direction being used:
  For \cref{alg_sfmax_GradientAscend}, the update direction (in Step
  \ref{step_update}) is $\v^k$, while for classical Frank-Wolfe it is
  $\v^k - \x^k$, i.e.,
  $\x^{k+1}\leftarrow \x^k + \gamma_k(\v^k - \x^k)$.
\end{remark}

To prove the approximation guarantee, we first derive the following
lemma.
\begin{lemma}\label{lemma_31}
  The output solution $\x^K$ lies in $\P$. Assuming $\x^*$ to be the
  optimal solution, one has,
  \begin{flalign}\label{eq26}
    \dtp{\v^k}{\nabla f(\x^k)}\geq \alpha [f(\x^*) -f(\x^k)] -
    \frac{1}{2}\delta \gamma_k L D^2 , \ \ \forall k = 0,..., K-1.
  \end{flalign}
\end{lemma}
\begin{theorem}[Approximation guarantee]\label{thm_fw}
  For error levels $\alpha \in (0, 1], \delta\in [0, \bar \delta]$,
  with $K$ iterations, \cref{alg_sfmax_GradientAscend} outputs
  $\x^K \in \P$ such that,
  \begin{equation}\label{eq8}
    f(\x^K)   \geq  (1-e^{-\alpha})f(\x^*)
    -\frac{LD^2 (1+\delta)}{2} \sum_{k=0}^{K-1}\gamma_k^2 + e^{-\alpha}f(\zero).	
  \end{equation}
\end{theorem}
\cref{thm_fw} gives the approximation guarantee for 
any step size $\gamma_k$.  By observing that
$\sum_{k=0}^{K-1}\gamma_k =1$ and
$\sum_{k=0}^{K-1}\gamma_k^2 \geq K^{-1}$ (see the proof in
\cref{app_proof_c9}), with constant step size, we obtain the following
``tightest'' approximation bound,

\begin{corollary}\label{cor_9}
  For a fixed number of iterations $K$, and constant step size
  $\gamma_k =\gamma = K^{-1}$, \cref{alg_sfmax_GradientAscend}
  provides the following approximation guarantee:
  \begin{equation}
    f(\x^K) \geq (1-e^{-\alpha})f(\x^*) -\frac{LD^2 (1+\delta)}{2K}+
    e^{-\alpha}f(\zero).
  \end{equation}
\end{corollary}

Corollary \ref{cor_9} implies that with a constant step size $\gamma$,
1) when $\gamma \rightarrow 0$ ($K\rightarrow \infty$),
\cref{alg_sfmax_GradientAscend} will output the solution with the
worst-case guarantee $(1-1/e)f(\x^*)$ in the error-free case if
$f(\zero) = 0$; and 2) The \submodularfw has a sub-linear convergence
rate for monotone {DR}-submodular maximization over a down-closed
convex constraint.

\paragraph{Remarks on computational complexity.}  It can be seen that
when using a constant step size, \cref{alg_sfmax_GradientAscend} needs
$O(\frac{1}{\epsilon})$ iterations to get $\epsilon$-close to the
best-possible function value $(1-e^{-1})f(\x^*)$ in the error-free
case.  When $\P$ is a polytope in the positive orthant, one iteration
of \cref{alg_sfmax_GradientAscend} costs approximately the same as
solving a positive LP, for which a nearly-linear time solver exists
\citep{allen2015nearly}.

\section{Experiments}

\setkeys{Gin}{width=0.7\textwidth}
\begin{figure}[htbp]
  \center \subfloat[\submodularfw utility,
  $K=50$ \label{fig_nqp_iter}]{%
    \hspace{-.3cm}
    \includegraphics[]{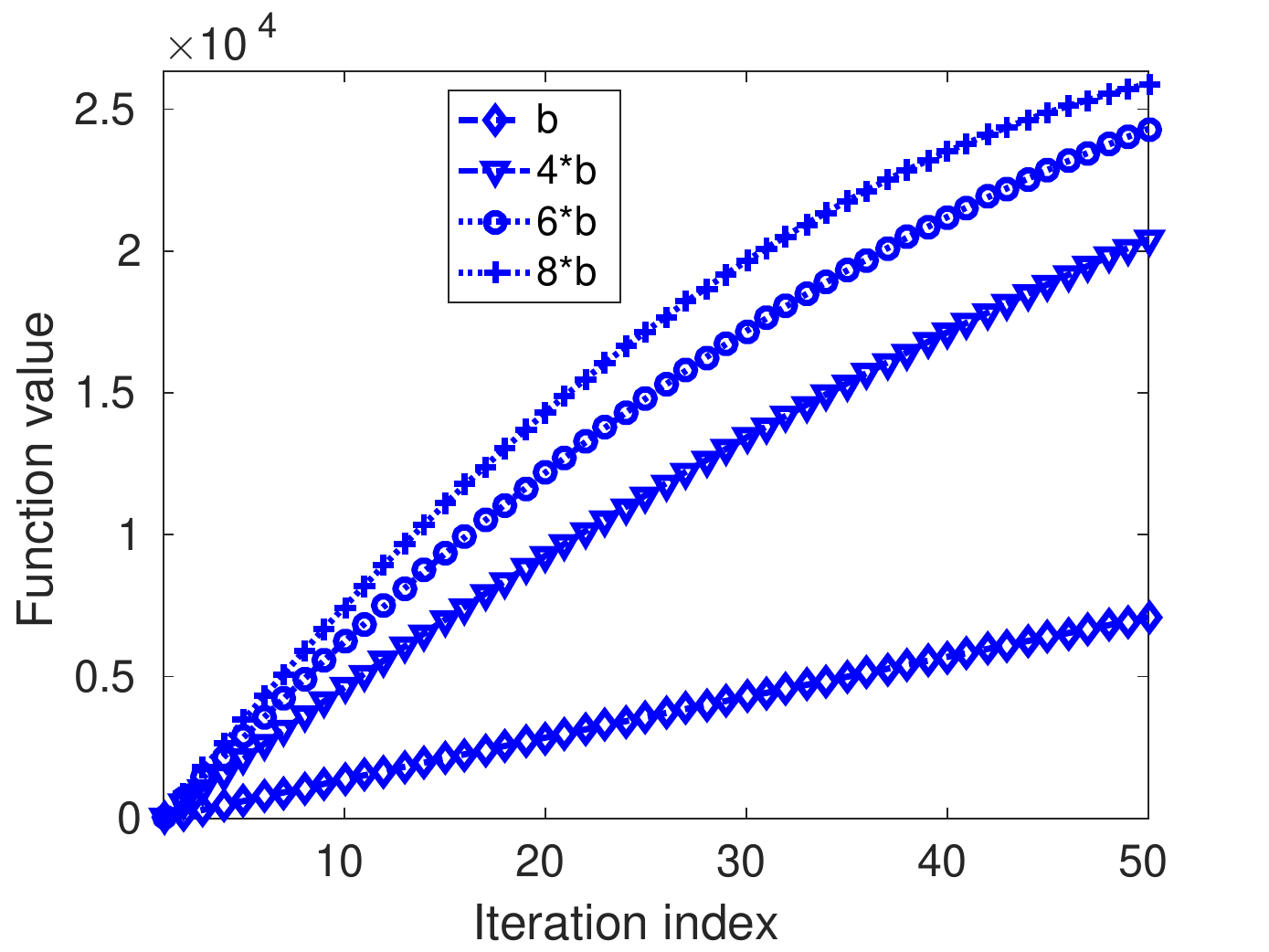}
  }

  \subfloat[Returned function value for monotone submodular
  QP instances \label{fig_nqp}]{%
    \hspace{-.3cm}
    \includegraphics[]{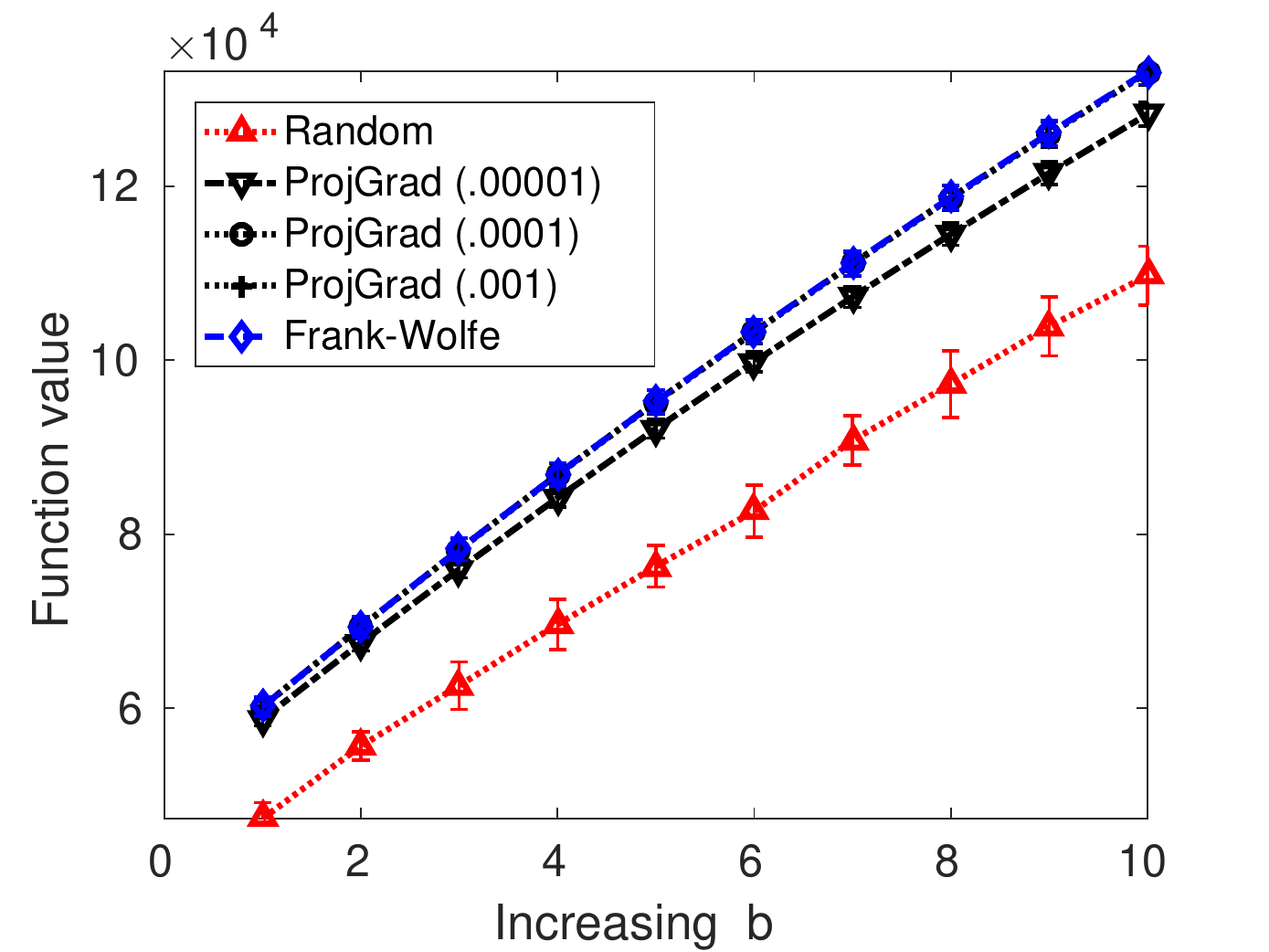}
  }
  \caption{Monotone SQPs (both \submodularfw and \pga
    (\algname{ProjGrad}) were ran for 50 iterations). \algname{Random}
    algorithm: return a randomly sampled point in the constraint. a)
    \submodularfw function value for four instances with different
    $\b$; b) QP function value returned w.r.t. different $\b$. }
  \label{fig_mono_sqp}
\end{figure}

\subsection{Monotone DR-Submodular QP}  

We have randomly generated monotone DR-submodular QP functions of the
form $f(\x) = \frac{1}{2} \x^{\trans} \bmH \x + \bh^{\trans} \x $,
where $\bmH\in \R^{n \times n}$ is a random matrix with
\textit{uniformly} distributed non-positive entries in $[-100,0]$,
$n=100$. We further have generated a set of $m = 50$ linear
constraints 
to construct the positive polytope
$\P = \{\x\in \R^n | \bmA\x\leqco \b, \zero\leqco \x\leqco \bar \bu\}$,
where $\bmA$ has uniformly distributed entries in $[0,1]$,
$\b= \mathbf{1}, \bar \bu = \mathbf{1}$.  To make the gradient
non-negative, we set $\bh = -\bmH^{\trans} \bar \bu$.
We have empirically tuned the constant  \stepsize  for \pga and ran all
algorithms for $50$ iterations.

\cref{fig_nqp_iter} shows the utility
obtained by \submodularfw v.s. the 
iteration index for four function instances with different values of $\b$.
\cref{fig_nqp} depicts the average utility obtained by different
algorithms with increasing values of $\b$. The result is the average
of 20 repeated experiments. For \pga, we plot the curves for three
different values of \stepsizes. %
One can observe that the performance of \pga fluctuates with different
\stepsizes. With the best-tuned \stepsize, \pga performs close to
\submodularfw.

\subsection{Influence Maximization with Marketing Strategies}

Follow the application in \cref{app_influence_max_marketing_strategies}, 
we consider the following simplified influence  model for experiments.

\paragraph{Simplified Influence Model for Experiments.}

For the general influence models, it is hard to evaluate
\cref{influence_general_marketing}.  To ease the experiments, we
consider $F(S)$ to be a facility location objective, for which the
expected influence has a closed-form expression, as shown by \citet[Section
4.2]{bian2019optimalmeanfield}.
Here each user may represent an ``opinion leader'' in social networks,
and there is a bipartite graph describing the influence strength of
each opinion leader to the population.

\subsubsection{Experimental Results}

We used the UC Irvine forum
dataset\footnote{\url{http://konect.uni-koblenz.de/networks/opsahl-ucforum}}.
It is a bipartite network containing user posts to forums. The users
are students at the University of California, Irvine. An edge
represents a forum message on a specific forum.  It has in total 899
users, 522 forums and 33,720 edges (posts).

For a specific (user, forum) pair, we determine the edge weight as the
number of posts from that user on the forum. This weighting indicates
that the more one user has posted on a forum, the more he has
influenced that particular forum. With this processing, we have 
7,089 unique edges between users and forums.

We experimented with the independent marketing actions in \cref{app_activations_influence_max} for
simplicity. For a user $i$, we set the parameter $p_i \in [0, 1]$
based on the following heuristic: Firstly, we calculate the
``degree'' of user $i$ as the number of forums he has posted on:
$d_i = \|W_{i:}\|_0$. Then we set $p_i = \sigma(- d_i)$, $\sigma(\cdot)$ is
the logistic sigmoid function.  Remember that $p_i$ is the probability
of user $i$ become activated with one unit of investment, so this
heuristic means that the more influence power a user has, the more
difficult it is to activate him, because he might charge more than other
users with less influence power.
Since it is too time consuming to experiment on the whole bipartite
graph, we experimented on different subgraphs of the original
bipartite graph.

\setkeys{Gin}{width=0.8\textwidth}
\begin{figure}[htbp]
  \center \subfloat [50 users, 10 forums] {
    \includegraphics[]{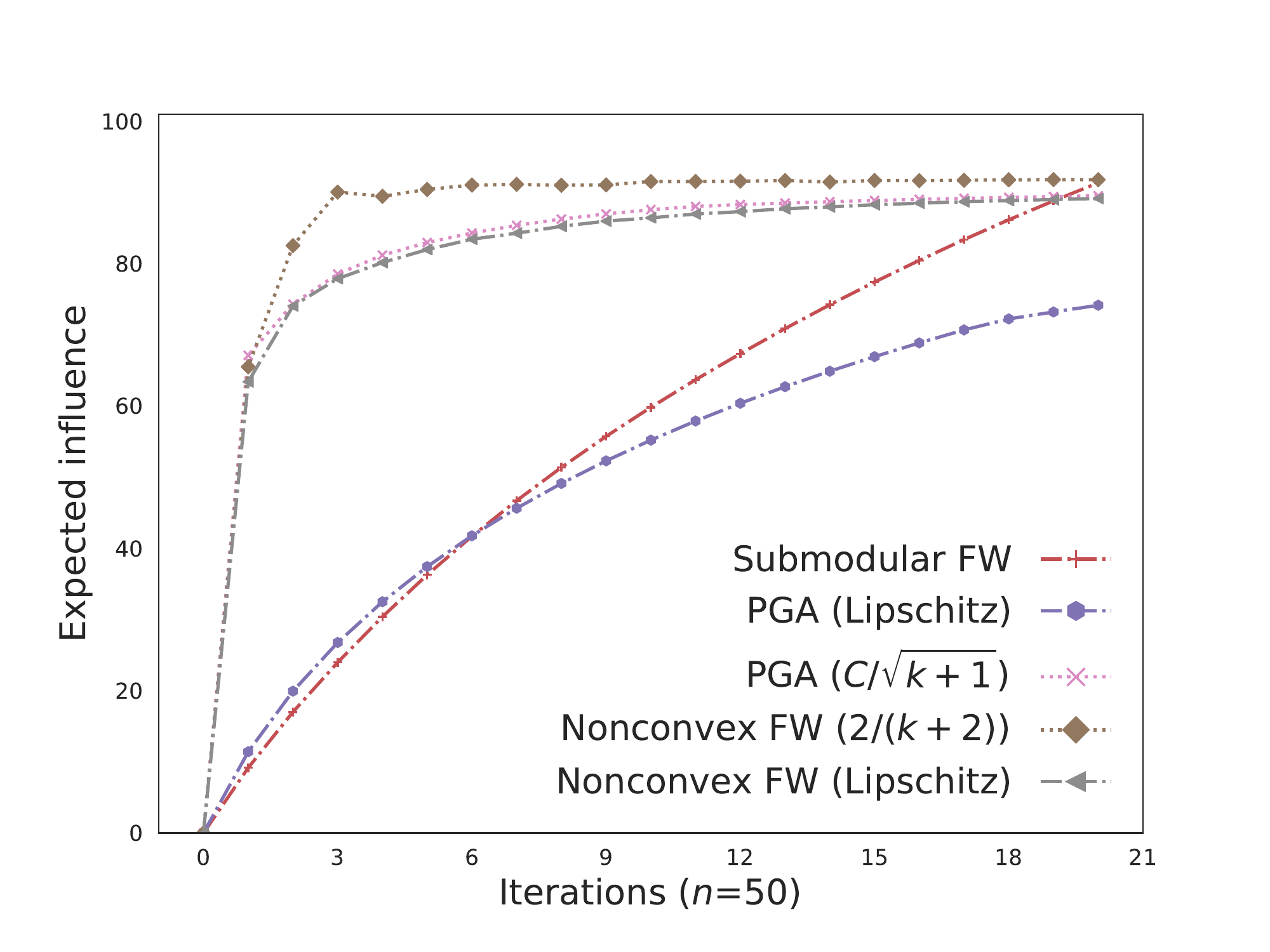}}
  \vspace{-0.8cm} \subfloat [100 users, 10 forums] {
    \includegraphics[]{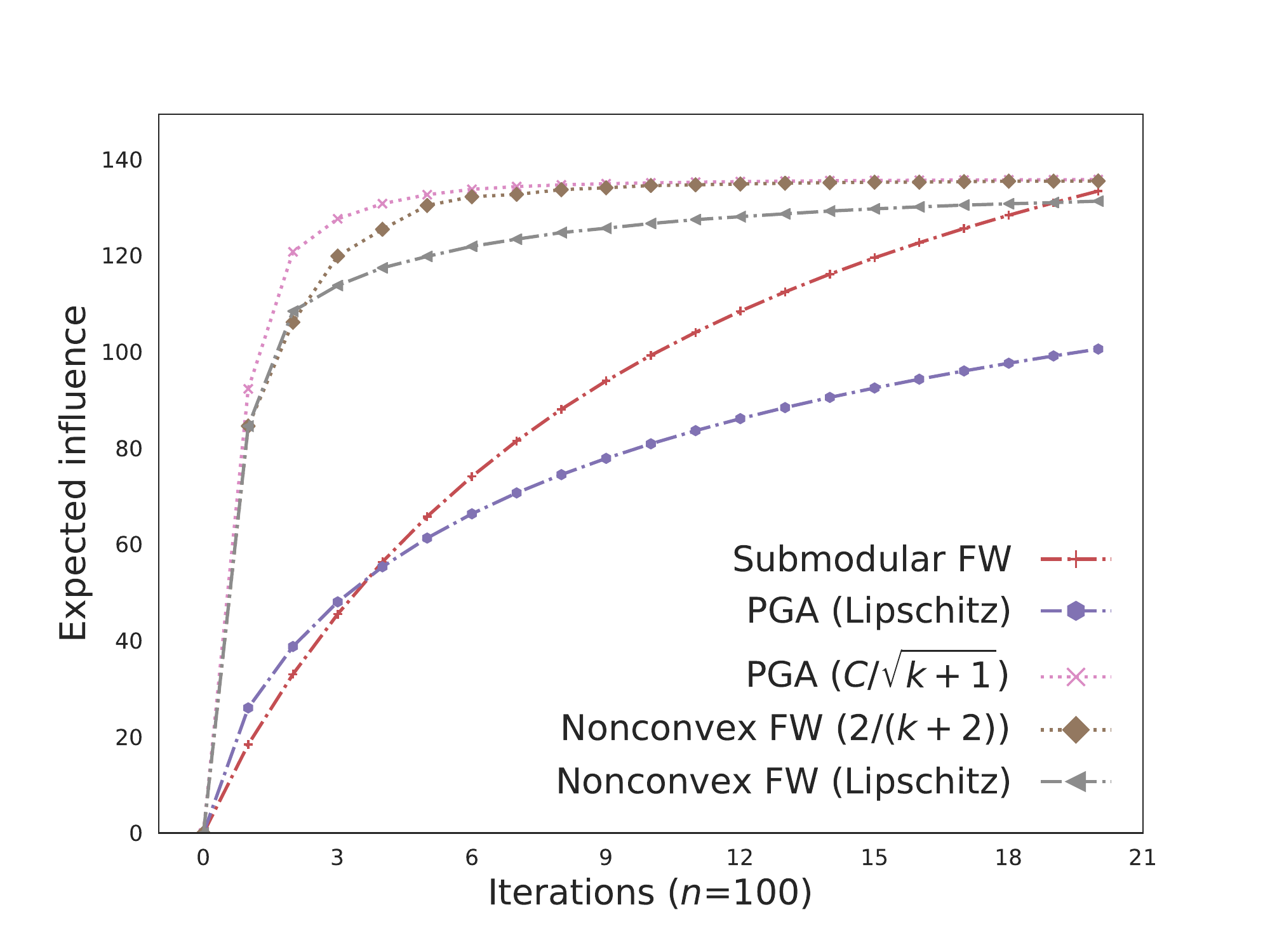}}
  \caption{Expected influence w.r.t. iterations of different
    algorithms on real-world graphs with 50 and 100 users.}
  \label{fig_traj_influence_50_100}
\end{figure}  

\setkeys{Gin}{width=0.8\textwidth}
\begin{figure}[htbp]
  \center \subfloat [150 users, 20 forums] {
    \includegraphics[]{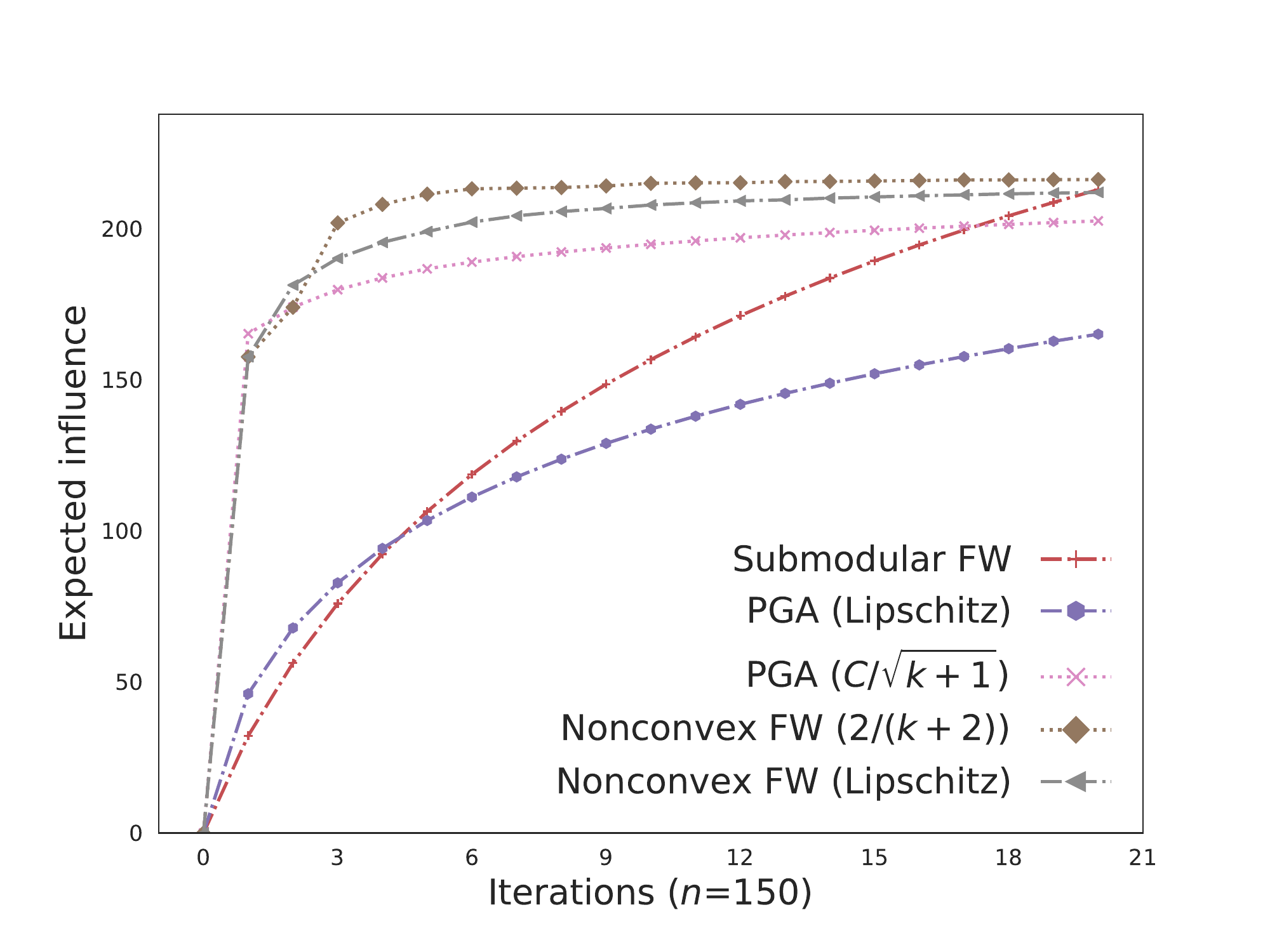}}
  \vspace{-0.8cm} \subfloat [200 users, 20 forums] {
    \includegraphics[]{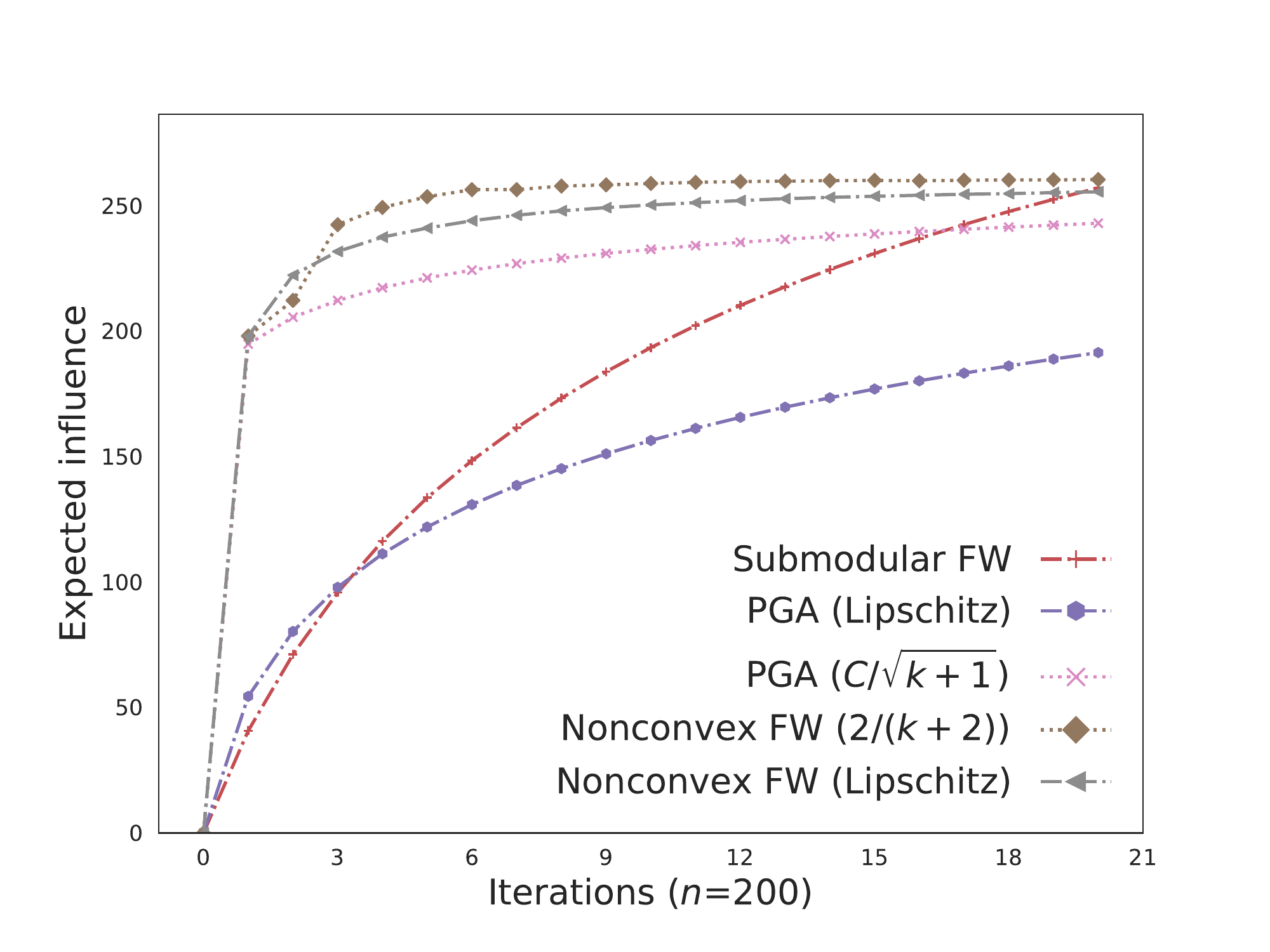}}
  \caption{Expected influence w.r.t. iterations of different
    algorithms on real-world graphs with 150 and 200 users.}
  \label{fig_traj_influence_150_200}
\end{figure}  

\cref{fig_traj_influence_50_100,fig_traj_influence_150_200} document
the trajectories of expected influence of different algorithms.  We
can see that \submodularfw has a very stable performance: It can
always reach a fairly good solution, no matter what kind of setting
you have. And it does not need to tune the \stepsizes or any
hyperparameters. One drawback is that it converges relatively slowly in the
beginning.

For \pga algorithms, we tested with two \stepsize rules: the Lipschitz
rule ($1/L$) which has the 1/2 approximation guarantee; the diminishing
\stepsize rule ($C/\sqrt{k+1}$), which does not have a formal
theoretical guarantee.
One general observation is that both \stepsize rules need a careful
tunning of hyperparameters, and the performance crucially depends on
the quality of hyperparameters. For example, for \pga, if the
\stepsize is too small, it may converge too slowly; if the \stepsizes
are too large, it tends to fluctuate.

For \nonconvexfw algorithms, we also tested two \stepsize rules: the
``oblivious'' rule ($2/(k+2)$)) and the Lipschitz rule. Apparently the
Lipschitz \stepsize rule needs a careful tunning of the Lipschitz
parameter $L$, while the oblivious rule does not.  With a careful
tuning of $L$, both \nonconvexfw variants converge very fast and
converge to the highest function value.

\section{Conclusions}

In this chapter we studied the problem of maximizing monotone
continuous DR-submodular functions. We started with the
inapproximability results of this problem. Then we presented two
classes of algorithms with constant factor approximation
guarantees. The first class of algorithms are based on the
Local-Global relation and have a 1/2 approximation ratio. The second
class of algorithms contain a Frank-Wolfe variant, termed
\submodularfw. It works by following the concave direction in each
iteration.  Finally, we demonstrated the effectiveness of the
algorithms through experiments on both synthetic and real-world data.

\section{Additional Proofs}
\label{sec_proofs_monotonemax}

\subsection{Proof of \cref{prop_np}}

\begin{proof}[Proof of \cref{prop_np}]

  On a high level, the proof idea follows from the reduction from the
  problem of maximizing a monotone submodular set function subject to
  cardinality constraints.
	
  Let us denote $\Pi_1$ as the problem of maximizing a monotone
  submodular set function subject to cardinality constraints, and
  $\Pi_2$ as the problem of maximizing a monotone continuous
  DR-submodular function under general down-closed polytope
  constraints.
  Following \citet{DBLP:journals/siamcomp/CalinescuCPV11}, there exist
  an algorithm $\A$ for $\Pi_1$ that consists of a polynomial time
  computation in addition to polynomial number of subroutine calls to
  an algorithm for $\Pi_2$. For details on $\A$ see the following.

  First of all, the multilinear extension
  \citep{calinescu2007maximizing} of a monotone submodular set
  function is a monotone continuous submodular function, and it is
  coordinate-wise linear, thus falls into a special case of monotone
  continuous DR-submodular functions. Evaluating the multilinear extension and its gradients can be done using sampling methods, thus resulted in a randomized algorithm. 
	
  So the algorithm $\A$ shall be: 1) Maximize the multilinear
  extension of the submodular set function over the matroid polytope
  associated with the cardinality constraint, which can be achieved by
  solving an instance of $\Pi_2$.  We call the solution obtained the
  fractional solution; 2) Round the fractional solution to a feasible
  integeral solution using polynomial time rounding technique in
  \citet{ageev2004pipage,calinescu2007maximizing} (called the pipage
  rounding). Thus we prove the reduction from $\Pi_1$ to $\Pi_2$.
	
  Our reduction algorithm $\A$ implies the NP-hardness and
  inapproximability of problem $\Pi_2$.
	
  For the NP-hardness, because $\Pi_1$ is well-known to be NP-hard
  \citep{calinescu2007maximizing,feige1998threshold}, so $\Pi_2$ is
  NP-hard as well.
	
  For the inapproximability: Assume there exists a polynomial
  algorithm ${\mathscr B}$ that can solve $\Pi_2$ better than $1-1/e$,
  then we can use ${\mathscr B}$ as the subroutine algorithm in the
  reduction, which implies that one can solve $\Pi_1$ better than
  $1 - 1/e$. Now we slightly adapt the proof of inapproximability on
  max-k-cover of \citet{feige1998threshold}, since max-k-cover is a
  special case of $\Pi_1$.  According to the proof of  Theorem 5.3 in
  \citet{feige1998threshold} and our reduction $\A$, we have a
  reduction from approximating 3SAT--5 to problem $\Pi_2$. Using the
  rest proof of Theorem 5.3 in \citet{feige1998threshold}, we reach
  the result that one cannot solve $\Pi_2$ better than $1 - 1/e$,
  unless RP = NP.
\end{proof}

\subsection{Proof of \cref{coro_nonconvex_fw}}

\begin{proof}[Proof of \cref{coro_nonconvex_fw}]
	Firstly, according to Theorem  1 of \citet{lacoste2016convergence}, 
	\nonconvexfw is known to converge to a  stationary point with a rate of 
	$1/\sqrt{k}$. 
	
	Then according to \cref{coro_1half}, any stationary point is a
	1/2 approximate solution.
\end{proof}

\subsection{Proof of \cref{lemma_31}}
\begin{proof}
  It is easy to see that $\x^K$ is a convex combination of points in
  $\P$, so $\x^K\in\P$.
	
  Consider the point
  $\v^*:=(\x^*\vee \x) - \x = (\x^* - \x)\vee \zero\geqco \zero$.  Because
  $\v^*\leqco \x^*$ and $\P$ is down-closed, we get $\v^*\in \P$.
	
  By monotonicity, $f(\x+\v^*) = f(\x^*\vee \x) \geq f(\x^*)$.
	
  Consider the function $g(\xi):= f(\x+\xi \v^*), \xi\geq 0$.
  $\frac{d g(\xi)}{d \xi} = \dtp{\v^*}{\nabla f(\x+\xi \v^*)}$.  From
  Proposition \ref{prop_concave}, $g(\xi)$ is concave, hence
  \begin{flalign}
    g(1) - g(0) = f(\x+\v^*) - f(\x) \leq \frac{d g(\xi)}{d \xi}
    \Bigr|_{\xi = 0} \times 1 = \dtp{\v^*}{ \nabla f(\x)}.
  \end{flalign}
  Then one can get
  \begin{flalign}
    &\dtp{\v}{\nabla f(\x)} \overset{(a)}{\geq} \alpha \dtp{\v^*}{
      \nabla f(\x)} -\frac{1}{2}\delta \gamma L D^2 \geq \\
    &\alpha (f(\x+\v^*) - f(\x)) -\frac{1}{2}\delta \gamma L D^2 \geq
    \alpha (f(\x^*) -f(\x)) -\frac{1}{2}\delta \gamma L D^2,
  \end{flalign}
  where $(a)$ is resulted from the LMO step of \cref{alg_sfmax_GradientAscend}.
\end{proof}

\subsection{Proof of \cref{thm_fw}}
\begin{proof}[Proof of \cref{thm_fw}]
  From the Lipschitz assumption of $f$ (\cref{eq_smooth}):
  \begin{flalign}
    f(\x^{k+1}) - f(\x^k) & = f(\x^k + \gamma_k \v^k) - f(\x^k)
    \\\notag
    &\geq \gamma_k \dtp{\v^k}{\nabla f(\x^k)} -
    \frac{\cg}{2}\gamma_k^2 \|\v^k\|^2 \quad (\text{Lipschitz
      smoothness}) \\\notag &\geq \gamma_k \alpha [f(\x^*) - f(\x^k)]
    - \frac{1}{2}\gamma_k^2 \delta LD^2 - \frac{\cg}{2}\gamma_k^2 D^2.
    \quad (\text{Lemma \ref{lemma_31}})
  \end{flalign}
  After rearrangement,
  \begin{flalign}
    f(\x^{k+1}) - f(\x^*) \geq (1-\alpha\gamma_k) [f(\x^k) - f(\x^*)]-
    \frac{LD^2\gamma_k^2 (1+\delta)}{2}.
  \end{flalign}
  Therefore,
  \begin{flalign}
    f(\x^K) - f(\x^*) \geq \prod_{k=0}^{K-1}
    (1-\alpha\gamma_k)[f(\zero) - f(\x^*)] - \frac{LD^2 (1+\delta)}{2}
    \sum_{k=0}^{K-1}\gamma_k^2 .
  \end{flalign}
  One can observe that $\sum_{k=0}^{K-1}\gamma_k = 1$, and since
  $1-y \leq e^{-y}$ when $y\geq 0$,
  \begin{flalign}
    f(\x^*) - f(\x^K) &\leq [f(\x^*) - f(\zero)]e^{-\alpha
      \sum_{k=0}^{K-1}\gamma_k} + \frac{LD^2 (1+\delta)}{2}
    \sum_{k=0}^{K-1}\gamma_k^2 \\ 
    & = [f(\x^*) -
    f(\zero)]e^{-\alpha} +\frac{LD^2 (1+\delta)}{2}
    \sum_{k=0}^{K-1}\gamma_k^2.
  \end{flalign}
  After rearrangement, we get,
  \begin{align}
    f(\x^K) \geq (1-1/e^{\alpha})f(\x^*) -\frac{LD^2 (1+\delta)}{2}
  \sum_{k=0}^{K-1}\gamma_k^2 + e^{-\alpha}f(\zero).
  \end{align}
\end{proof}

\subsection{Proof of \cref{cor_9}}\label{app_proof_c9}

\begin{proof}[Proof of \cref{cor_9}]
  Fixing $K$, to reach the tightest bound in \cref{eq8} amounts to
  solving the following problem:
  \begin{flalign}
    &\min \sum_{k=0}^{K-1}\gamma_k^2\\\notag &\text{ s.t. }
    \sum_{k=0}^{K-1}\gamma_k = 1, \gamma_k \geq 0.
  \end{flalign}
  Using Lagrangian method, let $\lambda$ be the Lagrangian multiplier,
  then
  \begin{align}
    L(\gamma_0, \cdots, \gamma_{K-1}, \lambda) =
  \sum_{k=0}^{K-1}\gamma_k^2 + \lambda \left[\sum_{k=0}^{K-1}\gamma_k
  - 1\right].
  \end{align}
  It can be easily verified that when
  $\gamma_0 = \cdots =\gamma_{K-1} = K^{-1}$,
  $\sum_{k=0}^{K-1}\gamma_k^2$ reaches the minimum (which is
  $K^{-1}$). Therefore we obtain the tightest worst-case bound in
  Corollary \ref{cor_9}.
\end{proof}


\def\dir{chapters/box-non-mono-sub-max}
\chapter{Maximizing Non-Monotone Continuous Submodular Functions with
  a Box Constraint}
\chaptermark{Box Constrained Continuous Submodular Maximization}
\label{chapter_doublegreedy}

\begin{chapquote}{Master Oogway}
	Yesterday is history, tomorrow is a mystery, and today is a gift. That’s why they call it the present.
\end{chapquote}

In this chapter we focus on maximizing continuous submodular
functions, with only one hypercube constraint (also called ``box
constraint''),
\begin{align}\label{problem_box}
  \underset{{ \x \in [\a, \;\b]}}{{\text {maximize}}} \;\; f(\x),   
\end{align}
where $f: \X \rightarrow \R$ is continuous submodular or
DR-submodular.

The problem of maximizing a general non-monotone continuous submodular
function under box constraints\footnote{It is also called
  ``unconstrained'' maximization in the combinatorial optimization
  community, since the domain $\X$ itself is also a box. Note that the
  box can be in the negative orthant here.} has various real-world
applications, including revenue maximization with continuous
assignments, multi-resolution summarization, mean-field inference for
probabilistic log-submodular models and its PA (see
\cref{chapter_mean_field} for details), etc.

\section{Hardness and Inapproximability Results}

Though only with a box constraint, problem \labelcref{problem_box} is
still extremely hard to solve to optimal.  The following proposition
shows the NP-hardness of the problem.

\begin{proposition}[Hardness and Inapproximability]\label{prop_np2}
  The problem of maximizing a generally non-monotone continuous
  submodular function subject to a box constraint is NP-hard.
  Furthermore, there does not exist a polynomial-time $(1/2 +\epsilon)$-approximation
  $\forall \epsilon>0$, unless RP = NP.
\end{proposition}

\def\sectitle{Submodular-DoubleGreedy: A 1/3 Approximation Algorithm
  for Submodular Maximization}
\section[Submodular-DoubleGreedy: A 1/3 Approximation]{\sectitle}

We now describe our algorithm for maximizing a non-monotone continuous
submodular function subject to box constraints.  It provides a
1/3-approximation, is inspired by the double greedy algorithm of
\citet{buchbinder2012tight}, and
can be viewed as a procedure performing coordinate-ascent on
\textit{two} solutions.
Since it only uses the submodularity property (instead of the DR
property), we call it \submodulardg.

What it given to the algorithm is the problem $\max_{\x\in [\underline{\bu}, \bar \bu]} f(\x)$, $f$ is
continuous submodular, and the requirement that  $f(\underline{\bu}) + f(\bar \bu)\geq 0$.
We view the process as two particles starting from
$\x^0=\underline{\bu}$ and $\y^0 = \bar{\bu}$, and following a certain
``flow'' toward each other.  The pseudo-code is given in
\cref{alg_uscfmax_DoubleGreedy}.
\begin{algorithm}[htbp]
  \caption{\submodulardg algorithm for maximizing non-monotone
    continuous submodular functions \citep{bian2017guaranteed}}\label{alg_uscfmax_DoubleGreedy}
  \KwIn{$\max_{\x\in [\underline{\bu}, \bar \bu]} f(\x)$, $f$ is
    continuous submodular, $f(\underline{\bu}) + f(\bar \bu)\geq 0$}
  {$\x^0 \leftarrow \underline{\bu}$, $\y^0 \leftarrow \bar \bu$\;}
  \For{$k = 1 \rightarrow n$}{ {let $e_k$ be the coordinate being
      operated\;}
		
    {find $\hat u_a$ \text{ s.t. }
      $f(\sete{x^{k-1}}{e_k}{\hat u_a}) \geq
      \max_{u_a\in[\underline{u}_{e_k}, \bar u_{e_k}]}
      f(\sete{x^{k-1}}{e_k}{ u_a}) - \delta$,
      $\delta_a \leftarrow f (\sete{x^{k-1}}{e_k}{\hat u_a}) -
      f(\x^{k-1}$)\tcp*{$\delta\in
        [0, \bar \delta]$
        is the additive error level }\label{step_1d_1}} {find
      $\hat u_b$ \text{ s.t. }
      $f(\sete{y^{k-1}}{e_k}{\hat u_b})\geq \max_{u_b\in
        [\underline{u}_{e_k}, \bar u_{e_k}]}
      f(\sete{y^{k-1}}{e_k}{u_b}) - \delta$,
      $\delta_b \leftarrow f (\sete{y^{k-1}}{e_k}{\hat u_b}) -
      f(\y^{k-1})$\;\label{step_1d_2}}
    {\textbf{If} $\delta_a\geq \delta_b$:
      $\x^{k}\leftarrow (\sete{x^{k-1}}{e_k}{\hat u_a})$,
      $\y^{k}\leftarrow (\sete{y^{k-1}}{e_k}{\hat u_a})$ \;}
    {\textbf{Else}: \quad\quad
      $\y^{k}\leftarrow (\sete{y^{k-1}}{e_k}{\hat u_b})$,
      $\x^{k}\leftarrow(\sete{x^{k-1}}{e_k}{\hat u_b})$\;} } {Return
    $\x^n$ (or $\y^n$)\tcp*{note that $\x^n = \y^n$}}
\end{algorithm}
We proceed in $n$ rounds that correspond to some arbitrary order of
the coordinates.
At iteration $k$, we consider solving a one-dimensional (1-D) subproblem
over coordinate $e_k$ for each particle, and moving the particles
based on the calculated local gains toward each other.  Formally, for
a given coordinate $e_k$, we solve a 1-D subproblem to
find the value of the first solution $\x$ along
coordinate $e_k$ that maximizes $f$, i.e.,
$\hat u_a = \arg\max_{u_a} f(\sete{x^{k-1}}{e_k}{u_a}) - f(\x^{k-1})$,
and calculate its marginal gain $\delta_a$.  We then solve
another 1-D subproblem to find the value of the second 
solution  $\y$ along coordinate $e_k$ that maximizes $f$, i.e.,
$\hat u_b = \arg\max_{u_b} f(\sete{\y^{k-1}}{e_k}{u_b}) -
f(\y^{k-1})$,
and calculate the second marginal gain $\delta_b$.  
Then we update the solutions by comparing the two marginal gains.
After comparing the two marginal gains, we select the superior
solution.  If changing $x_{e_k}$ to be $\hat u_a$ has a larger
benefit, we change \textit{both} $x_{e_k}$ and $y_{e_k}$ to be
$\hat u_a$. Otherwise, we change \textit{both} of them to be
$\hat u_b$.  After $n$ iterations the particles should meet at point
$\x^n = \y^n$, which is the final solution. Note that
\cref{alg_uscfmax_DoubleGreedy} can tolerate additive error $\delta$
in solving each 1-D subproblem (Steps \ref{step_1d_1},
\ref{step_1d_2}).

It is worth mentioning, that the assumptions required by the algorithm
\submodulardg
are submodularity of $f$, $f(\underline{\bu}) + f(\bar \bu)\geq 0$ and
the (approximate) solvability of the 1-D subproblem.  For proving the
approximation guarantee, the idea is to bound the loss in the
objective value from the assumed optimal objective value between every
two consecutive steps,
which is then used to bound the
maximum  loss after $n$ iterations.

We can show that \submodulardg has a 1/3 approximation guarantee:
\begin{theorem}\label{thm_double}
  Assuming the optimal solution to be $\optcont$, the output of
  \cref{alg_uscfmax_DoubleGreedy} has function value no less than
  $ \frac{1}{3} f(\optcont) - \frac{4n}{3}\delta$, where
  $\delta\in [0, \bar \delta]$ is the additive error level for solving
  each 1-D subproblem.
\end{theorem}

\paragraph{Remark on Time Complexity.}
It can be seen that the time complexity of
\cref{alg_uscfmax_DoubleGreedy} is $O(n*\texttt{cost\_1D})$, where
$\texttt{cost\_1D}$ is the cost of solving the 1-D subproblem.
Solving a 1-D subproblem is usually computationally inexpensive.

\section[DR-DoubleGreedy: An Optimal 1/2
Approximation]{DR-DoubleGreedy: An Optimal 1/2 Approximation for
  DR-Submodular Maximization}
\label{sec_dr_doublegreedy}

Unfortunately, problem \labelcref{problem_box} is generally hard
though $f$ is DR-submodular: The $1/2$ hardness result
in \cref{prop_np2} can be easily translated
to problem \labelcref{problem_box} when $f$ satisfies the DR property:

\begin{observation}\label{obs_dr_submodular_max}
	The problem of maximizing a generally non-monotone continuous DR-submodular 
	function subject to box-constraints is NP-hard. Furthermore, there is no polynomial-time $(1/2+\epsilon)$-approximation for any $\epsilon>0$, unless RP = NP.
\end{observation}

The following question arises
naturally: Is it possible to achieve the optimal $1/2$ approximation
ratio (unless RP=NP)
by properly utilizing the extra DR property?
To affirmatively answer this question, we propose a new Double Greedy
algorithm for continuous DR-submodular maximization called
\drdg (Since it explicitly utilizes the DR property) and prove a $1/2$ approximation
ratio.

\subsection{The Algorithm and Its Guarantee}

\begin{algorithm}[ht]
	\caption{\algname{DR-DoubleGreedy}$(f, \a, \b)$ for continuous DR-submodular maximization with a box constraint \citep{bian2019optimalmeanfield}
	}\label{alg_cont_doublegreedy} 
	\KwIn{ $\max_{\x \in [\a, \b]}f(\x)$,  $f(\x)$ is
		{\color{blue}DR}-submodular,  $[\a,\b]\subseteq  \X$ 
	}
	{$\x^0 \leftarrow \a$,
		$\y^0 \leftarrow \b$\;}
	\For{$k = 1 \rightarrow n$}{
		{ let $v_k$  be the coordinate being operated\;}
		{find
			$ u_a$ \text{ such that }
			$f(\sete{x^{k-1}}{\ele_k}{u_a}) \geq \max_{u'}
			f(\sete{x^{k-1}}{\ele_k}{ u'}) -
			\frac{\delta}{n}$, 
			
			$\delta_a \leftarrow f (\sete{x^{k-1}}{\ele_k}{u_a}) -
			f(\x^{k-1}$)\label{1d_1} \;}

		{find $u_b$ \text{ such that }
			$f(\sete{y^{k-1}}{\ele_k}{u_b})\geq \max_{u'}
			f(\sete{y^{k-1}}{\ele_k}{u'}) -
			\frac{\delta}{n}$, 
			
			$\delta_b \leftarrow f (\sete{y^{k-1}}{\ele_k}{u_b}) -
			f(\y^{k-1})$\label{1d_2} \;}

		{$\x^k \leftarrow \sete{x^{k-1}}{\ele_k}{(
				\frac{\delta_a}{\delta_a + \delta_b}   u_a +
				\frac{\delta_b}{\delta_a + \delta_b}
				u_b)}$\tcp*{update $\ele_k^\text{th}$ coordinate  to be a
				\emph{\color{blue}convex} combination of $u_a$ \& 
				$u_b$}} 
		
		{$\y^k \leftarrow \sete{y^{k-1}}{\ele_k}{(
				\frac{\delta_a}{\delta_a + \delta_b}   u_a +
				\frac{\delta_b}{\delta_a + \delta_b}   u_b)}$\;} 
		
	}
	\KwOut{$\x^n$ or $\y^n \; (\x^n = \y^n)$ }
\end{algorithm}

The pseudocode of \algname{DR-DoubleGreedy} is summarized in
\cref{alg_cont_doublegreedy}. It describes a one-epoch algorithm, sweeping
over the $n$ coordinates in one pass.
Like the previous Double Greedy algorithms, the procedure maintains
two solutions $\x, \y$, that are initialized as the lower bound $\a$
and the  upper bound $\b$, respectively.  In iteration $k$, it operates on
coordinate $\ele_k$, and solves the two 1-D subproblems
$\max_{u'} f(\sete{x^{k-1}}{\ele_k}{ u'})$ and
$\max_{u'} f(\sete{y^{k-1}}{\ele_k}{u'})$, based on $\x^\pare{k-1}$
and $ \y^\pare{k-1}$, respectively. It also allows   solving 1-D subproblems approximately with 
additive error $\delta \geq 0$ ($\delta =0$ recovers the error-free case).
Let $u_a$ and $u_b$ be the solutions of these 1-D subproblems.

Unlike previous Double Greedy algorithms, we change coordinate
$\ele_k$ of $\x^\pare{k-1}$ and $ \y^\pare{k-1}$ to be a \emph{convex}
combination of $u_a$ and $u_b$, weighted by respective gains
$\delta_a$, $\delta_b$. This convex combination is the key step that
utilizes the DR property of $f$, and it also plays a crucial role in
the proof.

Note that the 1-D subproblem has a closed-form solution for
many specific problem instances. For example, for 
ELBO in \cref{opt_problem_meanfield} (and similarly for
PA-ELBO in  \cref{pa_elbo}).  For coordinate $i$, the partial
derivative of the multilinear extension is $\nabla_i\multi(\x)$, and for
the entropy term, it is $\nabla H(x_i) = \log \frac{1-x_i}{x_i}$. Then $x_i$
should be updated as
$x_i \leftarrow \sigma(\nabla_i \multi(\x)) = \bigl(1+ \exp(-
\nabla_i \multi(\x)\bigr)^{-1}$,
where $\sigma$ is the logistic sigmoid function.
\begin{restatable}[]{theorem}{restatheoremone}
	\label{thm_cont_doublegreedy}
	Assume the optimal solution of  $\max_{\x \in [\a, \b]}f(\x)$  is $\optcont$, then for
	\cref{alg_cont_doublegreedy}    it holds,
	\begin{flalign}
	f(\x^n) \geq \frac{1}{2}f(\optcont) + \frac{1}{4}[f(\a) + f(\b)] - \frac{5\delta}{4}.
	\end{flalign}
\end{restatable}
\paragraph{Proof Sketch of \cref{thm_cont_doublegreedy}.}
The high level 
proof strategy is to bound the change of an intermediate variable
$\intermed^\pare{k} := (\optcont \vee \x^k) \wedge \y^k$ through the
course of \cref{alg_cont_doublegreedy}, which is the common framework
in the analysis of all existing Double Greedy variants
\citep{buchbinder2012tight,gottschalk2015submodular,bian2017guaranteed,soma2017non}
\footnote{Note
that \citet{buchbinder2012tight} analyzed in the  appendix a Double
Greedy variant (Alg. 4 therein) for maximizing the multilinear
extension of a submodular \textit{set} function, which is a special
case of continuous DR-submodular functions.
However, that variant 
cannot be applied for the general DR-submodular 
objective in \labelcref{problem_box}; Furthermore,  the analysis 
for that variant  is not applicable nor generalizable   for  \labelcref{problem_box}, since it only shows the guarantee wrt. the optimal solution 
that must be {binary}. While 
the optimal solution to \labelcref{problem_box} could be any fractional point
in $[\a,\b]$.}. 
The novelty of our method results from the update of $\x$, $\y$, which
plays a key role in achieving the optimal $1/2$ approximation ratio.
Furthermore, in the analysis we find a way to utilize the DR property
directly, resulting in a succinct proof.

\subsection{Comparision with 
Algorithm of \citet{niazadeh2018optimal}}

Along with the development of our work\footnote{Our work was released
  earlier than \citet{niazadeh2018optimal}.},
\citet{niazadeh2018optimal} proposed an optimal algorithm for
DR-submodular maximization, which is based on a zero-sum game
analysis. Their algorithm (Algorithm 4 in \citet{niazadeh2018optimal},
termed \algname{BSCB}: Binary-Search Continuous Bi-greedy) needs to
estimate the partial derivative of the objective, which is not needed
in our algorithm.  Furthermore, our algorithm is arguably easier to
interpret and to implement than \algname{BSCB}.
We have performed extensive experiments (see \cref{sec_exp} for
details on experimental statistics) to compare them; the results show
that both algorithms generate promising solutions, however, our
algorithm produces better solutions than \algname{BSCB} in most of the
experiments.

For further comparison of these two algorithms, we 
provide a simple example to  show that
DR-DoubleGreedy behaves 
very different from BSCB.
Consider the 2-D DR-submodular quadratic program: 
\begin{align}
  f(\x) = 0.5\x^\trans \BH \x + \bh^\trans \x, \BH = [-1, -1; -1, -2],
  \bh = [0.5; 1]. 
\end{align}

Define
$g([x_1; x_2]) := \frac{\partial f(\x)}{\partial x_1} = - x_1 - x_2 +
0.5$.  Consider the box-constrained DR-submodular maximization problem:

\begin{align}
  \max f(\x), \zero \leqco \x \leqco \one.
\end{align}

Starting with coordinate 1,

- For BSCB: $g([z; 0]) = -z + 0.5$, $g([z; 1]) = -z - 0.5$. In order to find the equilibrium, we set $g([z; 0]) * (1-z) + g([z; 1])* z$ to be 0, which amounts to $-2z + 0.5 = 0$, so $z = 1/4$.

- For \drdg:

Solving 1-D subproblem $u_a = \argmax_{x_1} f([x_1; 0])$ one gets $u_a = 0.5$, and $\delta_a = f([0.5; 0]) - f([0; 0]) = 1/8$.
Solving 1-D subproblem $u_b = \argmax_{x_1} f([x_1; 1])$ one gets $u_b = 0$, and $\delta_b = f([0; 1]) - f([1;1]) = 1$.
So $u = (1/9) * u_a + (8/9) * u_b = 1/18$. 

Note that every step is in closed form in the above derivation.

\section{Experiments on Box Constrained Submodular Maximization}

We show experimental results of the \submodulardg algorithm in this section, and leave the results of \drdg to \cref{chapter_mean_field}.
We experimented with the problem of revenue maximization with
continuous assignments (see \cref{app_revenue_max_alternative} for details), which is a continuous submodular objective.
Without loss of generality, we considered maximizing the revenue from
selling one product (corresponding to $q=1$, see \cref{supp_revenue}
for more details on this model).  It can be observed that the
objective in \cref{eq_re} is generally non-smooth and
\textit{discontinuous} at any point $\x$ which contains the element of
$0$. Since the subdifferential can be empty, we cannot use the
subgradient-based method and could not compare with
\pga.

We considered the following baselines: a) \algname{Random}: uniformly
sample $k_s$ solutions from the constraint set using the hit-and-run
sampler \citep{kroese2013handbook}, and select the best one.  b)
\algname{SingleGreedy}: for non-monotone submodular functions
maximization over a box constraint, we greedily increase each
coordinate, as long as it remains feasible. This approach is similar
to the coordinate ascent method.  In all of the experiments, we use
random order of coordinates for \submodulardg.

\setkeys{Gin}{width=0.8\textwidth}
\begin{figure}[htbp]
  \centering \subfloat[$\alpha =\beta = \gamma=
  10$ \label{fig_revenue2}]{
    \includegraphics[]{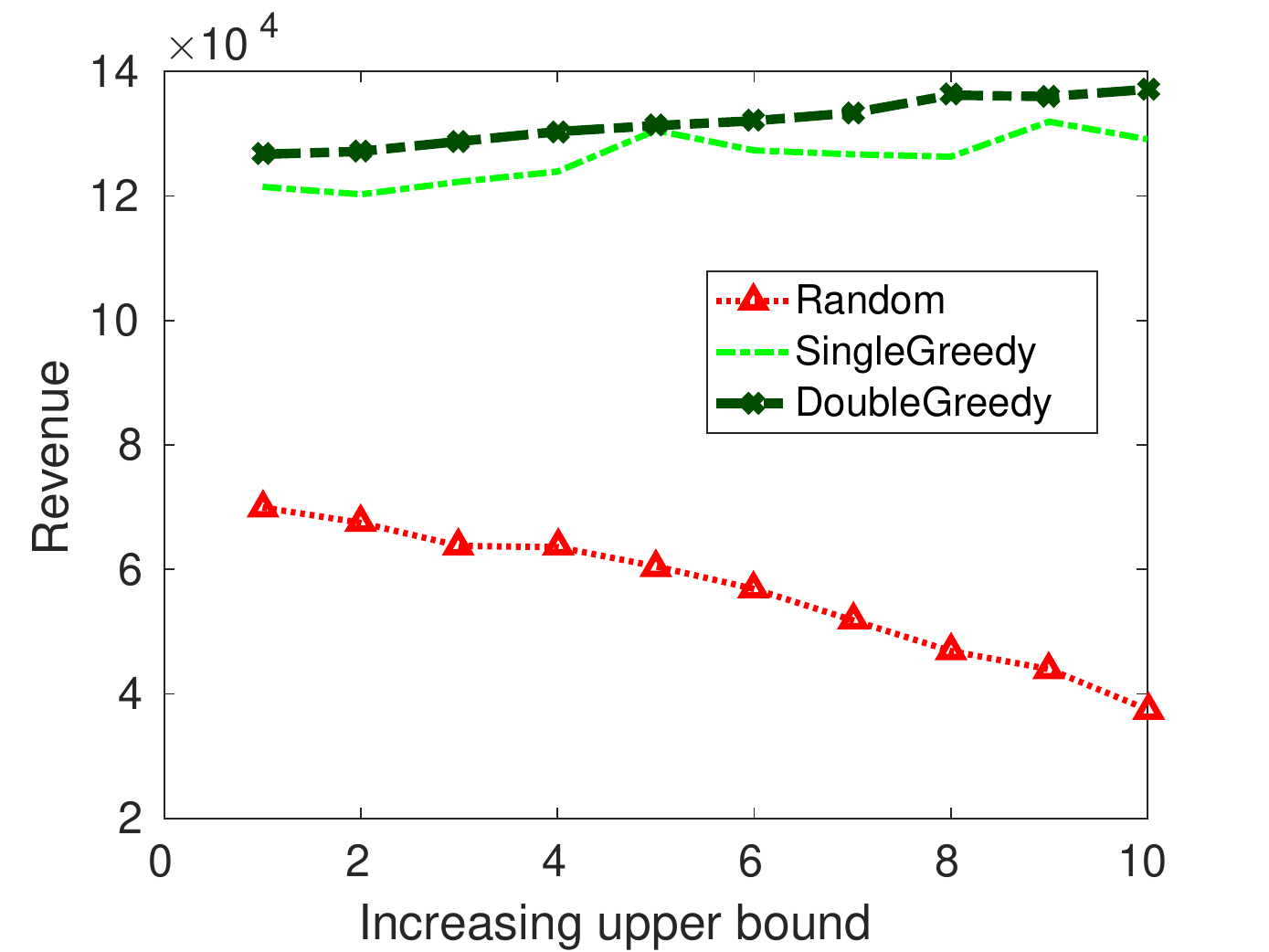}
  }\\
  \subfloat[$ \alpha =10, \beta = 5, \gamma =
  10$ \label{fig_revenue3}]{
    \includegraphics[]{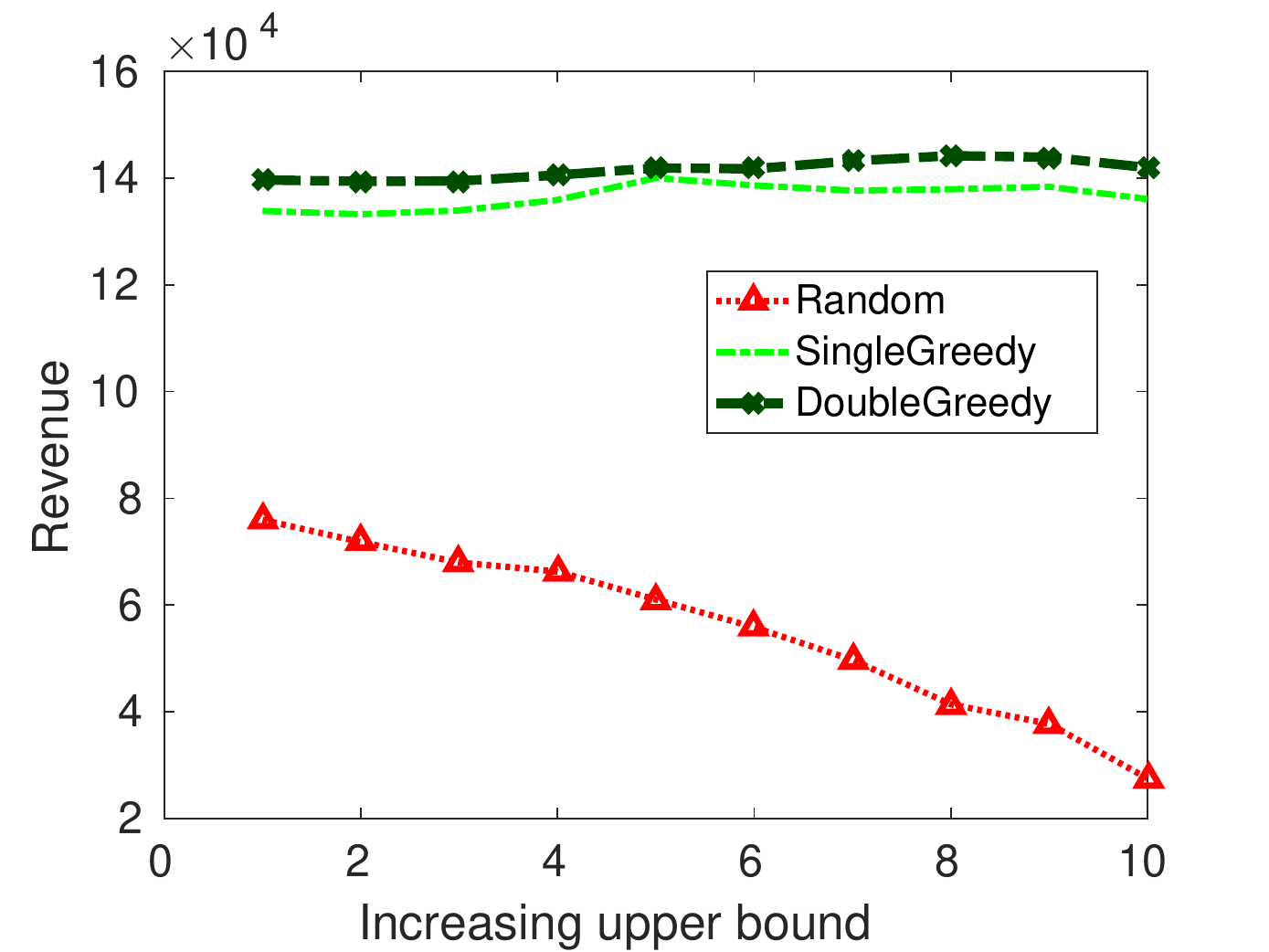}
  }
  \caption{Returned revenues for different experimental settings. In the legend,
    \algname{DoubleGreedy} means \submodulardg.  a, b) Revenue
    returned with different upper bounds on the Youtube social network
    dataset.}
  \label{haha}
\end{figure}

We performed our experiments on the top 500 largest communities of the
YouTube social
network\footnote{\url{http://snap.stanford.edu/data/com-Youtube.html}}
consisting of 39,841 nodes and 224,235 edges. The edge weights were
assigned according to a uniform distribution $U(0, 1)$.  See
\cref{fig_revenue2} and \cref{fig_revenue3} for an illustration of
revenue for varying upper bound ($\bar \bu$) and different
combinations of the parameters $(\alpha, \beta, \gamma)$ in the model
(\cref{eq_re}).  For different values of the upper bound,
\submodulardg outperforms the other baselines, while
\algname{SingleGreedy} maintaining only one intermediate solution
obtained a lower utility than \submodulardg.

\section{Conclusions}

In this chapter we have studied the problem of box-constrained
submodular maximization over continuous domains. We started by
presenting the inapproximability result by a reduction from the
problem of non-monotone submodular set function maximization. Then we 
proposed two algorithms: \submodulardg for maximizing continuous
submodular functions with a 1/3 approximation guarantee, and \drdg for
maximizing continuous DR-submodular functions with the tight 1/2
approximation guarantee. Finally we validated efficiency of our
algorithm on the revenue maximization problem.

\section{Additional Proofs}

\subsection{Proof of  \cref{prop_np2}}

\begin{proof}[Proof of  \cref{prop_np2}]
	
  The main proof follows the reduction of the problem to maximize an
  unconstrained non-monotone submodular set function.
	
  Let us denote $\Pi_1$ as the problem of maximizing an unconstrained
  non-monotone submodular set function, and $\Pi_2$ as the problem of
  maximizing a box constrained non-monotone continuous submodular
  function.
  Following the Appendix A of \citet{buchbinder2012tight}, there exist
  an algorithm $\A$ for $\Pi_1$ that consists of a polynomial time
  computation in addition to a polynomial number of subroutine calls to
  an algorithm for $\Pi_2$. For details see bellow.
	
  Given a submodular set function $F: 2^{\groundset}\rightarrow \R_+$,
  its multilinear extension \citep{calinescu2007maximizing} is a
  function $f: [0,1]^\groundset \rightarrow \R_+$, whose value at a
  point $\x\in [0,1]^\groundset$ is the expected value of $F$ over a
  random subset $R(\x)\subseteq \groundset$, where $R(\x)$ contains
  each element $e\in \groundset$ independently with probability $x_e$.
  Formally,
  $f(\x):= \mathbb{E} [R(\x)] = \sum_{S\subseteq \groundset} F(S)
  \prod_{e\in S} x_e\prod_{e'\notin S}(1-x_{e'})$.
  It can be easily seen that $f(\x)$ is a non-monotone continuous
  submodular function.

  Then algorithm $\A$ can be: 1) Maximize the multilinear extension
  $f(\x)$ over the box constraint $[0, 1]^\groundset$, which can be
  achieved by solving an instance of $\Pi_2$. Obtain the fractional
  solution $\hat \x\in [0, 1]^n$; 2) Return the random set
  $R(\hat \x)$. According to the definition of multilinear extension,
  the expected value of $F(R(\hat \x))$ is $f(\hat \x)$.  Thus proving
  the reduction from $\Pi_1$ to $\Pi_2$.
	
  Given the reduction, the hardness result follows from the hardness
  of unconstrained non-monotone submodular set function maximization.
	
  The inapproximability result comes from that of the unconstrained
  non-monotone submodular set function maximization in
  \citet{feige2011maximizing} and \citet{dobzinski2012query}.
\end{proof}

\subsection{Proof of \cref{thm_double}}
\label{supp_double}

On a high level, the proof follows the proof idea of the
\algname{DoubleGreedy} algorithm for bounded integer lattice in
\citet{gottschalk2015submodular}.

To better illustrate the proof, we reformulate
\cref{alg_uscfmax_DoubleGreedy} into its equivalent form in
\cref{alg_uscfmax_DoubleGreedy_a}, where we split the update into two
steps: when $\delta_a\geq \delta_b$, update $\x$ first while keeping
$\y$ fixed, and then update $\y$ first while keeping $\x$ fixed
($\x^{i}\leftarrow (\sete{x^{i-1}}{e_i}{\hat u_a})$,
$\y^{i}\leftarrow \y^{i-1}$; $\x^{i+1}\leftarrow \x^{i}$,
$\y^{i+1}\leftarrow (\sete{y^{i}}{e_i}{\hat u_a})$ ), when
$\delta_a < \delta_b$, update $\y$ first.  This iteration index change
is only used to ease the analysis.

To prove the theorem, we first prove the following lemmas.

\begin{algorithm}[h]
  \caption{\submodulardg algorithm reformulation (for analysis
    only) \citep{bian2017guaranteed}}\label{alg_uscfmax_DoubleGreedy_a}
  \KwIn{$\max f(\x)$, $\x\in [\underline{\bu}, \bar \bu]$, $f$ is
    generally non-monotone, $f(\underline{\bu}) + f(\bar \bu)\geq 0$}
  {$\x^0 \leftarrow \underline{\bu}$, $\y^0 \leftarrow \bar \bu$\;}
  \For{$i = 1, 3, 5,\cdots, 2n-1$}{ {find $\hat u_a$ \text{ s.t. }
      $f(\sete{x^{i-1}}{e_i}{\hat u_a}) \geq
      \max_{u_a\in[\underline{u}_{e_i}, \bar u_{e_i}]}
      f(\sete{x^{i-1}}{e_i}{u_a}) - \delta$,
      $\delta_a \leftarrow f(\sete{x^{i-1}}{e_i}{\hat u_a}) -
      f(\x^{i-1})$
      \tcp*{$\delta\in [0, \bar \delta]$ is the additive error level.
      }} {find $\hat u_b$
      $\text{ s.t. } f(\sete{y^{i-1}}{e_i}{\hat u_b})\geq
      \max_{u_b\in[\underline{u}_{e_i}, \bar u_{e_i}]}
      f(\sete{y^{i-1}}{e_i}{u_b}) - \delta$,
      $\delta_b \leftarrow f(\sete{y^{i-1}}{e_i}{\hat u_b}) -
      f(\y^{i-1})$
      \;} \If{$\delta_a\geq \delta_b$}{
      {$\x^{i}\leftarrow (\sete{x^{i-1}}{e_i}{\hat u_a})$,
        $\y^{i}\leftarrow \y^{i-1}$ \;} {$\x^{i+1}\leftarrow \x^{i}$,
        $\y^{i+1}\leftarrow (\sete{y^{i}}{e_i}{\hat u_a})$ \;}} \Else{
      {$\y^{i}\leftarrow (\sete{y^{i-1}}{e_i}{\hat u_b})$,
        $\x^{i}\leftarrow \x^{i-1}$\;} {$\y^{i+1}\leftarrow \y^{i}$,
        $\x^{i+1}\leftarrow (\sete{x^{i}}{e_i}{\hat u_b})$\;}} }
  {Return $\x^{2n}$ (or $\y^{2n})$ \tcp*{note that
      $\x^{2n} = \y^{2n}$}}
\end{algorithm}

Lemma \ref{lemma_42} is used to demonstrate that the objective value
of each intermediate solution is non-decreasing,
\begin{lemma}\label{lemma_42}
  $\forall i = 1, 2,..., 2n$, one has,
  \begin{flalign}
    f(\x^i) \geq f(\x^{i-1}) -\delta, \;\; f(\y^i) \geq
    f(\y^{i-1})-\delta.
  \end{flalign}
\end{lemma}

\begin{proof}[Proof of Lemma \ref{lemma_42}]
  Let $j:= e_i$ be the coordinate that is going to be changed. From
  submodularity,
  \begin{flalign}
    f(\sete{x^{i-1}}{j}{\bar u_j}) +
    f(\sete{y^{i-1}}{j}{\underline{u}_j}) \geq f(\x^{i-1}) +
    f(\y^{i-1}).
  \end{flalign}
  So one can verify that $\delta_a + \delta_b \geq -2\delta$. Let us
  consider the following two situations:
	
  \textcolor{emp_color}{1)} If $\delta_a \geq \delta_b$, $\x$ is
  changed first.
	
  We can see that the Lemma holds for the first change (where
  $\x^{i-1}\rightarrow \x^i, \y^i = \y^{i-1}$).  For the second
  change, we are left to prove $f(\y^{i+1}) \geq f(\y^i) -\delta$.
  From submodularity:
  \begin{flalign}
    f(\sete{y^{i-1}}{j}{\hat u_a}) + f(\sete{x^{i-1}}{j}{\bar u_j})
    \geq f(\sete{x^{i-1}}{j}{ \hat u_a }) + f(\y^{i-1}).
  \end{flalign}
  Therefore,
  $f(\y^{i+1}) - f(\y^i) \geq f(\sete{x^{i-1}}{j}{\hat u_a}) -
  f(\sete{x^{i-1}}{j}{ \bar u_j}) \geq -\delta$,
  the last inequality comes from the selection rule of $\delta_a$ in
  the algorithm.
	
  \textcolor{emp_color}{2)} Otherwise, $\delta_a < \delta_b$, $\y$ is
  changed first.
	
  The Lemma holds for the first change
  ($\y^{i-1}\rightarrow \y^i, \x^i = \x^{i-1}$).  For the second
  change, we are left to prove $f(\x^{i+1}) \geq f(\x^i) -
  \delta$. From submodularity,
  \begin{flalign}
    f(\sete{x^{i-1}}{j}{\hat u_b}) +
    f(\sete{y^{i-1}}{j}{\underline{u}_j}) \geq
    f(\sete{y^{i-1}}{j}{\hat u_b}) + f(\x^{i-1}).
  \end{flalign}
  So
  $f(\x^{i+1}) - f(\x^i) \geq f(\sete{y^{i-1}}{j}{\hat u_b}) -
  f(\sete{y^{i-1}}{j}{\underline{u}_j}) \geq -\delta$,
  the last inequality also comes from the selection rule of
  $\delta_b$.
\end{proof}

Let $OPT^i := (\optcont\vee \x^i)\wedge \y^i$, it is easy to observe
that $OPT^0 = \optcont$ and $OPT^{2n} = \x^{2n} = \y^{2n}$.
\begin{lemma}\label{lemma_43}
  $\forall i = 1, 2,..., 2n$, it holds,
  \begin{flalign}
    f(OPT^{i-1}) - f(OPT^{i}) \leq f(\x^i) - f(\x^{i-1}) + f(\y^i) -
    f(\y^{i-1}) +2\delta.
  \end{flalign}
\end{lemma}
Before proving Lemma \ref{lemma_43}, let us get some intuition about
it.  We can see that when changing $i$ from 0 to $2n$, the objective
value changes from the optimal value $f(\optcont)$ to the value
returned by the algorithm: $f(\x^{2n})$.  \cref{lemma_43} is then used
to bound the objective loss from the assumed optimal objective in each
iteration.

\begin{proof} 
	
  Let $j:= e_i$ be the coordinate that will be changed.
	
  First of all, let us assume $\x$ is changed, $\y$ is kept unchanged
  ($\x^i \neq \x^{i-1}, \y^i = \y^{i-1}$), this could happen in four
  situations: \textcolor{emp_color}{1.1)} $x^i_j \leq \opti{j}$ and
  $\delta_a\geq \delta_b$; \textcolor{emp_color}{1.2)}
  $x^i_j \leq \opti{j}$ and $\delta_a < \delta_b$;
  \textcolor{emp_color}{2.1)} $x^i_j > \opti{j}$ and
  $\delta_a\geq \delta_b$; \textcolor{emp_color}{2.2)}
  $x^i_j > \opti{j}$ and $\delta_a < \delta_b$.
  Let us prove the four situations one by one.
	
  \textbf{If} $x^i_j \leq \opti{j}$, the Lemma holds in the following
  two situations:
	
  {\textcolor{emp_color}{\textbf{1.1})}} When
  $\delta_a \geq \delta_b$, it happens in the first change:
  $x^i_j = \hat u_a \leq \opti{j}$, so $OPT^i = OPT^{i-1}$; According
  to Lemma \ref{lemma_42}, $\delta_a + \delta_b \geq -2\delta$, so
  $f(\x^i) - f(\x^{i-1}) + f(\y^i) - f(\y^{i-1}) +2\delta \geq 0$, so
  the Lemma holds;

  \textcolor{emp_color}{\textbf{1.2})} When $\delta_a < \delta_b$, it
  happens in the second change:
  $x^i_j = \hat u_b \leq \opti{j}, y^i_j = y^{i-1}_j = \hat u_b$, and
  since $OPT^{i-1} = (\optcont\vee \x^{i-1})\wedge \y^{i-1}$, so
  $OPT^{i-1}_j =\hat u_b$ and $OPT^i_j =\hat u_b$, so one still has
  $OPT^i = OPT^{i-1}$. So it amouts to prove that
  $\delta_a + \delta_b \geq -2\delta$, which is true according to
  Lemma \ref{lemma_42}.
	
  \textbf{Else if} $x^i_j > \opti{j}$, it holds that
  $OPT^i_j = x^i_j$, all other coordinates of $OPT^{i-1}$ remain
  unchanged.  The Lemma holds in the following two situations:
	
  \textcolor{emp_color}{\textbf{2.1})} When $\delta_a \geq \delta_b$,
  it happens in the first change. One has
  $OPT^i_j = x^i_j = \hat u_a$, $x^{i-1}_j = \underline{u}_j$, so
  $OPT^{i-1}_j = \opti{j}$. And
  $x^i_j =\hat u_a > \opti{j}, y^{i-1}_j = \bar u_j$.  From
  submodularity,
  \begin{flalign}\label{eq_a}
    f(OPT^i) + f(\sete{y^{i-1}}{j}{\opti{j}}) \geq f(OPT^{i-1}) +
    f(\sete{y^{i-1}}{j}{\hat u_a}).
  \end{flalign}
  Suppose by virtue of contradiction that,
  \begin{flalign}\label{eq_b}
    f(OPT^{i-1}) - f(OPT^{i}) > f(\x^i) - f(\x^{i-1}) + 2\delta.
  \end{flalign}
  Summing \cref{eq_a,eq_b} we get:
  \begin{flalign}\label{eq_19}
    0 > f(\x^i) - f(\x^{i-1}) + \delta + f(\sete{y^{i-1}}{j}{\hat
      u_a}) - f(\sete{y^{i-1}}{j}{\opti{j}}) + \delta.
  \end{flalign}
  Because $\delta_a\geq \delta_b$ then from the selection rule of
  $\delta_b$,
  \begin{flalign}\label{eq_20}
    \delta_a = f(\x^i) - f(\x^{i-1}) \geq \delta_b \geq
    f(\sete{y^{i-1}}{j}{c}) - f(\y^{i-1}) - \delta, \forall
    \underline{u}_j \leq c \leq \bar u_j.
  \end{flalign}
  Setting $c = \opti{j}$ and substitite (\ref{eq_20}) into
  (\ref{eq_19}), one can get,
  \begin{flalign}
    0 > f(\sete{y^{i-1}}{j}{\hat u_a}) - f(y^{i-1}) +\delta =
    f(\y^{i+1}) - f(\y^{i})+\delta,
  \end{flalign}
  which contradicts with Lemma \ref{lemma_42}.
	
  \textcolor{emp_color}{\textbf{2.2})} When $\delta_a < \delta_b$, it
  happens in the second change.
  $y^{i-1}_j =\hat u_b, x^i_j = \hat u_b > \opti{j}, OPT^i_j =\hat
  u_b, OPT^{i-1}_j = \opti{j}$. From submodularity,
  \begin{flalign}\label{eq_a1}
    f(OPT^i) + f(\sete{y^{i-1}}{j}{\opti{j}}) \geq f(OPT^{i-1}) +
    f(\sete{y^{i-1}}{j}{\hat u_b}).
  \end{flalign}
  Suppose by virtue of contradiction that,
  \begin{flalign}\label{eq_b1}
    f(OPT^{i-1}) - f(OPT^{i}) > f(\x^i) - f(\x^{i-1}) + 2\delta .
  \end{flalign}
  Summing Equations \ref{eq_a1} and \ref{eq_b1} we get:
  \begin{flalign}
    0 > f(\x^i) - f(\x^{i-1})+\delta + f(\sete{y^{i-1}}{j}{\hat u_b})
    - f(\sete{y^{i-1}}{j}{\opti{j}}) +\delta.
  \end{flalign}
  From Lemma \ref{lemma_42} we have
  $f(\x^i) - f(\x^{i-1}) +\delta \geq 0$, so
  $0 > f(\sete{y^{i-1}}{j}{\hat u_b}) - f(\sete{y^{i-1}}{j}{\opti{j}})
  +\delta$, which contradicts with the selection rule of $\delta_b$.
	
  The case when $\y$ is changed, $\x$ is kept unchanged is similar,
  the proof of which is omitted here.
\end{proof}

With Lemma \ref{lemma_43} at hand, one can prove Theorem
\ref{thm_double}: Taking a sum over $i$ from 1 to $2n$, one can get,
\begin{flalign}\notag
  f(OPT^0) - f(OPT^{2n}) &\leq f(\x^{2n}) - f(\x^{0}) + f(\y^{2n}) -
  f(\y^{0}) + 4n \delta \\\notag & = f(\x^{2n}) + f(\y^{2n}) -
  (f(\underline{\bu}) + f(\bar \bu)) + 4n \delta\\ 
  &\leq
  f(\x^{2n}) + f(\y^{2n}) + 4n \delta.
\end{flalign}
Then it is easy to see that
$f(\x^{2n}) = f(\y^{2n}) \geq \frac{1}{3} f(\optcont) -
\frac{4n}{3}\delta$.

\subsection{Proof of \cref{obs_dr_submodular_max}}

\begin{proof}[Proof of \cref{obs_dr_submodular_max}]
  The proof is very similar to the that of \citep[Proposition
  5]{bian2017guaranteed}, so we just briefly explain here.  One
  observation is that the multilinear extension of a submodular set
  function is also continuous DR-submodular, so we can use the same
  reduction as in \citep[Proposition 5]{bian2017guaranteed} to prove
  the hardness results as above.
\end{proof}

\subsection{Detailed Proof of \cref{thm_cont_doublegreedy}}

\begin{proof}[Detailed Proof of \cref{thm_cont_doublegreedy}]

  Define $\intermed^\pare{k} := (\optcont \vee \x^k) \wedge \y^k$. It
  is clear that $\intermed^\pare{0} = \optcont$ and
  $\intermed^\pare{n} = \x^\pare{n} = \y ^\pare{n}$.  One can notice
  that as \cref{alg_cont_doublegreedy} progresses,
  $\intermed^\pare{k}$ moves from $\optcont$ to $\x^n$ (or $\y^n$).

  Let $r_a = \frac{\delta_a}{\delta_a + \delta_b}, r_b = 1-r_a$,
  $u = r_a u_a + (1-r_a) u_b$.

  Firstly, using DR-submodularity, we prove that in each iteration, if
  we were to flip the 1-D subproblem solutions of $\x$ and $\y$, it
  still does not decrease the function value (in the error-free case
  $\delta=0$).
  \begin{restatable}[]{lemma}{restalemmaone}
    \label{lemma1}
    For all $k=1,...,n$, it holds that,
    \begin{align}
      f(\sete{x^\pare{k-1}}{\ele_k}{u_b}) - f(\x^\pare{k-1}) \geq -{\delta}/{n},  \\\notag 
      f(\sete{y^\pare{k-1}}{\ele_k}{u_a}) - f(\y^\pare{k-1} )  \geq -{\delta}/{n}.
    \end{align}
  \end{restatable}

  \begin{proof}[Proof of \cref{lemma1}]
	
    One can observe that $\x^\pare{k-1} \leqco \y^\pare{k-1}$, so from
    DR-submodularity:
    $f(\sete{x^\pare{k-1}}{\ele_k}{u_b}) - f(\x^\pare{k-1}) \geq
    f(\sete{y^\pare{k-1}}{\ele_k}{u_b}) -
    f(\sete{y^\pare{k-1}}{\ele_k}{a_{\ele_k}}) \geq
    -\frac{\delta}{n}$.

    Similarly, because of that $\x^\pare{k-1} \leqco \y^\pare{k-1}$ and
    $u_a \leq b_{\ele_k}$, from DR-submodularity:
    $f(\sete{y^\pare{k-1}}{\ele_k}{u_a}) - f(\y^\pare{k-1} ) \geq
    f(\sete{x^\pare{k-1}}{\ele_k}{u_a}) -
    f(\sete{x^\pare{k-1}}{\ele_k}{b_{\ele_k}} ) \geq -\frac{\delta}{n}
    $.
  \end{proof}

  Then using the new update rule and the DR property, we show that the
  loss on intermediate variables
  $f( \intermed^\pare{k-1} ) - f(\intermed^\pare{k})$ can be upper
  bounded by the increase of the objective value in $\x$ and $\y$
  times $1/2$.

  \begin{restatable}[]{lemma}{restalemmatwo}
    \label{lemma2}
    For all $k=1,...,n$, it holds that,
    \begin{align}
      &f( \intermed^\pare{k-1} )	 - f(\intermed^\pare{k}) \\\notag 
      \leq
      & \frac{1}{2}   \left[ f(\x^\pare{k}) -f(\x^\pare{k-1} ) + f(\y
	^\pare{k} ) - f(\y ^\pare{k-1})   \right]  + \frac{2.5 \delta }{n}. 
    \end{align}
  \end{restatable}

  \begin{proof}[Proof of \cref{lemma2}]

    Step I:
	
    Let us try to lower bound the RHS of \cref{lemma2}.
	
    \begin{align}
      f(\x^\pare{k}) - f(\x^\pare{k-1}) & =
                                          f(\sete{x^\pare{k-1}}{\ele_k}{r_a
                                          u_a + r_b u_b}) -
                                          f(\x^\pare{k-1}) \\ \notag 
                                        &
                                          \overset{\textcircled{1}}{\geq}
                                          r_a
                                          f(\sete{x^\pare{k-1}}{\ele_k}{u_a})
                                          + r_b
                                          f(\sete{x^\pare{k-1}}{\ele_k}{u_b})
                                          - f(\x^\pare{k-1})\\\notag  
                                        & = r_a
                                          [f(\sete{x^\pare{k-1}}{\ele_k}{u_a})
                                          - f(\x^\pare{k-1})] + r_b
                                          [f(\sete{x^\pare{k-1}}{\ele_k}{u_b})
                                          - f(\x^\pare{k-1})]\\ \notag 
                                        &  \overset{\textcircled{2}}{\geq}  r_a \delta_a  - r_b \frac{\delta}{n},
    \end{align}  
    where \textcircled{1} is because of that $f$ is concave along one
    coordinate, \textcircled{2} is from \cref{lemma1}.

    Similarly,
    \begin{align}
      f(\y^\pare{k}) - f(\y^\pare{k-1}) & =
                                          f(\sete{y^\pare{k-1}}{\ele_k}{r_a
                                          u_a + r_b u_b}) -
                                          f(\y^\pare{k-1}) \\ \notag 
                                        & \geq  r_a
                                          f(\sete{y^\pare{k-1}}{\ele_k}{u_a})
                                          + r_b
                                          f(\sete{y^\pare{k-1}}{\ele_k}{u_b})
                                          - f(\y^\pare{k-1})\\\notag 
                                        & = r_a
                                          [f(\sete{y^\pare{k-1}}{\ele_k}{u_a})
                                          - f(\y^\pare{k-1})] + r_b
                                          [f(\sete{y^\pare{k-1}}{\ele_k}{u_b})
                                          - f(\y^\pare{k-1})]\\ \notag 
                                        & \geq - r_a \frac{\delta}{n} + r_b \delta_b. 
    \end{align}

    So it holds that
    \begin{align}\label{eq_step1}
      f(\x^\pare{k}) - f(\x^\pare{k-1}) +  f(\y^\pare{k}) -
      f(\y^\pare{k-1}) \geq r_a\delta_a + r_b\delta_b -
      \frac{\delta}{n} =  \frac{\delta_a^2 + \delta_b^2}{\delta_a +
      \delta_b} - \frac{\delta}{n}.  
    \end{align}

    Step II:
	
    Now let us upper bound the LHS of \cref{lemma2}.
	
    Notice that
    $\intermed^\pare{k-1} := (\optcont \vee \x^\pare{k-1}) \wedge
    \y^\pare{k-1}$.
    For $\intermed^\pare{k-1}$, its $\ele_k$-th coordinate is
    $x^\star_{\ele_k}$. From $\intermed^\pare{k-1}$ to
    $\intermed^\pare{k}$, its $\ele_k$-th coordinate changes to be
    $u$. So,
	
    \begin{align}
      f(\intermed^\pare{k-1}) - f(\intermed^\pare{k}) & = f(
                                                        \sete{\intermed^\pare{k-1}}{\ele_k}{x^\star_{\ele_k}}
                                                        )  - f(
                                                        \sete{\intermed^\pare{k-1}}{\ele_k}{u}
                                                        ) 
    \end{align}
    Let us consider the following two situations:
	
    \begin{enumerate}
    \item $x^\star_{\ele_k} \leq u$.
		
      In this case:
      \begin{align}
        & f(\intermed^\pare{k-1}) - f(\intermed^\pare{k})\\\notag
        & = f( \sete{\intermed^\pare{k-1}}{\ele_k}{x^\star_{\ele_k}} )
          - f( \sete{\intermed^\pare{k-1}}{\ele_k}{u} ) \\\notag  
        & \overset{\textcircled{3}}{\leq}   f(
          \sete{y^\pare{k-1}}{\ele_k}{x^\star_{\ele_k}} )  - f(
          \sete{y^ \pare{k-1}}{\ele_k}{u} )    \\  \notag 
        & =    f( \sete{y^\pare{k-1}}{\ele_k}{x^\star_{\ele_k}} )  -
          f( \sete{y^ \pare{k-1}}{\ele_k}{r_au_a + r_b u_b} )     \\
        \notag   
        &  \overset{\textcircled{4}}{\leq}  r_a [f(
          \sete{y^\pare{k-1}}{\ele_k}{x^\star_{\ele_k}} )  - f(
          \sete{y^ \pare{k-1}}{\ele_k}{u_a} )]  + r_b [f(
          \sete{y^\pare{k-1}}{\ele_k}{x^\star_{\ele_k}} )  - f(
          \sete{y^ \pare{k-1}}{\ele_k}{u_b} ) ] \\ \notag 
        &  \leq  r_a [f( \sete{y^\pare{k-1}}{\ele_k}{x^\star_{\ele_k}} )  - f( \sete{y^ \pare{k-1}}{\ele_k}{u_a} )] + r_b \frac{\delta}{n} \quad \text{(selection rule of \cref{alg_cont_doublegreedy})} \\\notag
        & \overset{\textcircled{5}}{\leq} r_a [f(
          \sete{y^\pare{k-1}}{\ele_k}{u_b} ) + \frac{\delta}{n}  - (f(
          \y^\pare{k-1})  - \frac{\delta}{n} )] + r_b \frac{\delta}{n}
        \\\notag 
        & \leq r_a \delta_b  + (2r_a + r_b)\frac{\delta}{n}, 	 
      \end{align}
      where \textcircled{3} is because
      $\intermed^\pare{k-1} \leq \y^\pare{k-1}$ and DR-submodularity
      of $f$, \textcircled{4} is from concavity of $f$ along one
      coordinate, \textcircled{5} is because of the selection rule of
      \cref{alg_cont_doublegreedy} and \cref{lemma1}.

    \item $x^\star_{\ele_k} > u$:
		
      In this case:
      \begin{align}
        & f(\intermed^\pare{k-1}) - f(\intermed^\pare{k})\\\notag
        & = f( \sete{\intermed^\pare{k-1}}{\ele_k}{x^\star_{\ele_k}} )
          - f( \sete{\intermed^\pare{k-1}}{\ele_k}{u} ) \\ \notag 
        & \leq   f( \sete{x^\pare{k-1}}{\ele_k}{x^\star_{\ele_k}} )  -
          f( \sete{x^ \pare{k-1}}{\ele_k}{u} )    \text{\quad
          ($\intermed^\pare{k-1}\geqco \x^\pare{k-1}$ \&
                                                       DR-submodularity)}   \\ \notag  
        & =    f( \sete{x^\pare{k-1}}{\ele_k}{x^\star_{\ele_k}} )  -
          f( \sete{x^ \pare{k-1}}{\ele_k}{r_au_a + r_b u_b} )     \\
        \notag   
        &  \leq  r_a [f( \sete{x^\pare{k-1}}{\ele_k}{x^\star_{\ele_k}}
          )  - f( \sete{x^ \pare{k-1}}{\ele_k}{u_a} )]  + r_b [f(
          \sete{x^\pare{k-1}}{\ele_k}{x^\star_{\ele_k}} )  - f(
          \sete{x^ \pare{k-1}}{\ele_k}{u_b} ) ] \\ \notag 
        &  \leq r_a \frac{\delta}{n} +  r_b [f(
          \sete{x^\pare{k-1}}{\ele_k}{x^\star_{\ele_k}} )  - f(
          \sete{x^ \pare{k-1}}{\ele_k}{u_b} ) ] \\\notag 
        &   \leq r_a \frac{\delta}{n} +   r_b  [  ( f(
          \sete{x^\pare{k-1}}{\ele_k}{u_a} ) + \frac{\delta}{n} )   -
          ( f(\x^\pare{k-1})  - \frac{\delta}{n} )  ] \\\notag 
        & = r_b \delta_a + ( 2r_b + r_a)\frac{\delta}{n} 	
      \end{align}

    \end{enumerate}
    We can conclude that in both the above cases, it holds that
    \begin{align}\label{eq_step2}
      f(\intermed^\pare{k-1}) - f(\intermed^\pare{k}) \leq
      \frac{\delta_a \delta_b}{\delta_a + \delta_b} +
      \frac{2\delta}{n}.
    \end{align}	
	
    Combining \cref{eq_step1} and \cref{eq_step2} we can get,
    \begin{align}
      \frac{1}{2} [f(\x^\pare{k}) - f(\x^\pare{k-1}) +  f(\y^\pare{k})
      - f(\y^\pare{k-1})] \geq f(\intermed^\pare{k-1}) -
      f(\intermed^\pare{k})   - \frac{2.5\delta}{n}.
    \end{align}
    Thus we reach \cref{lemma2}.
	
  \end{proof}

  Now we can finalize the proof. For \cref{lemma2}, let us sum for
  $k = 1,...,n$, we can get,
  \begin{align}
    f(\optcont) - f(\x^\pare{n}) \leq \frac{1}{2} [ f(\x^\pare{n}) -
    f(\a) + f(\y^\pare{n}) - f(\b)  ] + 2.5\delta.
  \end{align}
  After rearrangement, one can show that
  $ f(\x^n) \geq \frac{1}{2} f(\optcont) + \frac{1}{4} [f(\a) + f(\b)]
  - \frac{5\delta}{4}$.

\end{proof}

\def\dir{chapters/constrained-non-monotone}
\chapter{Maximizing Non-Monotone Continuous DR-Submodular Functions 
  with a Down-Closed Convex Constraint}
\chaptermark{Convex Constrained Continuous DR-Submodular Maximization}
\label{chapter_constrained_nonmonotone}

\begin{chapquote}{Wang Yangming}
Yin and Yang are one vital force -- the primordial aura.
\end{chapquote}

In this chapter, we study the problem of maximizing a non-monotone
continuous DR-submodular function subject to a down-closed convex
constraint, i.e.,
\begin{align}\label{problem_nonmonotone_max}
  \max_{\x \in \P\subseteq \X} f(\x),    
\end{align}
where $f: \X \rightarrow \R$ is DR-submodular, $\P$ is convex and down-closed.

Non-monotone DR-submodular maximization is strictly harder than the
monotone setting.
For the simple situation with only one unit hypercube constraint ($\P = [0, 1]^n$), we have the 1/2 inapproximability result, as shown
in \cref{obs_dr_submodular_max}.

\section{Two-Phase Algorithm: Applying the Local-Global
  Relation}
\label{subsec_local_alg}

\begin{algorithm}[htbp]
\caption{The \algname{Two-Phase} Algorithm \citep{biannips2017nonmonotone}
}\label{lg_fw}
 
\KwIn{$\max_{\x \in \P} f(\x)$,
  stopping tolerances $\epsilon_1, \epsilon_2$, \#iterations
  $K_1, K_2$}
 
{$\x \leftarrow \algname{Non-convex Frank-Wolfe}(f, \P, K_1,
  \epsilon_1, \x^\pare{0}) $ \tcp*{$\x^\pare{0}\in \P$}}
	
{$\Q \leftarrow \P\cap \{\y\in \R_+^n\;|\;\y\leqco \bar \u -\x \}$\;}
{
  $\z \leftarrow \algname{Non-convex Frank-Wolfe}(f, \Q, K_2,
  \epsilon_2, \z^\pare{0}) $ \tcp*{$\z^\pare{0}\in \Q$}}
 
\KwOut{$\argmax \{f(\x), f(\z)\}$ \;}
\end{algorithm}

By directly applying the local-global relation in
\cref{subsec_local_global}, we present the \twophase algorithm in
\cref{lg_fw}. It is generalized from the ``two-phase'' method in
\citet{chekuri2014submodular,gillenwater2012near}. It invokes a
non-convex solver (we use the \nonconvexfw by
\citet{lacoste2016convergence}; pseudocode is included in
\cref{nonconvex_fw} of \cref{sec_nonconvex_fw}) to find approximately
stationary points in $\P$ and $\Q$, respectively, then returns the
solution with the larger function value.

Though we use \nonconvexfw as a subroutine here, it is noteworthy that
any algorithm that is guaranteed to find an approximately stationary
point can be plugged into \cref{lg_fw} as a subroutine.
We give an improved approximation bound by considering more properties
of DR-submodular functions.  Borrowing the results of
\citet{lacoste2016convergence}, we get the following,
\begin{theorem}\label{rate_local_fw}
  The output of \cref{lg_fw} satisfies,
  \begin{align}\label{eq_local_rates}
    &\max \{f(\x), f(\z)\}  \geq 
    \frac{\mu}{8}\left(\|\x -\x^*\|^2 + \|\z - \z^*\|^2\right )\\\notag 
    &    + \frac{1}{4}\left[f(\x^*)  - \min \left\{\frac{\max \{2h_1,
      C_f(\P)\}}{\sqrt{K_1+1}} , \epsilon_1\right\}   -
      \min\left\{\frac{\max \{2h_2, C_f(\Q)\}}{\sqrt{K_2+1}} ,
      \epsilon_2\right\} \right],  
  \end{align}  
  where $h_1 := \max_{\x\in\P}f(\x) - f(\x^\pare{0})$,
  $h_2 := \max_{\z\in\Q}f(\z) - f(\z^\pare{0})$ are the initial
  suboptimalities,
  $C_f(\P) : = \sup_{\x, \v\in \P, \gamma\in (0, 1], \y = \x + \gamma
    (\v - \x )}\frac{2}{\gamma^2}(f(\y) - f(\x) - {(\y -
    \x)^\trans}{\nabla f(\x)})$
  is the curvature of $f$ w.r.t.  $\P$, and $\z^*= \x\vee \x^* -\x$.
\end{theorem}
\cref{rate_local_fw} indicates that \cref{lg_fw} has a $1/4$
approximation guarantee and $1/\sqrt{k}$ rate of convergence.
However, it has good empirical performance as demonstrated by the
practical experiments.  Informally, this can be partially explained by
the term $\frac{\mu}{8}\left(\|\x -\x^*\|^2 + \|\z - \z^*\|^2\right )$
in
\labelcref{eq_local_rates}: if $\x$ strongly deviates from $\x^*$,
then this term will augment the bound; if $\x$ is close to $\x^*$, by
the smoothness of $f$, it should be close to optimal.

\section{Shrunken FW: Follow Concavity and Shrink
  Constraint}\label{sec_fw_variant}

\begin{algorithm}[htbp]
	\caption{The \shrunkenfw Algorithm
		for Non-monotone
		{DR}-submodular
		Maximization \citep{biannips2017nonmonotone}}\label{fw-non-monotone}
	\KwIn{$\max_{\x \in \P} f(\x)$
				; \#iterations $K$; \stepsize $\gamma= 1/K$.}
	{$\x^\pare{0} \leftarrow \zero$, $t^\pare{0}\leftarrow 0$, $k\leftarrow 0$\tcp*{$k:$ iteration index, $t^\pare{k}:$ cumulative \stepsize}}
	\While{$t^\pare{k} <  1$}{
		{$\v^\pare{k} \leftarrow   \argmax_{\v\in\P, \textcolor{blue}{\v\leqco {\bar \u}-\x^\pare{k}}} \dtp{\v}{ \nabla f(\x^\pare{k})}$\tcp*{\emph{ \color{blue} shrunken LMO}} \label{new_lmo}}
		{use uniform \stepsize $\gamma_k = \gamma$;  set $\gamma_k \leftarrow \min\{\gamma_k, 1-t^\pare{k} \}$\;}
		{$\x^\pare{k+1}\leftarrow \x^\pare{k} + \gamma_k \v^\pare{k}$, $t^\pare{k+1} \leftarrow t^\pare{k} + \gamma_k$,  $k\leftarrow k+1$\;}
	}
	\KwOut{$\x^\pare{K}$
		\tcp*{suppose there are $K$ iterations in total}
	}
\end{algorithm}

\cref{fw-non-monotone} summarizes the \shrunkenfw variant, which  is inspired by the  unified continuous
greedy algorithm in \citet{feldman2011unified} for
maximizing the multilinear extension of a submodular
set function.

It initializes the solution $\x^\pare{0}$ to be $\zero$, and maintains
$t^\pare{k}$ as the cumulative \stepsize. At iteration $k$, it
maximizes the linearization of $f$ over a ``shrunken'' constraint set
$\{\v \mid \v\in \P, \v\leqco \bar \u-\x^\pare{k}\}$, which is
different from the classical LMO
of Frank-Wolfe-style algorithms (hence we refer to it as the
``shrunken LMO''). Then it employs an update step in the direction
$\v^\pare{k}$ chosen by the LMO with a uniform step size
$\gamma_k = \gamma$.
The cumulative \stepsize $t^\pare{k}$ is used to ensure that the
overall \stepsizes sum to one, thus the output solution $\x^\pare{K}$
is a convex combination of the LMO outputs, hence also lies in $\P$.

The shrunken LMO (Step \labelcref{new_lmo}) is the key difference
compared to the \submodularfw variant in
\citet{bian2017guaranteed} (detailed in
\cref{alg_sfmax_GradientAscend}). Therefore, we call
\cref{fw-non-monotone} \shrunkenfw.  The extra constraint
$\v\leqco {\bar \u}-\x^\pare{k}$ is added to prevent too aggressive
growth of the solution, since in the non-monotone setting such
aggressive growth may hurt the overall performance.

The next theorem states the guarantees of \shrunkenfw in
\cref{fw-non-monotone}.
\begin{theorem}\label{thm-e}
  Consider \cref{fw-non-monotone} with uniform step size $\gamma$.
  For $k = 1,..., K$ it holds that,
  \begin{flalign}
    f(\x^\pare{k}) \geq t^\pare{k} e^{-t^\pare{k}}f(\x^*) - \frac{L
      D^2}{2}k\gamma^2 - O(\gamma^2)f(\x^*).
  \end{flalign}
\end{theorem}
By observing that $t^\pare{K} = 1$ and applying \cref{thm-e}, we get
the following \namecref{coro_e}:
\begin{corollary}\label{coro_e}
  The output of \cref{fw-non-monotone} satisfies
  	\begin{flalign}
  f(\x^\pare{K}) \geq \frac{1}{e} f(\x^*) - \frac{L D^2}{2K} -
  O\left(\frac{1}{K^2}\right)f(\x^*).
   \end{flalign}
\end{corollary}

\cref{coro_e} shows that \cref{fw-non-monotone} enjoys a sublinear
convergence rate towards some point $\x^\pare{K}$ inside $\P$, with a
$1/e$ approximation guarantee.

\paragraph{Proof sketch of \cref{thm-e}: }
The proof is by induction. To prepare the building blocks, we first of
all show that the growth of $\x^\pare{k}$ is indeed bounded,
\begin{restatable}[Bounding the growth of $\x^\pare{k}$]{lemma}{restalemmatwo}
  \label{prop_non_fw}
  Assume $\x^\pare{0} = \zero$. For $k=0,..., K-1$, it holds,
  \begin{align}
   x_i^\pare{k}\leq \bar u_i[1-(1-\gamma)^{t^\pare{k}/\gamma}],
  \forall i\in [n].
  \end{align}
\end{restatable}

Then the following \namecref{lem_nonmonotone_fw} provides a lower
bound, which depends on the global optimum,

\begin{restatable}[Generalized from Lemma 7 of
  \citet{chekuri2015multiplicative}]{lemma}{restalemmathree}
  \label{lem_nonmonotone_fw}
  Given $\bmtheta\in (\zero, \bar \u]$, let
  $\lambda' = \min_{i\in [n]} \frac{\bar u_i}{\theta_i}$. Then for all
  $\x\in [\zero, \bmtheta]$, it holds,
  \begin{align}
 f(\x\vee \x^*) \geq (1-\frac{1}{\lambda'})f(\x^*).
  \end{align}
\end{restatable}

Then the key ingredient for induction  is the relation between  $f(\x^{\pare{k+1}})$
and $f(\x^{\pare{k}})$ indicated by:  
\begin{restatable}{claim}{restaclaimthree}
  \label{claim3_1}
  For $k = 0,...,K-1$ it holds,
  	\begin{align}
  f(\x^{\pare{k+1}}) \geq (1-\gamma) f(\x^{\pare{k}}) +
  \gamma(1-\gamma)^{t^\pare{k}/\gamma} f(\x^*) -\frac{L
    D^2}{2}\gamma^2.
	\end{align}
\end{restatable}			
It is derived by a combination of the quadratic lower
bound in \cref{eq_quad_lower_bound}, \cref{prop_non_fw} and
\cref{lem_nonmonotone_fw}.

\subsection{Remarks on the Two Algorithms.}
Notice that though the \twophase algorithm has an inferior guarantee
compared to \shrunkenfw, it is still of interest: i) It preserves 
flexibility in using a wide range of existing solvers for finding an
(approximately) stationary point. ii) The guarantees that we present
rely on a worst-case analysis. The empirical performance of the
\twophase algorithm is often comparable or better than that of
\shrunkenfw. This suggests to explore more properties in concrete
problems that may favor the \twophase algorithm.

\section{Experiments}
\label{sec_exp_nonmonotone_max}

\subsection{Maximizing  Softmax Extensions}

With some derivation, one can see the derivative of the softmax
extension in \cref{eq_softmax} is:
 \begin{align} 
\nabla_i f(\x) = \tr{ \{ [{\diag(\x)(\bmL-\bmI) +\bmI }]^{-1} [(\bmL -
  \bmI)_i] \}}, \forall i\in [n],
 \end{align}
where $(\bmL - \bmI)_i$ denotes the matrix obtained by zeroing all
entries except for  the $i^\text{th}$ row of $(\BL - \BI)$. Let
$\BC:= ({\diag(\x)(\bmL-\bmI) +\bmI })^{-1}, \BD:=(\BL - \BI)$, one
can see that $\nabla_i f(\x) = \BD_{i\cdot}^\trans \BC_{\cdot i}$,
which gives an efficient way to calculate the gradient $\nabla f(\x)$.

\setkeys{Gin}{width=0.75\textwidth}
     \begin{figure}[htbp]
   \center 
  \vspace{-0.4cm}
          \subfloat
          {
          \includegraphics[]{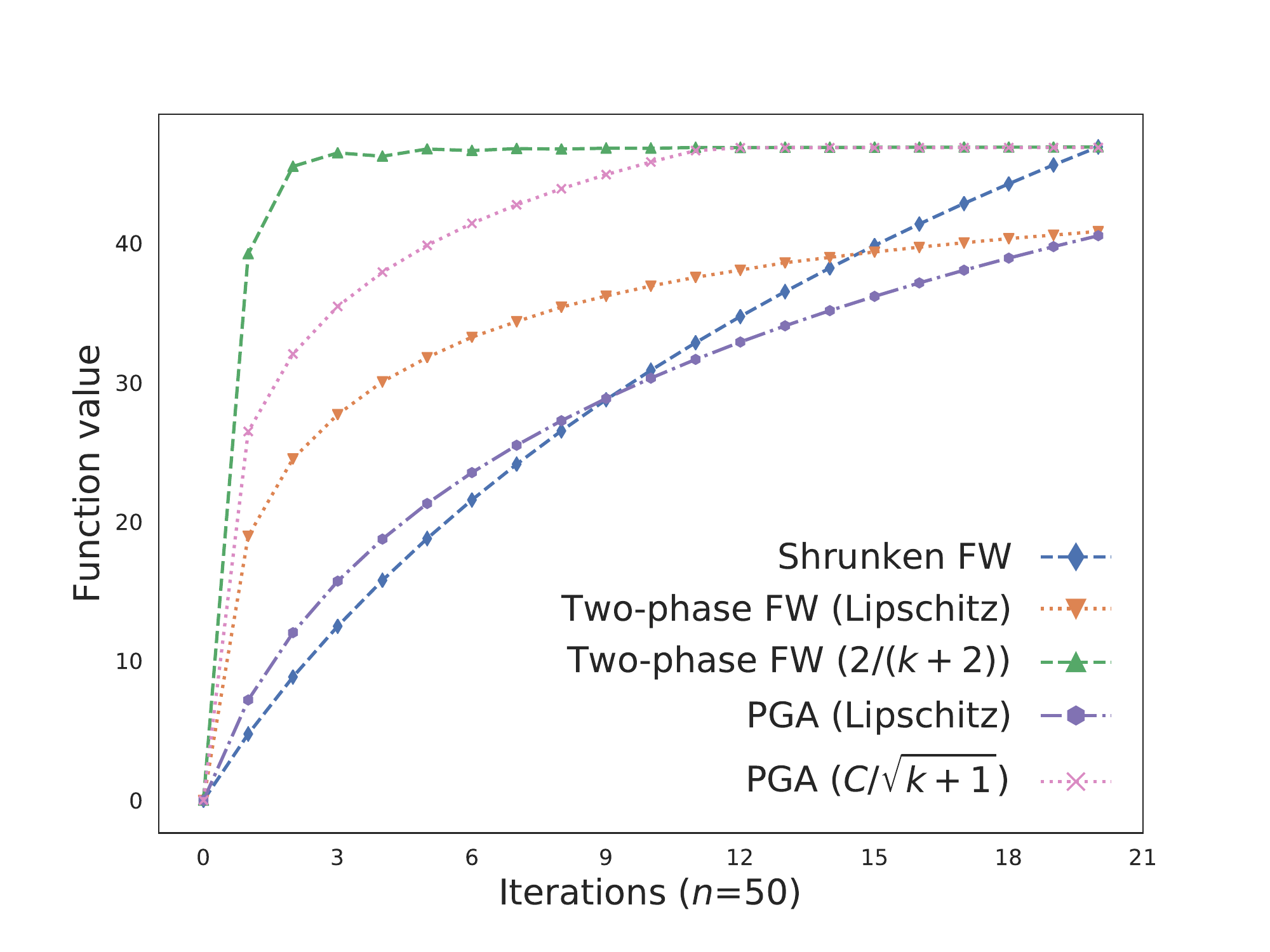}}
     \vspace{-1.4cm}   
    \subfloat
    {
    \includegraphics[]{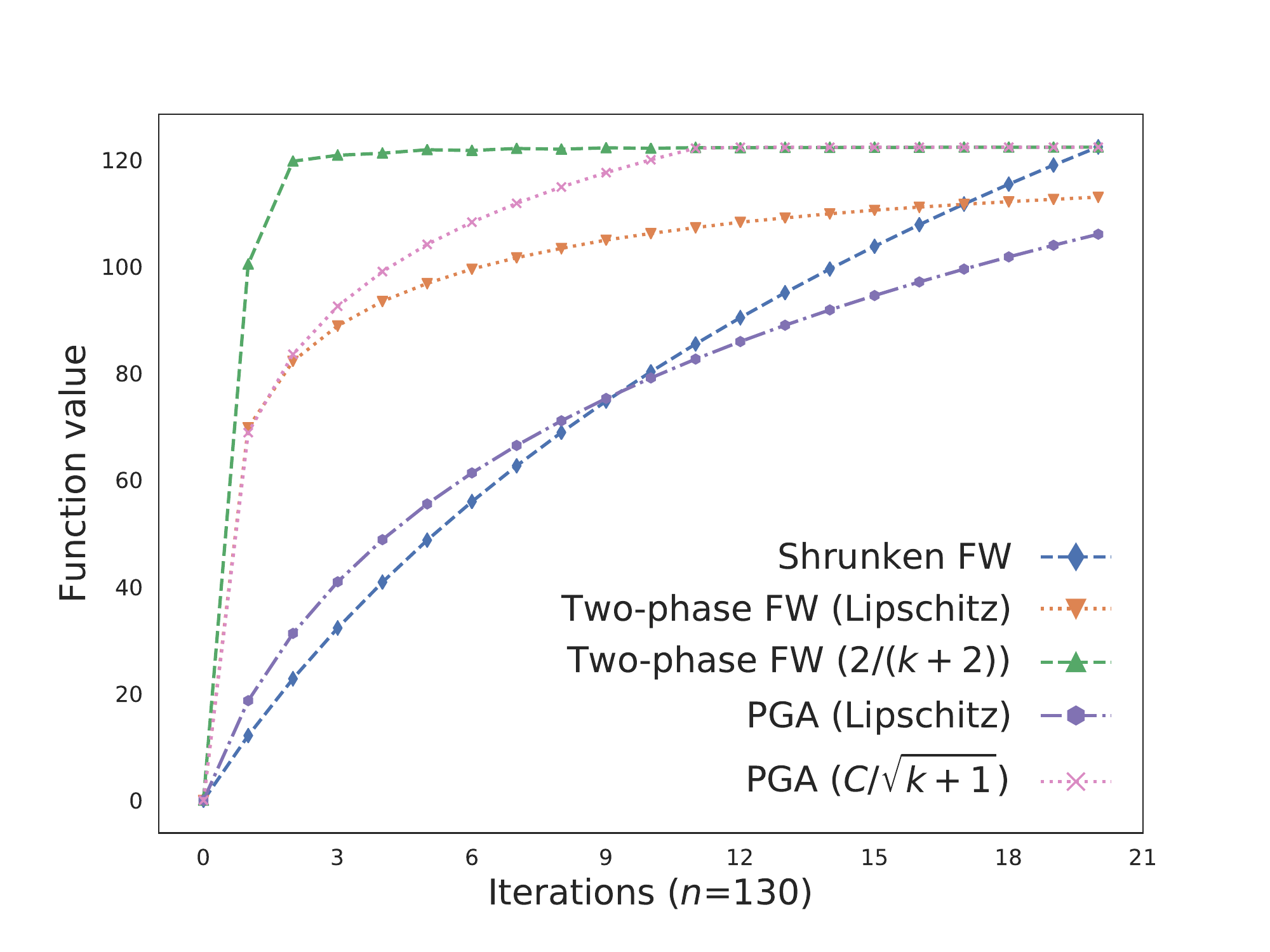}}
     \vspace{-1.4cm}  
  \subfloat
  {
  \includegraphics[]{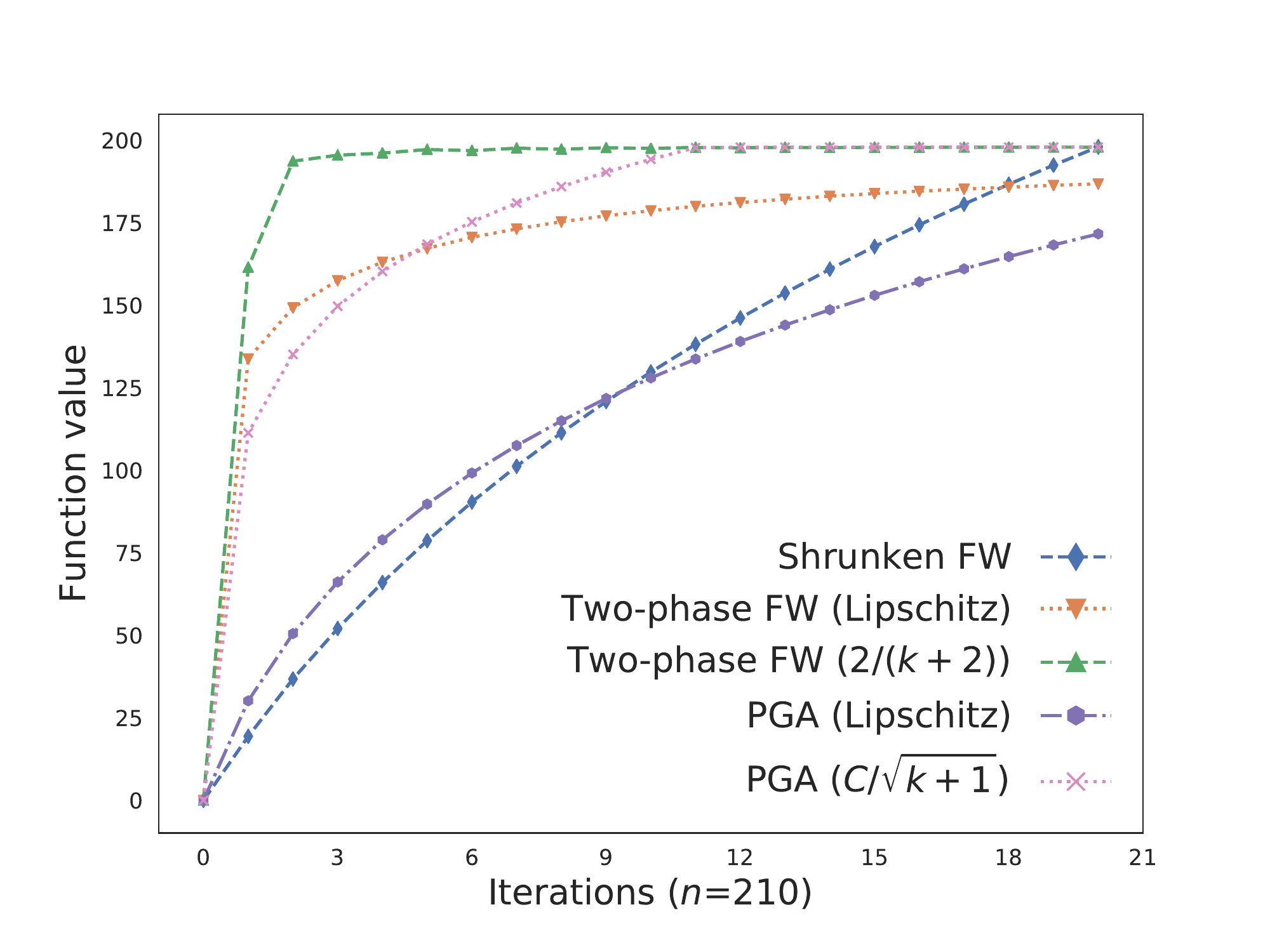}}
\caption{Trajectories of different solvers  on  Softmax instances with one cardinality  constraint.}
     \label{fig_softmax_syn}
\end{figure}   
 
 \paragraph{Results on synthetic data.}
 We generate the softmax objectives (see \labelcref{eq_softmax}) in
 the following way: first generate the $n$ eigenvalues $\d\in \R_+^n$,
 each evenly distributed in $[0, 10]$, and set $\BD =
 \diag(\d)$.
 After generating a random unitary matrix $\BU$, we set
 $\BL = \BU \BD\BU^\trans$.  One can verify that $\BL$ is positive
 semidefinite and has eigenvalues as the entries of $\d$.
Then we generate one cardinality constraint in the form of 
$\BA \x \leqco \b$, where $\BA =  \one^{1\times n}$ and 
$\b = 0.5n$. 

Function value trajectories returned by different solvers are shown in
\cref{fig_softmax_syn}.
One can observe that \twophasefw has the fastest convergence.
\shrunkenfw converges slower, however, it can always return a high
function value finally.  The performance of \pga highly depends on the
hyperparameters of the \stepsizes.

\subsection{Revenue Maximization with Continuous Assignments}

Given a social connection graph with nodes denoting $n$ users and
edges encoding their connection strength, the viral marketing
strategy suggests to choose a small subset of buyers to give them some
product for free, to trigger a cascade of further adoptions through
``word-of-mouth'' effects, in order to maximize the total revenue
\citep{hartline2008optimal}.

One model with ``discrete'' product assignments was used by
\citet{soma2017non} and \citet{durr2019non}, which is motivated by the
observation that giving a user more free products increases the
likelihood that the user will advocate this product.  It is also natural to consider continuous
product assignments which is suitable for products that will be given
to users to try for a certain period (e.g., a new software).

The model  can be viewed
as a simplified variant of the Influence-and-Exploit (IE) strategy of
\citet{hartline2008optimal}.
In the \emph{influence} stage, for each of the user $i$, if giving him
$x_i$ units of products for free, the user becomes an advocate of the
product with probability $1 - q^{x_i}$ (independently from other
users), where $q\in (0, 1)$ is a parameter.  In the \emph{exploit}
stage: suppose that a set $S$ of users advocate the product while the
complement set $\groundset \setminus S$ of users do not. Now the
revenue comes from the users in $\groundset \setminus S$, and they
will be influenced by the advocates with revenue proportional to
the edge weights.  So the expected revenue is a function
$f: \R_+^\groundset \rightarrow \R_+$:
\begin{align} 
f(\x)  
&= \epe[S]{\sum_{i\in S} \sum_{j\in \groundset\setminus S}W_{ij} } 
~=~ \sum_{i\in \groundset} \sum_{j\in \groundset\setminus \{i\}} W_{ij} (1- q^{x_i})q^{x_j},
\end{align}
where $\BW$  is the adjacency matrix of the social connection graph.

\subsubsection{Experimental Setting}

We experimented with several real-world graphs from the Konect network
collection
\citep{kunegis2013konect}\footnote{\url{http://konect.uni-koblenz.de/networks}}
and the SNAP\footnote{\url{http://snap.stanford.edu/}} dataset. 
The graph datasets and corresponding experimental parameters are
documented in \cref{tab_dataset}.

\begin{table}[htbp]
	\begin{center}
	\caption{Graph datasets and corresponding experimental parameters}
		\label{tab_dataset}
		
		\begin{tabularx}{\textwidth}{|r|X|X|X|X|}
			\hline
			Dataset &   $n$  & \#edges & $q$ & budget $b$  \\
			\hline
			\hline 
			``Reality Mining''  & 96 & 1,086,404 (multiedge) &   0.75 &   $0.2nu$  \\
			\hline
			``Residence hall'' & 217 & 2,672 & 0.75  & $0.4nu$  \\
			\hline
			``Infectious'' & 410 & 17,298 & 0.7  & $0.2nu$  \\
			\hline 	
			``U. Rovira i Virgili'' & 1,133  & 5,451& 0.8 & 		$0.2nu$  \\
						\hline 	
			``ego Facebook'' & 4,039  & 88,234& 0.9 & 		$0.1nu$  \\
			\hline 
		\end{tabularx}
	\end{center}
\end{table}

For a specific example, the ``Reality Mining''
\citep{eagle2006reality}
dataset\footnote{\url{http://konect.uni-koblenz.de/networks/mit}, and\\
  \url{http://realitycommons.media.mit.edu/realitymining.html}}
contains the contact data of 96 persons through tracking 100 mobile
phones.  The dataset was collected in 2004 over the course of nine
months and represents approximately 500,000 hours of data on users'
location, communication and device usage behavior.
Here one contact could mean a phone call, Blueteeth sensor proximity
or physical location proximity.  We use the number of contacts as the
weight of an edge, by assuming that the more contacts happen between
two persons, stronger the connection strength should be.

\subsubsection{Experimental Results}

\paragraph{Results on a small graph with visualization.}

We firstly tested on a small graph for a sanity check.  It has to be
small enough in order to visualize the results on the graph. To
achieve this test goal, we select a subgraph from the ``Reality Mining'' dataset
by taking the first 5 users/nodes, the nodes and number of contacts
amongst nodes are shown in \cref{fig_reality_subgraph}. For ease of
illustration, we label the 5 users as ``A, B, C, D,
E''. One can see that there are different level of contacts between
different users, for example, there are 22,194 contacts between A and
B, while there are only 82 contacts between E and C.

\cref{fig_sub_Trajectories_reality} traces the trajectories of
different algorithms when maximizing the revenue objective. They were
all run for 20 iterations. One can see that \shrunkenfw and
\twophasefw reach higher revenue than \pga algorithms. Notice that
\shrunkenfw and \twophasefw with oblivious \stepsizes do not need to
tune any hyperparameters, while the others need to adapt the Lipschitz
parameter $L$ and the constant $C$ to determine the \stepsizes.

\setkeys{Gin}{width=0.9\textwidth}
\begin{figure}[htbp]
	\center 
	\subfloat[The ``Reality Mining'' subgraph. \label{fig_reality_subgraph}]{
		\includegraphics[]{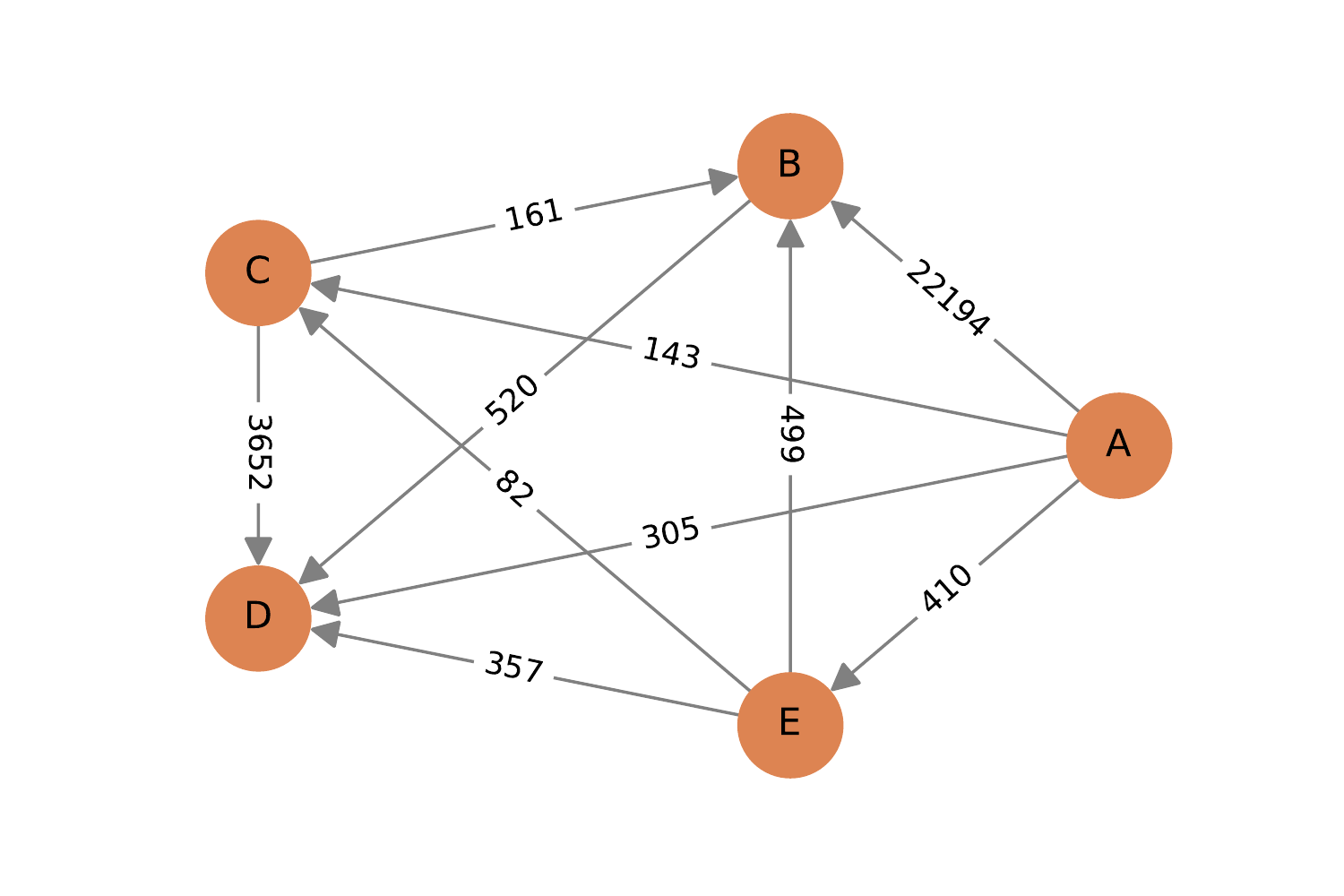}}
	
	\subfloat[Trajectories of algorithms with 20 iterations \label{fig_sub_Trajectories_reality}]{
		\includegraphics[]{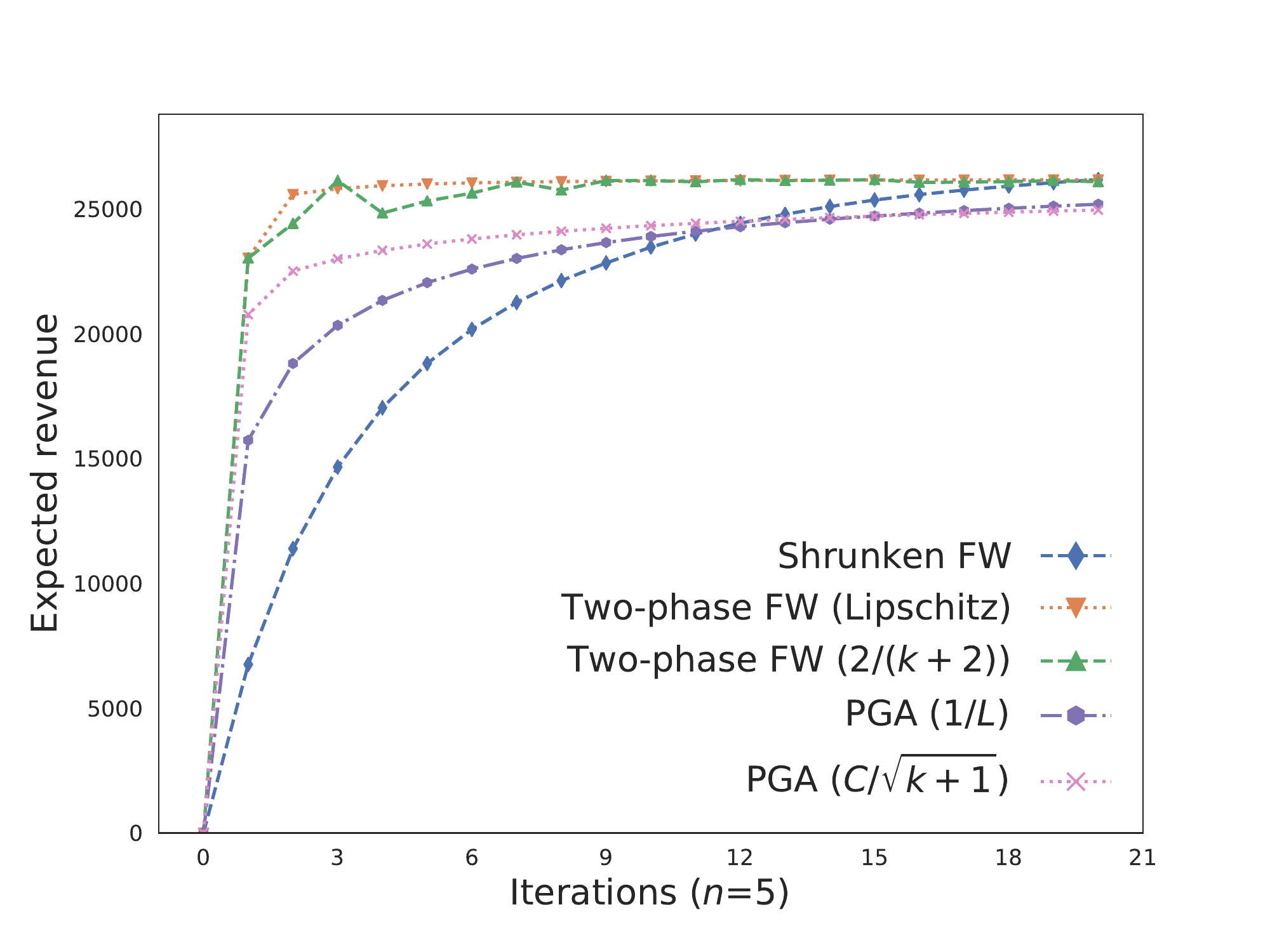}}
	\caption{Results  on real-world graphs with one cardinality constraint, where $b=0.2*n*u$.}
	\label{fig_reality}
\end{figure}

One may ask the question: How does the assignment look like for
different algorithms? In order to show this behavior, we visualize the
assignments in \cref{fig_reality_assignments}.
One can see that \shrunkenfw assigns user A the most free products
(6.1), followed by user C (3.3), then user E (0.6). All other users
get $0$ assignment.  This is consistent with the intuition: one can
observe that user A most strongly influences others users (22,194+
410 + 143), while user D exerts zero influence on others. \twophasefw
provides similar result, while \pga is conservative in assigning free
products to users.

\setkeys{Gin}{width=0.8\textwidth}
\begin{figure}[htbp]
\center 
\subfloat{
\includegraphics[]{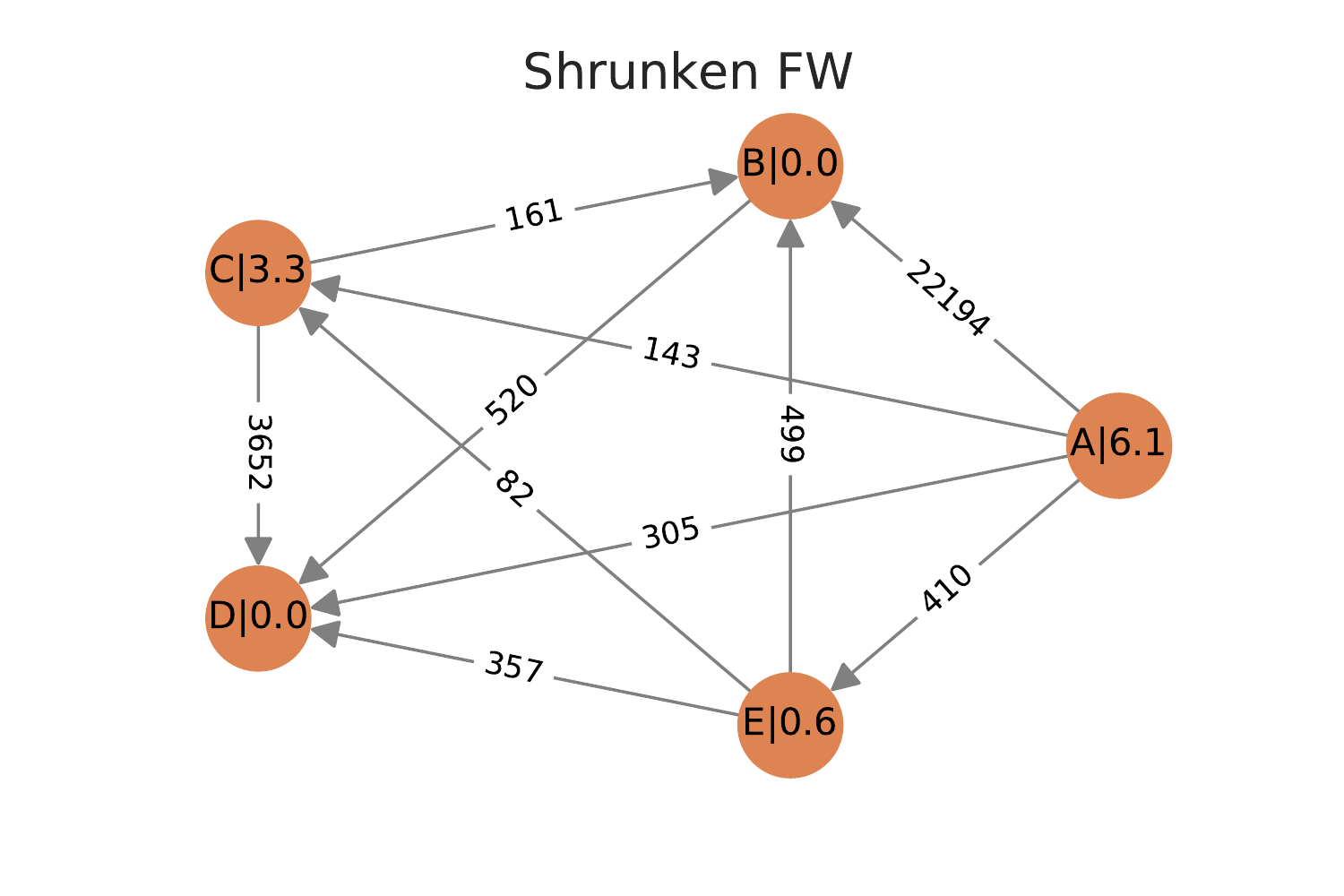}}
\vspace{-0.8cm}
\subfloat{
\includegraphics[]{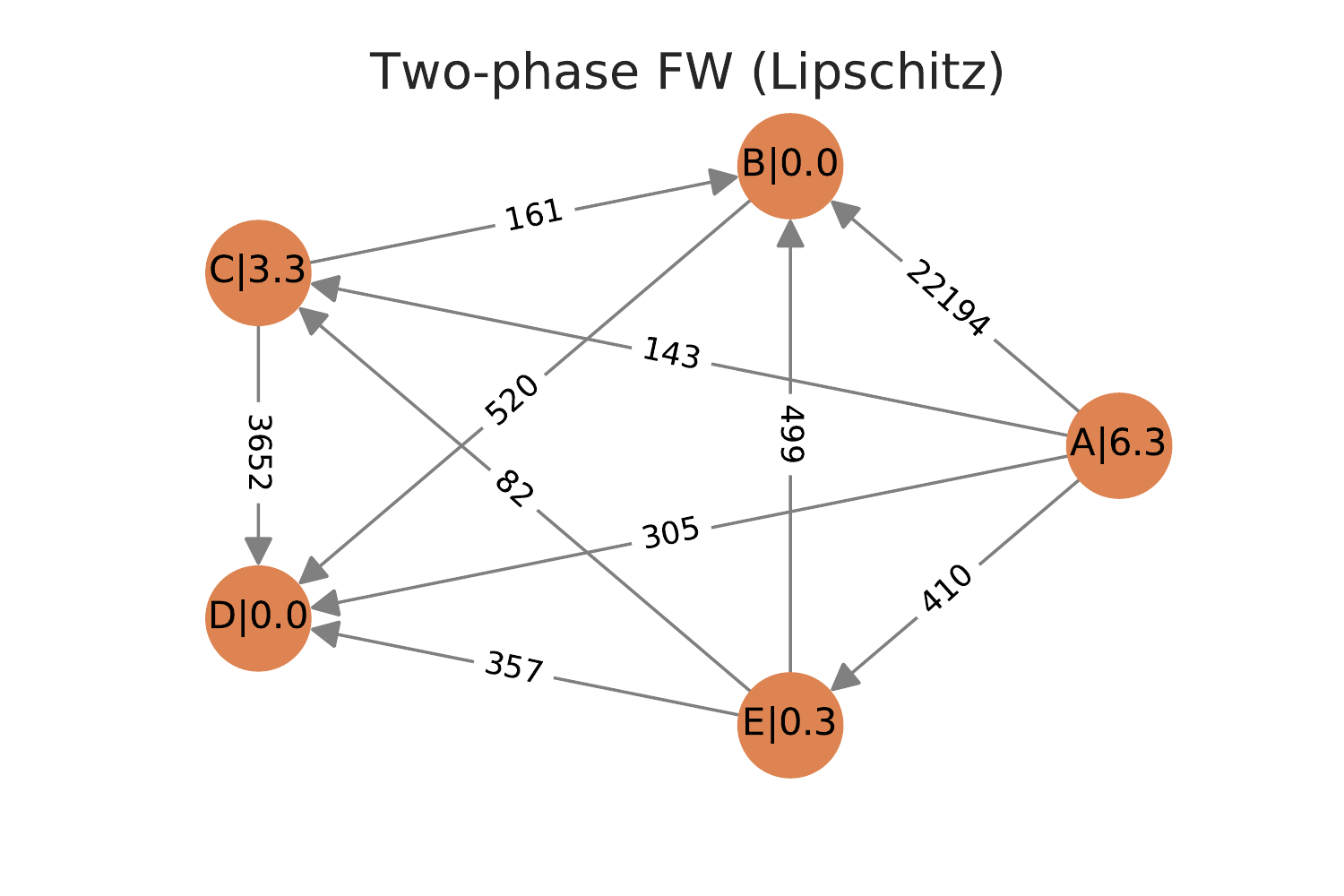}}
\vspace{-0.8cm}
\subfloat
{
\includegraphics[]{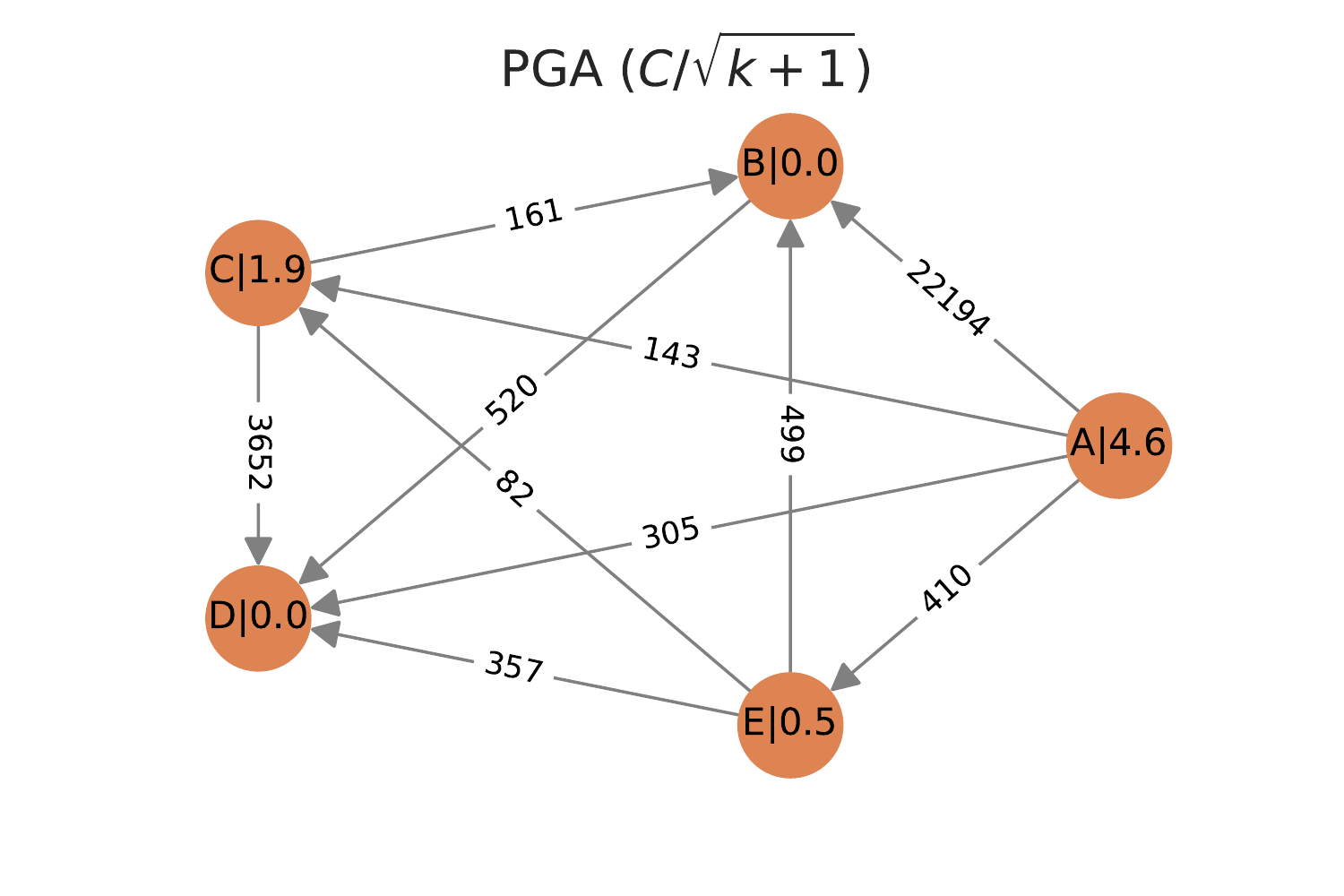}}	
\vspace{-0.4cm}
\caption{Assignments to the users returned by different algorithms.}
\label{fig_reality_assignments}
\end{figure}

\paragraph{Results on big graphs.}

Then we looked at the behavior of the algorithms on the original big
graph, which is plotted in \cref{fig_traj_big_graph} and
\cref{fig_traj_big_graph2}, for real-world graphs with at most
$n = 4,039$ nodes.

One can observe that usually \twophasefw algorithm achieves the highest
objective value, and also converges with the fastest rate. \shrunkenfw converges
slower than \twophasefw, but it always reaches competitive function
value, since it has a $1/e$ approximation guarantee.  \pga algorithms
need to tune parameters for the \stepsize, and converges to lower
objective values.

\setkeys{Gin}{width=0.9\textwidth}
\begin{figure}[htbp]
	\center 
	\subfloat
	[``Residence hall'' dataset]
	{
		\includegraphics[]{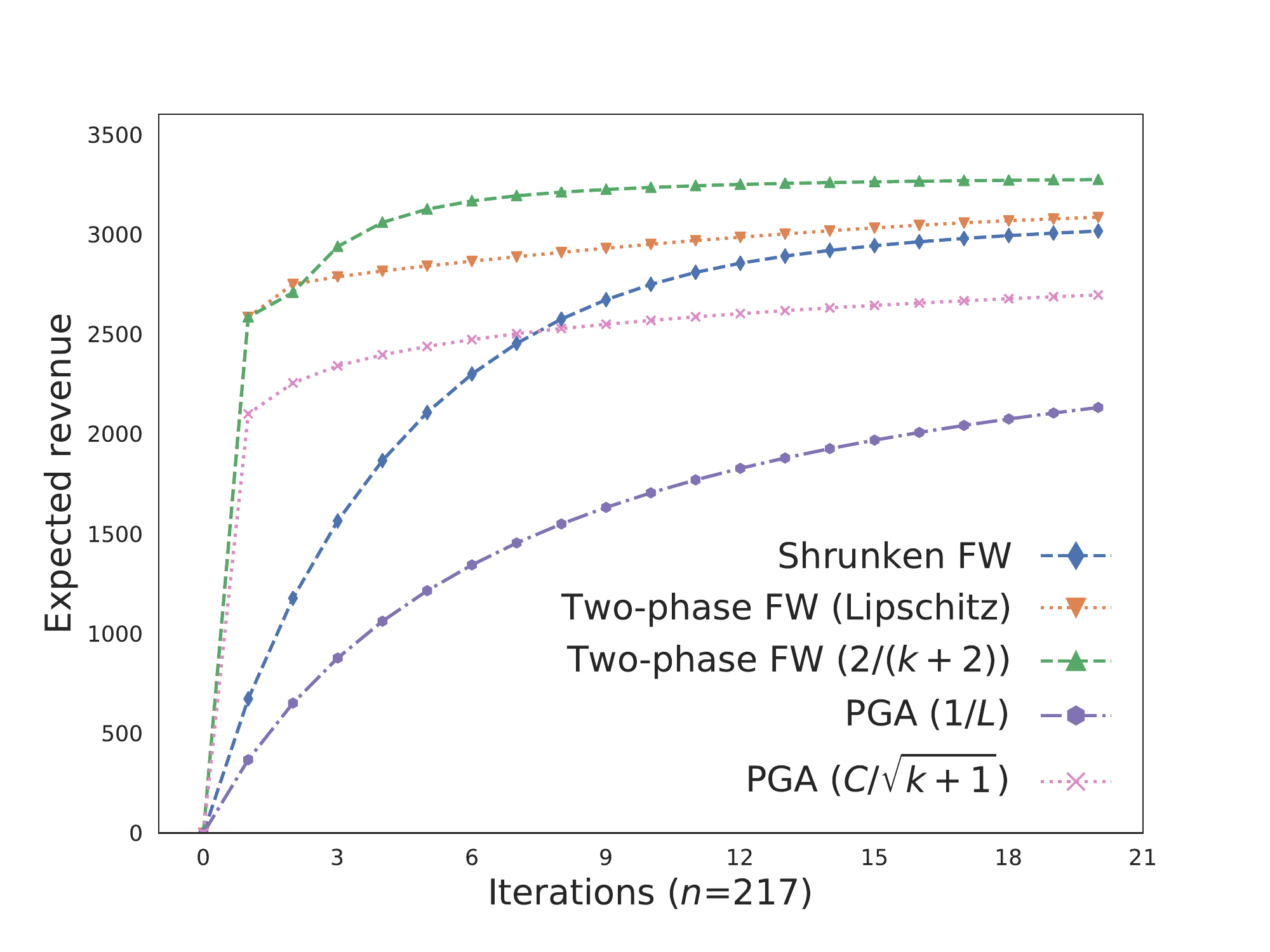}}
	\vspace{-0.4cm}
	\subfloat
	[``Infectious'' dataset]
	{
		\includegraphics[]{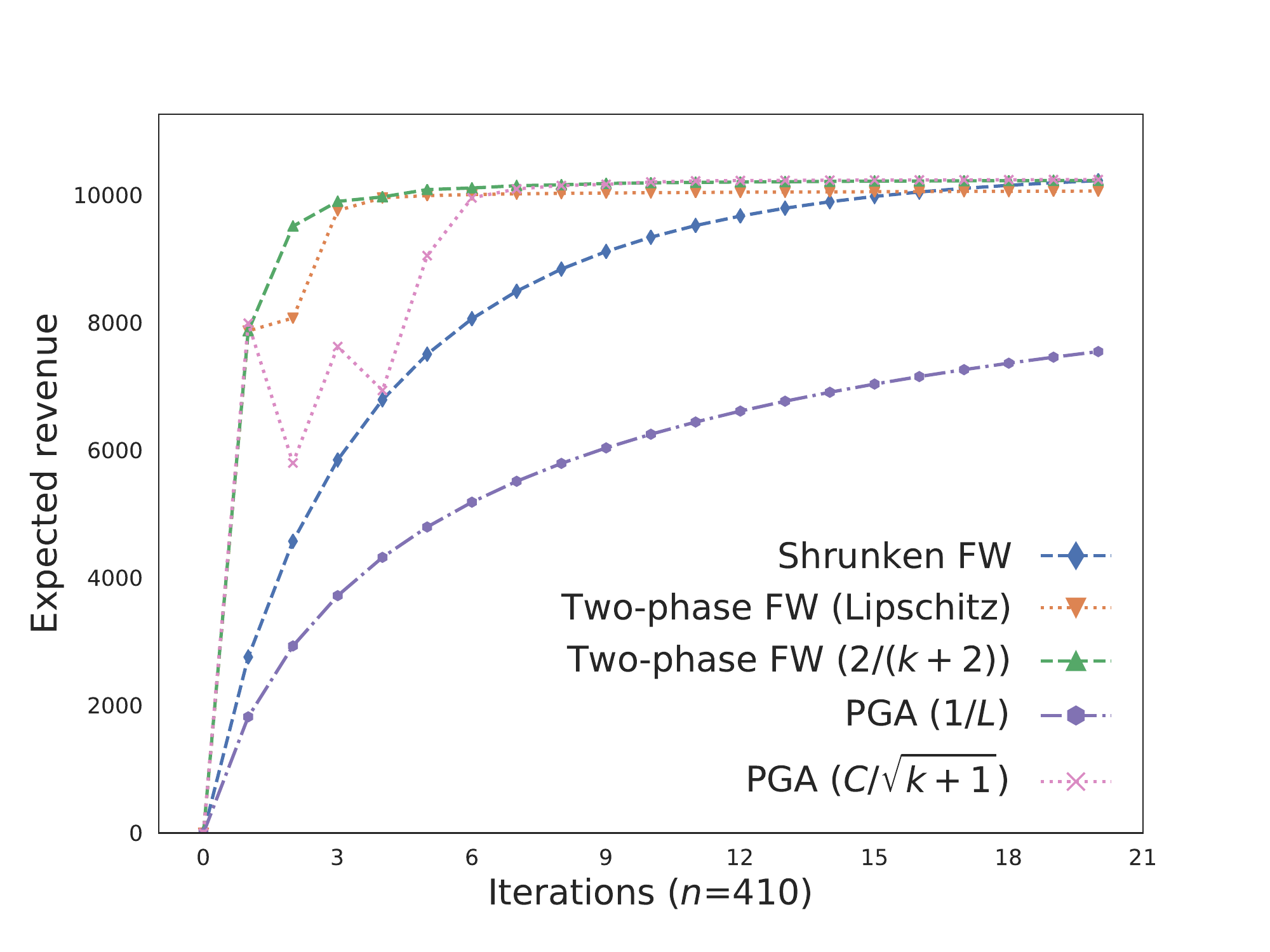}}
	\caption{Trajectory of different algorithms on real-world graphs.}
	\label{fig_traj_big_graph}
\end{figure}

\setkeys{Gin}{width=0.9\textwidth}
\begin{figure}[htbp]
	\center 
	\subfloat
	[``U. Rovira i Virgili'' dataset]
{
	\includegraphics[]{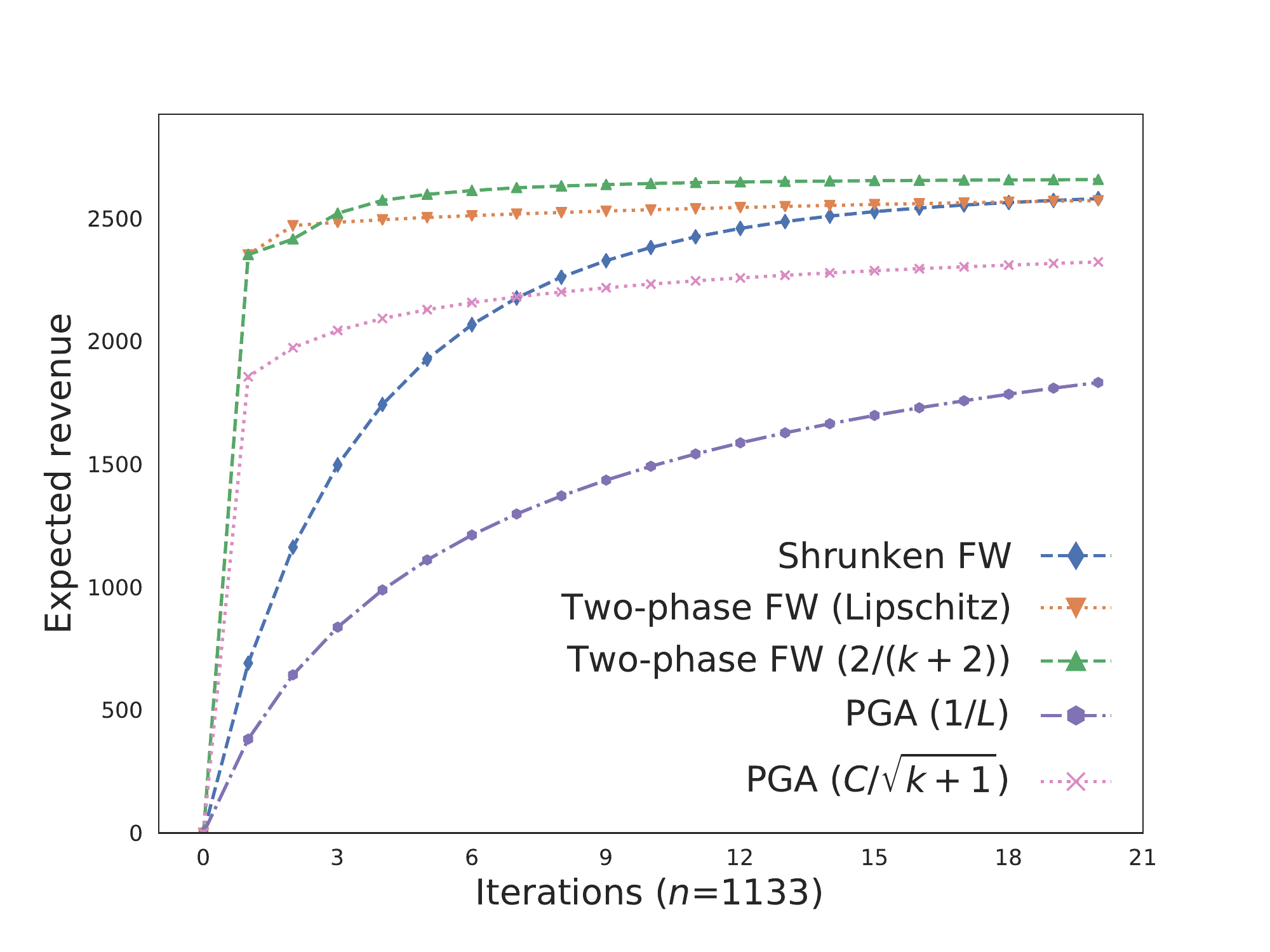}}	
	
	\subfloat
	[``ego Facebook'' dataset]
	{
		\includegraphics[]{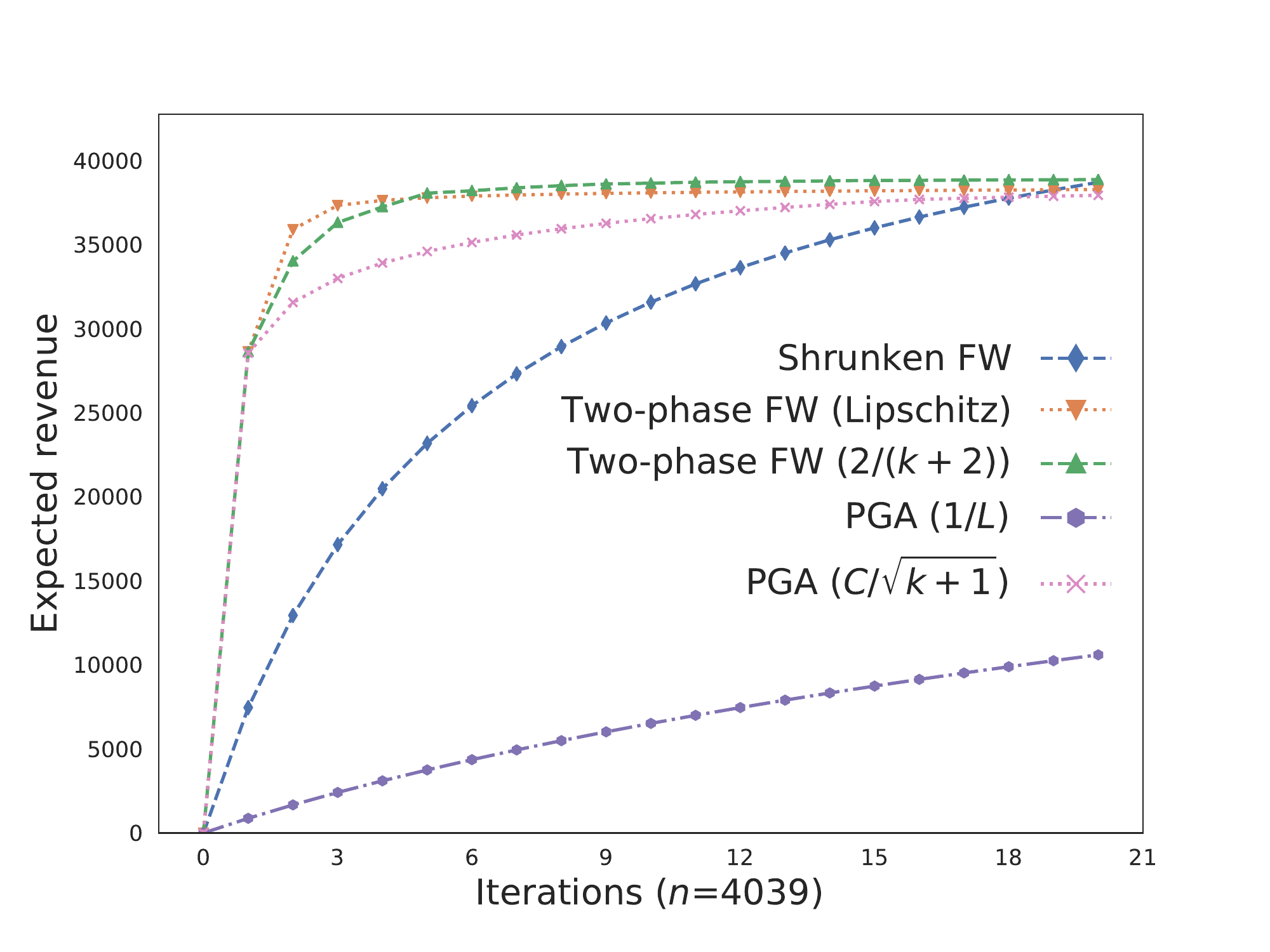}}
	\caption{Trajectories of different algorithms on real-world graphs.}
	\label{fig_traj_big_graph2}
\end{figure}

\section{Conclusions}

In this chapter we have investigated the problem of constrained
non-monotone DR-submodular maximization with a down-closed convex
constraint. We proposed two different algorithms for solving this
problem: the \twophase algorithm with a $1/4$ approximation guarantee
and the \shrunkenfw with a $1/e$ approximation guarantee.  We
extensively demonstrated the efficacy of the proposed algorithms over
the problems of DPP MAP inference and revenue maximization with
continuous assignments.

\section{Additional Proofs}

\subsection{Proof of \cref{rate_local_fw}}

\begin{proof}[Proof of \cref{rate_local_fw}]
	
Let $g_{\P}(\x), g_{\Q}(\z)$ to the non-stationarity of $\x$ and
$\z$, respectively. Since we are using 
the \nonconvexfw (\cref{nonconvex_fw}) as
subroutine, according to \citet[Theorem 1]{lacoste2016convergence}, one can  get,
\begin{align} 
&	g_{\P}(\x) \leq \min\left\{\frac{\max \{2h_1, C_f(\P)\}}{\sqrt{K_1+1}} , \epsilon_1\right\}, \\
&	 g_{\Q}(\z) \leq  \min\left\{\frac{\max \{2h_2, C_f(\Q)\}}{\sqrt{K_2+1}} , \epsilon_2 \right\}.
\end{align}
Plugging the above into \cref{local_global} we reach the  conclusion in \labelcref{eq_local_rates}.
\end{proof}

\subsection{Detailed Proofs for \cref{thm-e}}

\subsubsection{Proof of \cref{prop_non_fw}}

\restalemmatwo*

\begin{proof}[Proof of \cref{prop_non_fw}]
	We prove  by induction. 	
	First of all, it holds when $k=0$, since $x_i^\pare{0}=0$,
	and $t^\pare{0}=0$ as well. 	
	Assume it holds for $k$. Then for $k+1$, we have
	\begin{align} 
	x_i^\pare{k+1} & = x_i^\pare{k} + \gamma v_i^\pare{k}\\
	& \leq x_i^\pare{k} + \gamma ({\bar u_i} - x_i^\pare{k}) \quad \text{(constraint of shrunken LMO)}\\\notag 
	& = (1-\gamma) x_i^\pare{k} + \gamma {\bar u_i}\\
	& \leq (1-\gamma){\bar u_i}[1- (1-\gamma)^{t^\pare{k}/\gamma} ]+ \gamma {\bar u_i} \quad \text{ (induction) } \\\notag 
	& =  \bar u_i [1- (1-\gamma)^{t^\pare{k+1}/\gamma}].
	\end{align}
\end{proof}

\subsubsection{Proof of \cref{lem_nonmonotone_fw}}

\restalemmathree*

\begin{proof}[Proof of \cref{lem_nonmonotone_fw}]
	
	Consider $r(\lambda)= \x^* + \lambda(\x\vee \x^* - \x^*)$, it is easy to
	see that $r(\lambda)\geqco 0, \forall \lambda \geq 0$. 
	
	Notice that $\lambda'\geq 1$. 
	Let $\y = \r(\lambda') =  \x^* + \lambda'(\x\vee \x^* - \x^*)$, it is easy to see that $\y \geqco 0$, it also holds that $\y\leqco \bar u$: Consider one coordinate $i$, 1) if $x_i\geq x_i^*$, then $y_i = x_i^* + \lambda'(x_i - x_i^*)\leq \lambda'x_i \leq \lambda'\theta_i \leq \bar u_i$; 2)  if $x_i< x_i^*$, then $y_i = x_i^* \leq \bar u_i$. So $f(\y) \geq 0$. 
	
	Note that 
	\begin{align}
	\x\vee \x^* = (1-\frac{1}{\lambda'})\x^* + \frac{1}{\lambda'}\y = (1-\frac{1}{\lambda'})r(0) + \frac{1}{\lambda'}r(\lambda'), 
	\end{align}
	
	since $f$ is concave along $r(\lambda)$, so it holds that,
	\begin{align}
	f(\x\vee \x^*) \geq  (1-\frac{1}{\lambda'})f(\x^*) +  \frac{1}{\lambda'}f(\y) \geq (1-\frac{1}{\lambda'})f(\x^*).
	\end{align}
\end{proof}

\subsubsection{Proof of \cref{thm-e}}
\label{app_subsec_thm2_proof}

\begin{proof}[Proof of \cref{thm-e}]
	
  First of all, let us prove the \namecref{claim3_1}:
	
  \restaclaimthree*
	
  \begin{proof}[Proof of \cref{claim3_1}]
    Consider a point
    $\z^\pare{k}:= \x^\pare{k}\vee \x^* - \x^\pare{k}$, one can
    observe that: 1) $\z^\pare{k}\leqco \bar \u -\x^\pare{k}$; 2)
    since $\x^\pare{k}\geqco \zero, \x^*\geqco \zero$, so
    $\z^\pare{k}\leqco \x^*$, which implies that $\z^\pare{k}\in \P$
    (from down-closedness of $\P$). So $\z^\pare{k}$ is a candidate
    solution for the shrunken LMO (Step \labelcref{new_lmo} in \cref{fw-non-monotone}). We have,
    \begin{flalign}
      f(\x^{\pare{k+1}}) - f(\x^{\pare{k}}) & \geq \gamma\dtp{\nabla
        f(\x^\pare{k})}{\v^\pare{k}} - \frac{L}{2}\gamma^2
      \|\v^\pare{k}\|^2  (\text{Quadratic lower bound
        of \labelcref{eq_quad_lower_bound}}) \\ & \geq
      \gamma\dtp{\nabla f(\x^\pare{k})}{\v^\pare{k}} -
      \frac{L}{2}\gamma^2 D^2 \quad (\text{diameter of } \P) \\
      & \geq \gamma \dtp{\nabla f(\x^\pare{k})}{\z^\pare{k}} -
      \frac{L}{2}\gamma^2 D^2\quad (\text{shrunken LMO})\\  &
      \geq \gamma(f(\x^\pare{k}+\z^\pare{k}) - f(\x^\pare{k})) -
      \frac{L}{2}\gamma^2 D^2 \quad (\text{concave along
        $\z^\pare{k}$})\\ & = \gamma [f(\x^\pare{k}\vee \x^*) -
      f(\x^\pare{k})] - \frac{L}{2}\gamma^2 D^2\\ & \geq \gamma
      [(1-\frac{1}{\lambda'})f(\x^*) - f(\x^\pare{k})] -
      \frac{L}{2}\gamma^2 D^2 \quad (\text{\cref{lem_nonmonotone_fw}})
      \\ & =\gamma [ (1-\gamma)^{t^\pare{k}/\gamma} f(\x^*) -
      f(\x^\pare{k})] - \frac{L}{2}\gamma^2 D^2,
    \end{flalign}
    where the last equality comes from {setting }
    $\bmtheta : = \bar \u(1-(1-\gamma)^{t^\pare{k}/\gamma})$ {
      according to \cref{prop_non_fw}}, thus
    $\lambda' = \min_i \frac{\bar u_i}{\theta_i} =
    (1-(1-\gamma)^{t^\pare{k}/\gamma})^{-1}$.
		
    After rearrangement, we reach the claim.
  \end{proof}
  Then, let us prove \cref{thm-e} by \emph{induction}.
	
  First of all, it holds when $k = 0$ (notice that
  $t^\pare{0}=0$). Assume that it holds for $k$. 
  
  Then for $k+1$,
  considering the fact $e^{-t} - O(\gamma)\leq (1-\gamma)^{t/\gamma}$
  when $0< \gamma\leq t \leq 1$ and \cref{claim3_1} we get,
  \begin{align}
    & f(\x^{\pare{k+1}})\\ 
    &  \geq (1-\gamma)  f(\x^{\pare{k}})   +
      \gamma(1-\gamma)^{t^\pare{k}/\gamma} f(\x^*) -\frac{L
      D^2}{2}\gamma^2\\ 
    & \geq  (1-\gamma)  f(\x^{\pare{k}})   + \gamma [e^{-t^\pare{k}} -
      O(\gamma)] f(\x^*) -\frac{L D^2}{2}\gamma^2\\\notag  
    & \geq  (1-\gamma) [ t^\pare{k} e^{-t^\pare{k}}f(\x^*) - \frac{L
      D^2}{2}k\gamma^2 - O(\gamma^2)f(\x^*)]+ \gamma [e^{-t^\pare{k}}
      - O(\gamma)] f(\x^*) -\frac{L D^2}{2}\gamma^2\\\notag 
    & = [(1-\gamma) t^\pare{k} e^{-t^\pare{k}} + \gamma
      e^{-t^\pare{k}}   ]f(\x^*)  - \frac{L D^2}{2}\gamma^2
      [(1-\gamma)k + 1] - [(1-\gamma) O(\gamma^2) + \gamma
      O(\gamma)]f(\x^*)\\\label{eq_30} 
    & \geq  [(1-\gamma) t^\pare{k} e^{-t^\pare{k}} + \gamma
      e^{-t^\pare{k}}   ]f(\x^*) -  \frac{L D^2}{2}\gamma^2(k+1) -
      O(\gamma^2)f(\x^*). 
  \end{align}
  Let us consider the term
  $ [(1-\gamma) t^\pare{k} e^{-t^\pare{k}} + \gamma e^{-t^\pare{k}}
  ]f(\x^*)$.
  We know that the function $g(t) = te^{-t}$ is concave in $[0, 2]$,
  so
  $g(t^\pare{k}+\gamma) - g(t^\pare{k}) \leq \gamma g'(t^\pare{k})$,
  which amounts to,
  \begin{align}
    [(1-\gamma) t^\pare{k} e^{-t^\pare{k}} + \gamma e^{-t^\pare{k}}
    ]f(\x^*) & \geq (t^\pare{k} +\gamma) e^{-(t^\pare{k} +\gamma)}
    f(\x^*)\\\label{eq_34}
    &= t^{\pare{k+1}} e^{-t^{\pare{k+1}}} f(\x^*).
  \end{align}
  Plugging \cref{eq_34} into \cref{eq_30} we get,
  \begin{align}
    f(\x^{\pare{k+1}})    \geq t^{\pare{k+1}} e^{-t^{\pare{k+1}}}
    f(\x^*) -  \frac{L D^2}{2}\gamma^2(k+1) - O(\gamma^2)f(\x^*). 
  \end{align}
  Thus proving the induction, and proving the theorem as well.
\end{proof}

%

\def\dir{chapters/maxcut}
\chapter{Validating Greedy \MAXCUT\ Algorithms}
\label{chapter_greedy_maxcut}

\begin{chapquote}{Richard P. Feynman}
We are trying to prove ourselves wrong as quickly as possible, because only in that way can we find progress.
\end{chapquote}

\MAXCUT\ defines a classical NP-hard problem for graph partitioning
and it serves as a typical instance of the  non-monotone
Unconstrained Submodular Maximization (USM) problem.  Greedy
algorithms to approximately solve \MAXCUT\ rely on greedy vertex
labelling or on an edge contraction strategy. These algorithms have
been studied by measuring their approximation ratios in the worst case
setting but very little is known to characterize their robustness to
noise contaminations of the input data in the average case.

Adapting the framework of Approximation Set Coding of
\citet{Buhmann10isit}, we present a method to exactly measure the
cardinality of the algorithmic approximation sets of five greedy
\MAXCUT\ algorithms. Their information contents are explored for graph
instances generated by two different noise models: the edge reversal
model and Gaussian edge weights model. The results provide insights
into the robustness of different greedy heuristics and techniques for
\MAXCUT, which may be used for algorithm design of general USM
problems.

\section{Why Validating Greedy \MAXCUT\ Algorithms?}

Algorithms are mostly analyzed by measuring their runtime and memory
consumption for the worst possible input instance.  The robustness of
an algorithm to input fluctuations is rarely investigated although
such a property might often be indispensable in applications.  In
these scenarios, algorithms are also selected according to their
``robustness'' to noise perturbations of the input instance and their
insensitivity to randomization during algorithm execution.

Taking the \MAXCUT\ problem for example, in practice, instead of
having the graph $G$ as input to recover the maximal cut, one usually
only has access to multiple noisy observations of the graph $G$.
Assuming for simplicity, there are two noisy observations of the
underlying master graph $G$: $G'$ and $G''$, and we want to recover
the maximal cut with respect to $G$.  The ability of an algorithm to
recover the true solutions given only noisy observations is closely
related to the robustness/informativeness of the algorithm.

How should this ``robustness'' property be measured? Machine learning
requires that algorithms with random variables as input generalize
over these fluctuations. The algorithmic answer has to be stable
with respect to this uncertainty in the input instance. Approximation Set
Coding (ASC) quantifies the impact of input randomness on the solution
space of an algorithm by measuring the attainable resolution for the
algorithm's output.  In this chapter we employ this framework in an
exemplary way by estimating the robustness of \MAXCUT\ algorithms to
specific input instances. Thereby, we effectively perform an average
case analysis of the generalization properties of greedy \MAXCUT
algorithms.

\subsection{\MAXCUT  and Unconstrained Submodular Maximization}

Given an undirected graph $G=(\gndset, E; \BW)$ with vertex set
$\gndset=\{v_1,..., v_n\}$ and edge set $E$ with nonnegative weights
$w_{ij}, \forall (i,j)\in E$, the \MAXCUT\ problem aims to find a
partition of vertices into two disjoint subsets $S_1$ and $S_2$, such
that the cut value
$cut(S_1, S_2):=\sum_{i\in S_1} \sum_{j\in S_2} w_{ij}$ is maximized.

\MAXCUT\ is employed in various applications, such as in
semi-supervised learning (\cite{Wang:2013:SLU:2502581.2502605}), in
social network (\cite{agrawal2003mining}), in statistical physics and
in circuit layout design (\cite{barahona1988application}).  \MAXCUT\
is considered to be a typical case of the USM problem because its
objective can be formulated as a set function:
$f(S):=cut(S, \gndset\setminus S)
, S\subseteq \gndset$,
which is submodular, non-monotone, and symmetric
($f(S) = f(\gndset\setminus S)$).  Beside \MAXCUT, USM captures many
practical problems such as \textsc{MaxDiCut}
\citep{halperin2001combinatorial}, variants of \textsc{MaxSat} and the
maximum facility location problem
\citep{cornuejols1977uncapacitated,ageev19990}.

\subsection{Greedy Heuristics and Techniques}

The five algorithms investigated in this chapter (as summarized in
\cref{tab-alg-summarization}) belong to two greedy
\textit{heuristics}: double greedy and backward
greedy. The \textit{double} greedy algorithms conduct classical
forward greedy and backward greedy simultaneously: it works on two
solutions initialized as $\emptyset$ and the ground set \gndset,
respectively, then processes the
elements (vertices for \MAXCUT\ problem) one at a time, for which it
determines whether it should be added to the first solution or removed
from the second solution.  The \textit{backward} greedy algorithm
removes the smallest weighted edge in each step.  The difference of
the four double greedy algorithms originates from the greedy
{techniques} they use: sorting, randomization and the way to
initialize the first two vertices.

\subsection{Approximation Set Coding for Algorithm Analysis}

In analogy to Shannon's theory of communication, the ASC framework
(\cite{Buhmann10isit}, \cite{JB:mcpr:2011}, \cite{buhmann2013simbad})
determines distinguishable sets of solutions and, thereby, provides a
general principle to conduct model validation
(\cite{DBLP:dblp_journals/jmlr/ChehreghaniBB12},
\cite{zhousparse}). As an algorithmic variant of the ASC framework,
\citet{busse2012information, informativemst} define the
\textit{algorithmic $t$-approximation set} of an algorithm
$\mathscr{A}$ at step $t$ as the set of feasible solutions after $t$
steps, $C_t^{\mathscr{A}}(G)$, which  is the solution
set that  are still considered as viable by $\mathscr{A}$ after $t$
computational steps.

 Since we
investigate the average case behavior of algorithms, we have to
specify the probability distribution of the input instances. 
ASC follows the two-instance scenario (as shown in \cref{sec_intro_alg_information}) to generate the graph instances: 
First,  generate a
``master graph'' $G$, e.g., a complete graph with Gaussian distributed edge
weights.
In a second step, we generate two input graphs
$G^\prime,\;G^{\prime\prime}$ by independently applying a noise
process to edge weights of the master graph $G$.  
With an abuse of notation, we use $\graphrv, \graphrv'$ and $\graphrv''$ to denote the corresponding  random variables in this generative process.

The algorithmic analogy of \textit{information content}
(\cite{Buhmann10isit,informativemst}), i.e., algorithmic information
content $I^{\mathscr{A}}(\graphrv'; \graphrv'')$, is
computed as the maximum stepwise information
$I_t^{\mathscr{A}}(\graphrv'; \graphrv'')$:
\begin{align}\label{eq:ic}
  &  I^{\alg}(\graphrv'; \graphrv'') := 
    \max_{t}  I_t^{\alg} (\graphrv'; \graphrv'')\\
  &= 
    \max_t \E_{G', G''} \left[ 
    \log \left( |\C| \frac{
    |\Delta_t^{\mathscr{A}}(G',G'')|}{|C_t^{\mathscr{A}}(G')|
    |C_t^{\alg}(G'')|} \right)  
    \right],
\end{align}
where
$\Delta_t^{\alg}(G',G'') := C_t^{\mathscr{A}}(G') \cap
C_t^{\mathscr{A}}(G'')$
denotes the intersection of approximation sets, and
$\C$ is the solution space, i.e., all possible cuts.
The information content $ I^{\mathscr{A}}_t(\graphrv'; \graphrv'')$ measures how much
information is extracted by algorithm ${\mathscr{A}}$ at iteration $t$
from the input data that is relevant to the output
data.

\begin{table}
\begin{center}
\normalsize
\caption{Summary of Greedy \MAXCUT\   Algorithms \citep{ITW15_BianGB}}
\label{tab-alg-summarization}
\begin{tabular}{|c|c|c|c|c|c|}
\hline
\multirow{2}{*}{Name}  & Greedy     & \multicolumn{3}{c|}{Techniques}  \\
   \cline{3-5}
     & Heuristic &  Sorting & Randomization & Init. Vertices \\
  \hline
  \hline
  D2Greedy  & \multirow{4}{*}{Double} &  &  &\\
  \cline{1-1}  \cline{3-5}
  RDGreedy  &  &  & $\checkmark$  &\\
  \cline{1-1}  \cline{3-5}
  SG & &  &  & $\checkmark$\\
  \cline{1-1}  \cline{3-5}
  SG3  &  &  $\checkmark$  &   & $\checkmark$\\
  \hline
  EC & Backward & $\checkmark$  &  &\\
  \hline
\end{tabular}
\end{center}
\end{table}

\section{Greedy \MAXCUT\  Algorithms}

We investigate five greedy algorithms (\cref{tab-alg-summarization}) for \MAXCUT. According to the type of
greedy heuristic, they can be
divided into two categories: I)~\textit{Double Greedy}: SG, SG3,
D2Greedy, RDGreedy; II)~\textit{Backward Greedy}: Edge Contraction.

Besides the type of greedy heuristic, the difference between the
algorithms are mainly in three techniques: \textit{sorting} the
candidate elements, \textit{randomization} and the way
\textit{initializing the first two vertices}. In the following, we
briefly introduce one typical algorithm in each category and we
present details of the others  in \cref{sup:alg}.

\subsection{Double Greedy Algorithms}

D2Greedy (\cref{alg:d-usm-4-max-cut}) is the
{\textbf{D}}eterministic double greedy, RDGreedy is the
\textbf{R}andomized double greedy,
they were proposed by \citet{buchbinder2012tight} to solve the general
USM problem with ${1/3}$ and ${1/2}$ worst-case approximation
guarantee, respectively.  They use the same double greedy heuristic as
SG \citep{sahni1976p} and SG3 (presented in \citet{kahruman2007greedy}, it is a variant of SG), which are classical
greedy \MAXCUT\ algorithms.  
We prove in  \cref{app:equivalence-SG-D2Greedy} that, for
\MAXCUT, SG and D2Greedy use equivalent labelling criteria except for
initializing the first two vertices.

\begin{algorithm}
  \caption{D2Greedy \citep{buchbinder2012tight}}\label{alg:d-usm-4-max-cut}
\KwIn{Complete graph $G=(\gndset, E; \BW)$ with nonnegative edges}
\KwOut{A disjoint cut  and the cut value}
{$S^0 :=\emptyset$, $T^0 := \gndset$}\;
\For{$i=1$ to $n$}{
    {$a^i := f(S^{i-1} \cup \{v_i\})-f(S^{i-1})$}\;
    {$b^i := f(T^{i-1} \backslash \{v_i\})-f(T^{i-1})$}\;
    \If{$a^i\geq b^i$}
        {$S^i :=S^{i-1} \cup \set{v_i}$, $T^{i}:=T^{i-1}$ \tcp*[r]{expand $S$}} 
    \Else{$S^i :=S^{i-1}$, $T^{i}:=T^{i-1} \backslash \set {v_i}$ \tcp*[r]{shrink $T$}}
}
\KwRet{$S^n$, $\gndset \backslash S^n$, and $cut(S^n, \gndset \backslash S^n)$} 
\end{algorithm}

As shown in \cref{alg:d-usm-4-max-cut}, D2Greedy maintains two
solution sets: $S$ initialized as $\emptyset$, $T$ initialized as the
ground set \gndset. It labels all the vertices one by one: for vertex
$v_i$, it computes the objective gain of adding $v_i$ to $S$ and the
gain of removing $v_i$ from $T$, then labels $v_i$ to have higher
objective gain.

SG and D2Greedy differ in the initialization of the first two
vertices: SG picks first of all the maximum weighted edge and
distributes its two vertices to the two active subsets. Compared to
D2Greedy, the RDGreedy algorithm uses randomization technique when
labelling each vertex: it labels each vertex with probability
proportional to the objective gain. Compared to SG, SG3 sorts the
unlabelled vertices according to a certain score function (which is
proportional to the possible objective gains), and selects the vertex
with the maximum score to be the next one to be labelled.

\subsection{The Edge Contraction (EC) Algorithm}
\label{sec:ec}

EC \citep{kahruman2007greedy}, which is summarized in  \cref{alg:edge-contraction},  contracts
the smallest edge in each step. The two vertices of this contracted
edge become one ``super'' vertex, and the weight of an edge connecting
this super vertex to any other vertex is assigned as the sum of
weights of the original two edges.  EC belongs to the backward greedy
in the sense that it tries to remove the least expensive edge from the
cut set in each step.  We can easily derive a heuristic for the
{\textsc{Max-k-Cut}} problem by using $n-k$ steps instead of $n-2$
steps.

\begin{algorithm}
\caption{Edge Contraction (EC) \citep{kahruman2007greedy}}\label{alg:edge-contraction}
\KwIn{ Complete graph $G=(\gndset, E; \BW)$ with nonnegative edge weights}
\KwOut{A disjoint cut $S_1, S_2$ and cut value $cut(S_1, S_2)$}
\For{$i=1:n$}
    {$ContractionList(i):= \{i\}$\;}
\For{$i=1:n-2$}{
    {Find a minimum weight edge $(x, y)$ in $G$}\; %
    {$v:=contract(x,y)$, $\gndset:=\gndset\cup\{v\}\backslash\{x,y\}$ \tcp*[r]{contract}}   
    \For{$j\in \gndset\backslash \{v\}$}
        {$w_{vj}:= w_{xj}+w_{yj}$\;}
    {$ContractionList(v) := ContractionList(x) \cup ContractionList(y)$}\;
}
{Denote by $x$ and $y$ the only two vertices in $\gndset$}\;
\Return{$S_1:=ContractionList(x)$, 
$S_2:=ContractionList(y)$,
$cut(S_1, S_2):=w_{xy}$}
\end{algorithm}

\section{Counting Solutions in Approximation Sets}
\label{sec:counting}

To compute the information content according to \cref{eq:ic}, we need
to exactly compute the cardinalities of different solution sets. For
\MAXCUT{} problem, the solution space has the cardinality
$|\C| = 2^{n-1} - 1$. In the following we will present guaranteed
methods for \textit{exact} counting
$|C_t^{\mathscr{A}}(G')|, |C_t^{\mathscr{A}}(G'')|$ and
$|\Delta_t^{\mathscr{A}}(G',G'')|$ (sub-/superscripts omitted for
notational clarity).

\subsection{Counting Methods for Double Greedy Algorithms}
\label{sec:counting-sg3}

The counting methods for the double greedy algorithms are similar, so
we only discuss the method for SG3 here; details about other methods
and the corresponding proofs are in \cref{complement:counting} and
\cref{app:proof-SG3}, respectively.

For the SG3 (\cref{alg:sg3}), after step $t$ ($t = 1, ..., n-1$) there
are $k = n-t-1$ unlabelled vertices, and it is clear that
$|C(G')|=|C(G'')|=2^{k}$.

To count the intersection set $\Delta (G',G'')$, assume the solution
set pair of $G'$ is $(S_1', S_2')$, the solution set pair of $G''$ is
$(S_1'', S_2'')$, so the unlabelled vertex sets are
$T'=\gndset\backslash \{S_1' \cup S_2'\}$,
$T''=\gndset\backslash \{S_1'' \cup S_2''\}$, respectively.  Denote
$L:=T'\cap T''$ be the common vertices of the two unlabelled vertex
sets, so $l=|L|$ ($0\leq l\leq k$) is the number of common vertices in
the unlabelled $k$ vertices. Denote $M' :=T'\backslash L$,
$M'' :=T''\backslash L$ be the sets of different vertex sets between
the two unlabelled vertex sets.
Then,
\begin{equation}\notag
   | \Delta (G',G'') | = 
  \left\{\begin{array}{ll}
           \multirow{2}{*}{$2^l$}  & \textrm{if
                                     $(S_1''\backslash M', S_2''\backslash M')$ is
                                     matched by}\\
                                   & \textrm{$(S_1'\backslash M'', S_2'\backslash
                                     M'')$  or $(S_2'\backslash M'', S_1'\backslash
                                     M'')$}\\  
           0 & \textrm{otherwise.}\\
         \end{array}\right.
     \end{equation}

\subsection{Counting Method for the Edge Contraction Algorithm}
For the EC (\cref{alg:edge-contraction}), after step $t$
($t = 1,..., n-2$) there are $k = n-t$ ``super'' vertices
(i.e.~contracted ones).  It is straightforward to see that
$|C(G')|=|C(G'')|=2^{k-1}-1$.

To count the intersection $\Delta (G',G'')$, suppose there are $l$
($0\leq l\leq k$) common super vertices in the unlabelled $k$
vertices. Remove the $l$ common super vertices from each set, then
there are $h = k-l$ distinct super vertices in each set, denote them
by $P:=\{\mathbf{p}_1,\mathbf{p}_2,\cdots,\mathbf{p}_h\}$,
$Q:=\{\mathbf{q}_1,\mathbf{q}_2,\cdots,\mathbf{q}_h\}$,
respectively. Notice that
$\mathbf{p}_1\cup\mathbf{p}_2 \cup\cdots \cup \mathbf{p}_h =
\mathbf{q}_1 \cup\mathbf{q}_2\cup \cdots \cup\mathbf{q}_h$,
so after some contractions in both $P$ and $Q$, there exist some
common super vertices between $P$ and $Q$. Assume the maximum number 
of common super vertices after all possible contractions is $c^*$,
then it holds
\begin{equation}\label{}
 | \Delta (G',G'') | =2^{c^*+l-1}-1\;.
\end{equation}
To compute $c^*$, we propose a polynomial time algorithm
(\cref{alg:max-common}) with a theoretical guarantee in
\cref{theo:ec}, the proof is deferred to \cref{app:proof-EC}. The
algorithm finds the maximal number of common super vertices after all
possible contractions, that is used to count the volume of  $\Delta (G',G'')$ for EC.
\begin{theorem}[\cite{ITW15_BianGB}]
\label{theo:ec}
Given two distinct super vertex sets
$P:=\{\mathbf{p}_1,\mathbf{p}_2,\cdots,\mathbf{p}_h\}$,
$Q:=\{\mathbf{q}_1,\mathbf{q}_2,\cdots,\mathbf{q}_h\}$ (any two super
vertices inside $P$ or $Q$ do not intersect, and there is no common
super vertex between $P$ and $Q$), such that
$\mathbf{p}_1\cup\mathbf{p}_2 \cup\cdots \cup \mathbf{p}_h =
\mathbf{q}_1 \cup\mathbf{q}_2\cup \cdots \cup\mathbf{q}_h$,
\cref{alg:max-common} returns the maximum number of common super
vertices between $P$ and $Q$ after all possible contractions.
\end{theorem}
\begin{algorithm}
\begin{small}
  \caption{Common Super Vertex Counting \citep{ITW15_BianGB}
}\label{alg:max-common}
\KwIn{Two distinct super vertex sets $P$, $Q$ } \KwOut{Maximum number
  of common super vertices after all possible contractions} {$c:=0$\;}
\While{$P\neq \emptyset$}{ {Randomly pick $\mathbf{p}_i\in P$\;} {Find
    $\mathbf{q}_j\in Q$ s.t.
    $\mathbf{p}_i\cap \mathbf{q}_j\neq \emptyset$\;}
    \If{$\mathbf{q}_j\backslash\mathbf{p}_i\neq\emptyset$}{
        {For $\mathbf{p}_i$, find $\mathbf{p}_{i'}\in P\backslash
          \{\mathbf{p}_i\}$ s.t. $\mathbf{p}_{i'}\cap
          (\mathbf{q}_j\backslash\mathbf{p}_i)\neq\emptyset$\;}  
        {$\mathbf{p_{ii'}}:=\mathbf{p}_i\cup \mathbf{p}_{i'}$,
          $P:=P\cup\{\mathbf{p_{ii'}}\}\backslash \{\mathbf{p}_i,
          \mathbf{p}_{i'}\}$ %
        }\;         
    }
    \If{$\mathbf{p}_i\backslash\mathbf{q}_j\neq\emptyset$}{
        {For $\mathbf{q}_j$, find $\mathbf{q}_{j'}\in
          Q\backslash\{\mathbf{q}_j\}$ s.t. $\mathbf{q}_{j'}\cap
          (\mathbf{p}_i\backslash\mathbf{q}_j)\neq\emptyset$\;} 
        {$\mathbf{q_{jj'}}:=\mathbf{q}_j\cup \mathbf{q}_{j'}$,
          $Q:=Q\cup\{\mathbf{q_{jj'}}\}\backslash \{\mathbf{q}_j,
          \mathbf{q}_{j'}\}$ %
        }\; } \If{$\mathbf{p_{ii'}}==\mathbf{q_{jj'}}$}{ {Remove
          $\mathbf{p_{ii'}}$, $\mathbf{q_{jj'}}$ from $P$, $Q$,
          respectively\;} {$c:=c+1$\;}
    }
}
\Return{$c$}
\end{small}
\end{algorithm}

\section{Experiments}

We conducted experiments on two exemplary models: the edge reversal
model and the Gaussian edge weights model. Each model involves the
master graph $G$ and a noise type used to generate the two noisy
instances $G^\prime$ and $G^{\prime\prime}$. The width of the instance
distribution is controlled by the strength of the noise model. These
models provide the setting to investigate the algorithmic behavior.

\setkeys{Gin}{width=.9\textwidth, height=0.45\textheight}
\begin{figure}[htbp]
\begin{center}
\subfloat[Edge Reversal Model, $n$=100]{
\includegraphics{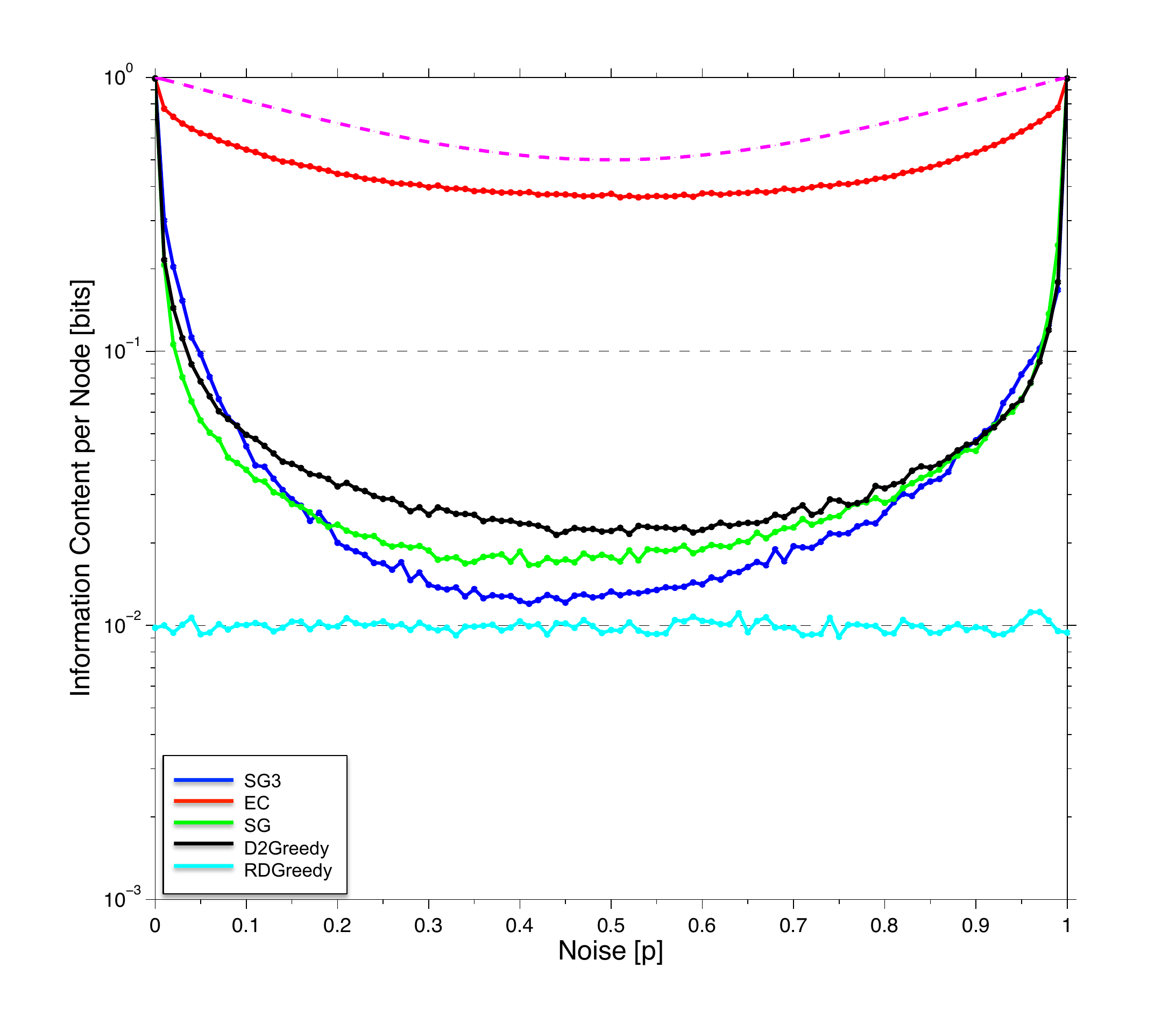}
}\\
\subfloat[Gaussian Edge Weights Model, $n$=100]{
\includegraphics{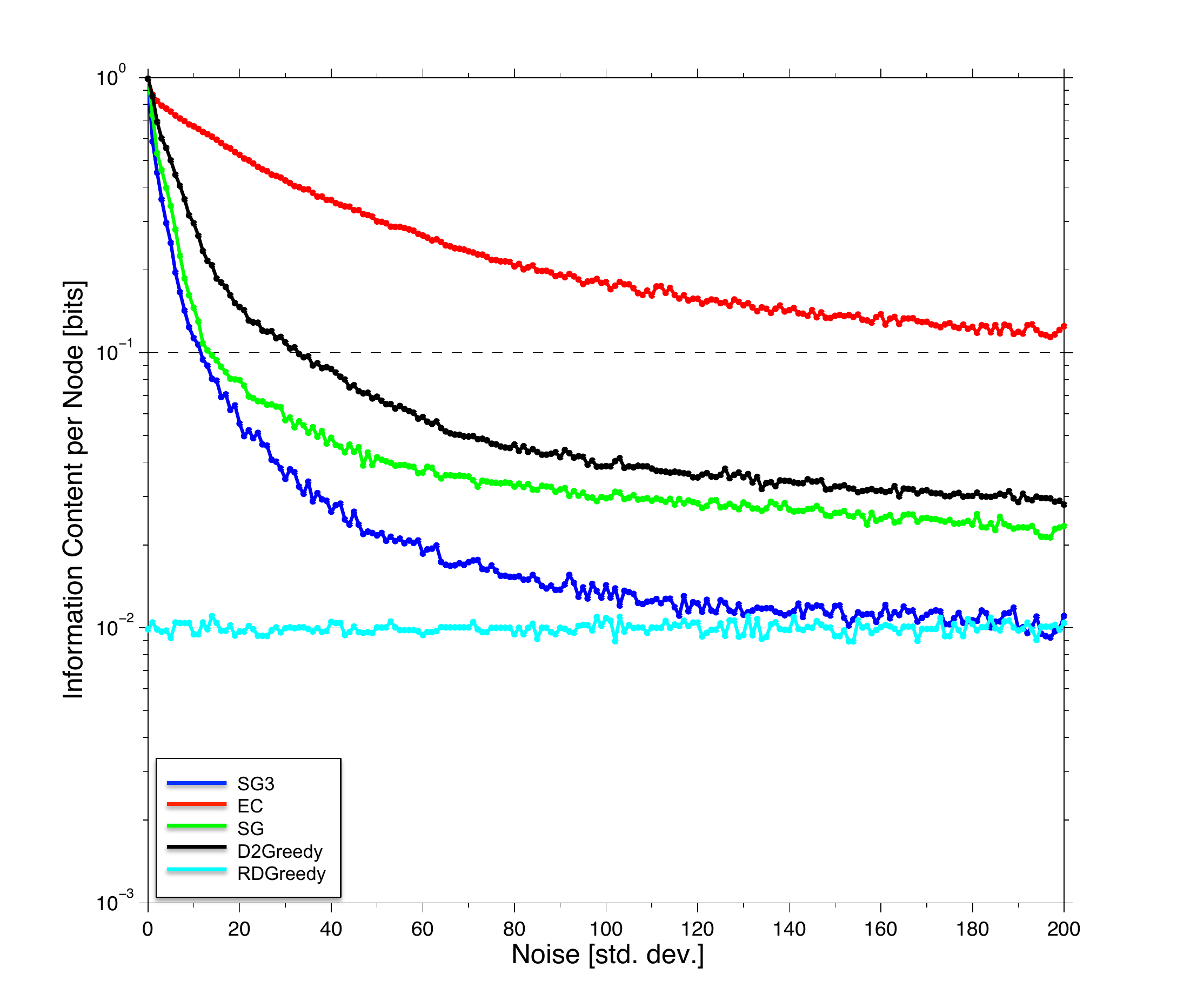}}
\end{center}
\caption{Information content per node.}
\label{fig:ic}
\end{figure}

\subsection{Experimental Setting}

\paragraph{Edge Reversal Model.}%
To obtain the master graph, we generate a balanced bipartite graph
$G_b$ with disjoint vertex sets $S_1$, $S_2$.  Then we assign
uniformly distributed weights in $[0, \frac{8}{n^2}]$ to all edges
inside $S_1$ or $S_2$ and we assign uniformly distributed weights in
$[1- \frac{8}{n^2}, 1]$ to all edges between $S_1$ and $S_2$, thus
generating graph $G'_b$. Then randomly flip edges in $G'_b$ to
generate the master graph $G$. Here, flip edge
$e_{ij}$ means changing its weight $w_{ij}$ to $1-w_{ij}$ with
probability $p_m$, and $(flip\; e_{ij} ) \sim Ber
(p_m)$; $p_m = 0.2$ is used to generate the master graph $G$.
Noisy graphs $G'$, $G''$ are generated by flipping the edges in $G$
with probability $p$, that is $(flip\; e_{ij} ) \sim Ber (p)$.

\paragraph{Gaussian Edge Weights Model.} %
The master graph $G$ is generated with Gaussian distributed edge
weights $w_{ij} \sim N (\mu, \sigma_m^2)$, $\mu = 600, \sigma_m = 50$,
negative edges are set to be $\mu$.  {Noisy graphs $G'$, $G''$} are
obtained by adding Gaussian distributed noise $n_{ij} \sim N (0,
\sigma^2)$, negative noisy edges are set to be $0$.

For both noise models, we conducted 1000 experiments on 
i.i.d. generated noisy graphs $G'$ and $G''$, and then we aggregated the
results to estimate the expectation in  \cref{eq:ic}.

\subsection{Results}

We plot the information content and stepwise information \textit{per
  node} in ~\cref{fig:ic} and ~\cref{fig:stepwise-info},
respectively. For the edge reversal model, we also investigate the
number of equal edge pairs between $G'$ and $G''$: $d = 0, ..., m$
($m$ is the total edge number), $d$ measures the consistency of the
two noisy instances.  The expected fraction of equal edge pairs is
$\mathbb{E} [d] = p^2 + (1 - p)^2$, and it is plotted as the dashed
magenta line in \cref{fig:ic}(a).

\subsection{Analysis}

Before discussing these results, let us revisit the stepwise
information and information content.  From the counting methods in
 \cref{sec:counting}, we derive the analytical form of $|\C|$,
$|C_t^{\mathscr{A}}(G')|$ and $ |C_t^{\mathscr{A}}(G'')|$ (e.g., 
$\mathscr{A}=SG3$), and we insert these values into the definition of
stepwise information,
\begin{align}\label{eq:stepwise-sg3}
  &I_t^{\alg } = \mathbb{E} \log \Bigl(|\C|
  \frac{|\Delta_t^{\mathscr{A}}(G',G'') |}{ |C_t^{\mathscr{A}}(G')|  |C_t^{\mathscr{A}}(G'')|} \Bigr)\\
   & = \mathbb{E} (\log(|\C| |\Delta_t^{\mathscr{A}}(G',G'') |) -
  \log(|C_t^{\mathscr{A}}(G')|  |C_t^{\mathscr{A}}(G'')|))\\ 
   & = \mathbb{E} \log |\Delta_t^{\mathscr{A}}(G',G'')| + 2t +
  \log(2^{n-1} - 1) -2(n - 1).
\end{align}

\setkeys{Gin}{width=0.9\textwidth,height=0.45\textheight}
\begin{figure}[htbp]
\begin{center}
  \subfloat[Edge Reversal Model, $n=100$ , $p = 0.65$]{
    \includegraphics{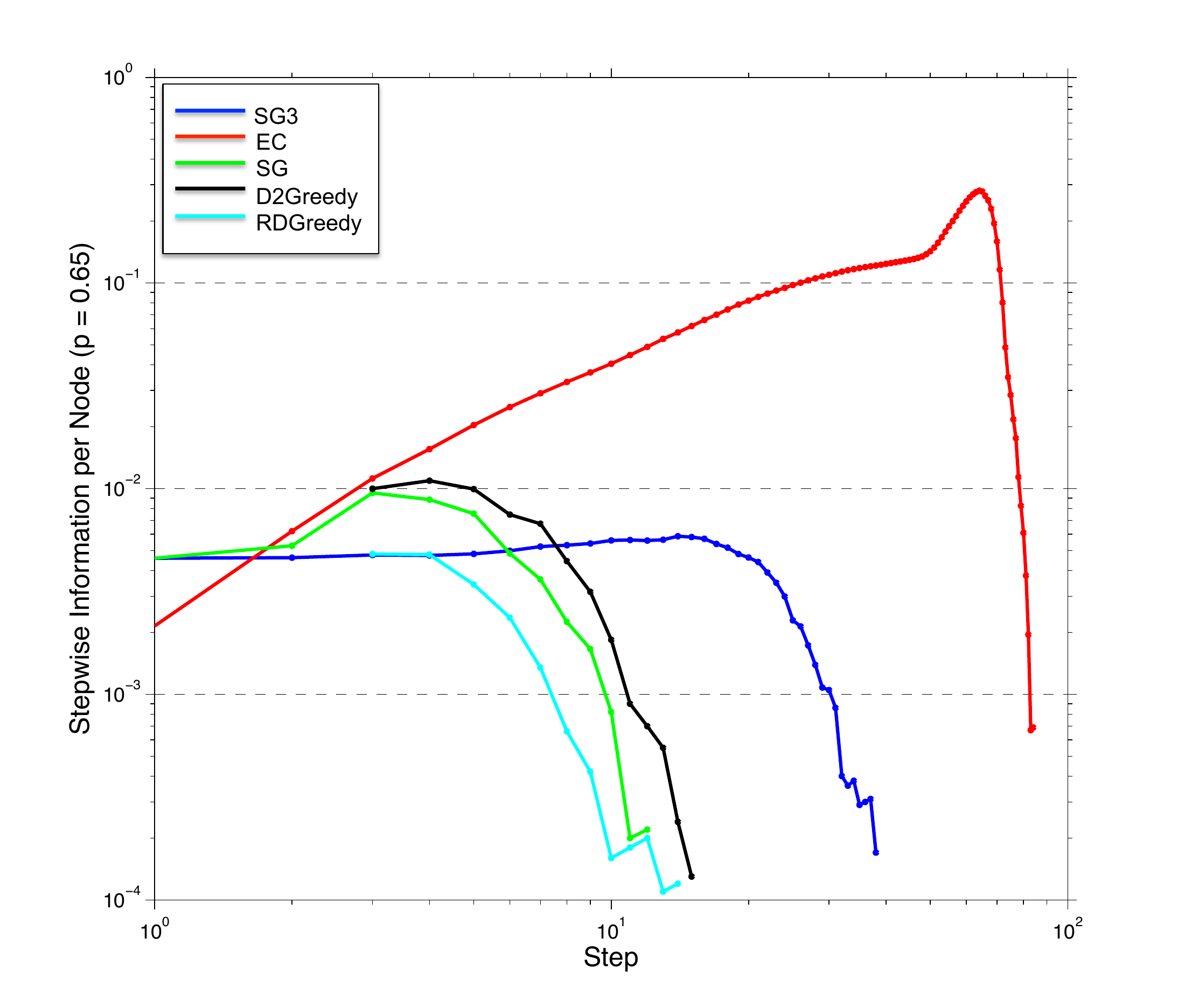}}\\
  \subfloat[Gaussian Edge Weights Model, $n=100$, $\sigma=125$ ]{
    \includegraphics{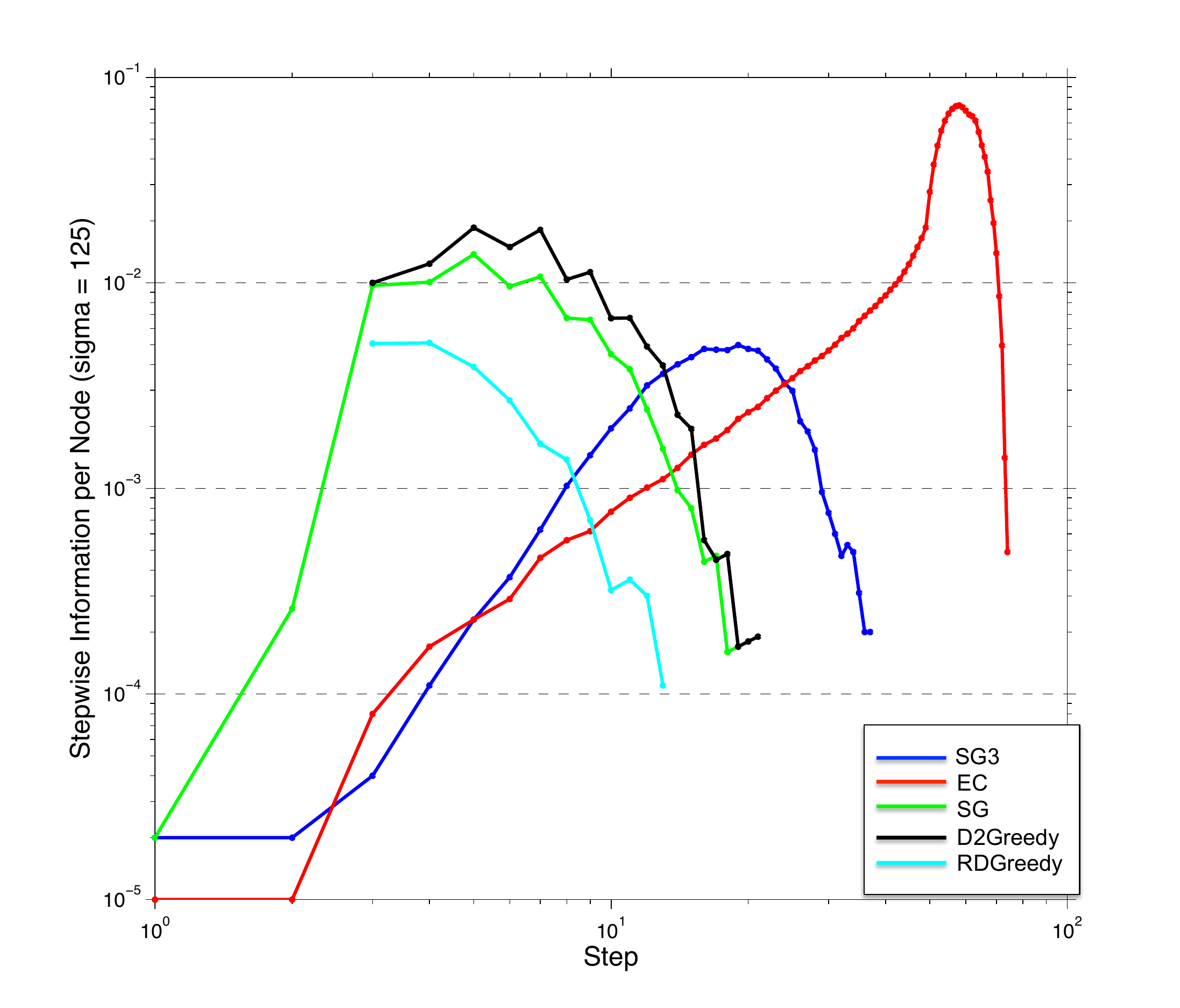}}
\end{center}
\caption{Stepwise information per node.}
\label{fig:stepwise-info}
\end{figure}

The information content is computed as the maximum stepwise
information $I^{\mathscr{A}} := \max_t I_t^{\mathscr{A}}$.  Notice
that $\log |\Delta_t^{\mathscr{A}}(G',G'') |$ measures the ability of
$\mathscr{A}$ to find common solutions for the two noisy instances
$G', G''$, given the underlying input graph $G$.

Our results support the following observations and analysis:

\begin{enumerate}

\item All investigated algorithms reach the maximum information
  content in the noise free limit ($G' = G''$), i.e., for $p=0, 1$ in
  the edge reversal model and for $\sigma=0$ in the Gaussian edge
  weights model.
In this circumstance, 
$\mathbb{E} \log |\Delta_t^{\mathscr{A}}(G',G'')| =
\log|C_t^{\mathscr{A}}(G')| = n - t -1$, so $I_t^{\mathscr{A}} = t +
\log(2^{n-1} - 1) -(n - 1)$, and the information content reaches
its maximum $\log(2^{n-1} - 1)$ at the final step $t = n-1$.

\item \cref{fig:ic}(a) demonstrates that the information content
  qualitatively agrees with the consistency between two noisy
  instances (the dashed magenta line), which reflects that
  $\log |\Delta_t^{\mathscr{A}}(G',G'') |$ is affected by the noise level.  

\item Stepwise information (\cref{fig:stepwise-info}) of the
  algorithms increase initially, but after reaching the optimal step
  $t^*$ (the step with highest information), it decreases and finally
  vanishes.

\item For the greedy heuristics, backward greedy is more informative
  than double greedy under both models.  EC (with the backward greedy strategy) achieves
  the highest information content. We explain this behavior by
  \emph{delayed decision making} of the backward greedy strategy.  With high probability it preserves consistent
  solutions by contracting low weight edges that have a low
  probability to be included in the cut. The same phenomena arises for
  the reverse-delete algorithm to calculate the minimum spanning tree
  of a graph \citep{informativemst}.

\item The information content of the four double greedy algorithms
achieve different rank orders for the two models.  SG3 is inferior to
other double greedy algorithms for the Gaussian edge weights model, but
this only occurs when $p\in [0.2, 0.87]$ for the edge reversal model.
This observation  results from that information content of one specific algorithm is affected by 
both the input master graph $G$ and the noisy instances $G', G''$, 
which are completely different under the two models.

\item Different greedy techniques cast different influences on the
  information content.  The four double greedy algorithms differ by
  the techniques they use (\cref{tab-alg-summarization}). (1) The
  randomization technique makes RDGreedy very fragile
  w.r.t. information content, though it improves the worst-case
  approximation guarantee for the general USM problem
  \citep{buchbinder2012tight}.  RDGreedy labels each vertex with a
  probability proportional to the objective gain, this randomization
  makes the consistency between $C_t^{\mathscr{A}}(G')$ and
  $C_t^{\mathscr{A}}(G'')$ very weak, resulting in small
  approximation set intersection $|\Delta_t^{\mathscr{A}}(G',G'')|$.
  (2) The initializing strategy for the first two vertices as used in SG
  decreases the information content (SG is outperformed by D2Greedy
  under both models) due to \textit{early} decision making. (3) The situation
  is similar for the sorting technique used in SG3 under Gaussian
  edge weights model, it is outperformed by both SG and D2Greedy.  But
  for the edge reversal model, this observation only holds when
  $p\in [0.2, 0.87]$.

\item SG and D2Greedy behave very similar under both models, which is
  caused by an equivalent processing sequence apart from initializing
  of the first two vertices (proved in
  \cref{app:equivalence-SG-D2Greedy}).

\end{enumerate}

\section{Conclusions and Discussions}

In this chapter we advocate an information theoretically guided
average case analysis of the generalization ability of greedy \MAXCUT\
algorithms.  We have presented provably correct methods to
\textit{exactly} compute the cardinality of approximation sets.  The
counting algorithms for approximate solutions enable us to explore the
information content of greedy \MAXCUT\ algorithms.  Based on the
observations and analysis, we propose the following conjecture:

\begin{conjecture}
  Different greedy heuristics (backward, double) and different
  processing techniques (sorting, randomization, initialization)
  sensitively influence the information content. The backward greedy
  with its delayed decision making consistently outperforms the double
  greedy strategies for different noise models and noise levels.
\end{conjecture}

In this work ASC has been employed as a descriptive tool to compare
algorithms. We could also use the method for algorithm design. A
meta-algorithm modifies the algorithmic steps of a \MAXCUT\ procedure
and measures the resulting change in information content. Beneficial
changes are accepted and detrimental changes are rejected. It is also
imaginable that design principles like delayed decision making are
systematically identified and then combined to improve the
informativeness of novel algorithms.

\section{Additional Details}

\subsection{Details of Double Greedy Algorithms}\label{sup:alg}

\begin{algorithm}
\begin{small}
\caption{\textbf{SG} \citep{sahni1976p}}\label{alg:sg}
\KwIn{A complete graph $G=(\gndset, E; \BW)$ with nonnegative edge
  weights $w_{ij}, \forall i,j\in \gndset, i\neq j$} \KwOut{A disjoint
  cut and the cut value} {Pick the maximum weighted edge $(x,y)$\;}
{$S_1:=\{x\}$, $S_2:=\{y\}$, $cut(S_1, S_2):=w_{xy}$\;}
\For{$i = 1:n-2$}{
    {If $w(i, S_1)>w(i, S_2)$, add $i$ to $S_2$\tcp*[r]{$w(i, S_k):= \sum_{j\in S_k}w_{ij}, k=1,2$}}
     {Else add $i$ to
      $S_1$\;}
    {$cut(S_1, S_2):=cut(S_1, S_2)+ \max\{w(i, S_1), w(i, S_2)\}$\;}
}
\Return{$S_1$, $S_2$, and $cut(S_1, S_2)$}
\end{small}
\end{algorithm}

\begin{algorithm}[htbp]
\begin{small}
\caption{\textbf{RDGreedy} (\cite{buchbinder2012tight})}\label{alg:r-usm-4-max-cut}
\KwIn{A complete graph $G=(\gndset, E; \BW)$ with nonnegative edge
  weights $w_{ij}, \forall i,j\in \gndset, i\neq j$} \KwOut{A disjoint
  cut and the cut value } {$S^0 :=\emptyset$, $T^0 := \gndset$\;}
\For{$i=1$ to $n$}{ {$a_i := f(S^{i-1} \cup \set{v_i})-f(S^{i-1})$\;}
  {$b_i := f(T^{i-1} \backslash \set{v_i})-f(T^{i-1})$\;}
  {$a_i':=\max\{a_i, 0\}$, $b_i':=\max\{b_i, 0\}$\;} {\textbf{With
      probability} $\frac{a_i'}{a_i'+b_i'}$ \textbf{do}:
    $S^i :=S^{i-1} \cup \set{v_i}$, $T^{i}:=T^{i-1}$ \tcp{If
      $a_i'=b_i'=0$, assume $\frac{a_i'}{a_i'+b_i'}=1$}}
  {\textbf{Else} (with the compliment probability
    $\frac{b_i'}{a_i'+b_i'}$) \textbf{do}:} {$S^i :=S^{i-1}$,
    $T^{i}:=T^{i-1} \backslash \set{v_i}$\;} } \Return{Two subsets:
  $S^{n}$, $\gndset \backslash S^{n}$, and
  $cut(S^{n}, \gndset \backslash S^{n})$}
\end{small}
\end{algorithm}
\begin{algorithm}
\begin{small}
\caption{\textbf{SG3} \citep{kahruman2007greedy}}\label{alg:sg3}
\KwIn{A complete graph $G=(\gndset, E; \BW)$ with nonnegative edge
  weights $w_{ij}, \forall i,j\in \gndset, i\neq j$} \KwOut{A disjoint
  cut $S_1, S_2$ and the cut value $cut(S_1, S_2)$} {Pick the maximum
  weighted edge $(x,y)$\;}
{$S_1:=\{x\}$, $S_2:=\{y\}$,$\gndset:=\gndset\backslash\{x,y\}$,
  $cut(S_1, S_2):=w_{xy}$\;}
\For{$i=1:n-2$}{
    \For{$j\in \gndset$}{
        {$score(j):=|w(j, S_1)-w(j, S_2)|$ \tcp*[r]{$w(j, S_k):= \sum_{j'\in S_k}w_{jj'}, k=1,2$}} 
    }
    {Choose the vertex $j^*$ with the maximum score\;}
    {If $w(j^*, S_1)>w(j^*, S_2)$, then add $j^*$ to $S_2$, else add it to $S_1$\;}
    {$\gndset:=\gndset\backslash\{j^*\}$\;}
    {$cut(S_1, S_2):=cut(S_1, S_2) + \max\{w(j^*, S_1), w(j^*, S_2)\}$\;}
}
\Return{$S_1$, $S_2$, and $cut(S_1, S_2)$}
\end{small}
\end{algorithm}

\subsection{Equivalence Between Labelling Criteria of SG and D2Greedy}
\label{app:equivalence-SG-D2Greedy}

\begin{claim}
Except for processing the first two vertices,
D2Greedy and SG conduct the same labelling strategy for each vertices.
\end{claim}

\begin{proof}
  To verify this, assume in the beginning of a certain step $i$, the
  solution set pair of SG is $(S_1, S_2)$, of D2Greedy is $(S, T)$
  (for simplicity omit the step index here).

  Note that the relationship between solution sets of SG and D2Greedy
  is: $S_1 \leftrightarrow S$ and
  $S_2 \leftrightarrow (\gndset \backslash T)$.

  For SG, the labelling criterion for vertex $i$ is:
  \begin{align}
    w(i, S_2) - w(i, S_1) = \sum_{i,  j\in S_2} w_{ij} -  \sum_{i,  j\in S_1} w_{ij}.
  \end{align}

  For D2Greedy, the labelling criterion for vertex $i$ is:

  \begin{align}
    \notag	a_i - b_i & = [f(S \cup \sett{v_i}) - f(S) ] - [f(T \backslash \sett{v_i}) -f(T) ] \\
    \notag		& = \left( \sum_{i\in S \cup
                                          \sett{v_i}, j\in V
                                          \backslash S \backslash
                                          \sett{v_i}} w_{ij} -
                                          \sum_{i\in S, j\in V
                                          \backslash S} w_{ij} \right)
                                          -  \\ 
                          & {\phantom = } \left( \sum_{i\in T
                            \backslash \sett{v_i}, j\in V \backslash T
                            \cup \sett{v_i}} w_{ij} -  \sum_{i\in T,
                            j\in V \backslash T} w_{ij} \right)  \\
    \notag				& = \left( \sum_{i, j\in V
                                          \backslash S \backslash
                                          \sett{v_i}} w_{ij} -
                                          \sum_{i\in S, j = i} w_{ij}
                                          \right) -  \\   
                          & {\phantom = } \left( \sum_{i\in T
                            \backslash \sett{v_i}, j = i} w_{ij} -
                            \sum_{i, j\in V \backslash T}
                            w_{ij}\right)  \\	 
    \notag				& = \left( \sum_{i, j\in (V
                                          \backslash T) \cup (T
                                          \backslash S \backslash
                                          \sett{v_i})} w_{ij} -
                                          \sum_{i, j\in S}
                                          w_{ij}\right) -  \\   
                          & {\phantom = } \left( \sum_{i, j \in (S)
                            \cup (T \backslash S \backslash
                            \sett{v_i})} w_{ij} -  \sum_{i, j\in V
                            \backslash T} w_{ij}\right)  \\ 
    \notag			& = 2\left(\sum_{i, j\in V \backslash
                                  T} w_{ij} - \sum_{i, j\in S} w_{ij}
                                  \right)\\ 
\label{eq:relation}   & =  2\left(\sum_{i, j\in S_2} w_{ij} - \sum_{i,
                        j\in S_1} w_{ij}  \right)\\ 
\notag 				& = 2[w(i, S_2) - w(i, S_1)],			   
\end{align}

\noindent where \cref{eq:relation} comes from the relationship between
solution sets of SG and D2Greedy.

So the labelling criterion for SG and D2Greedy is equivalent with each
other.
\end{proof}

\subsection{Counting Methods for  Double Greedy Algorithms}
\label{complement:counting}

\textbf{D2Greedy}: summarized in \cref{alg:d-usm-4-max-cut}, we have
proved that it has the same labelling criterion with SG, the
relationship between solution sets of SG and D2Greedy is:
$S_1 \leftrightarrow S$ and $S_2 \leftrightarrow (V \backslash T)$, we
will use $S_1$ and $S_2$ in the description of its counting methods.

In step $t$ ($t = 1, \cdots, n$) there are $k = n-t$ unlabelled
vertices, it is not difficult to know that the number of possible
solutions for each instance is

\begin{equation}
  |C(G')| = |C(G'')| = \left\{\begin{array}{ll}
                                2^{k}  & \textrm{if $S_1 \neq
                                         \emptyset$ and $S_2 \neq
                                         \emptyset$}\\

                                2^{k}-1  & \textrm{otherwise}\\
                              \end{array}\right.
\end{equation}

To count the intersection set (i.e. $|C(G')\cap C(G'')|$), assume the
solution sets of $G'$ is $(S_1', S_2')$, the solution sets of $G''$ is
$(S_1'', S_2'')$, so the unlabelled vertex sets are
$T'=V\backslash S_1'\backslash S_2'$,
$T''=V\backslash S_1''\backslash S_2''$, respectively.  Denote
$L:=T'\cap T''$ be the common vertices of the two unlabelled vertex
sets, so $l=|L|$ ($0\leq l\leq k$) is the number of common vertices in
the unlabelled $k$ vertices. Denote $M' :=T'\backslash L$,
$M'' :=T''\backslash L$ be the sets of different vertex sets between
the two unlabelled vertex sets.
Then,

\begin{enumerate}

\item if $(S_1'\backslash M'', S_2'\backslash M'')$ or
  $(S_2'\backslash M'', S_1'\backslash M'')$ matches
  $(S_1''\backslash M', S_2''\backslash M')$.

Assume w.l.o.g. that $(S_1'\backslash M'', S_2'\backslash M'')$
matches $(S_1''\backslash M', S_2''\backslash M')$:

  \begin{equation}\notag
  \begin{split}
    & |C(G^{'})\cap C(G^{''})| = \\
    & \left\{\begin{array}{ll}
               2^{l}  & \textrm{if $S_1' \cup S_1'' \neq \emptyset$
                        and $S_2' \cup S_2'' \neq \emptyset$}\\ 
               2^{l}-1  & \textrm{otherwise}\\
                    \end{array}\right.
\end{split}
\end{equation}

\item  otherwise, $|C(G^{'})\cap C(G^{''})|= 0$  
\end{enumerate}

\textbf{SG3}:  presented in  \cref{sec:counting-sg3}.

\textbf{SG}: summarized in \cref{alg:sg}, the methods to count its
approximation sets is the same as that of SG3.

\textbf{RDGreedy}: summarized in \cref{alg:r-usm-4-max-cut}, the
methods to count its approximation sets is the same as that of
D2Greedy.

\subsection{Proof of the Correctness of Method to Count
  $|C(G')\cap C(G'')|$ of SG3}
\label{app:proof-SG3}

\begin{proof}

  First of all, notice that $M'$ must be included in $S_1''\cup S_2''$
  and $M''$ must be included in $S_1'\cup S_2'$, because $M'$ has no
  intersection with $M''$, and we know that
  $S_1''\cup S_2''\cup M''=S_1'\cup S_2'\cup M'$. After removing $M'$
  from $S_1''\cup S_2''$, and $M''$ from $S_1'\cup S_2'$, the vertices
  in the pairs, $(S_1'\backslash M'', S_2'\backslash M'')$ and
  $(S_1''\backslash M', S_2''\backslash M')$, can not be changed by
  distributing any other unlabelled vertices , so if they can not
  match with each other, there will be no common solutions.

  If they can match, in the following, there is only one way to
  distribute $M'$ and $M''$ to have common solutions. And the vertices
  in the common set $L=T'\cap T''$ can be distributed consistently in
  the two instances, so in this situation $|C(G')\cap C(G'')|=2^l$.
\end{proof}

\subsection{Proof of Theorem \ref{theo:ec}}
\label{app:proof-EC}
\begin{proof}
First of all, 
We will prove the following claim, then use the claim to prove
\cref{theo:ec}.

\begin{claim}
In each step $t$ ($t = 0, \cdots, n-2$), the following conditions hold:
\begin{enumerate}

\item The remained super vertices in $P, Q$ are distinct with each
  other, that means any two super vertices inside $P$ or $Q$ do not
  have intersection, and there are no common super vertex between $P$
  and $Q$.

\item The common super vertex removed from $P, Q$, i.e.,
  $\mathbf{p_{ii'}}=\mathbf{q_{jj'}}$, is the smallest common super
  vertex containing $\mathbf{p}_i$ or $\mathbf{p}_{i'}$ (respectively,
  $\mathbf{q}_j$ or $\mathbf{q}_{j'}$)

\item The common super vertex removed from $P, Q$, i.e.,
  $\mathbf{p_{ii'}}=\mathbf{q_{jj'}}$, are ``unique'' (i.e., there
  does not exist $\mathbf{p_{ii''}}=\mathbf{q_{jj''}}$, such that
  $\mathbf{p_{ii''}} \neq \mathbf{p_{ii'}}$). That means, there is
  only one possible way to construct the removed common super vertex.

\end{enumerate}

\end{claim}

We will use inductive assumption to prove the {claim}.  First of all,
in the beginning (step $0$), the conditions hold. Assume the
conditions hold in step $t$.  In step $t+1$, there are two possible
situations:

\begin{itemize}
\item There are no common super vertex removed.

  Condition 1 holds because the contracted super vertices pair do not
  equal. Condition 2, 3 hold as well because there are no contracted
  super vertices removed.

\item There are common super vertex removed.

  Condition 1 holds because the only common super vertices pair have
  been removed from $P, Q$, respectively.\\ 

  To prove condition 2, notice that the smaller vertices for
  $\mathbf{p_{ii'}}$ are
  $\mathbf{p_{ii'}} \backslash \mathbf{p_{i}}=\mathbf{p_{i'}}$ and
  $\mathbf{p_{ii'}} \backslash \mathbf{p_{i}'}=\mathbf{p_{i}}$,
  respectively, for $\mathbf{q_{jj'}}$ are
  $\mathbf{q_{jj'}} \backslash \mathbf{q_{j}}=\mathbf{q_{j'}}$ and
  $\mathbf{q_{jj'}} \backslash \mathbf{q_{j}'}=\mathbf{q_{j}}$,
  according to Condition 1, they can not be common
  super vertices, so there are no smaller common super vertices.\\

To prove condition 3, assume there exists
$\mathbf{p_{ii''}}=\mathbf{q_{jj''}}$, such that $\mathbf{p_{ii''}}
\neq \mathbf{p_{ii'}}$ (respectively, $\mathbf{q_{jj''}} \neq
\mathbf{q_{jj'}}$), so $\mathbf{p_{i''}}\neq \mathbf{p_{i'}}$
($\mathbf{p_{j''}}\neq \mathbf{p_{j'}}$). From
\cref{alg:max-common} we know that $\mathbf{p_{i}} \cup
\mathbf{p_{i''}}=\mathbf{p_{ii''}} \supseteq \mathbf{q_{j}} \backslash
\mathbf{p_{i}}$ and $\mathbf{p_{i}}
\cup\mathbf{p_{i'}}=\mathbf{p_{ii'}} \supseteq \mathbf{q_{j}}
\backslash \mathbf{p_{i}}$ (respectively, $\mathbf{q_{j}} \cup
\mathbf{q_{j''}}=\mathbf{q_{jj''}} \supseteq \mathbf{p_{i}} \backslash
\mathbf{q_{j}}$ and $\mathbf{q_{j}} \cup
\mathbf{q_{j'}}=\mathbf{q_{jj'}} \supseteq \mathbf{p_{i}} \backslash
\mathbf{q_{j}}$), so that $\mathbf{p_{i''}} \supseteq \mathbf{q_{j}}
\backslash \mathbf{p_{i}}$ and $\mathbf{p_{i'}}\supseteq
\mathbf{q_{j}} \backslash \mathbf{p_{i}}$ (respectively,
$\mathbf{q_{j''}} \supseteq \mathbf{p_{i}} \backslash \mathbf{q_{j}}$
and $\mathbf{q_{j'}}\supseteq \mathbf{p_{i}} \backslash
\mathbf{q_{j}}$), that contradicts the known truth that
$\mathbf{p_{i'}}$ and $\mathbf{p_{i''}}$ (respectively,
$\mathbf{q_{j'}}$ and $\mathbf{q_{j''}}$) must be totally different
with each other (from Condition 1).

\end{itemize}

Then we use the claim to prove that the $c$ returned by
\cref{alg:max-common} is exactly the maximum number of common super
vertices after all possible contractions.  Because the three
conditions hold for each step, we know that finally all the common
super vertices are removed out from $P$ and $Q$.  From Condition 2 we
know that all the removed common super vertices are the smallest ones,
from Condition 3 we get that there is not a second way to construct
the common super vertices, so the resulted $c$ is the maximum number
of common super vertices after all possible contractions.

\end{proof}

\def\dir{chapters/Goemans-Williamson}
\chapter{Validating Goemans-Williamson's MaxCut Algorithm}
\label{chapter_Goemans_Williamson}

\begin{chapquote}{Carl Friedrich Gauss}
You have no idea, how much poetry there is in the calculation of a table of logarithms!
\end{chapquote}

In \cref{chapter_greedy_maxcut} we have investigated the robustness of
\textit{greedy} \MAXCUT\ algorithms by employing the Approximation Set
Coding framework \citep{Buhmann10isit,informativemst} for measuring
information content of algorithmic solutions. However, the methodology
used in the last chapter has to be generalized to analyze continuous
\MAXCUT\ algorithms, e.g., the Goemans-Williamson's \MAXCUT\ algorithm
using semidefinite programming relaxation
(\citet{goemans1995improved}, abbreviated as \MAXCUT-SDP).
In order to analyze the generalization performance of non-greedy,
continuous algorithms, we propose an information-theoretic algorithmic
regularization and validation strategy based on \emph{posterior
  agreement} (PA), and further theoretically justify it by presenting
the ``coding by posterior'' framework. The strategy regularizes
algorithms and ranks them according to the informativeness of their
output given noisy input.

The \maxcut-SDP algorithm firstly obtains a fractional solution by solving an SDP relaxation of the \maxcut problem, then rounds the fractional solution back  to a solution of graph cut using the technique of random hyperplane rounding. The vanilla \maxcut-SDP algorithm does not provide a sequence of posterior distributions of graph cut solutions. In order to study its generalization performance, we derive a generalization of the \maxcut-SDP algorithm in the following sense: i) We allow the SDP solver of \maxcut-SDP to stop at any running time $t$; ii) For a given stopping time $t$, we propose methods to evaluate the posterior distribution of cuts induced by the status of \maxcut-SDP at this time; iii) Given the evaluation of posteriors, we utilize the PA based approach to investigate the generalization ability of the \maxcut-SDP algorithm. 
Experimental comparison with representative 
greedy \MAXCUT\ algorithms shows that \MAXCUT-SDP with the best
known approximation ratio generalizes worse than greedy \MAXCUT\
algorithms under high noise level.

\section[Generalization Ability of Algorithms]{Information Content as
  Generalization Ability of Algorithms}

Classical algorithms usually search for a unique or a randomized
solution in the hypothesis class. Input noise often renders such
algorithmic solutions highly unstable. Therefore, we require an
algorithm to return a posterior distribution of solutions given the
noisy input. Such a posterior should concentrate on few solutions but
the posterior must be stable for equally likely inputs. We interpret
this tradeoff between precise localization in the hypothesis class and
stability of posteriors as the \emph{generalization} ability of an
algorithm.  Under this strategy, an algorithm should stop early to
recover the stable solutions (posterior distribution of solutions).
For an algorithm $\A$ to succeed with an informative and stable output
in the two-instance scenario, e.g., the \MAXCUT-SDP of
\citet{goemans1995improved}, we propose a general
information-theoretic regularization and validation strategy, which is
based on a provable analogue of information content for algorithms.

It is well-known that when training machine learning models, e.g.,
training neural network with the stochastic gradient descent
algorithm, one should stop the algorithm early to recover
generalizable models, which is called the ``early-stopping'' strategy
\citep{girosi1995regularization,giles2001overfitting}.  With empirical success, few theory has
been proposed for this well-utilized strategy.
By analogue between the generalizable solutions and machine learning
models, this work also provides an information-theoretic verification
of this ``early-stopping'' strategy.  Though we use \MAXCUT\
algorithms as an illustrating example in this chapter, it is
noteworthy that the strategy applies generally to any algorithms in
the two-instance scenario.

\section{Algorithm Validation via Posterior Agreement}
\label{sec_coding}

The posterior agreement objective for algorithm validation is
motivated by the ``coding by posterior'' framework, which will be
presented in this section.  On a high level, it is based on an
analogue to the noisy communication channel in Shannon's information
theory \citep{cover2012elements}.

We denote as $\da$, $\da'$, $\da''\in \Z$ different data instances in
the two-instance scenario (details in \cref{sec_intro_alg_information}).
Often, a computational problem is associated with some cost function
$R(\h, \da)$, which measures how well a hypothesis $\h$ in the
hypothesis space $\H$ will solve the problem on input $\da$.  An
algorithm $\A$ maps the input space to the hypothesis space
$\A : \Z \rightarrow \H$.  We introduce parameters
$\bmtheta\in \Theta$ to enumerate a set of algorithms.  In
combinatorial optimization for example, $\bmtheta$ might denote the
approximation precision or stopping time.

In general, we assume that algorithm $\A$ assigns non-negative weights
$w_{\bmtheta}(\h, \da)$ to all hypotheses dependent on the input and
the parameters, i.e.,
\begin{equation}
w: \H\times \Z \times \Theta \rightarrow   [0, +\infty), \quad
(\h,\da,\bmtheta) \mapsto w_{\bmtheta}(\h, \da)\,.
\end{equation}
Gibbs weights $w_{\beta}(\h, \da) = \exp\bigl(-\beta R(\h, \da)\bigr)$
with inverse temperature $\beta$, for example, rank different
hypotheses according to how well they solve the problem in terms of
costs $R(\h, \da)$.

Such a weighting of hypotheses can be interpreted as a
{posterior distribution $\prob_{\bmtheta}(\h|\da)$} induced by algorithm
$\A$, and is defined as
\begin{equation}
  \prob_{\bmtheta}(\h|\da) := w_{\bmtheta}(\h, \da) /
  \sum\nolimits_{\h^\prime\in\H}w_{\bmtheta}(\h^\prime, \da),\quad   
  \forall \h\in \H \,. 
\end{equation}

For example, if we choose an indicator function as weights
\begin{equation}\label{eq_gate}
w_{\bmtheta}(\h, \da)= \mathbf{1} \bigl\{
R(\h, \da) \leq R(\h^{\perp}, \da) + \gamma(\bmtheta) \bigr\},
\end{equation} 
where $\gamma(\bmtheta)$ denotes a precision value determined by a
specific $\A$ and the empirical risk minimizer
$c^{\perp}(\da) ={\arg\min}_{\h\in\H}R(\h, \da)$ centers an
{``approximation set''} of size $\gamma(\bmtheta)$ in $\H$. In this manner 
we can recover the ``approximation set coding'' framework  of 
\citet{Buhmann10isit}.

The posterior $\prob_{\bmtheta}(\h|\da)$ effectively partitions the
hypothesis class into statistically equivalent solutions with high
weight values and discards hypotheses with vanishing weights.
$\prob_{\bmtheta}(\h|\da)$ plays the role of a codebook vector with the
associated Voronoi cell. 

To generate alternative posteriors for a
coding protocol we have to use the given data $\da$ and have to
transform the mapping from $\Z$ to $\H$. Such transformations should
not change the measurements represented by $\da$ but the algorithmic
mapping.

\begin{definition}[Transformation set]
Given data instance $\da$ and algorithm $\A$ with posterior $\prob_{\bmtheta}(\h|\da)$,  we define 
the \emph{{transformation set}} $\T$  as a set of mappings $\t: \Z \rightarrow
\Z$ such that the following two conditions are satisfied,
\begin{enumerate}
	
	\item $\A(\t\circ\da), \t\in \T$ generates an ``approximately uniform
	cover'' of the hypothesis space $\H$, i.e.,
	$\sum_{\t\in\T} \prob_{\bmtheta}(\h|\t\circ\da) \in \bigl[
	\frac{|\T|}{|\H|} (1 - \rho), \frac{|\T|}{|\H|} (1 + \rho)\bigr]$, for $0<\rho < 1$;
	
	\item For every transformation $\t\in\T$ there exists an associated
	transformation $\t^\H: \H \mapsto \H$ such that
	$w_{\bmtheta}(\h, \t\circ\da) = w_{\bmtheta}(\t^\H\circ\h,
        \da)$.
\end{enumerate}
\end{definition}

Given a posterior and transformations, we can define a \textit{virtual
	communication scenario}. It requires
a sender $\se$, a receiver $\re$, and a problem generator $\pg$ as a
noisy channel between $\se$ and $\re$.  Sender and
receiver agree on algorithm $\A$ and its induced posteriors. The communication scenario consists the 
following parts: 

\subsection{Code Book Generation}

The communication code is generated by the procedure:

\begin{enumerate}
	
	\item Sender $\se$ and receiver $\re$ obtain the sample  $\da'$
	from the problem generator $\pg$.
	\item Sender $\se$ and receiver $\re$ calculate the posterior
	$\prob_{\bmtheta}(\h|\da')$.
	\item A set of transformations
	$\CB = \{\t_1,\cdots, \t_M \}\subseteq \T$ is generated uniformly with
	associated posteriors $\prob_{\bmtheta}(\h|\t_j\circ \da'),\; 1\leq j\leq M$.  
	\item $\se$ and $\re$ agree on a transformation set $\CB$ and 
	posteriors $\prob_{\bmtheta}(\h|\t_j\circ \da'), 1\leq j\leq M$.
	
\end{enumerate} 

The posteriors $\prob_{\bmtheta}(\h|\t_j\circ \da'), \t_j\in \CB$ play the
role of codebook vectors in Shannon's theory of communication.

\subsection{Communication Protocol}

\begin{enumerate}
	
\item The sender $\se$ selects a transformation $\t_s\in \CB$ as
  message and sends it to the problem generator $\pg$.
	
\item $\pg$ generates the instance
  $\da''$ and applies the transformation $\t_s$ to $\da''$, yielding
  $\datilde := \t_s\circ \da''$.
	
\item $\pg$ sends $\datilde$ to $\re$ without revealing $\t_s$.
	
\item $\re$ calculates the posterior $\prob_{\bmtheta}(\h|\datilde)$.
	
\item $\re$ estimates the message $\t_s$ by using the decoding rule:
  \begin{flalign}
    \hat \t = \arg\max_{\t\in \CB} \E_{\h\sim
      \prob_{\bmtheta}(\h|\t\circ\da')} \prob_{\bmtheta}(\h|\t_s\circ
    \da'')
    \\
    = \arg\max_{\t\in \CB}\sum\nolimits_{\h\in\H}
    \prob_{\bmtheta}(\h|\t\circ\da') \prob_{\bmtheta}(\h|\datilde).
  \end{flalign}
\end{enumerate}

\subsection{Error Analysis of the Virtual Communication Protocol}
\label{sec:erroranalysis}

The probability of a communication error amounts to
\begin{align}
& \prob(\hat{\t} \neq \t_s | \t_s)  \\
&= \prob\Bigl(\max_{\t_j\in
	\CB\backslash \t_s} \E_{\prob(\h|\t_j\circ \da')} [\prob(\h |\datilde)]
\geq \E_{\prob(\h|\t_s\circ \da')} [\prob(\h |\datilde)] 
\Bigr)\\\label{eq:union-bound}  
& \stackrel{(a)}{\leq} \sum_{\t_j\in \CB\backslash
	\t_s} \prob \left( \E_{\prob(\h|\t_j\circ \da')} [\prob(\h |\datilde)] \geq
\E_{\prob(\h|\da')} [\prob( \h |\da'')] 
\right)\\ 
& \stackrel{(b)}{\leq} \sum_{\t_j\in \CB\backslash \t_s}
\E_{\da', \da''} 
\frac{\E_{\t_j} \E_{\prob(\h|\t_j\circ \da')} [\prob(\h |\datilde)] 
}{\E_{\prob(\h|\da')} [\prob( \h | \da'')]}\, ,
\end{align}
by applying the union bound $(a)$ and Markov's inequality $(b)$.

Abbreviating
$Z_{\T} :=\E_{\t_j}\E_{\prob(\h|\t_j\circ \da')} [\prob(\h |\datilde)]$,
we derive
\begin{flalign} 
& Z_{\T}\\\notag 
&=\E_{\t_j}\E_{\prob(\h|\t_j\circ \da')} \prob(\h |\datilde)
= \E_{\t_j} \sum_{\h\in\H} \prob(\h|\t_j\circ \da') \prob(\h | \datilde) \\ 
& = \sum_{\h\in\H} \prob(\h | \datilde) \E_{\t_j} \prob(\h|\t_j\circ\da')
=  \sum_{\t_j\in\T}\prob(\t{_j}) \prob(\h|\t_j\circ\da) \\\label{eq_cover}
& \leq  (1 + \rho) |\H|^{-1},
\end{flalign}
where  \cref{eq_cover} arises from the approximately uniform
coverage of $\H$ by the posteriors, i.e.,
$\sum_{\t\in\T} \prob(\h|\t\circ\da) \in \bigl[\frac{|\T|}{|\H|}(1 -
\rho), \frac{|\T|}{|\H|}(1 +
\rho\bigr)]$
and $\prob(\t) = 1/|\T|$.  
Substituting  \cref{eq_cover} into  \cref{eq:union-bound} we
derive the error bound
\begin{flalign} 
\prob(\hat{\t} \neq \t_s | \t_s)
& \leq   \sum_{\t_j\in T\backslash\t_s}\! \E_{\da', \da''}\left[
\bigl(\cardH \E_{\prob(\h|\da')} [\prob( \h |  \da'')]\bigr)^{-1}
\right] \\\label{eq9}
& = (M - 1) \E_{\da', \da''}\left[  
\bigl(\cardH k(\da',\da'') \bigr)^{-1} \right]\\\label{eq_Eexp} 
&\le M  \E_{\da', \da''}\left[
\exp \bigl( - \log (\cardH k(\da',\da'')) \bigr)
\right],
\end{flalign}

where  \cref{eq9} comes from the definition of
posterior agreement in  \cref{eq:pa}. 

We then analyze 
$\hat{I} := \log\bigl( |\H| k(\da',\da'') \bigr)$ in 
\cref{eq_Eexp} at its expected value
\begin{equation}\label{eq:i}
I := \E_{\da', \da''} \left[ \log\bigl( |\H| k(\da',\da'')
\bigr)\right].
\end{equation}
To control the fluctuations $\Delta_{\da',\da''} :=\hat{I}-I$, we
assume that for all $\epsilon>0, \delta>0$, there exists
$n_0\in\mathbb{N}$ s.t. for all $n>n_0$ 
\begin{equation} 
\prob\left( |\Delta_{\da',\da''}| \ge \epsilon I \right) <\delta \,.
\end{equation} 
This assumption of asymptotically vanishing fluctuations 
yields the following upper bound
\begin{equation}
\E_{\da', \da''}\bigl[
\exp(-\hat{I}) \bigr] \le  \exp\bigl( -I(1-\epsilon) \bigr).
\end{equation}
Since $\epsilon$ can be chosen arbitrarily small in the asymptotic
limit, the error probability is bounded with high probability by
\begin{equation}\label{eq_errorbound}
  \prob(\hat{\t} \neq \t_s | \t_s) \le  \exp \bigl( - I +\log
  (M(1+\rho)) \bigr) \,. 
\end{equation}
For $I$ exceeding the effective total rate $\log (M(1+\rho))$, the
error vanishes asymptotically since $I = O(\log|\H|)$. This bound
suggests that we should maximize $I$ in \cref{eq:i} when searching for
informative algorithms, thus verifying the definition of algorithmic
information content defined in \cref{eq_ic}.

\subsection{Connection to Classical Mutual Information}\label{sec:classical_mi}

Let $\graphrv'$ and $\graphrv''$ be the two random variables of the
noisy graph instances in the two-instance scenario (as specified in \cref{sec_intro_alg_information}).  With a bit abuse
of notation, we use $G', G''$ as the realizations of $\graphrv'$ and
$\graphrv''$, respectively.  Then we make a connection between the
classical mutual information $\I(\graphrv'; \graphrv'')$ and the
algorithmic information content in \cref{eq_ic}.

We start by expanding the joint distribution
$\prob(G^\prime, G^{\prime\prime})$ with cut variables $\h \in \H$ and
transformations $\tau \in \CB$:
\begin{flalign} 
  \I(\graphrv'; \graphrv'') & = \E_{\da', \da''} \log
  \frac{\prob(G^\prime,
    G^{\prime\prime})}{\prob(G^\prime)\prob(G^{\prime\prime})} \\
  & = \E_{\da', \da''} %
  \log \frac{\sum_c \sum_\t \prob(G^\prime, G^{\prime\prime}| c, \t)
    \prob( c, \t) }%
  {\prob(G^\prime)\prob(G^{\prime\prime})}. \label{eq:classical_ic_joint} 
\end{flalign} 
The conditional distribution
$\prob(G^\prime, G^{\prime\prime}| c, \t)$ of $\da',\da''$ factorizes
due to conditioning on $\h$ and $\t$,
\begin{align}
  \prob(G^\prime, G^{\prime\prime}| c, \t) 
  & {=}{} \prob(G^\prime | c, \t) \prob( G^{\prime\prime} | c,
    \t ) \label{eq:mut_inf_lemma_1}  \\
  &  \overset{(a)}{=}{} \frac{\prob(c | G^\prime, \t)}{\prob(c|\t)}
    \prob(G^\prime|\t)  
    \frac{\prob(c |G^{\prime\prime}, \t)}{\prob(c|\t) }
    \prob(G^{\prime\prime}|\t),\, 
\end{align}
since $\da',\da''$ are independent given $c$ and $\tau$.  The
transformation $\t$ plays the role of a latent variable.  Step $(a)$
applies the Bayes rule twice. Substitute \cref{eq:mut_inf_lemma_1}
into \cref{eq:classical_ic_joint} and assume
$\prob(c)=\vert \mathcal{C}\vert^{-1}$, we get (detailed derivation of
\cref{eq_deferred_proof} is deferred to \cref{sec:proof})
\begin{align}\label{eq_deferred_proof}
\I(\graphrv'; \graphrv'')
& =  \E_{\da', \da''} \log
\sum\nolimits_c \Bigl[\prob(c|G')\prob(c|G'') / \prob(c) \Bigr] \\ 
& =  
\E_{\da', \da''} \log |\H| \sum\nolimits_c \prob(c|G')\prob(c|G'') = I(\graphrv'; \graphrv'').
\end{align} 
Thus we reach the algorithmic
information content in  \cref{eq_ic}.

\section{\MAXCUT\ Algorithm using SDP Relaxation}\label{sec:interpret_sdp}
\begin{algorithm}[t]
\caption{\MAXCUT-SDP \citep{goemans1995improved}}\label{alg:maxcut-sdp}
\KwIn{undirected graph $G =(V, E; \BW)$ with non-negative weights $\BW$} \KwOut{cut
  $c = (S, V\setminus{S})$} {solve problem $(R)$, obtaining an optimal set of
  vectors $\v_i \in S_{n-1}$\;}
{let $\r$ be a vector \textit{uniformly} distributed on $S_{n-1}$\;}
\KwRet{$S := \{i \;|\; \v_i\cdot \r \geq 0, \forall i\in V\}$ and
  $V\setminus{S}$}
\end{algorithm}

In this section we give a geometric interpretation of Goemans-Williamson's \MAXCUT\
algorithm using semidefinite programming relaxation
(\citet{goemans1995improved}, abbreviated as \MAXCUT-SDP),  which will facilitate deriving
methods to calculate the posterior of cuts. 
\ALG \ref{alg:maxcut-sdp} summarizes the \MAXCUT-SDP algorithm: It rounds
the solution to a non-linear programming relaxation, which can be
interpreted as SDP,
then it solves the SDP using standard algorithms, such as
interior-point methods \citep{helmberg1996interior}, bundle method or
block coordinate descent \citep{waldspurger2015phase}. 

Concretely,
\MAXCUT\ is formulated as the NP-complete
integer program:
\begin{flalign}
\begin{array}{llr}
     & \max \frac{1}{2}\sum_{i<j} W_{ij}(1 - v_i v_j) &\\
(Q)\;\;\; & \text{s.t.} \;\; v_i \in \{-1, 1\}  & \forall i \in V
\end{array}
\end{flalign}
then $(Q)$ is relaxed to define the following non-linear problem, 
\begin{flalign}\label{eq:r}
\begin{array}{llr}
     & \max \frac{1}{2}\sum_{i<j} W_{ij}(1 - \v_i^\trans  \v_j) &\\
(R) & \text{s.t.} \;\; \v_i \in S_{n-1}  & \forall i \in V
\end{array}
\end{flalign}
where $S_{n-1}$ is the $(n-1)$-dimensional unit sphere, i.e., $S_{n-1} = \{ \v\in \R^n \;|\; \|\v\|_2 = 1 \}$.
Arrange the $n$ vectors $\v_1, \cdots, \v_n$ to be the $n$ columns of a $n\times n$ matrix $\BD$, that is,   $\BD= (\v_1, \v_2, \cdots, \v_n)$. Let  $\BX:= \BD^{\trans} \BD$, then the $ij^\text{th}$ entry of $\BX$ is 
$X_{ij} = \v_i^\trans \v_j$. One can observe that  $(R)$ equals to the following SDP problem
with only equality constraints:
\begin{flalign}\label{eq:sd}
\begin{array}{llr}
         &  \max \frac{1}{2}\sum_{i<j} W_{ij}(1 - X_{ij})\\
(\text{SDP}) & \text{s.t.} \;\; X_{ii} = 1, \forall i \in V,   \\ & \qquad \BX \; \text{is symmetric positive semidefinite}.\;
\end{array}
\end{flalign}
We use one classical interior-point method
\citep{helmberg1996interior} to solve the SDP problem in (\ref{eq:sd}).

\setkeys{Gin}{width=0.8\textwidth,height=.51\textwidth}
\begin{figure}[htbp]
\center
\includegraphics[]{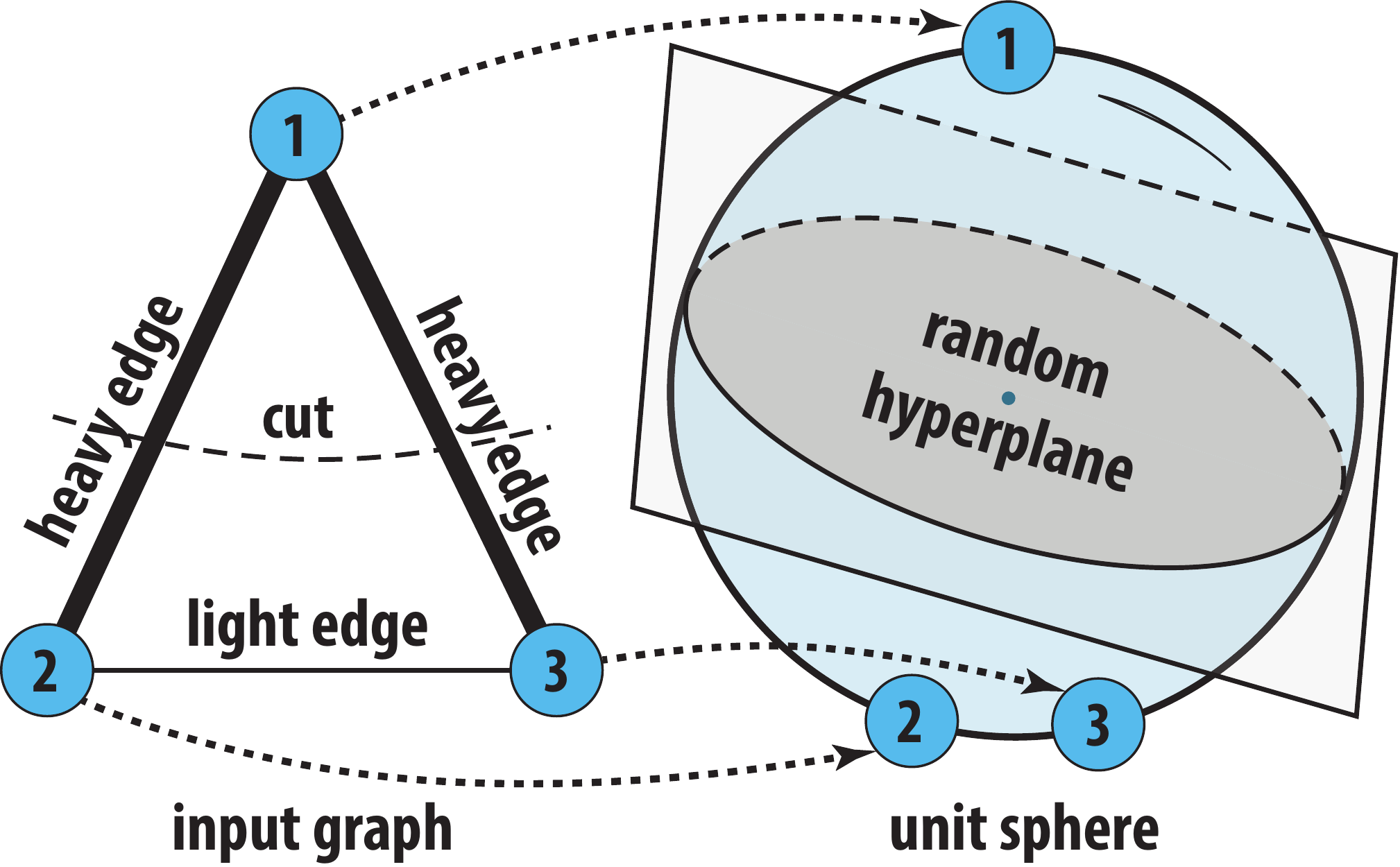} 
\caption{A geometric view of  \cref{alg:maxcut-sdp}}\label{fig:geometric}
\end{figure}

A geometric view in 
\cref{fig:geometric} explains the essence of  \cref{alg:maxcut-sdp}:
It maps vertices to vectors on the unit sphere. 
A feasible SDP solution corresponds to a
point configuration on the unit sphere, while a feasible solution to
\MAXCUT\ assigns a sign variable $\{\pm 1\}$ with every graph vertex.
An optimal solution of SDP tends to send adjacent vertices with heavy
edges to antipodal points, thereby maximizing
$(1 - \v_i^\trans \v_j)/2$.  A rounding technique is required that
separates most far away pairs, and hence keeps close
pairs together.  \textit{Random hyperplane rounding} works gracefully: A
random hyperplane through the origin partitions the sphere into two
halves, which correspond to cut parts (see \cref{fig:geometric}).
The ratio between the expected cut value over the maximum cut value is
never worse than $\alpha \approxeq 0.87856$, which is the expected
approximation guarantee of \maxcut-SDP.

\section{Calculate Posterior Probability of  Cuts}\label{sec:calc_posterior}

Given the geometric interpretation of the \MAXCUT-SDP algorithm
in  \cref{sec:interpret_sdp}, we can derive the scheme to calculate posterior probability of cuts here, which will be used to
evaluate the posterior agreement in  \cref{eq:pa}. 

In step $t$, the relaxed SDP problem (\ref{eq:r}) outputs the $n$
intermediate vectors $O_t = \{\v_1, \v_2, \cdots, \v_n\}$.  Each cut
$c := (S, \bar{S})$, where $S = \{1,\ldots, \ell\}$, 
$\bar{S} := V\setminus S =  \{\ell+1,\ldots, n\}$, induces a set,
\begin{flalign}\label{eq:cut-induced-set}
  B(c) &= \{\b_1, \cdots, \b_n\} := \{\v_1, \cdots, \v_{\ell}, -\v_{\ell + 1},
    \cdots, -\v_{n}\}.\notag
\end{flalign}
Let the corresponding matrix with columns specified by $\{\b_1, \cdots, \b_n\}$ be $\BB$.
$B(c)$ is used to define a polygonal intersection  cone:
\begin{definition}[Polygonal \textit{intersection}
  cone]\label{def:intersection}
The cone $C$ determined by  the cut $c = (S, \bar{S})$
is the intersection of $n$ half-spaces:
\begin{equation}
  C = C(c) := \{\x\in \mathbb{R}^n, \|\x\|_2\leq 1 \;|\; \x^\trans \b_i
  \geq 0, \forall i\in V   \}.
\end{equation}
\end{definition}
It determines the
posterior of the corresponding cut by  \cref{the:pro} (all proof of lemmas are in  \cref{sec:proof}) in the following,
\begin{lemma}\label{the:pro}
The posterior $\prob(c|G)$ of
a cut $c$ is
\begin{align}\label{eq:rounding}
  \prob(c|G)  
   & = \frac{2*\text{\emph{unit spherical area of} }    C(c)}{\text{\emph {area of
    unit sphere}}}   \\
 &  = \frac{2*\text{\emph{volume of }} C(c)}{\text{\emph{volume of  unit ball}}}\\
 &   =
    \frac{2*\text{\emph{solid angle of }}C(c)}{\text{\emph{solid angle of unit
    sphere}}} \,.
\end{align}
\end{lemma}

Ensured by  \cref{the:pro}, the cut
probabilities are measured either by spherical area, by volume
or by solid angle.  
Without loss of generality, we calculate the solid angle to derive
$\prob(c|G)$. Since it is convenient to express the method of calculating
 solid angle in terms of the spanning cone definition, let us
transform the intersection cone $C$ into the spanning cone,
\begin{definition}[Polygonal \textit{spanning}
  cone]\label{def:spanning}
  According to \citet{tiel1984convex}, a polygonal spanning cone is
  spanned by a set of $n$ linearly independent unit vectors
  $ \BA= (\a_1, \a_2, \cdots, \a_n)$ in $\mathbb{R}^n$:
  \begin{flalign}\notag 
    C' 
    :=  \{\x\in
    \mathbb{R}^n, \|\x\|_2\leq 1\; |\; \x = \sum\nolimits_{i=1}^n \lambda_i \a_i,
    \lambda_i \geq 0, 1\leq i \leq n \}.
  \end{flalign}
\end{definition}
Given an intersection cone $C$, one can get an equivalent spanning
cone $C'$ by taking $\BA^{\trans} = \BB^{-1}$, which is ensured by,
\begin{lemma}\label{lemma:def-equi}
  Given one intersection cone $C$ (Definition \ref{def:intersection}) and
  one spanning cone $C'$ (Definition \ref{def:spanning}), if  
  $\exists\; k_1, \cdots, k_n >0$, \text{s.t.,}
  $\BA^{\trans} = \text{diag}(k_1, \cdots, k_n) \BB^{-1}$, then
  $C' = C$.
\end{lemma}
\begin{algorithm}[htbp]
  \caption{Calculate posterior of each cut \citep{bian2016information}}\label{alg:prob-calc}
  \KwIn{independent vectors $\{\v_1, \v_2, \cdots, \v_n\}$ on
    $S_{n-1}$} 
  \KwOut{posterior of each cut
    $\mathbb{P}(c|G), \forall c\in \H$} \For{\emph{each cut} $c\in \H$}{ {get
      the cut induced set $B(c)$, let the corresponding matrix be $\BB$\;
    }

    {$\BA^{\trans} \leftarrow \BB^{-1}$\tcp*{ensured by Lemma
        \ref{lemma:def-equi}}}
		
    {compute $\mathbb{P}(c|G)$ by   \cref{eq:combination,eq:solid-angle}\;} }
  \KwRet{$\mathbb{P}(c|G), \forall c\in \H$}
\end{algorithm}
Now we have the spanning cone $C'$ associated with the cut $c$, we
borrow the results of $n$-dimensional solid angle calculating
\citep{hajja2002measure,ribando2006measuring}: The solid angle of a
spanning cone $C'$ from  \cref{def:spanning} is given by:
\begin{equation}\label{eq:solid-angle}
E = |\det(\BA)| \int_S \|\BA\s\|_2^{-n} dS, 
\end{equation}
where 
the integral is calculated over a unit sphere $\|\s\|_2 = 1$ in
the positive orthant given by $s_i \geq 0$.

Combined with the fact that the solid angle subtended by $S_{n-1}$ is
$\Omega_n = \frac{2\pi^{\frac{n}{2}}}{\Gamma(\frac{n}{2})}$
($\Gamma(\cdot)$ is the Gamma function),
according to  \cref{eq:rounding}, 
\begin{equation}\label{eq:combination}
\prob(c|G) = 2E/\Omega_n = E\cdot  \Gamma(n/2)/\pi^{\frac{n}{2}}.
\end{equation}
The complete procedure\footnote{If the $n$ intermediate vectors
$O_t = \{\v_1, \v_2, \cdots, \v_n\}$ have mutual dependencies,
one can add small perturbations to them in order to make them
independent, and the perturbation would still be insignificant w.r.t.
vector positions.} to calculate
the posterior probability of each cut is summarized in 
\cref{alg:prob-calc}.
The way to exactly evaluate the surface integral (\cref{eq:solid-angle}) is in \cref{supp_evalute_surf}. It involves an $(n-1)$-variate integral, which is computationally intractable, we only
use it in the low dimensional case as ground truth.\\

\paragraph{Sampling to approximate posterior of cut.}
For the high dimensional case, ensured by Lemma \ref{the:pro}, we
propose one simple and efficient sampling method in 
\cref{alg:prob-sampling} to approximate the posterior: In each
iteration it uniformly samples one hyperplane with normal vector $\r$
and records the cut $c$ separated by that hyperplane, then it
estimates $\prob(c|G)$ by the statistics of each cut's frequency
of occurrence.
Theoretical analysis of approximation guarantee of 
\cref{alg:prob-sampling} and space-efficient implementation of it are in
  \cref{supp_ana_sampling} and
\cref{sec:space-eff}, respectively.

\begin{algorithm}[t]
	\caption{Approximate cut's posterior by
		sampling \citep{bian2016information}}\label{alg:prob-sampling}
	\KwIn{$\{\v_1, \v_2, \cdots, \v_n\}$ on $S_{n-1}$,
		$\#\text{samplings}$} \KwOut{approximate posterior of each cut
	} {initialize $\text{count}(c) \leftarrow 0, \forall c\in \H$\;}
	\For{\emph{each $\r$ \textit{uniformly} sampled from $S_{n-1}$}}{
		{$\tilde{c}\leftarrow (S,  V\setminus {S})$, where
			$S = \{i \;|\; \v_i\cdot \r \geq 0, \forall i\in V\}$\;}
		{$\text{count}(\tilde{c})\leftarrow \text{count}(\tilde{c}) + 1$;}
	}
	\KwRet{$\prob(c|G) \approxeq 
		{\emph{count}(c)}/{\#\emph{samplings}}, \forall c\in \H$}
\end{algorithm}

\section{Experiments}\label{sec:exp}

We compare the \MAXCUT-SDP algorithm (abbreviated as ``\algname{SDP}''
in the following) with two representative greedy \MAXCUT\ algorithms:
The double greedy \algname{D2Greedy} (\texttt{D}eterministic
\texttt{D}ouble \texttt{Greedy} algorithm in
\cite{buchbinder2012tight}), and the backward greedy \algname{EC}
(\texttt{E}dge \texttt{C}ontraction algorithm in
\citet{kahruman2007greedy}).  The way to evaluate their posteriors
(approximation sets) can be found in \cref{chapter_greedy_maxcut}.
Let $W_{\A}(G)$ be the cut value generated by an algorithm  $\A$ on graph $G$, 
$W_{*}(G)$ be the optimal cut value of $G$. 
The  approximation ratio of an algorithm $\A$ is the worst-case bound $\min_G \frac{W_{\A}(G)}{W_{*}(G)}$,
which  ranks the
three algorithms as
$\algname{SDP} \succ \algname{D2Greedy} \succ  \algname{EC}$.
Since finding $W_{*}(G)$ for NP-complete problem is non-trivial, we use
$\frac{W_{\A}(G)}{W(G)}$  as a natural 
lower bound of $\frac{W_{\A}(G)}{W_*(G)}$, where $W(G)$ denotes the  total weight of $G$.

\subsection{Experimental Setting}  

 We experimented with the
{Gaussian edge weights model} \citep{informativemst}: 
The graph instances are generated in a two-step fashion: Firstly, a
random ``master'' graph $G$ is generated with Gaussian distributed
edge weights $W_{ij} \sim N (\mu, \sigma_m^2)$,
$\mu = 300, \sigma_m = 50$, negative edges are set to be $\mu$.
Secondly, noisy graphs $G'$, $G''$ are obtained by adding Gaussian
distributed noise $n_{ij} \sim N (0, \sigma^2)$, negative edges are
set to be zero.  We perform 1000 repeated noisy samplings to estimate
the expectation over $(G', G'')$ in  \cref{eq_ic}.

\subsection{Results and Analysis}  
 \cref{fig:stepwise} shows the temporal information content
($I_t^{\mathscr{A}}$ in  \cref{eq_ic})
for two $\sigma$
values: 10 and 58. For all the algorithms,
$I_t^{\mathscr{A}}$
increases at the beginning. After reaching some optimal step
$t^*$,
where the highest $I_t^{\A}$
($I^{\A}$
in  \cref{eq_ic}) is achieved, it decreases and finally vanishes.
This observation confirms the principle of regularization by early stopping
at time $t^*$ when the maximum $I_t^{\A}$ is reached.

\setkeys{Gin}{width=0.9\textwidth,height=0.7\textwidth}
\begin{figure}[htbp]
\begin{center}
\subfloat[$\sigma  = 10$]{
  \includegraphics{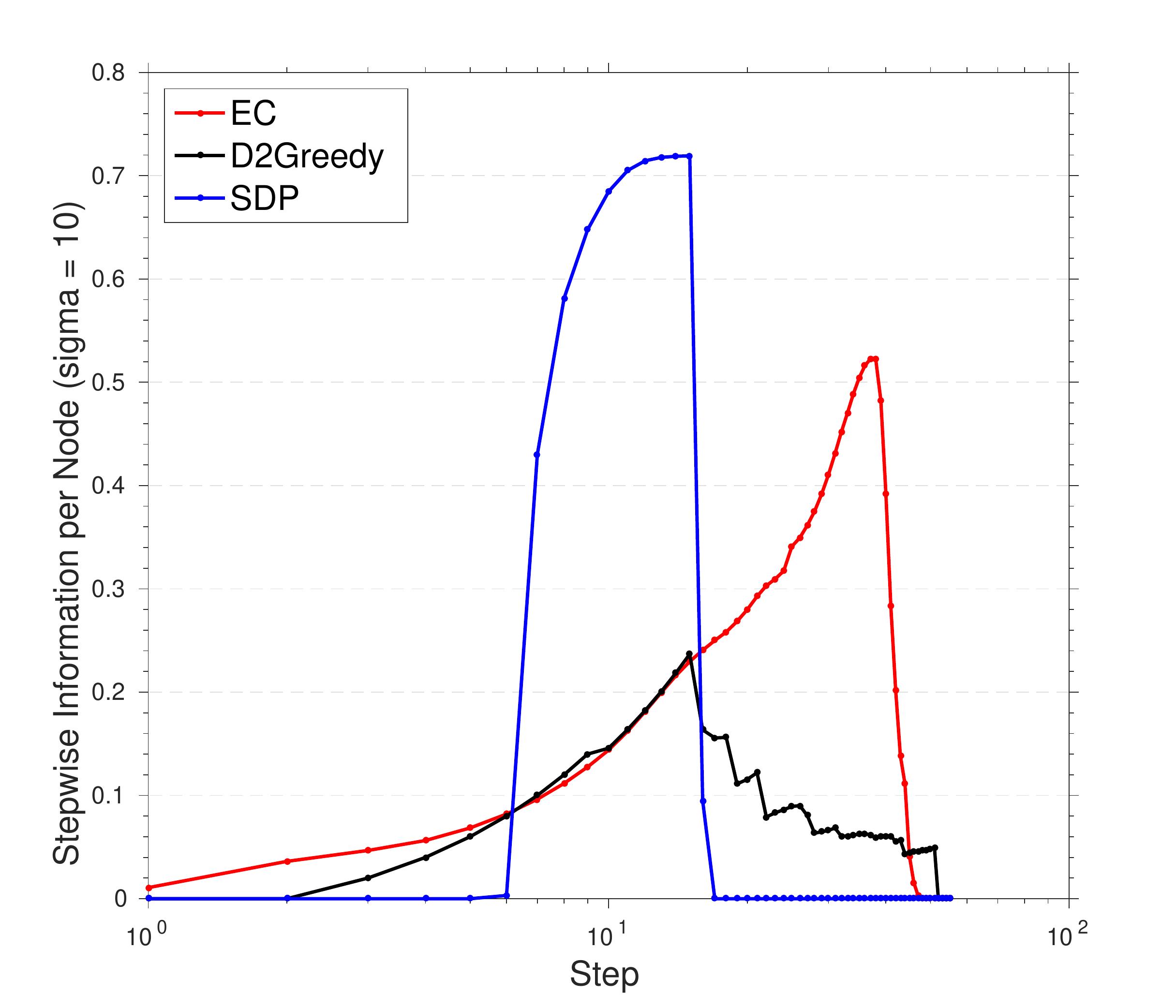}
}\\
\subfloat[$\sigma  = 58$]{
  \includegraphics{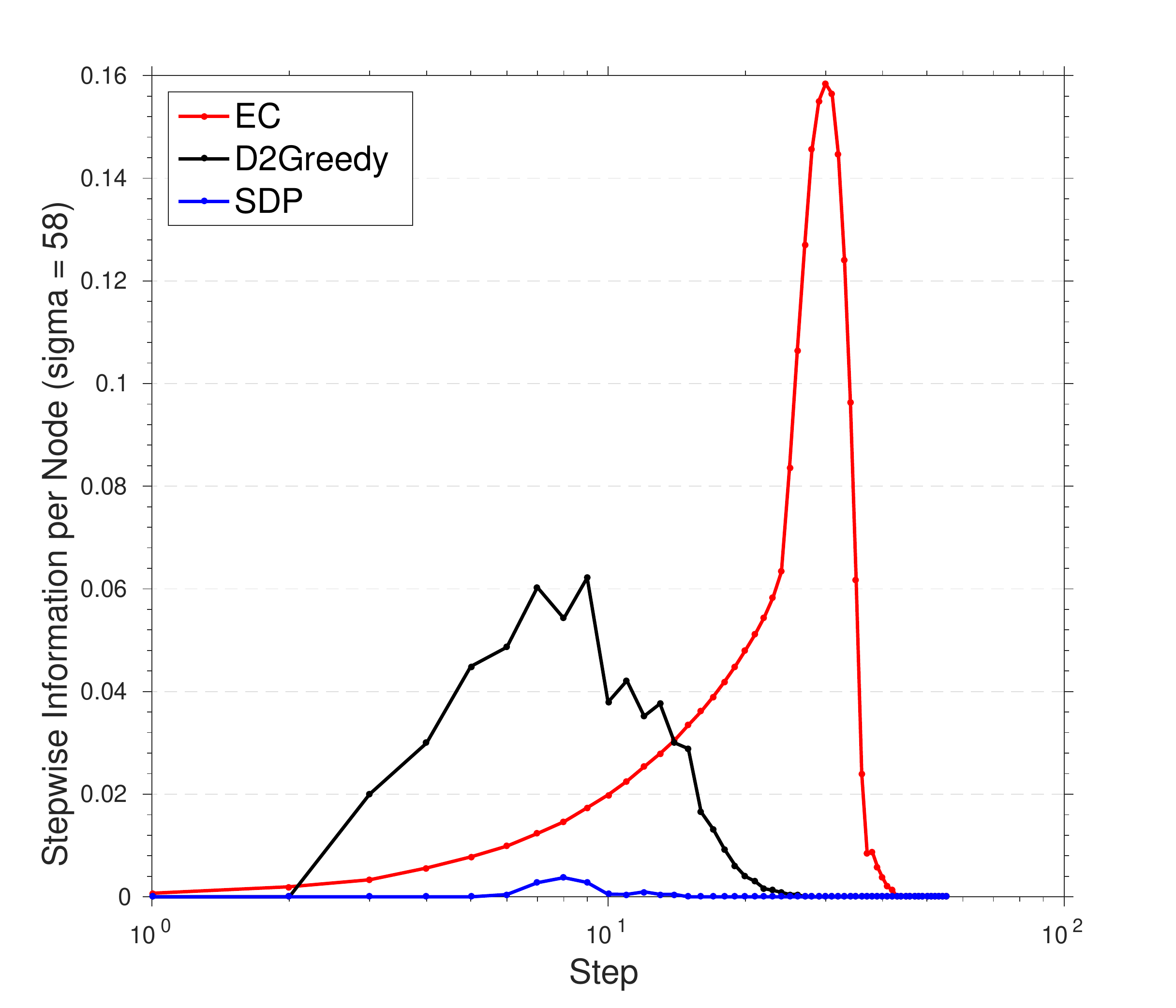}}
\end{center}
\caption{$I_t^{\mathscr{A}}$  {per vertex} w.r.t. $t$. $n = 50$.}
\label{fig:stepwise}
\end{figure}

\cref{fig:ic-gaussian} shows (a) the information content, and (b)
the fully overlapping curves
${W_{\A}(G')}/{W(G')}$,
${W_{\A}(G'')}/{W(G'')}$,
respectively. $\sigma$
controls the noise level, larger $\sigma$
means larger noise.  In the noiseless case, $G'=G''$,
so $\prob(c|G')=\prob(c|G'')$,
and $
I_t^{\mathscr{A}} = \mathbb{E}_{G', G''} [ \log ( |\H| \sum_{c\in
  \H}\prob^2(c|G')) ]
$.  All algorithms start with uniform distribution of solutions when
$t
=
0$; as the algorithm proceeds, the distribution of solutions
concentrates more and more on a small support, $\sum_{c\in
  \H}\prob^2(c|G')$ increases and reaches a maximum in the final step, so
all algorithms reach the maximum $I_t^{\A}$
in the final step. For greedy algorithms (\algname{D2Greedy} and
\algname{EC}), there is only one final solution with probability 1 in
the last step, so $\sum_{c\in
  \H}\prob^2(c|G') = 1$ and the maximum $I_t^{\A}$ is $\log (|\H|) =
\log(2^{n-1} -
1)$, as shown by  \cref{fig:ic-gaussian}(a).  For \algname{SDP},
however, when $\sigma
=
0$, \algname{SDP} can only approximately solve the input graphs, in
the last step there are several solutions with non-zero probability,
which renders its information content less than $\log (|\H|)$.

It is worth noting that for greedy algorithms (\algname{D2Greedy} and
\algname{EC}), the higher the approximation ratio is for noisy graphs,
the lower is the information content achieved by the algorithm. This
behavior is quite intuitive since high approximation ratio means
better adaptation to empirical fluctuations and, therefore,
overfitting to noisy graphs. Consequently, there will be less
agreement between the solutions of the two noisy graphs and
the information content of the algorithm drops. A similar conclusion
has also been drawn in \citet{bousquet2008tradeoffs}.

However, for the non-greedy algorithm \algname{SDP}, there are two factors
affecting its information content: The approximation ratio and its
\emph{probabilistic weighting strategy} to down-weight solutions without
discarding them.  \algname{SDP} keeps all the possible solutions,
instead of removing the bad solutions as greedy algorithms do, it
assigns less probabilistic weights to them, so it can capture some
uncertainty in the input.

\setkeys{Gin}{width=0.85\textwidth}
\begin{figure}[htbp]
\begin{center}
\subfloat[$I^{\mathscr{A}}$ \textit{per vertex} w.r.t. $\sigma$]{
\includegraphics{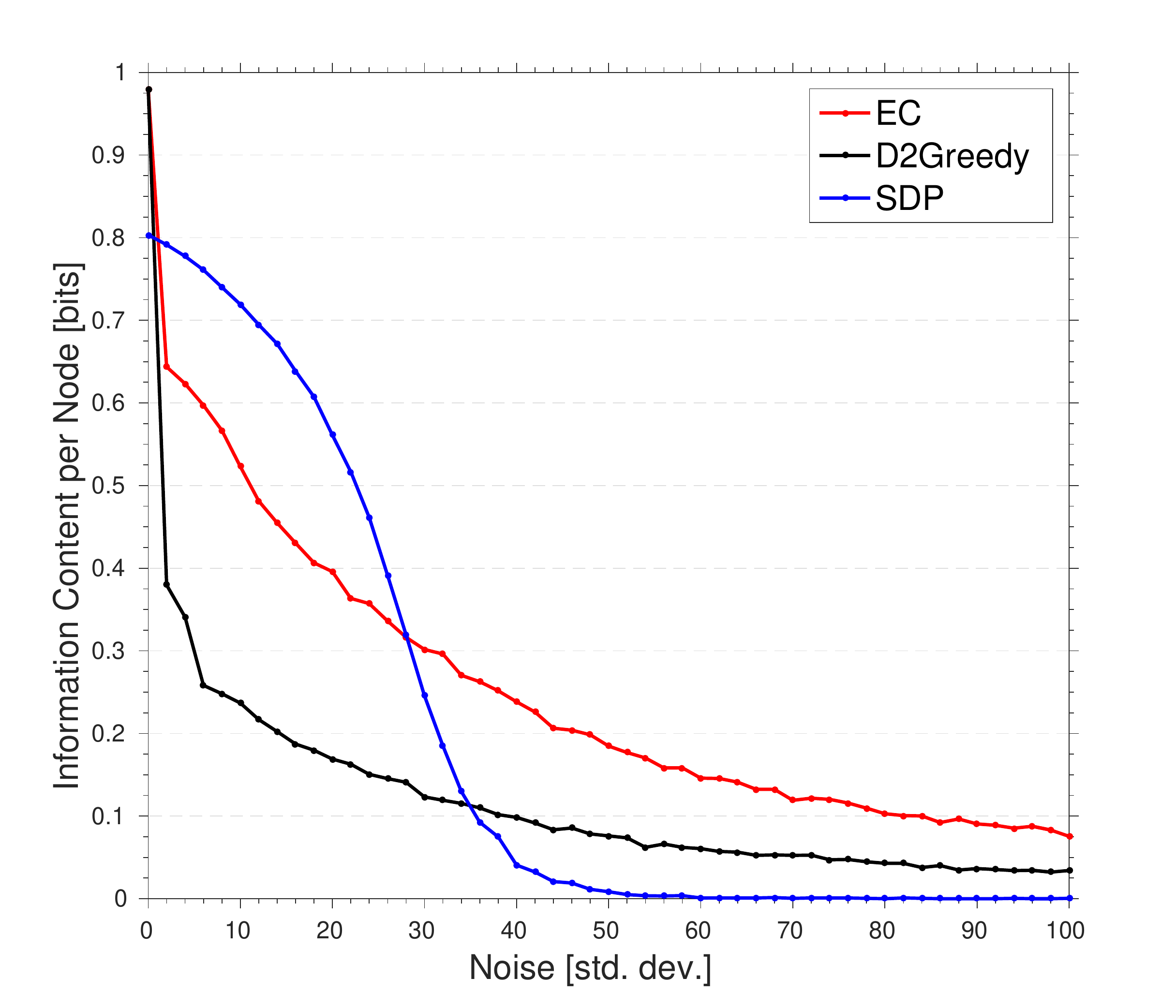}
}\\
\vspace{-.3cm}
\subfloat[$\frac{W_{\A}(G')}{W(G')}$, $\frac{W_{\A}(G'')}{W(G'')}$]{
\includegraphics{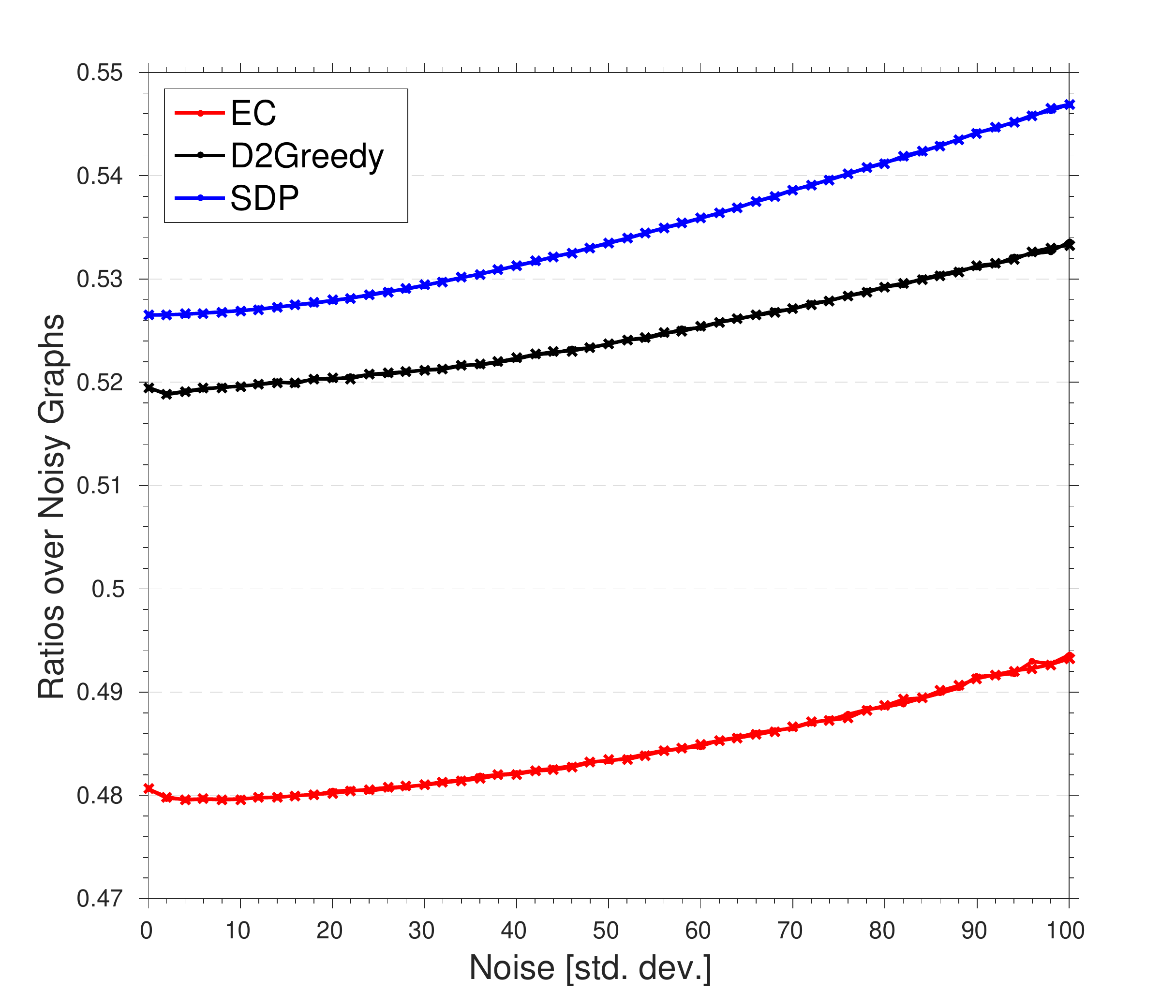}}
\end{center}
\vspace{-.3cm}
\caption{Information content and lower bounds of approximation ratios.}
\label{fig:ic-gaussian}
\end{figure}

The information content of \algname{SDP} shows the influence of
both factors: For low noise, the probabilistic weighting
strategy dominates, \algname{SDP} outperforms greedy algorithms in
information content; while in high noise level, the influence of
approximation ratio dominates, and \algname{SDP} is inferior to greedy
algorithms.

\section{Conclusions and Discussions}\label{sec:disc_con}

The objective 
$ I^{\mathscr{A}} $ in  \cref{eq_ic}
measures the information content of an algorithm given a noisy source
of instances. We have theoretically justified this criterion, and  applied it to study the robustness of
\MAXCUT\ algorithms with different approximation ratios. Of particular
interest is the SDP based algorithm by 
\citet{goemans1995improved}, since it pursues a non-greedy strategy for the \MAXCUT problem.

Comparison of \algname{SDP} with two representative greedy \MAXCUT\
algorithms (\algname{D2Greedy} and \algname{EC}) demonstrates that the
ability of this approximation algorithm to achieve a high
approximation ratio might decrease its generalization ability. The
property of an algorithm to efficiently find a good empirical minimum
might increase its fragility due to noise adaptation. This observation
could be generalized or even proved for general
approximation algorithms provided that the algorithms operate in a
similar settings or use similar optimization strategies.

The posterior agreement based criterion also enables a meta-algorithm to
search for more informative  algorithms. Algorithms are
usually tuned by parameter adaptation or by modifying the algorithmic
strategy in the spirit of genetic programming. Thereby, the
meta-algorithm will search through the space of  algorithms
guided by maximal gradient ascent on posterior agreement. With a
validation criterion as posterior agreement, we enable algorithm
engineering to explore multi-objective optimization of algorithms
with respect to time, space and robustness.

\section{Additional Details}\label{sec:proof}

\subsection{Detailed Proof in  \cref{sec:classical_mi}}

The classical mutual information is closed related to the  information content defined  in  \cref{eq_ic}.
The classical mutual information $\I(\graphrv'; \graphrv'')$ is defined as
\begin{flalign}
\I(\graphrv'; \graphrv'') 
& = \E_{\da', \da''} \log
\frac{\prob(G^\prime, G^{\prime\prime})}{\prob(G^\prime)\prob(G^{\prime\prime})} \label{eq:MI}
\end{flalign}

\begin{figure}[htbp]
	\center 
	\includegraphics[width=0.6\textwidth]{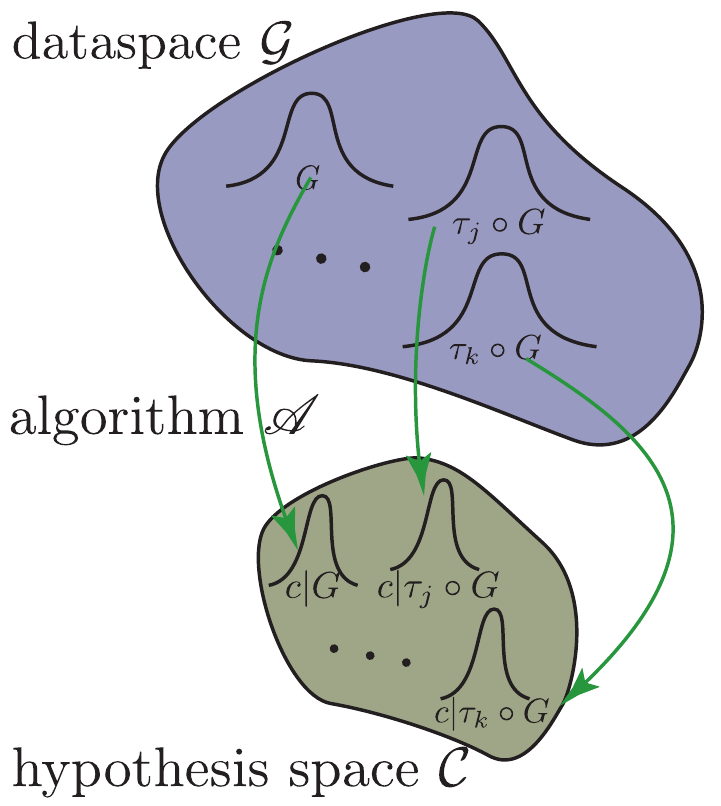} 
	\caption{Illustration of the mixture distribution}\label{fig_mix}
\end{figure}

From the definition of virtual communication scenario, the data instances 
$G', G''$ can be treated to be drawn from a mixture distribution, as illustrated in  \cref{fig_mix}.

We consider the special transformations that map from input space to output spaces, then the joint probability can be factorized in the following way,

\begin{flalign} 
&\prob(\da', \da'') = \sum_{c\in \H} \sum_{\CB \in \T^M} \prob(G^\prime,
\da^{\prime\prime}| c, \CB) \prob(c, \CB)\\\notag 
& = \sum_{j=1}^{M} p_j  \sum_{\t_j \in \T}\sum_{c\in \H} \prob(G^\prime,
\da^{\prime\prime}|\tau_j, c ) \prob(\tau_j, c)\\\notag
& \overset{(a)}{=}  \sum_{j=1}^{M} p_j  \sum_{\tau_j \in \T}\sum_{c\in \H} \prob(G^\prime|\tau_j, c )\prob(G^{\prime\prime}|\tau_j, c )  \prob(c|\tau_j)\prob(\tau_j)\\\notag
&\overset{(b)}{=}  \sum_{j=1}^{M} p_j  \sum_{\tau_j \in \T}\prob(\tau_j)\sum_{c\in \H} \frac{\prob(c | \da^\prime,
	\t_j)}{\prob(c|\t_j)} \prob(\da^\prime|\t_j)  
\frac{\prob(c |\da^{\prime\prime}, \t_j)}{\prob(c|\t_j) }
\prob(G^{\prime\prime}|\t_j) \prob(c|\t_j) \\\notag
& \overset{(c)}{=} \sum_{j=1}^{M} p_j  \sum_{\tau_j \in \T} \prob(\da^\prime|\t_j)  \prob(G^{\prime\prime}|\t_j) \prob(\t_j)
\underbrace{\sum_{c\in \H} \frac{\prob(c | \da^\prime,
		\t_j)\prob(c |\da^{\prime\prime}, \t_j)}{\prob(c|\t_j)} }_{=\tilde{k}(\t_j\circ \da', \t_j\circ \da'')}.
\end{flalign}
Step $(a)$ exploits the fact that conditioning on the %
transformation $\t_j$ and hypothesis $c$ renders $G^\prime, G^{\prime\prime}$
statistically independent since the two instances are drawn
i.i.d. from the same component of the mixture distribution; $(b)$ applies Bayes rule twice.
In step $(c)$ we define the \textit{generalized posterior agreement}  as $$\tilde{k}(\da',  \da''):=\sum_{c\in \H} \frac{\prob(c | \da^\prime)\prob(c |\da^{\prime\prime})}{\prob(c)}.$$
From condition $2)$ of the definition of the transformation set $\T$, one can get that 
\begin{flalign}\label{eq25}
\tilde{k}(\t_j\circ \da', \t_j\circ \da'')=\tilde{k}(\da',  \da'')
\end{flalign}
Combining  \cref{eq25} with $(c)$ one can get,
\begin{flalign}
& \frac{\prob(\da', \da'')}{\prob(\da')\prob(\da'')}  \\
&= \tilde{k}(\da',  \da'')  \underbrace{\frac{\sum_{j=1}^{M} p_j  \sum_{\tau_j \in \T} \prob(\da^\prime|\t_j)  \prob(G^{\prime\prime}|\t_j) \prob(\t_j)}{\sum_{j=1}^{M} p_j  \sum_{\t_j \in \T} \prob(\da^\prime|\t_j) \prob(\t_j)\sum_{l=1}^{M} p_l  \sum_{\t_l \in \T} \prob(\da''|\t_l) \prob(\t_l)}}_{\text{to be proved $=|\T|$}}\\\label{eq27}
& = \tilde{k}(\da',  \da'') |\T|.
\end{flalign}

Let us prove  \cref{eq27} first of all. We simplify the term,
\begin{flalign}  
& \frac{\sum_{j=1}^{M} p_j  \sum_{\tau_j \in \T} \prob(\da^\prime|\t_j)  \prob(G^{\prime\prime}|\t_j) \prob(\t_j)}{\sum_{j=1}^{M} p_j  \sum_{\t_j \in \T} \prob(\da^\prime|\t_j) \prob(\t_j)\sum_{l=1}^{M} p_l  \sum_{\t_l \in \T} \prob(\da''|\t_l) \prob(\t_l)} = \\\notag
& \frac{\sum_{j=1}^{M} p_j  \sum_{\tau_j \in \T} \prob(\da^\prime|\t_j)  \prob(G^{\prime\prime}|\t_j) \prob(\t_j)}{\sum_{j=1}^{M}\sum_{l=1}^{M}\sum_{\t_j \in \T}  \sum_{\t_l \in \T} p_j p_l \prob(\t_j) \prob(\t_l) \prob(\da^\prime|\t_j)    \prob(\da''|\t_l)} = \\\notag 
& \frac{\sum_{j=1}^{M} p_j  \sum_{\tau_j \in \T} \prob(\da^\prime|\t_j)  \prob(G^{\prime\prime}|\t_j) \prob(\t_j)}{\sum_{j=1}^{M}\sum_{\t_j \in \T}   p_j^2 \prob(\t_j)^2 \prob(\da^\prime|\t_j)    \prob(\da''|\t_l)+ \sum_{j\neq l}\sum_{\t_j \in \T}  \sum_{\t_l \in \T} p_j p_l \prob(\t_j) \prob(\t_l) \prob(\da^\prime|\t_j)    \prob(\da''|\t_l)}  
\end{flalign}

We further make the simplifying assumption that $\prob(\t_j)  = \prob(\t_l)= 1/|\T|$, then
\begin{flalign}\notag 
& \Rightarrow  \frac{\frac{1}{|\T|}\sum_{j=1}^{M} p_j  \sum_{\tau_j \in \T} \prob(\da^\prime|\t_j)  \prob(G^{\prime\prime}|\t_j)}{\frac{1}{|\T|^2}\left(\sum_{j=1}^{M}\sum_{\t_j \in \T}   p_j^2 \prob(\da^\prime|\t_j)    \prob(\da''|\t_l)+ \sum_{j\neq l}\sum_{\t_j \in \T}  \sum_{\t_l \in \T} p_j p_l \prob(\da^\prime|\t_j)    \prob(\da''|\t_l)\right)} \\\notag
&  = |\T|  \frac{\sum_{j=1}^{M} p_j  \sum_{\tau_j \in \T} \prob(\da^\prime|\t_j)  \prob(G^{\prime\prime}|\t_j)}{\left(\sum_{j=1}^{M}p_j^2\sum_{\t_j \in \T}    \prob(\da^\prime|\t_j)    \prob(\da''|\t_l)+ \sum_{j\neq l}p_j p_l\sum_{\t_j \in \T}  \sum_{\t_l \in \T}  \prob(\da^\prime|\t_j)    \prob(\da''|\t_l)\right)} \\\notag
&  \overset{p_j = \Delta_{js}}{=} |\T|  \frac{ \sum_{\tau_s \in \T} \prob(\da^\prime|\t_s)  \prob(G^{\prime\prime}|\t_s)}{\sum_{\tau_s \in \T} \prob(\da^\prime|\t_s)  \prob(G^{\prime\prime}|\t_s)+ \sum_{j\neq l}\underbrace{\Delta_{js}\Delta_{ls}}_{\text{$=0$ for $l\neq j$}} \sum_{\t_j \in \T}  \sum_{\t_l \in \T}  \prob(\da^\prime|\t_j)    \prob(\da''|\t_l)} \\\notag
& = |\T|.
\end{flalign} 

Inserting  \cref{eq27} into
 \cref{eq:MI} proves the claim that the mutual information is identical to  the information content if assuming $\prob(c)=\vert \mathcal{C}\vert^{-1}$,
\begin{flalign}\notag 
\I(\graphrv'; \graphrv'')  
& = \E_{\da', \da''} \log
\frac{\prob(G^\prime, G^{\prime\prime})}{\prob(G^\prime)\prob(G^{\prime\prime})} \\\notag
& =  \E_{\da', \da''} \log \kert(\da', \da'') \\\notag
& = \E_{\da', \da''} \log  \sum_{c\in \H}  \frac{\prob(c | \da^\prime)\prob(c | \da'')}{\prob(c)}\\\notag
&= \E_{\da', \da''} \log |\H| \sum\nolimits_{c\in \H} \prob(c|G')\prob(c|G'')\\\notag
& = I(\graphrv'; \graphrv'').
\end{flalign}

\subsection{Proof of Lemma \ref{the:pro}}

\begin{proof}
  For a specific cut $c := (S, \bar{S})$, assume the collections of
  normal vectors of all hyperplanes that give the cut $c$ is $R(c)$,
  according to the random hyperplane rounding technique,
  \begin{flalign}
    R(c) =& \{\r\in S_{n-1} \;|\; \r\cdot \b_i \geq 0,\forall i\in V
    \}\\\notag & \cup \{\r\in S_{n-1} \;|\; \r\cdot \b_i \leq 0,
    \forall i\in V \}
  \end{flalign}
  So $R(c)$ is two times the unit spherical surface of
  $C(c)$. Considering the fact that normal vectors of all hyperplanes
  constitute the surface of unit sphere, we get the first equality in
  \cref{eq:rounding}.  Using simple geometrical knowledge, we can get
  the second and third equalities.
\end{proof}

\subsection{Proof of Lemma \ref{lemma:def-equi}}

\begin{proof}
 
  $C' = C \Leftrightarrow$ any point in $C'$ must be in $C$
  $\Leftrightarrow$
  $(\sum_1^{n} \lambda_i \a_i)\cdot \b_j \geq 0, 1\leq i, j \leq n$
  holds $\forall \lambda_i \geq 0$ $\Leftrightarrow$
  $(\lambda_1, \cdots, \lambda_n)\cdot \BA^\trans \BB \geq 0$ holds
  $\forall (\lambda_1, \cdots, \lambda_n) \geq 0$

  So if
  $\exists k_1, \cdots, k_n >0, \text{ s.t.}\; \BA^\trans =
  \text{diag}(k_1, \cdots, k_n) \BB^{-1}$,
  one can get that $(\lambda_1, \cdots, \lambda_n)\cdot \BA^\trans \BB \geq 0$
  holds $\forall (\lambda_1, \cdots, \lambda_n) \geq 0$, so $C' = C$.
\end{proof}

\subsection{The Way to Exactly Evaluate the Surface Integral}
 \label{supp_evalute_surf}

Exactly calculating probability of cuts involves evaluating the high
dimensional surface integral in  \cref{eq:solid-angle}, To do this,
we first of all parametrize it using spherical polar coordinates, then
transform it to be a multivariate integral. Writing
$\s = \sum_{1}^{n} s_i \e_i$, we get:
\begin{equation}\label{eq:a}
  \|\BA\s\|_2^2 = \sum_{i = 1}^{n} \a_i\cdot \a_i s_i^2 +
  2\sum_{i<j}\a_i\cdot \a_j s_i s_j =  1+ 2\sum_{i<j}\a_i\cdot \a_j
  s_i s_j.
\end{equation}
Plugging  \cref{eq:a} into \cref{eq:solid-angle} one can express the
surface integral in a more manageable form
\begin{flalign}\label{eq:manageable}
E = |\det(\BA)| \int_S (1+ 2\sum_{i<j}\a_i\cdot \a_j s_i s_j)^{-n/2} dS 
= |\det(\BA)| \int_S f^{-n/2}(\s) dS ,
\end{flalign}
where $f(\s) = 1+ 2\sum_{i<j}\a_i\cdot \a_j s_i s_j$.  Then
parametrizing by spherical polar coordinates
$\bmtheta = (\theta_1, \cdots, \theta_{n-1})$:
\begin{flalign}\label{eq:polar-trans}
s_i =\cos(\theta_i)\prod_{j=1}^{i-1}\sin(\theta_j), i = 1, \cdots, n-1; 
s_n=\prod_{i =1}^{n-1}\sin(\theta_i)
\end{flalign}
for $0\leq \theta_i \leq \pi/2, 0\leq i \leq n-1$,  considering that the Jacobian is
$\prod_{i =1}^{n-2}{\sin}^{n-1-i}(\theta_i)$, substitute  \cref{eq:polar-trans} to 
  \cref{eq:manageable} it reaches the multivariate
integral:
\begin{equation}\label{eq:multivariate}
E=|{\det}(\BA)| \int_{\theta_1}\cdots \int_{\theta_{n-1}} \frac{\prod_{i =1}^{n-2}\sin^{n-1-i}(\theta_i)}{f^{n/2}(\bmtheta)} d\theta_1  \cdots d\theta_{n-1}.
\end{equation}

\subsection{Theoretical Analysis of  \cref{alg:prob-sampling}}
\label{supp_ana_sampling}

We will show that for a cut $c$ with high ground truth probability
$p_c := \mathbb{P}(c|V)$, the estimated cut probability $\hat p_c$ by  uniform sampling in  \cref{alg:prob-sampling} will be close to $p_c$ with high probability. 

Let $k = \#\text{samplings}$, random variable $X_i = 1$ means
recovering cut $c$ in the $i^\text{th}$ sampling, $X_i = 0$
means not recovering $c$ in the $i^\text{th}$ sampling. So
$\hat p_c = \sum_{i=1}^k X_i / k$, from the Chernoff-Hoeffding theorem \citep{hagerup1990guided},
for $\epsilon >0$,
\begin{algorithm}[ht]
	\caption{Pseudo-code to calculate estimate of $\sum_c \prob(c|G')\prob(c|G'')$ when $k<|\H|$ \citep{bian2016information}}\label{alg:implementation}
	\KwIn{$\texttt{cutIndices}(G')$, $\texttt{cutIndices}(G'')\in \R^k$, wherein the indices are in ascending order}
	\KwOut{estimate of $\sum_c \prob(c|G')\prob(c|G'')$}
	{initialize $\texttt{idx1} = \texttt{idx2} = 1$, $\texttt{sum} = 0$\;} 
	\While{$\texttt{idx1}\leq k\; \&\&\; \texttt{idx2}\leq k$}{
		\If{$\texttt{cutIndices}(G')_{\texttt{idx1}} == \texttt{cutIndices}(G'')_{\texttt{idx2}}$}{
			{$\texttt{commonIdx} = \texttt{cutIndices}(G')_{\texttt{idx1}}$\;}
			{$\texttt{cutNum1} = \texttt{cutNum2} = 1$\;}
			{$\texttt{idx1}++, \texttt{idx2}++$ \;}
			\While{$\texttt{idx1}\leq k\; \&\&\; \texttt{commonIdx} ==  \texttt{cutIndices}(G')_{\texttt{idx1}}$}{
				{$\texttt{cutNum1}++, \texttt{idx1}++$\;}
			}
			\While{$\texttt{idx2}\leq k\; \&\&\; \texttt{commonIdx} ==  \texttt{cutIndices}(G'')_{\texttt{idx2}}$}{
				{$\texttt{cutNum2}++, \texttt{idx2}++$\;}
			}
			{$\texttt{sum} +=\texttt{cutNum1}* \texttt{cutNum2}$\;}
		}
		\ElseIf{$\texttt{cutIndices}(G')_{\texttt{idx1}} < \texttt{cutIndices}(G'')_{\texttt{idx2}}$}{
			{$\texttt{idx1}++$\;}
		}
		\Else{
			{$\texttt{idx2}++$\;}
		}
	}
	\KwRet{$\sum_c\prob(c|G')\prob(c|G'') \approxeq \texttt{sum}/k^2$}
\end{algorithm}
\begin{align}\notag  
\prob(\hat p_ c \geq p_c + \epsilon) &\leq \left [\left (\frac{p_c}{p_c + \epsilon}\right )^{p_c + \epsilon}\left (\frac{1-p_c}{1-p_c - \epsilon}\right)^{1-p_c-\epsilon}\right ]^k = e^{-D(p_c + \epsilon || p_c)k}\\ 
&\leq  \left( \frac{p_c}{p_c +\epsilon}\right)^{k(p_c + \epsilon)}\cdot e^{k\epsilon}, \\\notag 
\prob(\hat p_ c \leq p_c - \epsilon) &\leq \left [\left (\frac{p_c}{p_c - \epsilon}\right )^{p_c - \epsilon}\left (\frac{1-p_c}{1-p_c + \epsilon}\right)^{1-p_c + \epsilon}\right ]^k \\ 
& = e^{-D(p_c - \epsilon || p_c)k}\leq  \left( \frac{p_c}{p_c -\epsilon}\right)^{k(p_c - \epsilon)}\cdot e^{-k\epsilon},
\end{align}
where $D(\cdot)$ is the Kullback-Leibler divergence between two Bernoulli random variables.

So to ensure that with probability at most $\delta < 1$, 
the estimated probability $\hat p_c$ is at most $\epsilon$-distant from the true probability $p_c$, one 
need to ensure that:
\begin{equation} 
	\max(e^{-D(p_c + \epsilon || p_c)k}, e^{-D(p_c - \epsilon || p_c)k}) \leq \delta,
\end{equation} 
which is equivalent to:
\begin{equation}\label{eq:bound-sampling}
	k \geq \max\left( \frac{-\ln \delta}{D(p_c + \epsilon || p_c)}, \frac{-\ln \delta}{D(p_c - \epsilon || p_c)} \right),
\end{equation}
which gives the lower bound of the sampling number $k$ required to recover the ground truth $p_c$  
  with probability $\delta$ at a specific error level $\epsilon$.

\subsection{Space-Efficient Implementation of  \cref{alg:prob-sampling}}\label{sec:space-eff}
When sampling number  $k\ge |\H|$, use array $\texttt{cutFrequency}\in \R^{|\H|}$ to record 
cuts' frequency of occurrence, and the posterior agreement $\sum_c \prob(c|G')\prob(c|G'')$ is estimated as the inner product $\langle\texttt{cutFrequency}(G'), \allowbreak \texttt{cutFrequency}(G'')\rangle$.

When $k<|\H|$, use array $\texttt{cutIndices}\in \R^{k}$ to record indices of 
sampled cuts in each sampling, note that there would be duplicated cuts in $\texttt{cutIndices}$.
Then sort the array $\texttt{cutIndices}$ to make the indices in it be in ascending order. 
Finally, use the way described by  the  pseudo-code in  \cref{alg:implementation}  to calculate estimate of  posterior agreement $\sum_c \prob(c|G')\prob(c|G'')$.

\def\dir{chapters/mean-field-pa}
\chapter{Provable Mean Field Approximation via
	Continuous DR-Submodular Maximization}
\chaptermark{Provable Mean Field Approximation}
\label{chapter_mean_field}

\begin{chapquote}{Bruce Lee}
You must be shapeless, formless, like water. When you pour water in a cup, it becomes the cup. When you pour water in a bottle, it becomes the bottle. When you pour water in a teapot, it becomes the teapot. Water can drip and it can crash. Become like water my friend.
\end{chapquote}

Mean field inference 
in probabilistic  models
is generally a
highly non-convex problem.  
Existing optimization methods, e.g., coordinate ascent algorithms, can
only generate local optima.
In this chapter we discuss provable mean field methods for probabilistic
log-submodular models and its posterior agreement (PA) via continuous DR-submodular maximization.
The main algorithmic technique is the  \drdg algorithm  for {continuous} DR-submodular
maximization with box-constraints.  
We validate the superior performance of our algorithms with
baseline results on real-world datasets.

\section{Why Do We Need Provable Mean Field Methods?}

Consider the following scenario: You want to build a recommender
system for $n$ products to sell. Let $\groundset$ contain all the
products. The system is expected to recommend a subset of products
$S\subseteq \groundset$ to the user.  This recommendation 
should reflect relevance and diversity of the user's choice, such that
it will raise the readiness to buy.  

The two most important components
in building such a system are (1) learning a utility function $F(S)$,
which measures the utility of any subset of products, and (2) inference,
i.e., finding the subset $\opt$ with the highest utility given the
learnt utility function $F(S)$.
The above task can be achieved by using a class of probabilistic
graphical models that devise a distribution on all subsets of
$\groundset$. Such a distribution is known as a point process. 
Specifically, it defines $p(S)\propto \exp(F(S))$, which renders subset
of products $S$ with high utility to be very likely suggested.  In
general, inference in point processes is \#P-hard.  One resorts to
approximate inference methods via either variational techniques
\citep{wainwright2008graphical} or sampling.

Both of the two components in the  recommender system example above can
be achieved via provable mean field methods since (i) the latter  provide
approximate inference given a utility function $F(S)$ and, (ii) by using
proper differentiation techniques, the iterative process of mean field
approximation can be unrolled to serve as a differentiable layer
\citep{zheng2015conditional}, thus enabling backpropagation of the
training error to parameters of $F(S)$.  Thereby, learning $F(S)$ in
an end-to-end fashion can utilize modern deep learning and stochastic
optimization techniques.

The most important property which we require on $F(S)$ is
\emph{submodularity},
which naturally models relevance and diversity. 
\citet{djolonga14variational} have used submodular functions $F(S)$ to
define two classes of point processes: 
$p(S)\propto \exp(F(S))$ is
termed probabilistic log-submodular models, while
$p(S)\propto \exp(-F(S))$ is called  probabilistic log-supermodular models.  
They
are strict generalizations of classical point processes, such as DPPs
\citep{kulesza2012determinantal}.
The variational techniques from
\citet{djolonga14variational,djolonga16mixed} focus on giving
tractable upper bounds of the log-partition functions.  This work
provides provable \emph{lower} bounds through mean field approximation, which
also completes the picture of variational inference for probabilistic
submodular models (PSMs).

\paragraph{Typical Application Domains.}
Recommender systems are just one illustrating example. There are
numerous scenarios that can benefit from the mean field method in this
work. These settings include, but not limited to, existing applications of
submodular models, such as diversity models
\citep{Tschiatschek16diversity,djolonga16mixed}, experimental design
using approximate submodular objectives
\citep{bianicml2017guarantees}, variable selection
\citep{Krause05nearoptimalnonmyopic}, data summarization
\citep{lin2011class},
dictionary learning \citep{krause2010submodular} etc.
Another category of applications is conducting model validation using
information-theoretic criteria. In order to infer the hyperparamters
in the model  $F(S)$, practitioners do validation by
splitting the training data into multiple folds, and then train models
on them.  Posterior Agreement (PA,
\citep{Buhmann10isit,bian2016information}) provides an
information-theoretic criterion for the models trained on these
folds, to measure the fitness of one specific hyperparameter
configuration.  We will show in \cref{sec_pa} that PA can be efficiently
approximated by the techniques developed in this work.

\subsection{A Shortcoming of Classical Mean Field Method}
The most frequently used algorithm for mean field approximation is
the \algname{\nmf} algorithm\footnote{It is known under various names in
the literature, e.g., iterated conditional modes (ICM), naive mean
field algorithm, etc.}. It maximizes the ELBO objective in a
coordinate-wise manner, which is detailed in \cref{alg_ca}.
\algname{\nmf} has been shown to reach stationary points/local
optima. However, local optima may be arbitrarily poor,  and \algname{\nmf} would
get stuck in these poor local optima without extra techniques, which
motivates our pursuit to develop provable methods.

Below we show that there may exist poor local optima for 
problems with the same structure as the ELBO objective, which will be formalized in  \labelcref{opt_problem_meanfield}.

\begin{algorithm}[ht]
	\caption{The \ca algorithm}\label{alg_ca} 
	\KwIn{ $\max_{\x \in [\a, \b]}f(\x)$,  $f(\x)$ is
		{\color{blue}DR}-submodular,  $[\a,\b]\subseteq  \X$, \#iterations $K$
	}
	{Initialize $\x^0 \in  [\a, \b]$, $k \leftarrow 1$\;}
	\While{$k\leq K$}{
		{ let $v_k$  be the coordinate being operated\;}
		{find
			$ u_a$ \text{ such that }
			$f(\sete{x^{k-1}}{\ele_k}{u_a}) \geq \max_{u'}
			f(\sete{x^{k-1}}{\ele_k}{ u'})$\;}

		{$\x^k \leftarrow \sete{x^{k-1}}{\ele_k}{u_a}$\;} 
		
		{$k \leftarrow k + 1$\;}
	}
	\KwOut{$\x^\pare{K}$ }
\end{algorithm}

\paragraph{There Exist Poor Local Optima.}
\label{sec_bad_localoptima}

If one  only assume the objective function $f(\x)$
to be continuous DR-submodular, and considering that 
the multilinear extension of a submodular set function 
is continuous DR-submodular, we can take the 
examples from literature on combinatorial optimization, e.g., 
\citet{feige2011maximizing}, 
to show that bad  local optima exist. 

Here we provide a \textit{stronger} example,
where we  assume that the objective function $f(\x)$
 has the 
same structure as the ELBO objective in  \labelcref{opt_problem_meanfield}. And still there exist 
bad local optima. 
These local optima have  arbitrarily
small objective value  compared to the global optimum. 
And \algname{\nmf}
will get stuck in this local optimum without extra techniques. 

Suppose that we have a directed graph $G = (\groundset, A)$ with 
 four vertices, $\groundset = \{ 1,2,3,4 \} $ and four  arcs, $A = \{ (1,2), 
(2,3), (3,2), (3,4)  \}$. The weights of the arcs are (let $b,c$  be
large positive numbers): $w_{1,2} = c$, $w_{2,3} = c$, $w_{3,4} = c$, $w_{3,2} = bc$.
Let 
$F(S)$ denote the sum of weights of arcs
leaving $S$.  Consider its ELBO (using techniques from \cref{gibbs_multilinear}), 
\begin{align}
f(\x)&  = \multi(\x) + \sum_{i \in\groundset} H(x_i) \\ \notag
& =  \sum_{(i,j)\in A} w_{ij} x_i (1- x_j) +  \sum_{i \in\groundset} H(x_i)\\ \notag
& = c x_1(1-x_2) + c x_2(1-x_3) + c x_3 (1-x_4) + bc x_3 (1-x_2) +  \sum_{i \in\groundset} H(x_i).
\end{align}

Consider  the point $\y = [0.5, 1, 0, 0.5 ]^\trans$,  it has  function value $f(\y) = c + 2\log 2$. Consider a second   point $\bar \x = [ 1,0,1,0 ]^\trans$,  while the global optimum $f(\x^\star)$ must be greater than $f(\bar \x) = (2+b)c$. When $b$ becomes  large, the ratio $\frac{f(\y)}{f(\x^\star)} \leq \frac{c + 2\log 2}{(2+b)c}$ can be arbitrarily small.

{\algname{\nmf}\ may get stuck on the point $\y = [0.5, 1, 0, 0.5 ]^\trans$.}
This can be illustrated by  considering the course  of \algname{\nmf}. 
Suppose wlog. that  \algname{\nmf}\ processes  coordinates in the order of $1\rightarrow 4$ (actually it is
the same with any orders). 

For coordinate 1, 
$\nabla_1 \multi(\x) = c (1-x_2) $, so $\nabla_1 \multi(\y) = 0$, after applying $\sigma(\nabla_1 \multi(\y) )$,
$y_1$ remains to be $0.5$. 

For coordinate 2, $\nabla_2 \multi(\x) = c (1-x_3) - bc x_3 $, so $\nabla_2 \multi(\y) = c$. When $c$ is sufficiently large (approaching infinity), after applying $\sigma(\nabla_2 \multi(\y))$, $y_2$ will still be 1.

For coordinate 3, $\nabla_3 \multi(\x) = - c x_2 + c (1-x_4) +  bc (1-x_2) $, so $\nabla_3 \multi(\y) = -0.5 c$. When $c$ is sufficiently large (approaching infinity), after applying $\sigma(\nabla_3 \multi(\y))$, $y_3$ will still be $0$.

For coordinate 4, 
$\nabla_4 \multi(\x) = -c x_3 $, so $\nabla_4 \multi(\y) = 0$, after applying $\sigma(\nabla_4 \multi(\y) )$,
$y_4$ remains to be $0.5$.

\section{Problem Statement and Related Work}

All of the mean field approximation problems investigated in this chapter
fall into the  problem of continuous DR-submodular maximization:
\begin{align}\label{opt_problem}
\underset{{ \x \in [\a, \;\b]}}{{\text {maximize}}} \;\; f(\x),   
\end{align}
where $f: \X \rightarrow \R$ is continuous DR-submodular.

\paragraph{Background and Related Work. 
}
\label{cont_subopt_bk}
Submodularity is one of the most well studied properties in combinatorial
optimization and many applications for machine learning, with strong
implications for both guaranteed minimization and approximate
maximization in polynomial time \citep{krause2012submodular}.  
Continuous extensions of submodular set functions play an 
important role in submodular optimization, representative 
instances include Lov{\'a}sz extension \citep{lovasz1983submodular}, multilinear extension \citep{calinescu2007maximizing,DBLP:conf/stoc/Vondrak08,chekuri2014submodular,chekuri2015multiplicative}
and the softmax extension for DPPs \citep{gillenwater2012near}.
These
guaranteed optimizations have been advanced to continuous domains
recently, for both minimization
\citep{bach2015submodular,staib2017robust} and maximization
\citep{bian2017guaranteed,biannips2017nonmonotone,Wilder2017RiskSensitiveSO,chen2018online,mokhtari2018decentralized}.
Specifically,  \citet{bach2015submodular} studies continuous
submodular minimization without constraints. He also discusses the
possibility of using the technique for mean field inference of
probabilistic log-supermodular models.
\citet{bian2017guaranteed,biannips2017nonmonotone} characterize
continuous submodularity using the weak  DR property and propose provable
algorithms for maximization.

Most related to this chapter  is the classical problem of {unconstrained
  submodular maximization (USMs)}, which has been studied in binary
\citep{buchbinder2012tight}, integer \citep{soma2017non} and
continuous domains \citep{bian2017guaranteed}.
For the general problem \labelcref{opt_problem}, at first glance one
may consider discretization-based methods: Discretizing the continuous
domain and transform problem \labelcref{opt_problem} to be an integer
optimization problem, then solve it using the reduction
\citep{ene2016reduction} or the integer Double Greedy algorithm
\citep{soma2017non}.  However, discretization-based methods are not
practical for problem \labelcref{opt_problem}: Firstly discretization will
inevitably introduce errors for the original continuous
problem \labelcref{opt_problem}; Secondly, the computational cost is
too high\footnote{e.g., the method from \citet{soma2017non} reaches
	$\frac{1}{2+\epsilon}$-approximation in
	$O( \frac{|\groundset|}{\epsilon}) \log( \frac{\Delta}{\delta})
	\log(B)(\theta + \log (B) )$
	time, $B$: \#grids of discretization, $\Delta$: the maximal
  positive marginal gain, $\delta$: minimum positive marginal gain}.
Thus we turn to continuous methods.
The \shrunkenfw in \citet{biannips2017nonmonotone} provides a
$1/e$ approximation guarantee and sublinear rate of convergence for problem 
\labelcref{opt_problem}, but it is still computationally too
expensive:
In each iteration it has to calculate the full gradient, which costs
$n$ times as much as computing a partial derivative.

Based on the above analysis, the most promising algorithm to consider
would be the \drdg  algorithm presented in  \cref{chapter_doublegreedy}, which
needs to solve  $\bigo{n}$ 1-D subproblems, and achieves a tight  $1/2$
guarantee for continuous \emph{DR-submodular} maximization.

{\palong} (PA) is developed as an information-theoretic criterion for
model selection \citep{gorbach2017model} and algorithmic validation
\citep{informativemst,bian2016information}.  It originates from the
approximation set coding framework proposed by \citet{Buhmann10isit}.
Recently, \citet{buhmannaposterior} prove rigorous asymptotics of PA on
two typical combinatorial problems: Sparse minimum bisection and
Lawler's quadratic assignment problem.
\citet{djolonga14variational,djolonga15scalable} study variational
inference for PSMs, they propose L-Field to give upper bounds for
log-supermodular models through optimizing the
subdifferentials.

\section{Application to Classical Mean Field Inference}

Mean field inference aims to approximate the intractable distribution
$p(S)\propto \exp(F(S))$ by a fully factorized surrogate distribution
$q(S|{\x}):= \prod_{i\in S}x_i \prod_{j\notin S}(1-x_j), \x\in[0,1]^n$.
This target can be achieved by maximizing the \text{(ELBO)} objective, which
provides a lower bound for the log-partition function,
$ \text{(ELBO)}\leq \log\parti = \log \sum_{S\subseteq
\groundset} \exp(F(S))$.  Specifically, the optimization problem is,
\begin{align}\notag 
  \max_{ \x \in [\zero, \one]} f(\x) 
  &  := 
    \overbrace{\E_{q(S\mid \x)}[F(S)]}^{\text{multilinear extension of }  F(S): \multi(\x)}\\\notag
 & - \sum\nolimits_{i=1}^{n} [x_i\log x_i + (1-x_i)\log(1-x_i)]  \\\label{opt_problem_meanfield}
  & \; = \multi(\x)   + \sum\nolimits_{i \in\groundset} H(x_i),
    \quad \text{(\textcolor{link_color}{ELBO})}
\end{align}
where $H(x_i) := -[ x_i \log x_i + (1-x_i)\log (1-x_i ) ] $ is the
binary entropy function and by default $0\log 0 =0$.
$\multi(\x) :=\E_{q(S\mid \x)}[F(S)]$ is the multilinear extension
\citep{DBLP:dblp_conf/ipco/CalinescuCPV07} of $F(S)$.  The above
(ELBO) is continuous DR-submodular w.r.t. $\x$, thus falling into the
general problem class \labelcref{opt_problem}.  

At first glance,
$\multi(\x)$ seems to require an exponential number of operations for
evaluation; we have shown in \cref{sec_app_multilinear} that
$\multi(\x)$ and its gradients can be computed precisely in polynomial
time for many classes of practical objectives.

\subsection{Mean Field Lower Bounds for PSMs}
\label{sec_meanfield_lowerbounds}

Maximizing \text{(ELBO)} to optimality provides  the tightest
lower bound of $\log \parti$ in terms of the KL divergence
$\kl{q}{p}$.  
We put details here.
For a probabilistic log-submodular model
$p(S) = \frac{1}{\parti}\exp(F(S))$,   
$\parti = \sum_{S\subseteq \groundset}\exp(F(S))$
is the  partition function.
Mean field inference aims to approximate $p(S)$ by
a fully factorized product  distribution
$q(S|{\x}):= \prod_{i\in S}x_i \prod_{j\notin S}(1-x_j), \x\in
[0,1]^n$,
by minimizing the distance measured w.r.t. the Kullback-Leibler
divergence between $q$ and $p$, i.e.,
$\kl{q}{p} =  \sum_{S\subseteq \groundset} q(S|{\x})
\log\frac{q(S|{\x})}{p(S)}$. $\kl{q}{p}$ is non-negative,
so 
\begin{align}\notag 
& 0\leq \kl{q}{p} =  \sum_{S\subseteq \groundset} q(S|{\x})
\log\frac{q(S|{\x})}{p(S)}\\
& = -  \E_{q(S|{\x})} [\log p(S)] - \entropy{q(S|{\x})}  \\   
&= 
-\sum_{S\subseteq \groundset}  F(S) \prod_{i\in S}x_i \prod_{j\notin S}(1-x_j)+ \\
& \sum\nolimits_{i=1}^{n} [x_i\log x_i + (1-x_i)\log(1-x_i)] + \log \parti,
\end{align}
where $\entropy{\cdot}$ is the entropy. So one can get 
$\log \parti \geq \sum_{S\subseteq \groundset}  F(S) \prod_{i\in S}x_i \prod_{j\notin S}(1-x_j)- \sum\nolimits_{i=1}^{n} [x_i\log x_i + (1-x_i)\log(1-x_i)]  = (\text{ELBO})$.

Multilinear extension $\multi(\x)$ of 
a submodular set function is continuous DR-submodular \citep{bach2015submodular}, 
and $- \sum\nolimits_{i=1}^{n} [x_i\log x_i + (1-x_i)\log(1-x_i)] $
is seperable and concave along each coordinate, so 
$(\text{ELBO})$ is DR-submodular w.r.t. $\x$.
Maximizing $(\text{ELBO})$ amounts to 
minimizing the Kullback-Leibler divergence.

\section[Application to Mean Field Inference  of PA]{Application to Mean Field Inference  of \palong (PA)}\label{sec_pa}

In addition to the traditional mean field objective (ELBO)
in \labelcref{opt_problem_meanfield}, here we further formulate a
second class of mean field objectives. They come from \palong (PA) for
probabilistic log-submodular models, which is an information-theoretic
criterion to conduct model and algorithmic validation
\citep{Buhmann10isit,buhmannaposterior,bian2016information}.

Let us again consider the recommender example: usually there are some
hyperparameters in the model/utility function $F(S)$ that require
adaptation to the input data.
One natural way to do so is through model validation: Split the
training data into multiple folds, train a model on each fold $\data$
one would infer a ``noisy'' posterior distribution $p(S\mid \data)$.  PA
measures the agreement between these ``noisy'' posterior distributions.

Assume w.l.o.g. that  there are
two folds of data $\data', \data''$  in the sequel. 
In the PA framework, we have two consecutive 
targets:  1) Direct inference based on the two posterior  distributions 
$p(S\mid \data')$ and $p(S\mid \data'')$. 
This task amounts to find the MAP solution of the PA distribution
(which is discussed in the next paragraph),
it can be approximated by standard mean field inference.
2) Use the log PA objective in \labelcref{pa_objective} as a criterion for
model validation/selection. Since in general the PA
objective \labelcref{pa_objective} is intractable, we will still use
mean field lower bounds and some upper bounds in
\citet{djolonga14variational} to provide estimations for it.

\subsection{Mean Field Approximation of the Posterior Agreement Distribution}
\label{sec_pa_distribution}

A probabilistic log-submodular model is a special case of a Gibbs
random field with unit temperature and $-F(S)$ as the energy
function. In PA framework, we explicitly keep $\beta$ as the inverse
temperature,
$p_{\beta}(S | \data) := \frac{\exp(\beta F(S |
  \data))}{\sum_{\tilde{S} \subseteq \groundset}\exp(\beta F(\tilde S
  | \data))}, \forall S \subseteq \groundset$,
where $\data$ is the dataset used to train the model $F(S\mid \data)$.
The \emph{PA distribution} is defined as,
\begin{align} 
\label{eq_pa_distribution}
  p^{\pa}(S)  \propto p_{\beta}(S \mid \data') p_{\beta}(S\mid
  \data'') \propto \exp[ \beta ( F(S| \data') \! +\! F(S| \data'')  )]. 
\end{align}
Note that its log partition function is still intractable. In order to
approximate $p^{\pa}(S)$, we use mean field approximation with a
surrogate distribution
$q(S|{\x}):= \prod_{i\in S}x_i \prod_{j\notin S}(1-x_j)$,
 \begin{align}\notag 
   & \log \parti^{\pa} =  \log  \sum\nolimits_{S \subseteq \groundset}
     \exp[ \beta ( F(S| \data') + F(S| \data'')  )] \\\label{pa_elbo} 
   & \geq \beta \ \mathbb{E}_{q(S|{\x})} [F(S|\data')] +\beta \
     \mathbb{E}_{q(S|{\x})} [F(S|\data'')] \\\notag 
     &   \quad   + \sum\nolimits_{i
     \in\groundset} H(x_i).    \quad\quad\qquad\qquad\qquad
     \text{(\textcolor{link_color}{PA-ELBO})} 
 \end{align}
Maximizing \text{(PA-ELBO)} in \cref{pa_elbo} still falls into
the general problem class of \cref{opt_problem}.
For \text{(PA-ELBO)} in   \labelcref{pa_elbo}, it is the sum of two multilinear extensions (weighted by $\beta >0$) and the binary entropy term, since the non-negative
sum of two DR-submodular functions is still DR-submodular, so 
\text{(PA-ELBO)} in \labelcref{pa_elbo} is also continuous DR-submodular. Thus it fits into the general optimization problem of \labelcref{opt_problem}.
 
Maximizing
\text{(PA-ELBO)} also serves as a building block for the second
target below.

\subsection{Lower Bounds for the  \palong\ Objective}

The \pa\ objective is used to measure the agreement between  two
posterior distributions motivated by an information-theoretic analogy
\citep{buhmannaposterior,bian2016information}. By introducing the same
surrogate distribution $q(S\mid\x)$, one can derive that,
\begin{align}\label{pa_objective}
&  \mathrm{log} \sum\nolimits_{S \subseteq \groundset}
  p_{\beta}(S \mid \data') p_{\beta}(S \mid \data'')   \quad
  \text{(log PA\  objective)} \\  
&  \!\!\geq  \underbrace{\entropy{q} \!+ \!\beta \ {\E}_q
  F(S|\data') \!+\!\beta \ {\E}_q F(S|\data'')}_{\text{(PA-ELBO)
  in \cref{pa_elbo}}}  \\\notag 
&  \quad   -\log \parti(\beta; \data') - \log \parti(\beta; \data'') ,
\end{align}
where $\entropy{q}$ is the entropy of $q$,  $\parti(\beta; \data')$ and $\parti(\beta; \data'')$ are the
partition functions of the two noisy distributions, respectively.
In order to find the best lower bound for \pa, one need to maximize
w.r.t. $q(S|\x)$ the (PA-ELBO) objective, at the same time, find the
upper bounds for
$\log \parti(\beta; \data') + \log \parti(\beta; \data'')$.  The
latter can be achieved using techniques of
\citet{djolonga14variational}.  We summarize the details in
\cref{upperbounds_pa} to make it self-contained.

\section{Multi-Epoch Extensions of DoubleGreedy Algorithms
}
Though \algname{DR-DoubleGreedy} reaches the optimal
$1/2$ approximation  guarantee with  one epoch, in practice
it usually helps to use its output as an initializer,
and continue optimizing coordinate-wisely  for
additional epochs. Since each step of  coordinate update
will never decrease the function value, the
approximation guarantees will hold.
We call this class of algorithms \algname{DoubleGreedy-MeanField},
abbreviated as \algname{DG-MeanField}, and summarize the pseudocode in
\cref{alg_dg_meanfield}.
\begin{algorithm}[htbp]
	\caption{\algname{DG-MeanField-$1/2$} \& \algname{DG-MeanField-$1/3$} \citep{bian2019optimalmeanfield} }\label{alg_dg_meanfield}
	\KwIn{ $\max_{\x \in [\a, \b]}f(\x)$,  e.g., from
	the ELBO \labelcref{opt_problem_meanfield} or  PA-ELBO \labelcref{pa_elbo} objective
	}
	{Option I: \algname{DG-MeanField-$1/3$}: run \algname{Submodular-DoubleGreedy} (detailed in \cref{alg_uscfmax_DoubleGreedy}) to get a ${1}/{3}$ initializer $\hat\x$}
	
	{Option II: \algname{DG-MeanField-$1/2$}: run \algname{DR-DoubleGreedy} (detailed in \cref{alg_cont_doublegreedy}) to get a $1/2$ initializer $\hat\x$  \;}
	{ beginning with $\hat\x$, optimize $f(\x)$ coordinate by coordinate for $T$ epochs \;}
\end{algorithm}

\section{Experiments}
\label{sec_exp}

The objectives under investigation are
ELBO \labelcref{opt_problem_meanfield} and
PA-ELBO \labelcref{pa_elbo} (We set  $\beta = 1$ in PA-ELBO).  We tested on the representative FLID
model on the following algorithms and baselines: The first category is
one-epoch algorithms, including \textcircled{1}
\algname{Submodular-DoubleGreedy} from \citet{bian2017guaranteed} with a
$1/3$ guarantee,  \textcircled{2} \algname{BSCB}  (Algorithm 4 in \cite{niazadeh2018optimal}, 
where we chose $\epsilon=10^{-3}$) with a $1/2$ guarantee 
and \textcircled{3} \algname{DR-DoubleGreedy}
(\cref{alg_cont_doublegreedy}) with a $1/2$ guarantee.
The second category contain multiple-epoch algorithms: \textcircled{4}
 \algname{CoordinateAscent-0}: initialized as $\zero$ and coordinate-wisely improving  the solution;
\algname{CoordinateAscent-1}: initialized as $\one$;
\algname{CoordinateAscent-Random}: initialized as a uniform  vector $U(\zero, \one)$.
\textcircled{5} \algname{\dgmf-${1}/{3}$}.  \textcircled{6}
\algname{\dgmf-${1}/{2}$} from \cref{alg_dg_meanfield}.
\textcircled{7} \algname{BSCB-Multiepoch}, which is the
multi-epoch extension of \algname{BSCB}: After the first
epoch, it continues to improve the solution coordinate-wisely.

{\begin{table}
		\begin{center}
			\footnotesize 
			\caption{Summary of results on ELBO objective \labelcref{opt_problem_meanfield} and PA-ELBO objective \labelcref{pa_elbo}. 
			}
			\label{tab_summ_elbo}
			\begin{tabular}{|c|c||c|c|c||c|c|c|}
				\cline{1-8} 
				&  &  \multicolumn{3}{|c||}{\textbf{ELBO} objective \labelcref{opt_problem_meanfield}} &  \multicolumn{3}{|c|}{\textbf{PA-ELBO} objective \labelcref{pa_elbo}}  \\
				\cline{1-8}
				Category	& $D$ &  \algname{Sub-DG} & \algname{BSCB} & \algname{DR-DG} &  \algname{Sub-DG} & \algname{BSCB} & \algname{DR-DG} \\
				\hline
				\hline
				\multirow{2}{*}{furniture}  & 2  & 2.078$\pm$0.091 & 2.771$\pm$0.123  & \textbf{3.035}$\pm$0.059  & 0.918$\pm$0.768 & 2.287$\pm$0.399  & \textbf{2.402}$\pm$0.159\\
				& 3  & 1.835$\pm$0.156 & 2.842$\pm$0.128  & \textbf{3.026}$\pm$0.099 & 1.296$\pm$1.176 & 2.536$\pm$0.439  & \textbf{2.693}$\pm$0.181 \\
				$n$=32 & 10  & 1.375$\pm$0.194 & \textbf{2.951}$\pm$0.161  & {2.917}$\pm$0.103 & 1.504$\pm$1.110 & 2.764$\pm$0.405  & \textbf{2.882}$\pm$0.248 \\
				\hline
				\multirow{2}{*}{carseats}  & 2  & 2.089$\pm$0.166 & 2.863$\pm$0.090  & \textbf{3.045}$\pm$0.069 & 1.015$\pm$1.081 & 2.106$\pm$0.228  & \textbf{2.348}$\pm$0.219\\
				& 3  & 1.890$\pm$0.146 & 3.003$\pm$0.110  & \textbf{3.138}$\pm$0.082 & 1.309$\pm$1.218 & 2.414$\pm$0.267  & \textbf{2.707}$\pm$0.208 \\
				$n$=34 & 10  & 1.390$\pm$0.232 & \textbf{3.100}$\pm$0.140  & {3.003}$\pm$0.157 & 1.599$\pm$1.317 & 2.684$\pm$0.271  & \textbf{2.915}$\pm$0.250 \\
				\hline
				\multirow{2}{*}{safety}  & 2  & 1.934$\pm$0.402 & 2.727$\pm$0.212  & \textbf{2.896}$\pm$0.098 & 1.370$\pm$1.203 & 2.049$\pm$0.280  & \textbf{2.341}$\pm$0.161\\
				& 3  & 1.867$\pm$0.453 & 2.830$\pm$0.191  & \textbf{2.970}$\pm$0.110 & 1.706$\pm$1.296 & 2.288$\pm$0.297  & \textbf{2.619}$\pm$0.167 \\
				$n$=36 & 10  & 1.546$\pm$0.606 & 2.916$\pm$0.191  & \textbf{2.920}$\pm$0.149 & 1.948$\pm$1.353 & 2.467$\pm$0.270  & \textbf{2.738}$\pm$0.187 \\
				\hline
				\multirow{2}{*}{strollers}  & 2  & 2.042$\pm$0.181 & 2.829$\pm$0.144  & \textbf{2.928}$\pm$0.060 & 0.865$\pm$0.952 & 1.933$\pm$0.256  & \textbf{2.202}$\pm$0.226\\
				& 3  & 1.814$\pm$0.264 & 2.958$\pm$0.146  & \textbf{2.978}$\pm$0.077 & 1.172$\pm$1.063 & 2.181$\pm$0.297  & \textbf{2.543}$\pm$0.254 \\
				$n$=40 & 10  & 1.328$\pm$0.544 & \textbf{3.065}$\pm$0.162  & 2.910$\pm$0.140 & 1.702$\pm$1.334 & 2.480$\pm$0.304  & \textbf{2.767}$\pm$0.336 \\
				\hline
				\multirow{2}{*}{media}  & 2  & 3.221$\pm$0.066 & 3.309$\pm$0.055  & \textbf{3.493}$\pm$0.051 & 0.372$\pm$0.286 &\textbf{ 1.477}$\pm$0.128  & 1.336$\pm$0.101\\
				& 3  & 3.276$\pm$0.082 & 3.492$\pm$0.083  & \textbf{3.712}$\pm$0.079 & 0.418$\pm$0.366 & 1.736$\pm$0.177  & \textbf{1.762}$\pm$0.095 \\
				$n$=58 & 10  & 2.840$\pm$0.183 & 3.894$\pm$0.122  & \textbf{3.924}$\pm$0.114 & 0.653$\pm$0.727 & 2.309$\pm$0.244  & \textbf{2.524}$\pm$0.130 \\
				\hline
				\multirow{2}{*}{health}  & 2  & 3.197$\pm$0.067 & 3.174$\pm$0.074  & \textbf{3.516}$\pm$0.043 & 0.548$\pm$0.282 & \textbf{1.655}$\pm$0.122  & 1.650$\pm$0.073\\
				& 3  & 3.231$\pm$0.055 & 3.306$\pm$0.108  & \textbf{3.707}$\pm$0.064 & 0.649$\pm$0.413 & 1.903$\pm$0.173  & \textbf{2.025}$\pm$0.083 \\
				$n$=62 & 10  & 2.633$\pm$0.115 & 3.508$\pm$0.120  & \textbf{3.675}$\pm$0.110 & 0.768$\pm$0.628 & 2.233$\pm$0.196  & \textbf{2.375}$\pm$0.101 \\
				\hline
				\multirow{2}{*}{toys}  & 2  & 3.543$\pm$0.047 & 3.454$\pm$0.091  & \textbf{3.856}$\pm$0.044 & 0.597$\pm$0.480 & 1.731$\pm$0.182  & \textbf{{1.761}}$\pm$0.133\\
				& 3  & 3.362$\pm$0.055 & 3.412$\pm$0.070  & \textbf{3.736}$\pm$0.051 & 0.578$\pm$0.520 & 1.738$\pm$0.192  & \textbf{1.802}$\pm$0.151 \\
				$n$=62 & 10  & 3.037$\pm$0.138 & 3.706$\pm$0.108  & \textbf{3.859}$\pm$0.119 & 0.758$\pm$0.871 & 2.140$\pm$0.242  & \textbf{2.330}$\pm$0.177 \\
				\hline
				\multirow{2}{*}{diaper}  & 2  & 3.500$\pm$0.058 & 3.517$\pm$0.058  & \textbf{3.636}$\pm$0.043 & 0.295$\pm$0.158 & \textbf{1.119}$\pm$0.063  & 0.665$\pm$0.116\\
				& 3  & 3.739$\pm$0.080 & 3.753$\pm$0.065  & \textbf{3.974}$\pm$0.065 & 0.337$\pm$0.240 & \textbf{1.429}$\pm$0.111  & 1.141$\pm$0.120 \\
				$n$=100 & 10  & 3.423$\pm$0.110 & 4.150$\pm$0.120  & \textbf{4.203}$\pm$0.086 & 0.386$\pm$0.504 & 1.969$\pm$0.201  & \textbf{2.009}$\pm$0.199 \\
				\hline
				\multirow{2}{*}{feeding}  & 2  & 3.942$\pm$0.041 & 3.808$\pm$0.024  & \textbf{3.970}$\pm$0.036 & 0.393$\pm$0.034 & \textbf{0.894}$\pm$0.022  & 0.501$\pm$0.029\\
				& 3  & 4.333$\pm$0.031 & 4.095$\pm$0.032  & \textbf{4.390}$\pm$0.031 & 0.503$\pm$0.072 & \textbf{1.232}$\pm$0.041  & 0.893$\pm$0.046 \\
				$n$=100 & 10  & 4.611$\pm$0.053 & 4.553$\pm$0.079  & \textbf{4.860}$\pm$0.056 & 0.608$\pm$0.239 & 1.808$\pm$0.087  & \textbf{1.820}$\pm$0.078 \\
				\hline
				\multirow{2}{*}{gear}  & 2  & 3.311$\pm$0.046 & 3.150$\pm$0.037  & \textbf{3.430}$\pm$0.040 & 0.232$\pm$0.068 & \textbf{1.019}$\pm$0.048  & 0.590$\pm$0.043\\
				& 3  & 3.538$\pm$0.048 & 3.347$\pm$0.045  & \textbf{3.721}$\pm$0.050 & 0.303$\pm$0.132 & \textbf{1.257}$\pm$0.085  & 1.020$\pm$0.064 \\
				$n$=100 & 10  & 3.065$\pm$0.083 & 3.550$\pm$0.050  & \textbf{3.670}$\pm$0.067 & 0.312$\pm$0.232 & \textbf{1.566}$\pm$0.130  & 1.514$\pm$0.072 \\
				\hline
				\multirow{2}{*}{bedding}  & 2  & 3.406$\pm$0.080 & 3.374$\pm$0.088  & \textbf{3.620}$\pm$0.062 & 0.525$\pm$0.121 & 1.932$\pm$0.194  & \textbf{2.001}$\pm$0.080\\
				& 3  & 3.648$\pm$0.106 & 3.564$\pm$0.083  & \textbf{3.876}$\pm$0.081 & 2.499$\pm$0.972 & 2.250$\pm$0.269  & \textbf{2.624}$\pm$0.066 \\
				$n$=100 & 10  & 3.355$\pm$0.161 & 3.799$\pm$0.144  & \textbf{3.912}$\pm$0.082 & \textbf{3.919}$\pm$0.045 & 2.578$\pm$0.358  & {3.157}$\pm$0.091 \\
				\hline
				\multirow{2}{*}{apparel}  & 2  & 3.560$\pm$0.094 & 3.527$\pm$0.046  & \textbf{3.784}$\pm$0.059 & 0.268$\pm$0.109 & \textbf{1.552}$\pm$0.141  & 1.513$\pm$0.191 \\
				& 3  & 3.878$\pm$0.092 & 3.755$\pm$0.062  & \textbf{4.140}$\pm$0.063 & 0.490$\pm$0.677 & 1.900$\pm$0.237  & \textbf{2.225}$\pm$0.136\\
				$n$=100 & 10  & 3.751$\pm$0.087 & 4.084$\pm$0.075  & \textbf{4.425}$\pm$0.066 & 0.820$\pm$1.372 & 2.351$\pm$0.337  & \textbf{2.967}$\pm$0.150 \\
				\hline
				\multirow{2}{*}{bath}  & 2  & 2.957$\pm$0.087 & 3.024$\pm$0.032  & \textbf{3.198}$\pm$0.056 & 0.197$\pm$0.090 & \textbf{1.101}$\pm$0.083  & 0.795$\pm$0.078\\
				& 3  & 3.062$\pm$0.085  & 3.195$\pm$0.058  & \textbf{3.448}$\pm$0.058 & 0.247$\pm$0.163 & \textbf{1.368}$\pm$0.134  & 1.269$\pm$0.059 \\
				$n$=100 & 10  & 2.497$\pm$0.135 & 3.426$\pm$0.076  & \textbf{3.438}$\pm$0.089 & 0.327$\pm$0.312 & 1.711$\pm$0.183  & \textbf{1.742}$\pm$0.098 \\
				\hline
			\end{tabular}
		\end{center}
	\end{table}
}

For all algorithms, we use the same random order to process
the coordinates within each  epoch.
We are trying to understand:
1)  In terms of continuous DR-submodular maximization,   how good are  the  solutions returned by one-epoch
algorithms?
2) How good are the realized lower bounds?
  For small scale problems 
  we can calculate the true
  log-partitions exhaustively, which servers as
  a natural upper bound of ELBO.
All algorithms and subroutines are implemented in Python3, and source
code are released on github (\url{https://github.com/bianan/optimal-dr-submodular-max}).

\setkeys{Gin}{width=1.1\textwidth, height=.6\textheight}
\begin{figure}[htbp]
	\center 
	\includegraphics[]{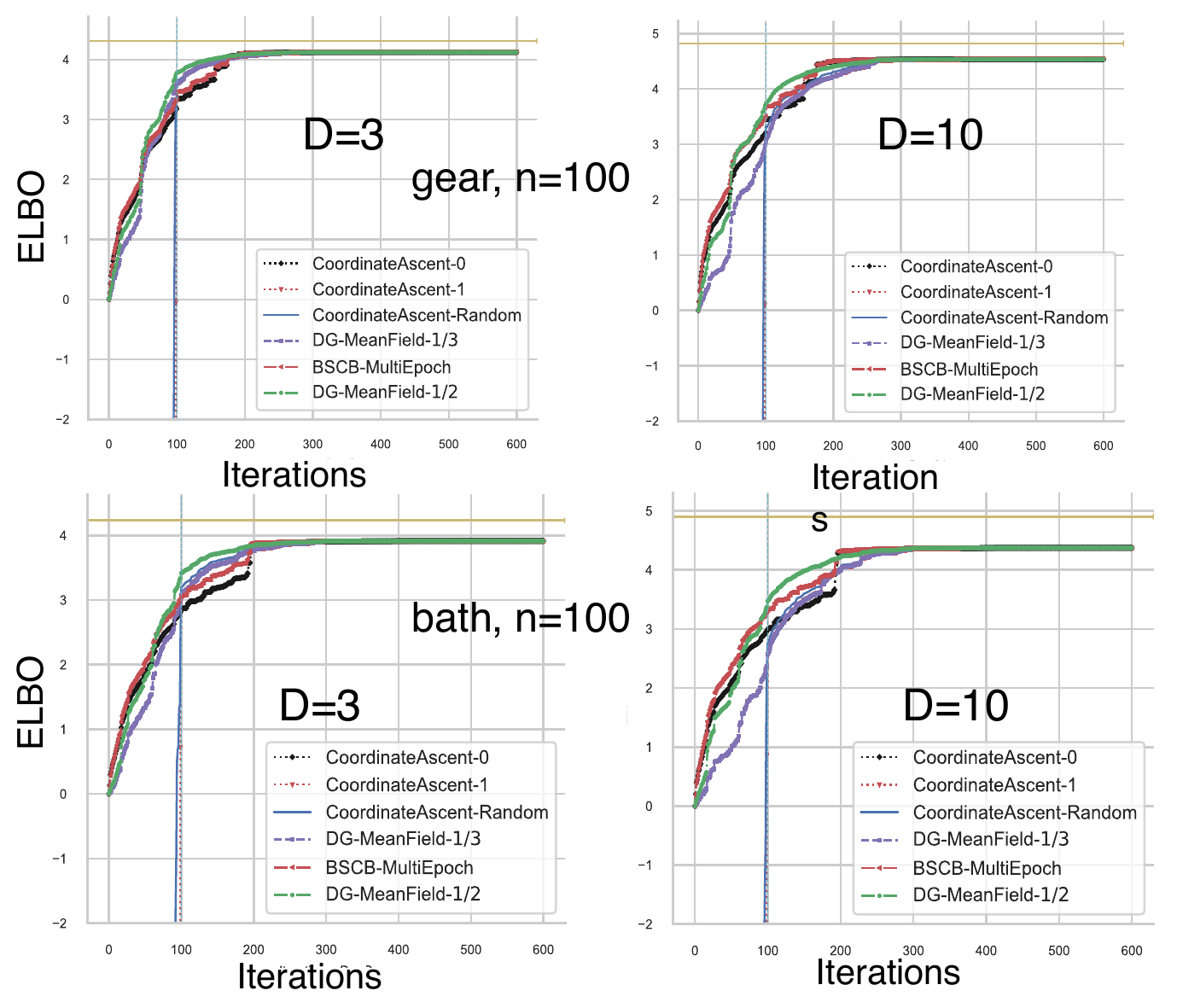}
	\caption{Typical trajectories of multi-epoch algorithms on ELBO objective  for Amazon data. 1st row:  ``gear'';
		2nd row: ``bath''.
		Cyan vertical line shows the one-epoch point. Yellow line shows  the true value of log-partition.}
	\label{fig3}
\end{figure}

\paragraph{Real-world Dataset. }
We  tested the mean field methods on the
trained FLID models from \citet{Tschiatschek16diversity} on Amazon Baby Registries
dataset.
After preprocessing, this dataset
has 13 categories, e.g., ``feeding'' and ``furniture''.
One category contains a certain number of registries
over the ground set of this category, e.g., ``strollers'' has
5,175 registries with $n=40$.
One can refer to \cref{tab_summ_elbo} for specific dimensionalities
on each of the category\footnote{More details on this dataset can be found in  \citet{gillenwater2014expectation}.}.
For each category, three
classes of
models were trained, with latent dimensions $D= 2, 3, 10$, repectively, on
10 folds of the data.

\subsection{Results on One-Epoch Algorithms}

\cref{tab_summ_elbo}  summarizes the outputs
of one-epoch algorithms for both ELBO and PA-ELBO
objectives. 
\algname{Sub-DG} stands for  \algname{Submodular-DoubleGreedy}, \algname{DR-DG} stands for  \algname{DR-DoubleGreedy}. Boldface numbers indicate the best mean of function values returned. For ELBO,  the mean and standard deviation were calculated
for 10 FLID models trained on 10 folds of the data. For PA-ELBO, the mean and standard deviation were calculated  for models trained over $45$ pairs of folds.
For each category, the results
of FLID models with three dimensionalities ($D= 2, 3, 10$)
are reported. 

\paragraph{ELBO Objective.}
The results are summarized in columns 3 to 5 in \cref{tab_summ_elbo}.
 The mean and standard deviation are calculated
for 10 FLID models trained on 10 folds of the data.
One can observe that both \algname{DR-DoubleGreedy} and \algname{BSCB}
improve over the baseline \algname{Submodular-DoubleGreedy}, which has only a 1/3 approximation guarantee.
Furthermore, \algname{DR-DoubleGreedy} generates better solutions
than \algname{BSCB} for almost all of the cases, though they have the same
approximation guarantees in the worst case.

\paragraph{PA-ELBO objective.}

The results are summarized in columns 6 to 8 in \cref{tab_summ_elbo}.
For each category, out of the 10 folds of data,
we have ${10\choose  2} = 45$ pairs of folds.
The mean and standard deviation are computed 
for these $45$ pairs for each category and each
latent dimensonality $D$.
One can still observe that \algname{DR-DoubleGreedy} and \algname{BSCB} significantly 
improve over \algname{Submodular-DoubleGreedy}. 
Moreover, \algname{DR-DoubleGreedy}  produces better
solutions than \algname{BSCB} in most of the experiments.

\subsection{Results on Multi-Epoch Algorithms}

\paragraph{ELBO Objective.}
\label{supp_elbo}

\cref{fig3} records typical trajectories of multi-epoch algorithms for ELBO objectives.
Note that the cyan vertical lines indicate the one-epoch
point.
It shows that after one epoch, \algname{\dgmf-$1/2$}
almost always returns the best solution, and
it  is also the fastest  one to converge.
However, \algname{\nmf} is quite sensitive to initializations.
After sufficiently many iterations, all  multi-epoch
algorithms converge to similar ELBO value.
This is consistent with the intuition since after one epoch,
all algorithms are using the same strategy: conducting
coordinate-wise maximization.
One can also observe that the obtained ELBO is
close to the true log partition functions (yellow lines).

\paragraph{PA-ELBO Objective.}

\setkeys{Gin}{width=1.05\textwidth, height=0.6\textheight}
\begin{figure}[htbp]
	\includegraphics[]{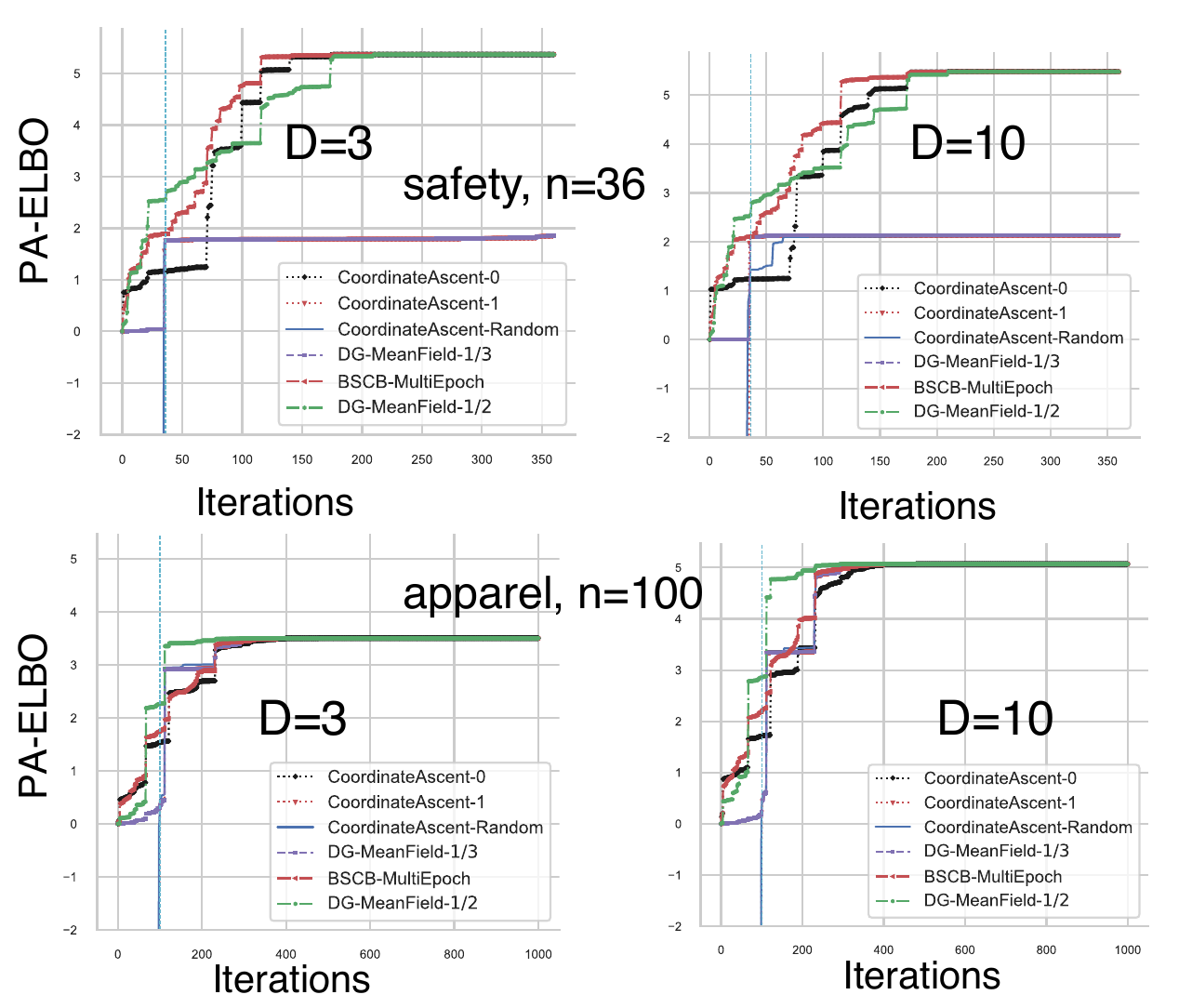}
	\caption{PA-ELBO  on Amazon data.  The figures  trace trajectories of multi-epoch algorithms. Cyan vertical line shows the one-epoch point.
	}
	\label{PA-ELBO_safety}
\end{figure}

\cref{PA-ELBO_safety} shows representative
results on PA-ELBO objectives.
One can see that after one epoch, \algname{\dgmf-$1/2$}
almost always returns the best solution.
In most of the experiments, \algname{\dgmf-$1/2$}  was the
fastest algorithm to converge.
However, \algname{\nmf} is quite sensitive to initializations.
After sufficiently many iterations, most  multi-epoch
algorithms converge to similar PA-ELBO value.
However, for \algname{\nmf} with unlucky initializations, e.g., for category ``safety'' (row 1), it may get stuck in poor
local optima.

\section{Conclusions}
Probabilistic structured models play an eminent
role in machine learning today, especially models with submodular
costs. Validating such models and their parameters remains an open
issue in applications. We have proposed provable mean field algorithms
for probabilistic log-submodular models and their posterior agreement
score. 
This optimization technique promises to open new
avenues for model inference by combining approximation guarantees of
submodular maximization with robustness of probabilistic inference.

\section{Additional Details}

\subsection{Complete Lower Bounds of the PA Objective}
\label{upperbounds_pa}

By giving  upper bounds for 
$\log \parti(\beta; \data') + \log \parti(\beta; \data'')$,
we can get the full lower bounds of the PA objective.

Let us take one $\log \parti(\beta; \data')$ for example. 
This  can be achieved using techniques of \citet{djolonga14variational}, which is done by 
 optimizing  supergradients of  $F(S\mid \data')$.
A representative supergradient
is the bar supergradient, which is defined as: if $i\in A$, $\bar \s^A = F_{\groundset -\{i\}} (\{i \}\mid \data' )$, if $i\notin A$, $\bar \s^A = F (\{i \} \mid \data')$, 
where $F_{B} (A\mid \data')$ is the marginal
gain of $A$ based on $B$. 
 Then, 
\begin{align}
\log \parti(\beta; \data') 
& \leq \min _{A} \log \parti^+(\bar \s^A, F(A\mid \data') - \bar \s^A(A) )  \\
& = \min _{A} F(A\mid \data') + \m (A\mid \data'),
\end{align}
where $\m ({\{i\} \mid \data'}) = \log ( 1 + e^ { -F_{\groundset -\{i\}} (\{i \}\mid \data') }) - \log ( 1 + e^{F(\{i\} \mid \data')})$. 

So the full lower bound of  PA objective in  \labelcref{pa_objective} is, 
\begin{align}\label{}
&  \mathrm{log} \sum\nolimits_{S \subseteq \groundset} p_{\beta}(S \mid \data') p_{\beta}(S \mid \data'') \\ 
&{=} -\left[ \sum\nolimits_{S \subseteq \groundset} 	q(S|\x) \right] \mathrm{log} \frac{\sum_{S\subseteq \groundset} q(S|\x)}{\sum\nolimits_{S \subseteq \groundset} p_\beta (S \mid \data') p_\beta (S \mid \data'')} \nonumber \\\notag 
&\stackrel{\text{log-sum inequality}}{\geq} -\sum\nolimits_{S \subseteq \groundset} q(S|\x) \ \mathrm{log} \frac{q(S|\x)}{p_\beta(S \mid \data') p_\beta(S \mid \data'')} 
 \\\notag 
&
= \entropy{q} + \mathbb{E}_q \mathrm{log} \ p_\beta(S \mid \data') +\mathbb{E}_q \log \ p_\beta(S\mid \data'') \\\notag 
&  = \underbrace{\entropy{q} + \beta \ \mathbb{E}_q F(S|\data') +\beta \ \mathbb{E}_q F(S|\data'')}_{\text{(PA-ELBO) in \labelcref{pa_elbo}}}  \!   -\log \parti(\beta; \data') - \!\log \parti(\beta; \data'').
\end{align}

Since the above holds for all $q$, so we get the lower bound,
\begin{flalign}
&  \mathrm{log} \sum\nolimits_{S \subseteq \groundset} p_{\beta}(S \mid \data') p_{\beta}(S \mid \data'')   \quad \text{(log \pa   objective)} \\
& \geq \max_{q}  \underbrace{\entropy{q} + \beta \ \mathbb{E}_q F(S|\data') +\beta \ \mathbb{E}_q F(S|\data'')}_{\text{(PA-ELBO) in \labelcref{pa_elbo}}}\\\notag 
& \quad  - 
 \min _{A} \left[F(A\mid \data') + \m (A\mid \data')\right] -  \min _{A} \left[F(A\mid \data'') + \m (A\mid \data'')\right].
\end{flalign}

\def\dir{chapters/discussion}

\chapter{Discussions and Future Work}
\label{chapter_disc_future_work}

\begin{chapquote}{Albert Einstein}
	The important thing is not to stop questioning. Curiosity has its own reason for existence.
\end{chapquote}

In this thesis we have studied how submodularity can be 
generalized as a unified structure that ensures provable non-convex optimization and algorithm validation. We
believe that the continuous generalization: continuous submodularity, will play a more and more significant role in the area of non-convex optimization.

Though lots of details have been discussed, we are still curious about the following open problems.

\section{Tighter Guarantees for Continuous DR-Submodular Maximization}

For monotone DR-submodular maximization with 
a down-closed convex constraint, we have studied several algorithms in \cref{chapter_max_monotone}. See \cref{tab_alg_monotone_max} for a summary of these algorithms. 
The algorithms motivated by local-global relation have a 1/2 approximation guarantee, while the optimal algorithm, \submodularfw,  has an approximation ratio of $1-1/e$.

\begin{table}[htbp]
	\begin{center}
		\normalsize
		\caption{Summary of algorithms for monotone DR-submodular maximization}
		\label{tab_alg_monotone_max}
		\begin{tabularx}{\textwidth}{|X|X|X|X|}
			\hline
			{Name}  & Technique     & {Approximation ratio} & Convergence rate  \\
			\hline
			\hline
			\nonconvexfw  & \multirow{2}{*}{local-global} &  $1/2$  &  $1/\sqrt{k}$  \\
			\cline{1-1}  \cline{3-4}
			\pga   &  &  $1/2$  &  $1/{k}$ \\
			\hline
			\submodularfw  & follow concavity & $1-1/e$  & $1/k$ \\
			\hline 			
		\end{tabularx}
	\end{center}
\end{table}

However,  in experiments one can usually observe, for excample, from \cref{fig_traj_influence_50_100,fig_traj_influence_150_200},  that 
\nonconvexfw has the fastest convergence rate and returns the best solution.
Similar phenomenon was also observed in experiments for non-monotone DR-submodular maximization algorithms. 

This observation motivates us to think of more properties of continuous DR-submodular functions that can help with explaining the practical performance of these algorithms. 
In this direction, we have proposed the strong DR-submodularity property (\cref{eq_strongly_dr}). Nevertheless, there should be more properties that shall be explored in the future work.

\section{Explore Submodularity over Arbitrary Conic Lattices}

Motivated by applications such as logistic regression with 
non-convex regularizers, we have  
studied generalized submodularity over 
the orthant conic lattice $(\X, \preceq_{\cone_{\bmalpha}})$ in \cref{sec_lattice}.
However, it
is noteworthy that the framework can be potentially generalized 
to arbitrary conic lattices, which may be of interest 
to model a larger group of applications.

\section{Sampling Methods for Estimating PA in Probabilistic Log-Submodular Models}

In \cref{chapter_mean_field} we provide a lower bound of the PA objective through mean filed approximation. However, it is not clear how large the deviation between the lower bound and the true objective is.  
This shortcoming of mean field approximation  naturally motivates us to consider sampling methods to estimate the PA objective, which amounts  to estimating the following three terms: 
\begin{align}
&  \log  \sum\nolimits_{S \subseteq \groundset}
\exp[ \beta ( F(S| \data') )],\\
&  \log  \sum\nolimits_{S \subseteq \groundset}
\exp[ \beta ( F(S| \data'')  )],\\
&  \log  \sum\nolimits_{S \subseteq \groundset}
\exp[ \beta ( F(S| \data') + F(S| \data'')  )].
\end{align}
All of them are in the form of a log partition function of some Gibbs distribution, which can be estimated using sampling methods such as Monte Carlo sampling.

\section{Negative Dependence for Continuous Random Variables}

Given  that the discrete random variables with negative dependence among each other has been formulated with probabilistic log-submodular models, 
it is natural to study continuous random variables with negative dependence. One would start by formulating  these distributions in a principled way, and then study approximate inference methods such as variational inference or sampling. 
It is noteworthy that 
\citet{karlin1980classesreverse} studied  the MTP$_2$ (stands for ``multivariate totally positive of order 2'') probability distribution, which is defined by a continuous function $f$ that are log-supmodular: $\forall \x, \y$, it holds that  $f(\x)f(\y) \leq f(\x\vee\y)f(\x\wedge \y) $. This definition implies positive dependency up to a logarithm operation.

\section{Incorporate Continuous Submodularity as Domain Knowledge into Deep Neural Net Architecture}

We have shown that continuous submodularity essentially captures the repulsion effect (or negative dependence) amongst different dimensionalities, which could be a valuable domain knowledge for modeling various practical scenarios. Its stronger version, continuous DR-submodularity, models the diminishing returns phenomenon. Continuous DR-submodularity has already been used as the domain knowledge in designing deep submodular set functions \citep{bilmes2017deep}, where the function induced by the submodular neural net is essentially a continuous DR-submodular function if feeding continuous inputs into the neural net. 

The negative dependence effect is prevalent in real-world applications. For instance, in financial areas, there exists the concept of substitutes and complements, which mean negative and positive dependence, respectively. However, a principled way to incorporate continuous submodularity into modern deep neural net architecture is still lacked. One can imagine some modular approaches for adding continuous submodularity into neural net architecture, and the subsequent specialized training algorithms for these continuous submodular models. 

\backmatter
\def\dir{frontbackmatter}
\manualmark
\markboth{\spacedlowsmallcaps{\bibname}}{\spacedlowsmallcaps{\bibname}} 
\refstepcounter{dummy}
\addtocontents{toc}{\protect\vspace{\beforebibskip}} 
\addcontentsline{toc}{chapter}{\tocEntry{\bibname}}
\printbibliography

\def\dir{frontbackmatter}
\chapter{Notation}\label{sec_notation}
\section*{General}
\begin{tabularx}{\textwidth}{l|X}
	\toprule
	{Symbol} & 	{Meaning} \\
	\midrule
    $\groundset=\{\ele_1, \ele_2,..., \ele_n\}$ &           the ground set of
    $n$ elements\\
    $\chara_i\in\R^n$ &  the characteristic vector for
    element $\ele_i$ (also the standard $i^\text{th}$ basis vector)   \\
    $\x\in \R^\groundset$ or  $\x\in \R^n$ &      an $n$-dimensional vector, whose $i^\text{th}$ entry is denoted as   $x_i$\\
    $\BA\in\R^{m\times n}$ &  an $m$ by $n$ matrix and  $A_{ij}$  is its ${ij}^\text{th}$ entry \\
    $f(\cdot)$     & a continuous function\\
    $F(\cdot)$    & a set function \\
    $\nabla f(\cdot)$  &  the gradient of a differentiable function $f(\cdot)$   \\
    $\nabla^2 f(\cdot)$  &  the Hessian of a twice differentiable function $f(\cdot)$\\
    $[n]$  &  $\{1,...,n\}$ for an
    integer $n \geq 1$\\
     $\x\leqco \y$ & 
    $x_i\leq y_i, \forall i$\\
    $\x\vee \y$  &  coordinate-wise maximum of $\x$ and $\y$   \\
    $\x \wedge \y $  & coordinate-wise minimum of $\x$ and $\y$   \\
    	$\ltwo{\x}$ &  $\ell_2$-norm\\
    $\lone{\x}$ &   $\ell_1$-norm\\
    $\sete{x}{i}{k}$ &  the operation of setting the
    $i^\text{th}$ element of $\x$ to $k$, while keeping all other elements
    unchanged, i.e., $\sete{x}{i}{k}=\x-x_i \bas_i + k\bas_i$\\
	\bottomrule
\end{tabularx}

\vfill

\section*{Algorithm Validation}
\begin{tabularx}{\textwidth}{l|X}
	\toprule
	{Symbol} &		{Meaning} \\
	\midrule
	    \alg   &   an algorithm \\
    \graphrv & the random variable of a graph\\
    $G$   & a realization  of \graphrv \\ 
     $I^{\alg}$  &  algorithmic information content of \alg\\
	$\I$ &  the classical mutual information\\
	\bottomrule
\end{tabularx}


\def\dir{frontbackmatter}
\chapter{Acronyms}

\begin{acronym}[OWL-QN] 
	 \acro{DR}{Diminishing Returns}
	 \acro{IR}{Increasing Returns}
	 \acro{PA}{Posterior Agreement}
    \acro{SDP}{Semidefinite Programming}
        
    \acro{L-BFGS}{Limited-memory Broyden-Fletcher-Goldfarb-Shanno}
    \acro{SGD}{Stochastic Gradient Descent}
    \acro{SOCP}{Second Order Cone Program}

    \acro{PCA}{Principal Component Analysis}
    
\end{acronym}

\pagebreak
\def\dir{frontbackmatter}
\pagestyle{empty}

\hfill

\vfill

\pdfbookmark[0]{Colophon}{colophon}
\section*{Colophon}
This document was typeset in \LaTeX~using the typographical look-and-feel
\texttt{classicthesis}. The bibliography is typeset using \texttt{biblatex}.

\end{document}